\documentclass{article}
\pdfoutput=1
\usepackage{etoolbox}

\newcommand{\Gexyu}{G_{\mathrm{ex} \to (y,u)}}
\newcommand{\Gexy}{G_{\mathrm{ex} \to y}}
\newcommand{\Gexu}{G_{\mathrm{ex} \to u}}
\newcommand{\Gexeta}{G_{\mathrm{ex} \to \eta}}

\newcommand{\Gweta}{G_{(w,e) \to \eta}}

\newcommand{\Bpinotfeed}{B_{\pinot,\mathrm{feed}}}

\newcommand{\Bpinotclin}{B_{\pinot,\mathrm{cl},\mathrm{in}}}

\newcommand{\natvhat}{\widehat{\matv}^{\,\mathrm{nat}}}
\newcommand{\Ghatexyu}{\Ghat_{\mathrm{ex}\to (y,u)}}
\newcommand{\Ghatexeta}{\Ghat_{\mathrm{ex}\to \eta}}
\newcommand{\Gweyu}{G_{(w,e)\to (y,u)}}
\newcommand{\ldcex}{\textsc{LDC-Ex}}

\newcommand{\midetanat}[1][t]{\mid \nateta_{1:#1}}
\newcommand{\midhateta}[1][t]{\mid \natetahat_{1:#1}}
\newcommand{\midhatetayg}[1][t]{\mid \Ghatexyu, \natetahat_{1:#1},\natvhat_{#1}}
\newcommand{\midpred}[1][t]{\mid \Gexyu,\natetahat_{1:#1},\natv_{#1}}

\newcommand{\Bpinotex}{B_{\pinot,\mathrm{ex}}}
\newcommand{\Bpinotclex}{B_{{\pi_{0}},\mathrm{cl},\mathrm{ex}}}
\newcommand{\Dpinotclex}{D_{\pinot,\mathrm{cl},\mathrm{ex}}}

\newcommand{\Cpinoteta}{C_{\pinot,\eta}}
\newcommand{\Dpinoteta}{D_{\pinot,\eta}}

\newcommand{\etalg}{\bm{\eta}^{\Alg}}

\newcommand{\Cpinotcleta}{C_{\pinot,\mathrm{cl},\eta}}
\newcommand{\Dpinotcleta}{D_{\pinot,\mathrm{cl},\eta}}
\newcommand{\dimeta}{d_{\eta}}

\newcommand{\midetaygst}[1][t]{\mid \Gexyu,\nateta_{1:#1},\natv_{#1}}
\newcommand{\natv}{\matv^{\mathrm{nat}}}
\newcommand{\valg}{\matv^{\Alg}}
\newcommand{\matvpi}{\matv^{\pi}}

\newcommand{\Gpicluu}{G_{\pi,\mathrm{cl},u \to u}}
\newcommand{\Gpicluy}{G_{\pi,\mathrm{cl},u \to y}}
\newcommand{\Gpicluyu}{G_{\pi,\mathrm{cl},u \to (y,u)}}

\newcommand{\natetahat}{\widehat{\bm{\eta}}^{\,\mathrm{nat}}}

	\newcommand{\Ayoulpi}{A_{\mathrm{yla}, \pi}}
	\newcommand{\Byoulpi}{B_{\mathrm{yla},\pi}}
	\newcommand{\Cyoulpi}{C_{\mathrm{yla},\pi}}
	\newcommand{\Dyoulpi}{D_{\mathrm{yla},\pi}}
	\newcommand{\Gyoulpi}{G_{\mathrm{yla},\pi}}
	\newcommand{\mateta}{\bm{\eta}}
\newcommand{\Gsttohat}{G_{\hat{\cdot} \to \star}}
	\newcommand{\Ghattost}{G_{\star \to \hat{\cdot} }}
	\newcommand{\Gyoulpibar}{\bar{G}_{\mathrm{yla},\pi}}
	\newcommand{\nateta}{\mateta^{\mathrm{nat}}}

\newcommand{\Ahat}{\widehat{A}}

	\newcommand{\Bhat}{\widehat{B}}
	\newcommand{\Chat}{\widehat{C}}
	\newcommand{\matetahat}{\widehat{\boldsymbol{\eta}}}
	\newcommand{\Ghatin}{G_{\hat{\cdot},\mathrm{in}}}
	\newcommand{\Ghateta}{G_{\hat{\cdot},\eta}}

\newtoggle{arxivupload}
\togglefalse{arxivupload}
\newtoggle{colt}
\togglefalse{colt}

\newcommand\numberthis{\addtocounter{equation}{1}\tag{\theequation}}

\usepackage{multirow,bm}
\usepackage{makecell}
\usepackage{longtable}
\usepackage{threeparttable}
\usepackage{longtable}
\usepackage{booktabs}
\usepackage[referable]{threeparttablex}
\usepackage{footnote}

\usepackage[utf8]{inputenc}

\usepackage{amsmath,amsfonts,amssymb,mathrsfs}
\usepackage{mathtools,natbib}
\iftoggle{colt}
{
	\PassOptionsToPackage{vlined,linesnumbered,ruled}{algorithm2e}

}
{
	\usepackage[breaklinks=true]{hyperref}
	\usepackage[linesnumbered,ruled]{algorithm2e}
	\usepackage{fullpage}
	\usepackage{amsthm}

}
\usepackage{breakcites}

\hypersetup{colorlinks,citecolor=blue}
\usepackage{verbatim} 
\usepackage{latexsym}
\usepackage{relsize}
\usepackage{thm-restate}
\usepackage{appendix}

\usepackage{cleveref}
\Crefname{equation}{Eq.}{Eqs.}
\Crefname{assumption}{Assumption}{Assumptions}
\Crefname{condition}{Condition}{Conditions}

\usepackage{xcolor}
\usepackage{dsfont}

\numberwithin{equation}{section}

\newcommand{\Sigmatil}{\widetilde{\Sigma}}
\newcommand{\Gtil}{\widetilde{G}}

\newcommand{\ftil}{\tilde{f}}

\iftoggle{colt}
{

}
{

}

\newcommand{\drc}{\textsc{Drc}}
\newcommand{\drcgd}{\textsc{Drc-Gd}}
\newcommand{\drcgdex}{\textsc{Drc-Gd-Ex}}
\newcommand{\drcex}{\textsc{Drc-Ex}}

\newcommand{\Differential}{\mathsf{D}}
	
\newcommand{\yst}{\maty^{\star}}

\newcommand{\matxhat}{\widehat{\matx}}

\newcommand{\dmin}{d_{\min}}
\newcommand{\dmax}{d_{\max}}

\newcommand{\midhatytmini}{\mid \natyhat_{1:t-i}}

\newcommand{\midhatt}{\mid \Ghat,\,\natyhat_{1:t}}

\newcommand{\midhattyu}{\mid \Ghat,\,\natyhat_{1:t},\natuhat_{1:t}}

\newcommand{\midstyu}{\mid \Gst,\,\naty_{1:t},\natu_{1:t}}
\newcommand{\midst}{\mid \Gst,\,\naty_{1:t}}
\newcommand{\midsty}{\mid \naty_{1:t}}

\newcommand{\calM}{\mathcal{M}}

\newcommand{\midhaty}{\mid \natyhat_{1:t}}

\newcommand{\sigmaw}{\sigma_{\matw}}
\newcommand{\sigmae}{\sigma_{\mate}}

\newcommand{\gradient}{\boldsymbol{g}}

\newcommand{\K}{\ensuremath{\mathcal K}}
\newcommand{\M}{\ensuremath{\mathcal M}}

\def\regret{\mbox{{Regret\ }}}
\def\newmethod{\mbox{{Observation Gradient\ }}}

\newcommand{\ignore}[1]{}

\newcommand{\maxs}[1]{\noindent{\textcolor{blue}{\{{\bf MS:} \em #1\}}}}
\newcommand{\mscomment}[1]{\noindent{\textcolor{blue}{\{{\bf MS:} \em #1\}}}}

\DeclareMathOperator{\BigOm}{\mathcal{O}}

\newcommand{\BigOh}[1]{\BigOm\left({#1}\right)}
\DeclareMathOperator{\BigOmtil}{\widetilde{\mathcal{O}}}
\newcommand{\BigOhTil}[1]{\BigOmtil\left({#1}\right)}

\newcommand{\matM}{\mathbf{M}}

\newcommand{\Mtil}{\tilde{M}}

\newcommand{\Mst}{M_{\star}}

\newcommand{\Gstyu}{G_{\star,(\maty \to \matu)}}

\newcommand{\uex}{\matu^{\mathrm{ex}}}

\newcommand{\twonorm}[1]{\|#1\|_2}

\newcommand{\Gst}{G_{\star}}

\newcommand{\R}{\mathbb{R}}

\newcommand{\Mclass}{\mathcal{M}}

\newcommand{\wadv}{\matw^{\mathrm{adv}}}
\newcommand{\eadv}{\mate^{\mathrm{adv}}}
\newcommand{\estoch}{\mate^{\mathrm{stoch}}}
\newcommand{\wstoch}{\matw^{\mathrm{stoch}}}

\newcommand{\alphaloss}{\alpha_{\mathrm{loss}}}
\newcommand{\betaloss}{\beta_{\mathrm{loss}}}
\newcommand{\Lbar}{\overline{L}}
\newcommand{\Lgrad}{L_{\mathbf{g}}}

\newcommand{\Gstw}{G_{\star,\matw}}

  \newcommand{\Diag}{\mathrm{Diag}}
    \newcommand{\powtoep}{\mathsf{PowToep}}
\newcommand{\opnorm}[1]{\|#1\|_{\op}}

\newcommand{\Lc}{L_{\mathrm{c}}}
\newcommand{\urc}{\mathsf{uRC}}
\newcommand{\mrc}{\mathsf{wmRC}}

\newcommand{\Mapprox}{M_{\mathrm{apprx}}}

\newcommand{\Lf}{L_{\mathrm{f}}}

\newcommand{\ypred}{\maty^{\mathrm{pred}}}
\newcommand{\Fpred}{F^{\mathrm{pred}}}
\newcommand{\fpred}{f^{\mathrm{pred}}}
\newcommand{\upred}{\matu^{\mathrm{pred}}}
\newcommand{\Capprox}{C_{\mathrm{approx}}}

\newcommand{\N}{\mathbb{N}}
\newcommand{\mats}{\mathbf{s}}
\newcommand{\closedloop}{\mathrm{cl}}

\newcommand{\Htwo}{\mathcal{H}_{2}}
\newcommand{\Hinf}{\mathcal{H}_{\infty}}

\newcommand{\radGst}{R_{\Gst}}

\newcommand{\mata}{\mathbf{a}}

\newcommand{\natu}{\matu^{\mathrm{nat}}}

\newcommand{\Gstuy}{G_{\star,u \to y}}
\newcommand{\Gstuu}{G_{\star,u \to u}}
\newcommand{\natuhat}{\widehat{\matu}^{\mathrm{nat}}}

\newcommand{\Ghatuu}{\widehat{G}_{u \to u}}
\newcommand{\Ghatuy}{\widehat{G}_{u \to y}}
\newcommand{\yalg}{\maty^{\Alg}}
\newcommand{\ualg}{\matu^{\Alg}}

\newcommand{\uexalg}{\matu^{\mathrm{ex},\Alg}}
\newcommand{\uexM}{\matu^{\mathrm{ex},M}}

\newcommand{\matDel}{\boldsymbol{\Delta}}
\newcommand{\Gls}{\widehat{G}_{\mathrm{LS}}}

\newcommand{\Cest}{C_{\mathrm{est}}}
\newcommand{\Restu}{R_{\matu,\mathrm{est}}}

\newcommand{\Rubar}{\overline{R}_{\matu}}

\newcommand{\radpsi}{R_{\psi}}

\newcommand{\fhat}{\widehat{f}}

\newcommand{\compsuffix}{0}

\newcommand{\Mcomp}{M_{\compsuffix}}

\newcommand{\Alg}{\mathsf{alg}}
\newcommand{\calN}{\mathcal{N}}

\newcommand{\Ghat}{\widehat{G}}
\newcommand{\natyhat}{\widehat{\maty}^{\mathrm{nat}}}
\newcommand{\truenaty}{\maty^{\mathrm{nat}}}
\newcommand{\matx}{\mathbf{x}}

\newcommand{\matu}{\mathbf{u}}
\newcommand{\matw}{\mathbf{w}}
\newcommand{\maty}{\mathbf{y}}

\newcommand{\iidsim}{\overset{\mathrm{i.i.d}}{\sim}}
\newcommand{\Regret}{\mathrm{Regret}}

\newcommand{\veczero}{\mathbf{0}}

\newcommand{\dimy}{d_{y}}
\newcommand{\dimx}{d_{x}}
\newcommand{\dimu}{d_{u}}

\newcommand{\Cpicl}{C_{\pi,\closedloop}}
\newcommand{\Cpiclu}{C_{\pi,\closedloop,u}}
\newcommand{\Apicl}{A_{\pi,\closedloop}}
\newcommand{\Bpicl}{B_{\pi,\closedloop}}
\newcommand{\Bpiclw}{B_{\pi,\closedloop,w}}
\newcommand{\Bpicle}{B_{\pi,\closedloop,e}}
\newcommand{\Gpicl}{G_{\pi,\closedloop}}
\newcommand{\Dpicl}{D_{\pi,\closedloop}}

\newcommand{\dimpi}{d_{\pi}}
\newcommand{\Bst}{B_{\star}}
\newcommand{\Ast}{A_{\star}}
\newcommand{\Cst}{C_{\star}}
\newcommand{\Cpi}{C_{\pi}}
\newcommand{\Api}{A_{\pi}}
\newcommand{\Bpi}{B_{\pi}}
\newcommand{\Dpi}{D_{\pi}}

\newcommand{\matzpi}{\matz^{\pi}}
\newcommand{\matupi}{\matu^{\pi}}
\newcommand{\matypi}{\maty^{\pi}}
\newcommand{\matxpi}{\matx^{\pi}}
\newcommand{\natx}{\matx^{\mathrm{nat}}}
\newcommand{\naty}{\truenaty}

\newcommand{\plusbrack}[1]{(#1)_+}
\newcommand{\plusbrackpls}[1]{(#1 + 1)_+}

\newcommand{\lip}{L}

\newcommand{\psiGst}{\psi_{\Gst}}

 \newcommand{\DelG}{\Delta_{G}}
 \newcommand{\epsG}{\epsilon_G}

\newcommand{\mcomp}{m_{\compsuffix}}
\newcommand{\calMcomp}{\calM_{\compsuffix}}

\newcommand{\radM}{R_{\Mclass}}

\newcommand{\radnat}{R_{\mathrm{nat}}}

\newcommand{\raddist}{R_{\mathrm{dist}}}

\newcommand{\I}{\mathbb{I}}
\newcommand{\mate}{\mathbf{e}}

\newcommand{\matz}{\mathbf{z}}
\newcommand{\calG}{\mathcal{G}}

\newcommand{\op}{\mathrm{op}}
\newcommand{\loneop}{\mathrm{\ell_1,op}}

\newcommand{\fro}{\mathrm{F}}
\newcommand{\calK}{\mathcal{K}}
\newcommand{\calF}{\mathcal{F}}

\newcommand{\calD}{\mathcal{D}}

\iftoggle{colt}
{
	\newtheorem{assumption}{Assumption}
	\newtheorem{condition}{Condition}[section]
	\newcommand*{\qed}{\hfill\ensuremath{\square}}%

	\newtheorem{lemma}{Lemma}[section]
	\newtheorem{claim}[lemma]{Claim}
	\newtheorem{proposition}[lemma]{Proposition}
	\newtheorem{corollary}[lemma]{Corollary}
  	\newtheorem{definition}[lemma]{Definition}

}
{
	\theoremstyle{plain}
	\newtheorem{theorem}{Theorem}
	\newtheorem{thminformal}{Informal Theorem}

	\newtheorem{lemma}{Lemma}[section]
	\newtheorem{claim}[lemma]{Claim}
	
	\newtheorem{corollary}{Corollary}[section]
	\newtheorem{proposition}{Proposition}[section]

	\theoremstyle{definition}

	\newtheorem{definition}{Definition}[section]
	\newtheorem{example}{Example}[section]

	\newtheorem{assumption}{Assumption}
	\newtheorem{condition}{Condition}[section]

}

\makeatletter
\newcommand{\neutralize}[1]{\expandafter\let\csname c@#1\endcsname\count@}
\makeatother
\makesavenoteenv{algorithm}

  \newenvironment{lemmod}[2]
  {\renewcommand{\thelemma}{\ref*{#1}#2}%
   \neutralize{lemma}\phantomsection
   \begin{lemma}}
  {\end{lemma}}

\newenvironment{thmmod}[2]
  {\renewcommand{\thetheorem}{\ref*{#1}#2}%
   \neutralize{theorem}\phantomsection
   \begin{theorem}}
  {\end{theorem}}

    \newenvironment{propmod}[2]
  {\renewcommand{\theproposition}{\ref*{#1}#2}%
   \neutralize{proposition}\phantomsection
   \begin{proposition}}
  {\end{proposition}}

     \newenvironment{defnmod}[2]
  {\renewcommand{\thedefinition}{\ref*{#1}#2}%
   \neutralize{definition}\phantomsection
   \begin{definition}}
  {\end{definition}}

   \newenvironment{asmmod}[2]
  {\renewcommand{\theassumption}{\ref*{#1}#2}%
   \neutralize{assumption}\phantomsection
   \begin{assumption}}
  {\end{assumption}}

   \newenvironment{condmod}[2]
  {\renewcommand{\thecondition}{\ref*{#1}#2}%
   \neutralize{condition}\phantomsection
   \begin{condition}}
  {\end{condition}}

\newtheorem*{theorem*}{Theorem}
\newtheorem*{lemma*}{Lemma}
\newtheorem*{corollary*}{Corollary}
\newtheorem*{proposition*}{Proposition}
\newtheorem*{claim*}{Claim}
\newtheorem*{fact*}{Fact}
\newtheorem*{observation*}{Observation}

\newtheorem*{definition*}{Definition}
\newtheorem*{remark*}{Remark}
\newtheorem*{example*}{Example}

\newcommand{\nablatwo}{\nabla^{\,2}}

\newcommand{\diam}{\mathrm{Diam}}
\newcommand{\zst}{z_{\star}}
\newcommand{\llangle}{\left\langle}
\newcommand{\rrangle}{\right\rangle}

\newcommand{\stocherr}{\boldsymbol{\epsilon}^{\mathrm{stoch}}}
\newcommand{\adverr}{\boldsymbol{\epsilon}}

\DeclareMathAlphabet{\mathbfsf}{\encodingdefault}{\sfdefault}{bx}{n}

\DeclareMathOperator*{\argmin}{arg\,min}

\let\Pr\relax
\DeclareMathOperator{\Pr}{\mathbb{P}}

\newcommand{\ceil}[1]{\lceil #1 \rceil}
\newcommand{\floor}[1]{\lfloor #1 \rfloor}

\newcommand{\poly}{\mathrm{poly}}

\newcommand{\reals}{\mathbb{R}}
\newcommand{\eps}{\varepsilon}

\renewcommand{\leq}{~\le~}
\renewcommand{\geq}{~\ge~}

\let\oldtfrac\tfrac
\renewcommand{\tfrac}[2]{\smash{\oldtfrac{#1}{#2}}}

\let\nablaold\nabla
\renewcommand{\nabla}{\nablaold\mkern-2.5mu}
\newcommand{\loss}{\ell}

\newcommand{\Exp}{\mathbb{E}}
\newcommand{\matdel}{\boldsymbol{\delta}}

\newcommand{\sige}{\sigma_{\mate}}
\newcommand{\sigw}{\sigma_{\matw}}

\newcommand{\Z}{\mathbb{Z}}
\newcommand{\Gcheck}{\check{G}}
\newcommand{\Torus}{\mathbb{T}}
\newcommand{\eithe}{e^{\iota \theta}}
\newcommand{\Toep}{\mathsf{Toep}}
\newcommand{\ToepTrans}{\mathsf{ToepTrans}}
\newcommand{\symbN}{\boldsymbol{N}}
\newcommand{\fronorm}[1]{\left\|#1\right\|_{\fro}}
\newcommand{\ltwonorm}[1]{\left\|#1\right\|_{\ell_2}}
\newcommand{\Transfer}{\mathsf{Transfer}}
\newcommand{\calH}{\mathcal{H}}
\newcommand{\Hmin}{\calH_{\min}}
\newcommand{\Hminh}[1]{\calH_{[#1]}}

\newcommand{\omnorm}[2]{\Hminh{#2}[#1]}
\newcommand{\omhnorm}[1]{\Hminh{h}[#1]}
\newcommand{\hminnorm}[1]{\Hmin[#1]}

\newcommand{\herm}{\mathsf{H}}
\newcommand{\Dpinotnoise}{D_{\pi_0,\mathrm{noise}}}

\newcommand{\scrG}{\mathscr{G}}
\newcommand{\scrGdinout}{\mathscr{G}^{\dout\times \din}}
\newcommand{\scrGdin}{\mathscr{G}^{\din}}

\newcommand{\din}{d_{\mathsf{in}}}
\newcommand{\dout}{d_{\mathsf{o}}}
\newcommand{\diag}{\mathrm{diag}}

\newcommand{\Sigmanoise}{\Sigma_{\mathrm{noise}}}
\newcommand{\C}{\mathbb{C}}
\newcommand{\rmd}{\mathrm{d}}
\newcommand{\scrP}{\mathscr{P}}
\newcommand{\calC}{\mathcal{C}}
\newcommand{\scrC}{\mathscr{C}}

\newcommand{\Gstwey}{G_{ (w,e) \to y}}

\newcommand{\Gstweycheck}{\Gcheck_{ (w,e)\to y}}

\newcommand{\Bnoise}{B_{\pi_0,\mathrm{noise}}}

\newcommand{\Apinotcl}{A_{\pi_0,\mathrm{cl}}}
\newcommand{\Bpinotcl}{B_{\pi_0,\mathrm{cl}}}
\newcommand{\Cpinotcl}{C_{\pi_0,\mathrm{cl}}}
\newcommand{\Dpinotcl}{D_{\pi_0,\mathrm{cl}}}

\newcommand{\Alc}{A_{LC}}
\newcommand{\Abk}{A_{BK}}
\newcommand{\Alcch}{\check{A}_{LC}}
\newcommand{\Abkch}{\check{A}_{BK}}

\newcommand{\Apinot}{A_{\pi_0}}
\newcommand{\Bpinot}{B_{\pi_0}}
\newcommand{\Cpinot}{C_{\pi_0}}
\newcommand{\Dpinot}{D_{\pi_0}}

\newcommand{\ToepRow}{\mathsf{ToepRow}}

\newcommand{\omegacheck}{\check{\omega}}

\newcommand{\hinfnorm}[1]{\|#1\|_{\Hinf}}

\newcommand{\dtheta}{\rmd\theta}
\newcommand{\Gnoise}{G_{\mathrm{noise}}}
\newcommand{\Gnoisech}{\Gcheck_{\mathrm{noise}}}

\newcommand{\Omegah}{\mathscr{W}_h}
\newcommand{\Achk}{\check{A}_{K}}

\newcommand{\psipicl}{\psi_{\pi,\mathrm{cl}}}
\newcommand{\Adelpi}{A_{\Delta \pi}}
\newcommand{\Bdelpi}{B_{\Delta \pi}}
\newcommand{\Cdelpi}{C_{\Delta \pi}}
\newcommand{\Ddelpi}{D_{\Delta \pi}}
\newcommand{\matxbar}{\bar{\matx}}

\newcommand{\ybarpinot}{\bar{\maty}^{\pinot}}

\newcommand{\matybar}{\bar{\maty}}
\newcommand{\uin}{\matu^{\mathrm{in}}}

\newcommand{\adelpi}{\mata^{\Delta}}
\newcommand{\Gpinot}{G_{\pinot}}
\newcommand{\Bpibar}{\bar{B}_{\pi}}
\newcommand{\Dpibar}{\bar{D}_{\pi}}

\newcommand{\udel}{\matu^{\Delta}}

\newcommand{\adel}{\mata^{\Delta}}

\newcommand{\xdel}{\matx^{\Delta}}

\newcommand{\Gpinotcluyu}{G_{\pinot,\mathrm{cl},u \to (y,u)}}
\newcommand{\upinot}{\matu^{\pinot}}
\newcommand{\uen}{\mathring{\matu}}
\renewcommand{\yen}{\maty^{\mathrm{en}}}

\newcommand{\Gpinotyyu}{G_{\pinot,y \to (y,u)}}
\newcommand{\Gpipinotbar}{\bar{G}_{\pinot\to \pi}}

\newcommand{\Dpipinot}{D_{\pinot \to \pi}}

\newcommand{\Cdelpicl}{C_{\Delta \pi,\mathrm{cl}}}
\newcommand{\Adelpicl}{A_{\Delta \pi,\mathrm{cl}}}
\newcommand{\Bdelpicl}{B_{\Delta \pi,\mathrm{cl}}}

\newcommand{\Gpinotcl}{G_{\pinot,\mathrm{cl}}}
\newcommand{\pinot}{\pi_0}
\newcommand{\ypinot}{\maty^{\pinot}}
\newcommand{\wpinott}[1][t]{\mathbf{w}_{\pi_0,#1}}
\newcommand{\epinott}[1][t]{\mathbf{e}_{\pi_0,#1}}

\newcommand{\matxtil}{\widetilde{\matx}}

\newcommand{\Apipinot}{A_{\pinot \to \pi}}
\newcommand{\Bpipinot}{B_{\pinot \to \pi}}
\newcommand{\Cpipinot}{C_{\pinot \to \pi}}
\newcommand{\Gpipinot}{G_{\pinot \to \pi}}

\newcommand{\psipinot}{\psi_{\pinot}}
\newcommand{\psipinotcl}{\psi_{\pinot,\mathrm{cl}}}

\newcommand{\Cdelpiclu}{C_{\Delta \pi,\mathrm{cl},u}}
\newcommand{\Bdelpicle}{B_{\Delta \pi,\mathrm{cl},e}}

\newcommand{\Bpiclin}{B_{\pi,\mathrm{cl},\mathrm{in}}}

\newcommand{\sopen}{\mathring{\mats}}
\newcommand{\sopenalg}{\sopen^{\Alg}}
\newcommand{\uenalg}{\mathring{\matu}^{\Alg}}

\newcommand{\yin}{\maty^{\mathrm{in}}}
\newcommand{\uout}{\matu^{\mathrm{out}}}

\newcommand{\Bpinotu}{B_{\pinot,u}}
\newcommand{\Gpicleu}{G_{\pi,\mathrm{cl},e\to u}}
\newcommand{\matv}{\mathbf{v}}

\title{Improper Learning for Non-Stochastic Control}
\author{{Max Simchowitz \thanks{UC Berkeley. \url{msimchow@berkeley.edu}}}\and { Karan Singh\thanks{Princeton University and Google AI Princeton. \url{karans@princeton.edu}}} \and { Elad Hazan\thanks{Princeton University and Google AI Princeton. \url{ehazan@princeton.edu}}} }

\begin{document}
\maketitle


\begin{abstract}
    We consider the problem of controlling a possibly unknown linear dynamical system with adversarial perturbations, adversarially chosen convex loss functions, and partially observed states, known as non-stochastic control. We  introduce a controller parametrization based on the denoised observations, and prove that applying online gradient descent to this parametrization yields a new controller which attains sublinear regret vs. a large class of closed-loop policies. In the fully-adversarial setting, our controller attains an optimal regret bound of $\sqrt{T}$-when the system is known, and, when combined with an initial stage of least-squares estimation, $T^{2/3}$ when the system is unknown;  both yield the first sublinear regret for the partially observed setting. 

    Our bounds are the first in the non-stochastic control setting that compete with \emph{all} stabilizing linear dynamical controllers, not just state feedback. Moreover, in the presence of semi-adversarial noise containing both stochastic and adversarial components, our controller attains the optimal regret bounds of $\mathrm{poly}(\log T)$ when the system is known, and  $\sqrt{T}$ when unknown. To our knowledge, this gives the first end-to-end $\sqrt{T}$ regret for online Linear Quadratic Gaussian controller, and applies in a more general setting with adversarial losses and semi-adversarial noise.

\end{abstract}
\newcommand{\vminpt}{\vspace{-.1in}}
\newcommand{\coltvminpt}{\iftoggle{colt}{\vspace{-.1in}}{}}

\newcommand{\ball}{\mathcal{B}}
\newcommand{\Otilde}{\widetilde{\mathcal{O}}}
\newcommand{\polylog}{\mathrm{poly\,log}}
\newcommand{\Kst}{K_{\star}}

\section{Introduction}

In recent years, the machine learning community has produced a great body of work applying modern statistical and algorithmic techniques to classical control problems. Subsequently, recent work has turned to a more general paradigm termed the \emph{non-stochastic control problem}: a model for dynamics that replaces stochastic noise  with adversarial perturbations in the dynamics.

In this non-stochastic model,  it is impossible to pre-compute an instance-wise optimal controller.
Instead, the metric of performance is regret, or total cost compared to the best in hindsight given the realization of the noise. Previous work has introduced new adaptive controllers that are learned using iterative optimization methods, as a function of the noise, and are able to compete with the best controller in hindsight. 

This paper presents a novel approach to non-stochastic control which unifies, generalizes,  and improves upon existing results in the literature. Notably, we provide the first sublinear regret guarantees for non-stochastic control with partial observation for both \emph{known} and \emph{unknown} systems. Our non-stochastic framework also leads to new results for classical stochastic settings: e.g., the first tight regret bound for  linear quadratic gaussian control (LQG) with an unknown system.

The non-stochastic linear control problem is defined using the following dynamical equations:
\iftoggle{colt}
{
	\begin{align}
\matx_{t+1} = \Ast \matx_t + \Bst \matu_t + \matw_t, \quad  \quad \maty_t = \Cst \matx_t + \mate_t, \label{eq:LDS_system}
\end{align}
}
{
	\begin{align}
\matx_{t+1} &= \Ast \matx_t + \Bst \matu_t + \matw_t\nonumber \\
 \maty_t &= \Cst \matx_t + \mate_t, \label{eq:LDS_system}
\end{align}
}
where $\matx_t$ is the \iftoggle{colt}{system state}{state of the system}, $\matu_t$ the control, $\matw_t,\mate_t$ \iftoggle{colt}{}{are} adversarially-chosen noise terms, and $\maty_t$ \iftoggle{colt}{}{is}  the observation. A learner iteratively chooses a control $\matu_t$ upon observing $\maty_t$, and suffers a loss $c_t(\matx_t,\matu_t)$ according to an adversarially-chosen loss function. Regret is defined as the difference between the sum of costs and that of the best controller in hindsight \iftoggle{colt}{among a possible controller class.}{, taken from some class of possible controllers.}

Our technique a classical formulation based on the Youla parametrization \cite{youla1976modern} for optimal control, which rewrites the state in terms of what we term ``Nature's y's ',   observations $\{\naty_t\}$ \textit{that would have resulted had we entered zero control at all times}. This yields a convex parametrization approximating possible stabilizing controllers we call \emph{Disturbance Response Control}, or \drc. By applying online gradient descent to losses induced by this convex controller parametrization, we obtain a new controller we call the \emph{Gradient Response Controller via Gradient Descent}, or \emph{\drcgd}. We show that \drcgd{} attains wide array of results for stochastic and nonstochastic control, described in Section~\ref{ssec:statment_of_results}. Among the highlights:
\iftoggle{colt}
{
	\begin{enumerate}
    \item $\Otilde(\sqrt{T})$ for controlling a known system with partial observation in the non-stochastic control model (Theorem~\ref{thm:know}), and $\Otilde(T^{2/3})$ regret (Theorem~\ref{thm:unknown}) when this system is \emph{unknown}. This is the first sublinear regret bound for either setting, and the former rate is tight (Theorem~\ref{thm:lb_known}).
    \item
    The first $\Otilde(\sqrt{T})$ regret bound for the LQG unknown system (Theorem~\ref{thm:fast_rate_unknown}). This bound is tight, even when the state is observed \citep{simchowitz2020lqr}, and extends to mixed stochastic and adversarial perturbations (\emph{semi-adversarial}). We also give $\polylog 
    \,T$ regret for semi-adversarial control with partial observation when the system is known (Theorem~\ref{thm:fast_rate_known}).
    \item Regret bounds hold against the class of linear dynamical controllers (Definition~\ref{def:lin_output_feedback}),  a much richer class than static feedback controllers previously considered for the non-stochastic control problem. This class is necessary to encompass $\Htwo$ and $\Hinf$ optimal controllers under partial observation, and is ubiquitous in practical control applications.
\end{enumerate}

}
{
	\begin{enumerate}
    \item 
    We give an efficient algorithm for controlling a known system with partial observation in the non-stochastic control model with $\Otilde(\sqrt{T})$ regret (Theorem~\ref{thm:know}), and an algorithm with $\Otilde(T^{2/3})$ regret (Theorem~\ref{thm:unknown}) when this system is \emph{unknown}. This is the first sublinear regret bound for either setting, and our rate for the known system is tight, even in far more restrictive settings (Theorem~\ref{thm:lb_known}).
    \item
    We give the first $\Otilde(\sqrt{T})$ regret bound for the classical LQG problem with an unknown system (Theorem~\ref{thm:fast_rate_unknown}). This bound is tight, even when the state is observed \citep{simchowitz2020lqr}, and extends to mixed stochastic and adversarial perturbations (\emph{semi-adversarial}). We also give $\polylog 
    \,T$ regret for semi-adversarial control with partial observation when the system is known (Theorem~\ref{thm:fast_rate_known}).
    \item
    Our regret bounds hold against the class of linear dynamical controllers (Definition~\ref{def:lin_output_feedback}), which is a much richer class than static feedback controllers previously considered for the non-stochastic control problem. This more general class is necessary to encompass $\Htwo$ and $\Hinf$ optimal controllers under partial observation, and is ubiquitous in practical control applications.
\end{enumerate}
}
\textbf{Organization: } \iftoggle{colt}{
	  The remainder of this section formally defines the setting, describes our results, and surveys the related literature.  Section~\ref{sec:benchmark} expounds the relevant assumptions and describes our regret bound, and Section~\ref{sec:DRC} describes our controller parametrization. Section~\ref{sec:alg_and_main} presents our algorithm and main results. All proofs are deferred to the appendix, whose organization and notation is 
    detailed in \Cref{app:a}. \Cref{app:comparison} states lower bounds and provides extended comparison to past work, Part~\ref{part:high_level} gives high-level proofs of main results, and Part~\ref{part:detail} supplies additional proof details, and allows for possibly unstable systems placed in feedback with stabilizing controllers.
}
{
We proceed to formally define the setting (Sec.~\ref{ssec:problem_setting}), describe our results (Sec.~\ref{ssec:statment_of_results}) , and survey the related literature (Sec.~\ref{ssec:prior_work}).  Section~\ref{sec:benchmark} expounds the relevant assumptions and describes our regret bound, and Section~\ref{sec:DRC} describes our controller parametrization. Section~\ref{sec:alg_and_main} presents our algorithm and main results.

 Sections~\ref{sec:main_known_system},~\ref{sec:unknown},~\ref{sec:proof_fast_rate_known}, and~\ref{ssec:fast_rate_unknown_sketch} prove our main theorems in the order in which they are presented. Finally, we present concluding remarks in \Cref{sec:conclusion}. Additional proofs are deferred to the appendix, whose organization is detailed in \Cref{ssec:app_organization}; notation is summarized in \Cref{ssec:notation_stabilized,ssec:notation_stabilized}. Notably, \Cref{app:comparison} states lower bounds and provides extended comparison to past work; proofs in the appendix are written for a more general ``strongly-stabilized system'' setting detailed in \Cref{sec:generalization_to_stabilized}. }

\subsection{Problem Setting\label{ssec:problem_setting}}

\paragraph{Dynamical Model:} We consider  \emph{partially observed} linear dynamical system (PO-LDS), a continuous state-action, partially observable Markov decision process (POMDP) described by \Cref{eq:LDS_system}, with linear state-transition dynamics, where the observations are linear functions of the state. Here, $\matx_t,\matw_t \in \R^{\dimx}$, $\maty_t,\mate_t \in \R^{\dimy}$, $\matu_t \in \R^{\dimu}$ and $\Ast,\Bst,\Cst$ are of appropriate dimensions. 
We denote by $\matx_t$ the state, $\matu_t$ the control input, $\maty_t$ is the output sequence, and $\matw_t, \mate_t$ are perturbations that the system is subject to. A \emph{fully observed} linear dynamical system (FO-LDS) corresponds to the setting where $\Cst = I$ and $\mate_t \equiv 0$, yielding a (fully observed) MDP where $\matx_t \equiv \maty_t$. We consider both the setting where $\Ast$ is stable (\Cref{asm:a_stab}), and in \Cref{sec:generalization_to_stabilized}, unstable systems where the controller is put in feedback with stabilizing controllers 
\coltvminpt
\paragraph{Interaction Model:}  A control policy (or learning algorithm) $\Alg$ iteratively chooses an adaptive control input $\matu_t = \Alg_t(\maty_{1:t}, \matu_{1:t-1}, \ell_{1:t-1})$ upon the observation of the output sequence $(\maty_1,\dots \maty_t)$, and the sequence of loss functions $(\ell_1,\dots \ell_{t-1})$, past inputs, and possibly  internal random coins. 
Let $(\yalg_t,\ualg_t)$ be the observation-action sequence from this resultant interaction. The cost of executing this controller is
\iftoggle{colt}
{
	$J_T(\Alg) = \sum_{t=1}^T \ell_t(\maty_t^{\Alg}, \matu_t^{\Alg})$. 
}
{
	\begin{align}
J_T(\Alg) = \sum_{t=1}^T \ell_t(\maty_t^{\Alg}, \matu_t^{\Alg}). \label{eq:cumulative_loss}
\end{align}

}
Notice that the learning algorithm $\Alg$ does not observe the state sequence $\matx_t$. Furthermore, it is unaware of the perturbation sequence $(\matw_t, \mate_t)$, except as may be inferred from observing the outputs $\maty_t$. Lastly, the loss function $\ell_t$ is only made known to $\Alg$ once the control input $\ualg_t$ is chosen. Morever generally, our results extend to achieving low regret on loss functions that depend on a finite history of inputs and outputs, namely $\ell_t(\maty_{t:t-h},\matu_{t:t-h})$.
\coltvminpt
\paragraph{Policy Regret:} Given a \emph{benchmark class} of comparator control policies $\pi \in \Pi $, our aim is to minimize the cumulative regret with respect to the best policy in hindsight:
\begin{align}
\Regret_T(\Pi) := J_T(\Alg) - \min_{\pi\in \Pi} J_T(\pi) = \sum_{t=1}^T \ell_t(\yalg_t,\ualg_t) - \min_{\pi\in \Pi} \sum_{t=1}^T \ell_t(\maty_t^{\pi},\matu_t^{\pi})  \label{eq:regret_def}
\end{align}
Note that the choice of the controller in $\Pi$ may be made with the complete foreknowledge of the perturbations and the loss functions that the controller $\pi$ (and the algorithm $\Alg$) is subject to. In this work, we compete with benchmark class of stabilizing linear dynamic controllers (LDC's) with internal state (see \Cref{def:lin_output_feedback} and \Cref{sec:benchmark}). This generalizes the state-feedback class $\matu_t = K\matx_t$ considered in prior work.
\coltvminpt
\paragraph{Loss and Noise Regimes:} We consider both the \emph{known system} setting where $\Alg$ has foreknowledge of the system~\Cref{eq:LDS_system}, and the \emph{unknown system} setting where $\Alg$ does not (in either case, the comparator is selected with knownledge of the system). We also consider two loss and noise regimes: the \emph{Lipschitz loss \& non-stochastic noise} regime where the losses are Lipschitz over bounded sets (\Cref{asm:lip}) and noises bounded and adversarial (\Cref{asm:bound}), and the \emph{strongly convex loss \& semi-adversarial} regime where the losses are smooth and strongly convex (\Cref{asm:strong_convexity}, and noise has a well-conditioned stochastic component, as well as an oblivious, possibly adversarial one (\Cref{asm:noise_covariance_ind}).  We term this new noise model \emph{semi-adversarial}; it is analogous to smoothed-adversarial and semi-random models considered in other domains \citep{spielman2004smoothed, moitra2016robust, bhaskara2014smoothed}.  In the first noise regime, the losses are selected by an adaptive adversary; in the second, an oblivious one.
\coltvminpt
\paragraph{Relation to LQR, LQG, $\Htwo$ and $\Hinf$: } The online \emph{LQG} problem corresponds to the problem where the system is driven by well-conditioned, independent Gaussian noise, and the losses $\ell_t(y,u) = y^\top Q y + u^\top R u$ are fixed quadratic functions. \emph{LQR} is the fully observed analogue of LQG. The solution to the LQR (resp. LQG) problems are known as the $\Htwo$-optimal controllers, which are well-approximated by a fixed state feedback controller (resp. LDC). In context of worst-case control, the $\Hinf$ program can be used to compute a minimax controller that is optimal for the worst-case noise, and is also well-approximated by a LDC. In contrast to worst-case optimal control methods, low regret algorithms offer significantly stronger guarantees of instance-wise optimality on each noise sequence. We stress that the $\Hinf$ and $\Htwo$ optimal control for partially observed systems are LDCs controllers; state feedback suffices only for full observations.

\subsection{Contributions} \label{ssec:statment_of_results}

We present Disturbance Response Controller via Gradient Descent, or \drcgd, a unified algorithm which achieves sublinear regret for online control of a partially observed LDS with \emph{both} adversarial losses and noises, even when the true system is \emph{unknown} to the learner. In comparison to past work, this consitutes the first regret guarantee for partially observed systems (known or unknown to the learner) with \emph{either} adversarial losses or adversarial noises.  Furthermore, our bounds are the first in the online control literature which demonstrate low regret with respect to the broader class of linear dynamic controllers or \emph{LDCs} described above (see also \Cref{def:lin_output_feedback}); we stress that LDCs are necessary to capture the $\Htwo$ and $\Hinf$ optimal control laws under partial observation, and yield strict improvements under full observation for certain non-stochastic noise sequences. In addition, all regret guarantees are non-asymptotic, and have polynomial dependence on other relevant problem parameters. Our guarantees hold in four different regimes of interest, summarized in \Cref{table:result_summary}, and described below.  \Cref{table:comparison_known,table:comparison_unknown} in the appendix gives a detailed comparison to past work.

For \emph{known} systems, \Cref{alg:improper_OGD_main} attains $\Otilde(\sqrt{T})$ regret for Lipschitz losses and adversarial noise (\Cref{thm:know}), which we show in \Cref{thm:lb_known} is optimal up to logarithmic factors, even when the state is observed and the noises/losses satisfy quite restrictive conditions. For strongly convex losses and semi-adversarial noise, we achieve $\polylog\,T$ regret (\Cref{thm:fast_rate_known}). This result strengthens the prior art even for full observation due to \cite{agarwal2019logarithmic} by removing extraneous assumptions on the gradient oracle, handling semi-adversarial noise, and ensuring bounded regret (rather than pseudo-regret). We do so via a regret bounds for ``conditionally-strongly convex loses'' (\Cref{sec:proof_fast_rate_known}), which may be of broader interest to the online learning community. 

For \emph{unknown} systems, \Cref{alg:unknown} attains $\Otilde(T^{2/3})$ regret for Lipschitz losses and adversarial noise, and $\Otilde(\sqrt{T})$-regret for strongly convex losses and semi-adversarial noise (\Cref{thm:unknown,thm:fast_rate_unknown}). The former result has been established under full observation but required ``strong controllability'' \citep{Hazan2019TheNC}; the latter $\Otilde(\sqrt{T})$ bound is novel even for full observation. This latter model subsumes both LQG (partial observation) and LQR (full observation); concurrent work demonstrates that $\sqrt{T}$ regret is optimal for the LQR setting \citep{simchowitz2020lqr}. 
As a special case, we obtain the first (to our knowledge) $\Otilde(\sqrt{T})$ end-to-end regret guarantee for the problem of online LQG with an unknown system, even the stochastic setting\footnote{An optimal $\sqrt{T}$-regret for this setting can be derived by combining \cite{mania2019certainty} with careful state-space system identification results of either \cite{sarkar2019finite} or \cite{tsiamis2019finite}; we are unaware of work in the literature which presents this result. The \drc{} parametrization obviates the system identification subtleties required for this argument. \label{note1} }.  
Even with full-observation LQR setting, this is the first algorithm to attain $\sqrt{T}$ regret for either adversarial losses or semi-adversarial noise. This is also the first algorithm to obtain $\sqrt{T}$ regret \emph{without} computing a state-space representation, demonstrating that learning methods based on improper, convex controller parametrizations can obtain this optimal rate. Adopting quite a different proof strategy than prior work (outlined in \Cref{ssec:fast_rate_unknown_sketch}), our bound hinges in part on a simple, useful and, to our knowledge, novel fact\footnote{Robustness for the batch (fixed-objective) setting was demonstrated by \cite{devolder2014first}}: strongly convex online gradient descent has a \emph{quadratic} (rather than linear) sensitivity to adversarial perturbations of the gradients (\Cref{prop:robust_ogd}).
\vminpt
\paragraph{Disturbance Response Control} Our results are based on novel perspective on the classical Youla parametrization, called Disturbance Response Control (\drc). \drc{} affords seemless generalization to partially-observed system, competes with linear dynamic controllers, and, by avoiding state space representations, drastically simplifies our treatment of the unknown system setting. Our regret guarantees are achieved by a \emph{remarkably simple} online learning algorithm we term Disturbance Response Controller via Gradient Descent (\drcgd): estimate the system using least squares (if it is unknown), and then 
run online gradient descent on surrogate losses defined by this parametrization. 

In \Cref{sec:generalization_to_stabilized}, we present a generalization called \drcex, or Disturbance Response Control with Exogenous Inputs, which combines \emph{exogenous} dictated by the \drc{} parametrization with a nominal stabilizing controller. This allows us to leverage the full strength of the Youla parametrization, and extend our results to arbitrary stabilizable and detectable systems. As we explain, the classical Youla parametrization requires precise system knowledge to implement. In constrast, our Nature's Y's perspective allows yields a novel formulation which is implementable under  \emph{inexact} system knowledge.

\begin{center}
\begin{table}[h]
\centering
\begin{tabular}{|c|c|c|}
\hline
\multicolumn{3}{|c|}{\makecell{ \textbf{Regret Rate}}}\\
\hline
Setting & Known & Unknown\\
\hline
\makecell{General 
Convex Loss \\
Adversarial Noise} & \makecell{$\Otilde(\sqrt{T})$\\ (\Cref{thm:know})} & \makecell{$\Otilde(T^{2/3})$\\ (\Cref{thm:unknown})}\\
\hline
\makecell{Strongly Convex \& Smooth Loss\\ 
Adversarial + Stochastic Noise} & \makecell{$\polylog \,T$\\ (\Cref{thm:fast_rate_known})} & \makecell{$\Otilde(\sqrt{T})$\\ (\Cref{thm:fast_rate_unknown})}\\
\hline
\end{tabular}
\caption{Summary of our results for online control. \Cref{table:comparison_known,table:comparison_unknown} in \iftoggle{colt}{}{the }appendix compare with prior work.  }\label{table:result_summary}
\end{table}
\end{center}
\coltvminpt
\coltvminpt
\coltvminpt

\subsection{Prior Work\label{ssec:prior_work}}
\paragraph{Online Control.} The field of online and adaptive control is vast and spans decades of research, see for example \cite{sastry2011adaptive,ioannou2012robust} for survey. Here we restrict our discussion to online control with low \emph{regret}, which measures the total cost incurred by the learner compared to the loss she would have incurred by instead following the best policy in some prescribed class; comparison between our results and prior art is summarized in \Cref{table:comparison_known,table:comparison_unknown} in the appendix. To our knowledge, all prior end-to-end regret bounds are for the \emph{fully observed} setting; a strength of our approach is tackling the more challenging partial observation case. 
\coltvminpt
\paragraph{Regret for classical control models. }
We first survey relevant work that assume either no perturbation in the dynamics at all, or i.i.d. Gaussian perturbations. 
Much of this work has considered obtaining low regret in the online LQR setting \citep{abbasi2011regret,dean2018regret,mania2019certainty,cohen2019learning} where a fully-observed linear dynamic system is drive by i.i.d. Gaussian noise via $\matx_{t+1} = \Ast \matx_t + \Bst\matu_t + \matw_t$, and the learner incurs constant quadratic state and input cost $\ell(x,u) = \frac{1}{2}x^\top Q x + \frac{1}{2}u^\top R u$. The optimal policy for this setting is well-approximated by a \emph{state feedback controller} $\matu_t = \Kst\matu_t$, where $\Kst$ is the solution to the Discrete Algebraic Ricatti Equation (DARE), and thus regret amounts to competing with this controller. Recent algorithms \cite{mania2019certainty,cohen2019learning} attain $\sqrt{T}$ regret for this setting, with polynomial runtime and polynomial regret dependence on relevant problem parameters. Further, \cite{mania2019certainty} present technical results can be used to establish $\sqrt{T}$-regret for the partially observed \emph{LQG} setting (see Footnote~\ref{note1}). 
\iftoggle{colt}
{
}
{
	
}
A parallel line 
by  \cite{cohen2018online} establish $\sqrt{T}$ in  a variant of online LQR where the system is known to the learner, noise is stochastic, but an adversary selects quadratic loss functions $\loss_t$ at each time $t$. Again, the regret is measured with respect to a best-in-hindsight state feedback controller.
\iftoggle{colt}
{
}
{

}
Provable control in the Gaussian noise setting via the policy gradient method was  studied in \cite{fazel2018global}. Other relevant work from the machine learning literature includes the technique of spectral filtering for learning and open-loop control of partially observable systems \citep{hazan2017learning, arora2018towards, hazan2018spectral}, as well as prior work on tracking adversarial targets \citep{abbasi2014tracking}.
\vminpt

\paragraph{The non-stochastic control problem.}

The setting we consider in this paper was established in 
\cite{agarwal2019online}, who obtain $\sqrt{T}$-regret in the more general and challenging setting where the Lipschitz loss function \emph{and the perturbations } are adversarially chosen. The key insight behind this result is combining an improper controller parametrization know as disturbance-based control with recent advances in online convex optimization with memory due to \cite{anava2015online}. Follow up work by \cite{agarwal2019logarithmic} achieves logarithmic pseudo-regret for strongly convex, adversarially selected losses and well-conditioned stochastic noise. Under the considerably stronger condition of controllability, the recent work by \cite{Hazan2019TheNC} attains $T^{2/3}$ regret for adversarial noise/losses when the system is \emph{unknown}. Analogous problems have also been studied in the tabular MDP setting \citep{even2009online,zimin2013online,dekel2013better}.
\vminpt
\paragraph{Convex Parameterization of Linear Controllers} There is a rich history of convex or \emph{lifted} parameterizations of controllers.  Nature's y's is equivalent to input-ouput parametrizations \cite{zames1981feedback,rotkowitz2005characterization,furieri2019input}, and in \Cref{sec:generalization_to_stabilized}, we extend to more general parametrizations encompassing the classical Youla or Youla-Kuc\v era parametrization \citep{youla1976modern,kuvcera1975stability}, and approximations to the Youla parametrization which require only \emph{approximate} knowledge of the system. More recently, \cite{goulart2006optimization} propose a parametrization over state-feedback policies, and \cite{wang2019system} introduce a generalization of Youla called system level synthesis (SLS); SLS is equivalent to the parametrizations adopted by \cite{agarwal2019online} et seq., and underpins the $T^{2/3}$-regret algorithm of \cite{dean2018regret} for online LQR with an unknown system; one consequence of our work is that convex parametrizations can achieve the optimal $\sqrt{T}$ in this setting. However, it is unclear if SLS (as opposed to input-output or Youla) can be used to attain sublinear regret under partial observation and adversarial noise.
\vminpt
\paragraph{Online learning and online convex optimization.}
We make extensive use of techniques from the field of online learning and regret minimization in games \citep{cesa2006prediction, shalev2012online, OCObook}. Of particular interest are techniques for coping with policy regret and online convex optimization for loss functions with memory \citep{anava2015online}. 
\vminpt
\paragraph{Linear System Identification:} To adress unknown systems, we make use of tools from the decades-old field linear system identification~\citep{ljung1999system}. To handle partial observation and ensure robustness to biased and non-stochastic noise, we take up the approach in \cite{simchowitz2019learning}; other recent approaches include \citep{oymak2019non,sarkar2019finite,tsiamis2019finite,simchowitz2018learning}.

\newcommand{\Rgst}{R_{\Gst}}
\section{Assumptions and Regret Benchmark\label{sec:benchmark}}

 In the main text, we assume the system is stable:
 \begin{assumption}\label{asm:a_stab} We assume that is $\rho(\Ast) < 1$, where $\rho(\cdot)$ denotes the spectral radius. 
 \end{assumption}
In Appendix~\ref{sec:generalization_to_stabilized}, we detail generalizations which apply to stabilizable and detectable, but potentially unstable systems. For simplicty, we assume $\matx_0=0$; further, we assume:

\begin{assumption}[Sub-quadratic Lipschitz Loss]\label{asm:lip} There exists a constant $L > 0$ such that non-negative convex loss functions $\ell_t$ obey that for all $(\maty,\matu),(\maty'\matu') \in \R^{\dimy + \dimu}$, and for the choice $R = \max\{\|(\maty,\matu)\|_2, \|(\maty',\matu')\|_2, 1\}$,\footnote{This characterization captures, without loss of generality, any Lipschitz loss function. The $L \cdot R$ scaling of the Lipschitz constant captures, e.g. quadratic functions whose lipschitz constant scales with radius.}
\begin{align*}
|\ell_t (\maty',\matu') - \ell_t (\maty,\matu)| \le \lip R\left\|\begin{bmatrix} \maty - \maty'\\
\matu - \matu'\end{bmatrix}\right\|_2  \quad \text{ and } \quad 0 \le \ell_t (\maty,\matu) \le L R^2.
\end{align*}
\end{assumption}

\paragraph{Linear Dynamic Controllers} Previous works on fully observable LDS consider a policy class of linear controllers, where $u_t=-Kx_t$ for some $K$. Here, for partially observable systems, we consider a richer class of controller with an internal notion of state. Such a policy class is necessary to capture the optimal control law in presence of i.i.d. perturbations (the LQG setting), as well as, the $H_\infty$ control law for partially observable LDSs \citep{bacsar2008h}.

\begin{definition}[Linear Dynamic Controllers]\label{def:lin_output_feedback}  A \emph{linear dynamic controller}, or \emph{LDC}, $\pi$ is a linear dynamical system $(\Api,\Bpi,\Cpi,\Dpi) $, with internal state $\sopen_t \in \R^{\dimpi}$, input $\yin_t\in \R^{\dimy},\uin_t \in \R^{\dimu}$, output $\uout_t\in \R^{\dimu}$, equipped with the dynamical equations:
\begin{align}
\sopen_{t+1} = \Api \sopen_t + \Bpi \yin_t  \quad \text{and} \quad \uout_t := \Cpi \sopen_t + \Dpi \yin_t \label{eq:pi_dynamics}.
\end{align}
The \emph{closed loop} iterates $(\matypi_t,\matupi_t)$ are the unique sequence of iterates satisfying both the LDS dynamical equations \Cref{eq:LDS_system} with $(\maty_t,\matu_t)  = (\matypi_t,\matupi_t) $ and LDC dynamical equations \Cref{eq:pi_dynamics} with $\yin_t = \matypi_t$ and $\uout_t = \uin_t = \matupi_t$.
\end{definition}

The dynamics governing $(\matypi_t,\matupi_t)$ are described by an augmented LDS, detailed in detailed in Lemma~\ref{l:augd}. Note that the optimal LQR and LQG controllers take the above form. The class of policies that our proposed algorithm competes is defined in terms of the Markov operators of these induced dynamical systems.

\begin{definition}[Markov Operator]\label{def:markov}
 The associated Markov operator of a linear system $(A,B,C,D)$ is the sequence of matrices $G=(G^{[i]})_{i\ge 0}\in(\R^{d_\maty\times d_\matu})^\N$, where $G^{[0]}=D$ and $G^{[i]}=CA^{i-1}B$ for $i\geq 1$. Let $\Gst$ (resp. $\Gpicleu$) be the Markov operator of the nominal system $(\Ast, \Bst, \Cst, 0)$  (resp. \iftoggle{colt}{of }{of the closed loop system} $(\Apicl, \Bpicle, \Cpiclu, \Dpi)$, given explicity by \Cref{l:augd}). We let $\|G\|_{\loneop} := \sum_{i \ge 0}\opnorm{G^{[i]}}$. 
\end{definition}

\begin{definition}[Decay Functions \& Policy Class]\label{def:transferclass} We say $\psi: \N \to \R_{>0}$ is a \emph{proper decay function} if $\psi$ is non-increasing and $\lim_{n \to \infty} \psi(n) = 0$. Given a Markov operator $G$, we define its \emph{induced decay function} $\psi_G(n) := \sum_{i \ge n}\|\Gst^{[i]}\|_{\op}$. For proper decay funciton $\psi$, the class of all controllers whose induced \emph{closed-loop system} has decay bounded by $\psi$ is denoted as follows:
\begin{align*}
\Pi(\psi) := \left\{\pi: \forall n \ge 0, \,\psi_{\Gpicleu}(n) \le \psi(n)\right\}.
\end{align*}
We define $R_{\psi} := 1 \vee \psi(0)$ and $\Rgst := 1 + \psiGst(0)$, where  $\psiGst(0) = \sum_{i \ge 0}\|\Gst^{[i]}\|_{\op} = \|\Gst\|_{\loneop}$.
\end{definition}
Note that the class $\Pi(\psi)$ \emph{does not} require that the controllers be internally stable ($\rho(\Ast) < 1$), only that they induce stable closed-loop dynamics.  
The decay function captures the decay of the response of the system to past inputs, and is invariant to state-space representation. 
For stable systems $G$, we can always bound the decay functions by $\psi_G(m) \le C\rho^{m}$ for some constants $C > 0,\,\rho \in (0,1)$; this can be made quantitative for strongly-stable systems \citep{cohen2018online}. While we assume $\Gst$ exhibits this decay in the main text, our results naturally extend to the stabilized systems via the \drcex{} parametrization (\Cref{sec:generalization_to_stabilized}). 

\paragraph{Regret with LDC Benchmark} We are concerned with regret accumulated by an algorithm $\Alg$ as the excess loss it suffers in comparison to that of the best choice of a LDC with decay $\psi$, specializing \Cref{eq:regret_def} with $\Pi \leftarrow \Pi(\psi)$:
\begin{align}
\Regret_T(\psi) := J_T(\Alg) - \min_{\pi\in \Pi(\psi)} J_T(\pi) = \sum_{t=1}^T \ell_t(\yalg_t,\ualg_t) - \min_{\pi\in \Pi(\psi)} \sum_{t=1}^T \ell_t(\maty_t^{\pi},\matu_t^{\pi})  \label{eq:regret_def2}.
\end{align}
Note that the choice of the LDC in $\Pi$ may be made with the complete foreknowledge of the perturbations and the loss functions that the controller $\pi$ (and the algorithm $\Alg$) is subject to. We remark that the result in this paper can be easily extended to compete with controllers that have fixed affine terms (known as a \emph{DC offset}), or periodic (time-varying) affine terms with bounded period.

\newcommand{\coltfootnote}[1]{\iftoggle{colt}{\footnote{#1}}{#1}}
\section{Disturbance Response Control\label{sec:DRC}}
The induced closed-loop dynamics for a LDC $\pi$ involves feedback between the controller $\pi$ and LDS, which makes the cost $J(\pi)$ non-convex in $\pi$, even in the fully observed LQR setting \citep{fazel2018global}.\coltfootnote{This has motivated a long line of work to consider  control parameterizations for which $J(\pi)$ is convex  \citep{youla1976modern,zames1981feedback}. For non-stochastic control, \cite{agarwal2019online} consider a parametrization which selects inputs as linear functions of the disturbances $\matw_t$, which can be exactly recovered under a full state observation.  But under partial observation, the disturbances $\matw_t$ cannot in general be recovered (e.g. whenever $\Cst$ does not possess a left inverse).} 
\iftoggle{colt}{}{
	
}
We propose representing our controllers with the classical Youla parametrization, which both ensures convexity and is ammenable to partial observation. Our formulation emphasizes a novel perspective we call ``Nature's Y's'', which allows us to execute these Youla controllers in the non-stochastic setting.

\paragraph{Nature's $y$'s} Define $\naty_t$ as the corresponding output of the system in the absence of any controller. Note that the sequence \emph{does not} depend on the choice of control inputs $\matu_t$.  In the analysis, we shall assume that $1 \vee \max_t \|\naty_t\| \leq \radnat$. 
\iftoggle{colt}{}{
Note that by appropriately modifying the definition of $\radGst$, $\radnat=(1+R_{\Gst})\raddist$ is always a valid upper bound.
}

\begin{definition}[Nature's y's]\label{def:naty} Given a sequence of  disturbances $(\matw_t,\mate_t)_{t \ge 1}$, we define the \emph{natures $\maty$'s} as the sequence $\naty_t := \mate_t + \sum_{s=1}^{t-1} \Cst \Ast^{t-s-1}\matw_s$.
\end{definition}

Throughout, we assume that the noises selected by the adversary ensure $\naty_t$ are bounded
\begin{assumption}[Bounded Nature's $y$]\label{asm:bound}
We assume that that $\matw_t$ and $\mate_t$ are chosen by an oblivious adversary, and that $\twonorm{\naty_t} \le \radnat$ for all $t$.\iftoggle{colt}
{
\footnote{Note that, if the system is stable and perturbations bounded, that $\naty_t$ will be bounded for all $t$.}
}{}
\end{assumption}
\iftoggle{colt}{
}
{
Note that, if the system is stable and perturbations bounded, that $\naty_t$ will be bounded for all $t$. 
}
The next lemma shows for any fixed system with known control inputs the output is completely determined given Nature's y's, even if $\matw_t, \mate_t$ are not known. In particular, this implies that the one of the central observations of this work: 
\iftoggle{colt}
{
	\emph{Nature's y's can be computed exactly given just control inputs and the corresponding outputs of a system.}
	More precisely:
}
{
	\begin{quote} \emph{Nature's y's can be computed exactly given just control inputs and the corresponding outputs of a system.}
\end{quote}
More precisely:
}

\begin{lemma}\label{lem:supy} For any LDS $(\Ast, \Bst, \Cst)$ subject to (possibly adaptive) control inputs $\matu_1, \matu_2, \dots \in \R^{\dimu}$, the following relation holds for the output sequence: $ \maty_t = \naty_t + \sum_{i=1}^{t-1} \Gst^{[i]} \matu_{t-i}$. 
\end{lemma}
\begin{proof}
This is an immediate consequence of the definitions of Nature's y's and that of a LDS.
\end{proof}

\paragraph{Disturbance Response Control} In the spirit of \cite{zames1981feedback}, we  show that any linear controller can be represented by its action on Nature's $y$'s, and that this leads to a convex parametrization of controllers which approximates the performance of any LDC controller.

\begin{definition}[Distrubance Response Controller]\label{defn:drc}
A Disturbance Response Controller (\drc), parameterized by a $m$-length sequence of matrices $M = (M^{[i]})_{i=0}^{m-1}$, chooses the control input as $\matu^{M}_t = \sum_{s=0}^{m-1} M^{[s]} \naty_{t-s}$. We let $\maty^M_t$ denote the associated output sequence, and $J_T(M)$ the loss functional.
\end{definition}
Define a class of Distrubance Response Controllers with bound length and norm 
\iftoggle{colt}
{
	$\Mclass(m,R) = \{M = ( M^{[i]})_{i=0}^{m-1} : \|M\|_{\loneop}\leq R  \}.$
}
{
	\begin{align*} 
	\Mclass(m,R) = \{M = ( M^{[i]})_{i=0}^{m-1} : \|M\|_{\loneop}\leq R  \}.
	\end{align*}
}
Under full observation, the state-feedback policy $\matu_t= K\matx_t$ lies in the set of \drc s $\Mclass(1,\|K\|_{\op})$. The following theorem, proven  in \Cref{ssec:policy_approximation}, states that \emph{all stabilizing LDCs can be approximated by \drc s}: 
\begin{theorem}\label{thm:policy_approximation} For \iftoggle{colt}{}{a} proper decay function $\psi$,  $\pi \in \Pi(\psi)$, and any $m \ge 1$, there exists 
\iftoggle{colt}{$M\in \M(m,\radM)$ such that $J_T(M) - J_T(\pi) \le 2LT\radM\radGst^2\radnat^2 \,\psi(m)$}{an $M\in \M(m,\radM)$ such that
\begin{align}
J_T(M) - J_T(\pi) \le 2LT\radM\radGst^2\radnat^2 \,\psi(m).\label{eq:J_subopt}
\end{align}
}
\end{theorem}
As $\psi(m)$ typically decays exponentially in $m$, we find that for any stabilizing LDC, there exists a \drc{} that approximately emulates its behavior. This observation ensures it sufficient for the regret guarantee to hold against an appropriately defined Disturbance Response class, as opposed to the class of LDCs. 
%
Note that the fidelity of the approximation in~\Cref{thm:policy_approximation} depends only on the magnitude of the true system response $\Gst$, and \emph{decay} of the comparator system $\Gpicl$, but not on the order of a state-space realization. 
\Cref{thm:policy_approximation2} in the appendix extends \Cref{thm:policy_approximation} to the setting where $\Gst$ may be unstable, but is placed in feedback with a stabilizing controller.



\section{Algorithmic Description \& Main Result\label{sec:alg_and_main}}

\paragraph{OCO with Memory:} Our regret bounds are built on reductions to the  online convex optimization (OCO) with memory setting as defined by \cite{anava2015online}: at every time step $t$, an online algorithm makes a decision $x_t\in \calK$, after which it is revealed a loss function $F_t:\calK^{h+1} \to \R$, and suffers a loss of $F_t[x_t,\dots,x_{t-h}]$. The \emph{policy regret} is
\iftoggle{colt}
{
  defined as $\sum_{t=h+1}^T F_t(x,\dots,x_{t-h}) - \min_{x\in \K} F_t(x,\dots x)$.
} 
{
\begin{align*} 
\textrm{PolicyRegret} = \sum_{t=h+1}^T F_t(x,\dots,x_{t-h}) - \min_{x\in \K} F_t(x,\dots x)
\end{align*}
}
\cite{anava2015online} show that Online Gradient Descent on the \emph{unary specialization} $f_t(x) := F_t(x,\dots,x)$ achieves a sub-linear policy regret bound \iftoggle{colt}{(\Cref{thm:anava}).}{, quoted as \Cref{thm:anava} in \Cref{sec:main_known_system}.} 

\paragraph{Algorithm:} Non-bold letters $M_0,M_1,\dots$ denote function arguments, and bold letters $\matM_0,\matM_1,\dots$ denote the iterates produced by the learner. We first introduce a notion of \emph{counterfactual cost}
that measures the cost incurred at the $t^{th}$ timestep had a non-stationary distrubance feedback controller $M_{t:t-h}=(M_t,\dots,M_{t-h})$ been executed in the last $h$ steps: This cost is entirely defined by Markov operators and Nature's y's,  without reference to an explicit realization of system parameters.

\begin{definition}[Counterfactual Costs and Dyamics]\label{defn:counterfactual} Given $M_{t:t-h} \in \calM(m,\radM)^{h+1}$, we define 
\iftoggle{colt}
{
  $\matu_t\left(M_t \midhaty\right) := \sum_{i=0}^{m-1} M_t^{[i]}\cdot\natyhat_{t-i},$, $\maty_t\left[M_{t:t-h}\midhatt\right] := \natyhat_t + \sum_{i=1}^{h} \Ghat^{[i]} \cdot \matu_{t-i}\left(M_{t-i}\midhatytmini\right)$, and $F_t\left[M_{t:t-h} \midhatt\right] := \loss_t\left(\maty_t\left[M_{t:t-h}\midhatt\right] , \matu_t\left(M_t\midhaty\right)\right)$.
}
{
  \begin{align*}
  &\matu_t\left(M_t \midhaty\right) := \sum_{i=0}^{m-1} M_t^{[i]}\cdot\natyhat_{t-i},\\
  &\maty_t\left[M_{t:t-h}\midhatt\right] := \natyhat_t + \sum_{i=1}^{h} \Ghat^{[i]} \cdot \matu_{t-i}\left(M_{t-i}\midhatytmini\right), \quad  \\
  &F_t\left[M_{t:t-h} \midhatt\right] := \loss_t\left(\maty_t\left[M_{t:t-h}\midhatt\right] , \matu_t\left(M_t\midhaty\right)\right)
\end{align*}
}
Overloading notation, for a given $M \in \cal(M,\radM)$, we let $\maty_t(M \mid \cdot) := \maty_t[M,\dots,M \mid \cdot]$ denote the unary (single-$M$) specialization of $\maty_t$, and lower case $f_t\left(M|\cdot\right) = F_t\left[M,\dots,M|\cdot \right]$ the specialization of $F_t$. Throughout, we use paranthesis for unary functions of $M_t \in \calM(m,\radM)$, and brackets for functions of $M_{t:t-h} \in \calM(m,\radM)^{h+1}$.
\end{definition}

For known $\Gst$, \Cref{alg:improper_OGD_main} compute $\natyhat_t$ exactly, and we simply run online gradient descent on the costs $f_t(\cdot\mid \Gst,\naty_{1:t})$. When $\Gst$ is unknown, we invoke \Cref{alg:unknown}, which first dedicates $N$ steps to estimating $\Gst$ via least squares (\Cref{alg:estimation}), and then executes online gradient descent (\Cref{alg:improper_OGD_main}) with the resulting estimate $\Ghat$. The following algorithms are intended for \emph{stable} $\Gst$. Unstable $\Gst$ can be handled by incorporating a nominal stabilizing controller (\Cref{sec:generalization_to_stabilized}).


\begin{algorithm}[h!]
\textbf{Input: }  Stepsize $(\eta_t)_{t \ge 1}$, radius $\radM$, memory $m$, Markov operator $\Ghat$.\\
Define $\calM = \calM(m,\radM) = \{M=(M^{[i]})_{i=0}^{m-1} : \|M\|_{\loneop}\leq \radM\}$.\\
Initialize $\matM_1\in \mathcal{M}$ arbitrarily.\\
\For{$t= 1,\dots,T$}
{
  Observe $\yalg_t$ and determine $\natyhat_t$ as $\natyhat_t \leftarrow \yalg_t - \sum_{i=1}^{t-1} \Ghat^{[i]}\ualg_{t-i}.$ \footnotemark
   \\
  Choose the control input as  $ \ualg_t \leftarrow \matu_t\left(\matM_t \midhaty\right) = \sum_{i=0}^{m-1} \matM_t^{[i]} \,\natyhat_{t-i}. $\\
  Observe the loss function $\ell_t$ and suffer a loss of $\ell_t(y_t, u_t)$.\\
  Recalling $f_t(\cdot| \cdot)$ from \Cref{defn:counterfactual},update the disturbance feedback controller as $ \matM_{t+1} = \Pi_{\calM}\left(\matM_t - \eta_t \partial f_t\left(\matM_t \midhatt\right)\right) $, where $\Pi_{\calM}$ denotes projection onto $\calM$.\footnotemark
}
\caption{Disturbance Response Control via Gradient Descent (\drcgd)\label{alg:improper_OGD_main}}
\end{algorithm}
\addtocounter{footnote}{-2}
\stepcounter{footnote}\footnotetext{This step may be truncated to $\natyhat_t \leftarrow \yalg_t - \sum_{i=t-h}^{t-1} \Ghat^{[i]}\ualg_{t-i}$; these are identical when $\Ghat$ is estimated from \Cref{alg:estimation}, and the analysis can be extended to accomodate this truncation in the known system case}
\stepcounter{footnote}\footnotetext{ To simplify analysis, we project onto the $\loneop$-ball $\calM(m,R) := \{M=(M^{[i]})_{i=0}^{m-1} : \|M\|_{\loneop}\leq R\}$. While this admits an efficient implementation (\Cref{ssec:efficient_projection}),  in practice one can instead project onto outer-approximations of the set, just as a Frobenius norm ball containing $\calM(m,R) $, at the expense of a greater dependence on $m$. }
\begin{algorithm}
\textbf{Input: } Number of samples $N$, system length $h$. \\
\textbf{Initialize} $\Ghat^{[i]} = 0$ for $i \notin [h]$.\\
\textbf{For } $t = 1,2,\dots,N$, play $\ualg_t \sim \mathcal{N}(0,I_{\dimu})$. \\
Estimate $\Ghat^{[1:h]} \leftarrow \argmin \sum_{t=h+1}^N\|\yalg_t - \sum_{i=1}^h \Ghat^{[i]}\ualg_{t-i}\|_2^2$ via least squares, and return $\Ghat$.
\caption{Estimation of Unknown System\label{alg:estimation}}
\end{algorithm}

\begin{algorithm}[h!]
\textbf{Input: }  Stepsizes $(\eta_t)_{t \ge 1}$, radius $\radM$, memory $m$, rollout $h$, Exploration length $N$, \\
Run the estimation procedure (Algorithm~\ref{alg:estimation}) for $N$ steps with system length $h$ to estimate $\Ghat$\\
Run the regret minimizing algorithm (Algorithm~\ref{alg:improper_OGD_main}) for $T-N$ remaining steps with estimated Markov operators $\Ghat$, stepsizes $(\eta_{t+N})_{t \ge 1}$, radius $\radM$, memory $m$, rollout parameter $h$.
\caption{\drcgd{} for Unknown System\label{alg:unknown}}
\end{algorithm}

\subsection{Main Results for Non-Stochastic Control}
For simplicity, we assume a finite horizon $T$; extensions to infinite horizon can be obtained by a doubling trick. To simplify presentation, we will also assume the learner has foreknowledge of relevant decay parameters system norms. Throughout, let $d_{\min} = \min\{\dimy,\dimu\}$ and $d_{\max} = \max\{\dimy,\dimu\}$. We shall present all our results for general decay-functions, and further specialize our bounds to when the system and comparator exhibit explicity geometric decay, and where the noise satisfies subgaussian magnitude bound:
\newcommand{\sigmanoise}{\sigma_{\mathrm{noise}}}

\begin{assumption}[Typical Decay and Noise Bounds]\label{asm:noise_decay_bound} Let $C > 0$, $\rho \in (0,1)$ and $\delta \in (0,1)$.  We assume that $\radnat^2 \le d_{\max}\radGst^2\sigmanoise^2 \log (T/\delta)$\footnote{For typical noise models, the magnitude of the covariates  scales with output dimension, not internal dimension}\iftoggle{colt}{, }{. We further assume } that the system decay $\psiGst$ satisfies $\sum_{i \ge n} \|\Cst \Ast^{i}\|_{\op} \le \psiGst$, and that $\psiGst$ and the comparator $\psi$ satisfies $\psi(n),\psiGst(n) \le C\rho^n$. 
\end{assumption} 
We explain the above assumption, relations between parameters, and analogues for the strong-stabilized setting adressed in \Cref{ssec:asm:noise_decay_bound}. 
For known systems, our main theorem is proved in~\Cref{sec:main_known_system}:
\begin{theorem}[Main Result for Known System]\label{thm:know} Suppose \Cref{asm:a_stab,asm:lip,asm:bound} hold, and fix a decay function $\psi$. When Algorithm~\ref{alg:improper_OGD_main} is run with exact knowledge of Markov parameters (ie. $\hat{G}=\Gst$), radius $\radM \ge R_{\psi}$, parameters $m,h \ge 1$ such that $\psiGst(h+1) \lesssim \radGst/T$ and $\psi(m) \lesssim \radM/T$, and step size $\eta_t = \eta = \sqrt{\dmin}/4Lh\radnat^2\radGst^2\sqrt{2mT}$, we have\footnote{If the loss is assumed to be \emph{globably Lipschitz}, then the term $\radnat^2\radGst^2\radM^2$ can be improved to $\radnat\radGst\radM$. }
  \begin{align*}
  \Regret_T(\psi) &\lesssim L\radnat^2\radGst^2\radM^2\sqrt{h^2d_{\min}m} \cdot \sqrt{T}.
  \end{align*} 
  In particular, under \Cref{asm:noise_decay_bound}, we obtain $\Regret_T(\psi) \lesssim  \poly (C,\frac{1}{1-\rho},\log \frac{T}{\delta})\cdot \sigmanoise^2\sqrt{\dmin\dmax^{2} T} $.
\end{theorem}
\Cref{thm:lb_known} in the appendix shows that $\sqrt{T}$ is the optimal rate for the above setting. For unknown systems, we prove in \Cref{sec:unknown}:
\begin{theorem}[Main Result for Unknown System]\label{thm:unknown} 
Fix a decay function $\psi$, time horizon $T$, and confidence $\delta \in (e^{-T},T^{-1})$. Let $m,h$ satisfy $\psi(m) \le \radM/\sqrt{T}$ and $\psiGst(h+1) \le 1/10\sqrt{T}$, and suppose $\radM \ge R_{\psi}$ and $\radnat\radM \ge \sqrt{\dimu  + \log(1/\delta)}$. Define the parameters
\begin{align}
C_{\delta} := \sqrt{\dmax + \log \tfrac{1}{\delta} + \log (1+\radnat)}. \label{eq:Cdelta}
\end{align}
Then, if \Cref{asm:a_stab,asm:lip,asm:bound} hold, and \Cref{alg:unknown} is run with estimation length $N  = (Th^2 \radM \radnat C_{\delta})^{2/3}$ and parameters $m,h,\radM$, step size $\eta_t = \eta = \sqrt{\dmin}/4Lh\radnat^2\radGst^2\sqrt{2mT}$, and if $T \ge c'h^4 C_{\delta}^5 \radM^2\radnat^{2} + \dmin m^3$ for a universal constant $c'$,  then with probability $1 - \delta - T^{-\Omega(\log^2 T)}$,
  \begin{align*}
  \Regret_T(\psi) &\lesssim L\radGst^2\radM^2\radnat^2  \left(h^2  \radM \radnat C_{\delta}\right)^{2/3} \cdot T^{2/3}.
  \end{align*}
  In particular, under Assumption~\ref{asm:noise_decay_bound}, we obtain $\Regret_T(\psi) \lesssim  \poly (C,\sigmanoise^2,\frac{1}{1-\rho},\log \tfrac{T}{\delta})\cdot \dmax^{5/3}L T^{2/3} $.
\end{theorem}

\subsection{Fast rates under strong convexity \& semi-adversarial noise}
  
  We show that OCO-with-memory obtains improved regret the losses are strongly convex and smooth, and when system is excited by persistent noise. We begin with a strong convexity assumption:
  \begin{assumption}[Smoothness and Strong Convexity]\label{asm:strong_convexity} For all $t$, $\alphaloss\, \preceq \nablatwo \ell_t(\cdot,\cdot) \preceq \betaloss \, I$.
  \end{assumption}
  The necessity of the  smoothness assumption is explained further in \Cref{sec:proof_fast_rate_known}. Unfortunately, strongly convex losses are not sufficient to ensure strong convexity of the unary functions $f_t(M \mid \cdot)$. Generalizing \cite{agarwal2019logarithmic}, we assume an \emph{semi-adversarial} noise model where disturbances decomposes as
  \iftoggle{colt}
  {
     $\matw_t = \wadv_t + \wstoch_t$ and $
  \mate_t = \eadv_t + \estoch_t$, 
  }
  {
      \begin{align*}
  \matw_t = \wadv_t + \wstoch_t \quad \text{and}\quad
  \mate_t = \eadv_t + \estoch_t,
  \end{align*}
  }
  where $\wadv_t$ and $\eadv_t$ are an adversarial sequence of disturbances, and $\wstoch_t$ and $\estoch_t$ are stochastic disturbances which provide persistent excitation. We make the following assumption:
  \begin{assumption}[Semi-Adversarial Noise]\label{asm:noise_covariance_ind} The sequences $\wadv_t$ and $\eadv_t$ and losses $\ell_t$ are selected by an oblivious adverary. Moreover, $\wstoch_1,\dots,\wstoch_T$ and $\estoch_1,\dots,\estoch_T$ are independent random variables, with $\Exp[\wstoch_t] = 0$, $\Exp[\estoch_t] = 0$ and
  \begin{align*}
  \Exp[\wstoch_t\, (\wstoch_t)^\top] \succeq \sigw^2 I, \quad \text{ and } \quad \Exp[\estoch_t(\estoch_t)^\top] \succeq \sige^2 I.
  \end{align*}
  \end{assumption}
  This assumption can be generalized slightly to require only a martingale structure (see \Cref{asm:noise_covariance_mart}). Throughout, we shall also assume bounded noise. Via truncation arguments, this can easily be extended to light-tailed excitations (e.g. Gaussian) at the expense of additional logarithmic factors, as in \Cref{asm:noise_decay_bound}. For known systems, we obtain the following bound, which we prove in \Cref{sec:proof_fast_rate_known}:
  \begin{theorem}[Logarithmic Regret for Known System]\label{thm:fast_rate_known} Define \iftoggle{colt}{ the effective strong convexity parameter $ \alpha_f:= \alphaloss \cdot \left(\sige^2 + \sigw^2 \left(\frac{\sigma_{\min}(\Cst)}{ 1+\|\Ast\|_{\op}^2}\right)^2\right) $}{the effective strong convexity parameter
  \begin{align}
  \alpha_f:= \alphaloss \cdot \left(\sige^2 + \sigw^2 \left(\frac{\sigma_{\min}(\Cst)}{ 1+\|\Ast\|_{\op}^2}\right)^2\right) \label{eq:alphaf_stable_sys}
  \end{align}
  }
  and assume \Cref{asm:a_stab,asm:strong_convexity,asm:bound,asm:noise_covariance_ind,asm:lip} hold.
  For a decay function $\psi$, if \Cref{alg:improper_OGD_main} is run with $\Ghat = \Gst$, radius $\radM \ge R_{\psi}$, parameters $1 \le h \le m$ satisfying $\psiGst(h+1) \le \radGst/T$, $\alpha \le \alpha_{f}$, $\psi(m) \le \radM/T$, $T \ge \alpha m \radM^2$, and step size $\eta_t = \frac{3}{\alpha t}$, we have that with probability $1-\delta$,
    \begin{align}
   \Regret_T(\psi) \lesssim \frac{L^2 m^{3}\dmin \radnat^4\radGst^4\radM^2}{\min\left\{\alpha, L\radnat^2 \radGst^2\right\}} \left(1 + \frac{\betaloss }{L \radM}\right)\cdot\log\frac{T}{\delta}\label{eq:str_convex_known}.
    \end{align}
      \iftoggle{colt}{Under }{In particular, under} Assumption~\ref{asm:noise_decay_bound}, we have $\Regret_T(\psi) \lesssim  \frac{L^2}{\alpha}\dmax^3 ( 1+ \betaloss/L) \poly (C,\sigmanoise^2,\frac{1}{1-\rho},\log \tfrac{T}{\delta})$.
  \end{theorem}
  Finally, for unknown systems, we show in \Cref{ssec:fast_rate_unknown_sketch} that \Cref{alg:unknown} attains optimal $\sqrt{T}$ regret:
  \begin{theorem}[$\sqrt{T}$-regret for Unknown System]\label{thm:fast_rate_unknown} Fix a decay function $\psi$, time horizon $T$, and confidence $\delta \in (e^{-T},T^{-1})$. Let $m \ge 3h \ge 1$  satisfy $\psi(\floor{\frac{m}{2}} - h) \le \radM/T$ and $\psiGst(h+1) \le 1/10T$, and suppose $\radM \ge 2R_{\psi}$ and $\radnat\radM \ge (\dimu  + \log(1/\delta))^{1/2}$. Finally, let $\alpha \ge \alpha_f$ for $\alpha_f$ as in \Cref{thm:fast_rate_known}, and $C_{\delta}$ as in \Cref{thm:unknown}. Then, if \Cref{asm:a_stab,asm:strong_convexity,asm:bound,asm:noise_covariance_ind,asm:lip} hold, and Algorithm~\ref{alg:unknown} is run with parameters $m,h,\radM$, step sizes $\eta_t = \frac{12}{\alpha t}$, appropriate $N$ and $T$ sufficiently large {\normalfont(Eq. \eqref{eq:T_and_N})}, we have with probability $1 - \delta - T^{-\Omega(\log^2 T)}$, 
  \iftoggle{colt}
  {
    \begin{align*}\Regret_T(\psi) \lesssim \text{\normalfont (RHS of Eq.~\eqref{eq:str_convex_known})} +  Lm^{1/2}h^2\radGst^3 \radM^3 \radnat^3 C_{\delta}  \left(1 + \tfrac{L}{\alpha } + \tfrac{\betaloss^2}{\alpha L}\right)^{1/2}\cdot \sqrt{T}
   \end{align*}
  }
  {
    \begin{align*}\Regret_T(\psi) \lesssim \text{\normalfont (RHS of Eq.~\eqref{eq:str_convex_known})} +  Lmh^2\radGst^3 \radM^3 \radnat^3 C_{\delta}  \sqrt{1 + \frac{L}{\alpha } + \frac{\betaloss^2}{\alpha L}}\cdot \sqrt{T}
   \end{align*}
  }
   

  \iftoggle{colt}
  {
    Under Assumption~\ref{asm:noise_decay_bound}, $\Regret_T(\psi) \lesssim  \poly (C,L,\betaloss,\frac{1}{\alpha},\sigmanoise^2,\frac{1}{1-\rho},\log \tfrac{T}{\delta})\cdot (\dmax^{2}\sqrt{T} + \dmax^3)$.
  }
  {
  In particular, under Assumption~\ref{asm:noise_decay_bound}, we obtain
    \begin{align*}\Regret_T(\psi) \lesssim  \poly (C,L,\betaloss,\frac{1}{\alpha},\sigmanoise^2,\frac{1}{1-\rho},\log \tfrac{T}{\delta})\cdot (\dmax^{2}\sqrt{T} + \dmax^3).
  \end{align*}
  }
  \end{theorem}
  \iftoggle{colt}
  {Due to space limitiations, examples for the LQR and LQG settings are deferred to \Cref{ssec:lqr_lqg}, and concluding remarks are provided in \Cref{sec:conclusion}.
  }
  {

We now demonstrate how our results specialize to the LQR and LQG settings:
\begin{example}[LQR]\label{exm:lqr} In the LQR setting, the observable state $\matx_t$ evolves as $\matx_{t+1}=\Ast \matx_t +\Bst\matu_t + \matw_t$, where $\matw_t\sim\mathcal{N}(0,\sigw^2)$, and the associated losses are fixed quadratics $l(\matx,\matu) = \matx^\top Q\matx + \matu^\top R\matu$. The optimal control\footnote{In strict terms, this is only true for the infinite horizon case. However, even in the finite horizon setting, such a control law (utilizing the infinite horizon controller) is at most $\log T$ sub-optimal additively.} is expressible as $\matu_t = -K\matx_t$ (trivially an LDC). Our framework realizes this setting by choosing $\Cst = I$ and $\sige^2 = 0$ (observations are noiselss). The strong convexity parameter is then $\alpha_f = \frac{\sigw^2}{(1+\|\Ast\|_{\op}^2)^2}\alpha_{loss}$, which degrades with the norm of $\Ast$, but does not vanish even as $\Ast$ becomes unstable. For LQR, Theorem~\ref{thm:fast_rate_unknown} guarantees a regret of $\tilde{O}(\sqrt{T})$ matching the previous results \citep{cohen2019learning,mania2019certainty}\iftoggle{colt}{.}{; the latter too require strong convexity of the loss functions in addition to the losses being quadratic.}
  \end{example}
  
   \begin{example}[LQG]\label{exm:lqg} In the LQG setting, the state evolves as Equation~\ref{eq:LDS_system}, where $\matw_t\sim\mathcal{N}(0,\sigw^2),\mate_t\sim\mathcal{N}(0,\sige^2)$, and the associated losses are fixed quadratics $l(\maty,\matu) = \maty^\top Q\maty + \matu^\top R\matu$. The optimal control for a known system may be obtained via the separation principle~\citep{bertsekas2005dynamic}, which involves the applying the LQR controller on a latent-state estimate $\matxhat_t$ obtained via Kalman filtering. This can be expressed as (see e.g. \citet[Eq.13]{mania2019certainty})
   \begin{align*}
       \matxhat_{t+1} = \Ast \matxhat_{t}  + \Bst \matu_t + L(\maty_{t}-\Cst \matx_t); \qquad
       \matu_t = K\matxhat_t
   \end{align*}
   Hence, $\matxhat_{t+1}  = A_{\pi}\matxhat_t$ for $\Api = \Ast + \Bst K - L \Cst$. This yields an LDC with $D_{\pi} = 0$, and $B_{\pi} = L$ and $C_{\pi} = K$. For an unknown LQG system, Theorem~\ref{thm:fast_rate_unknown} guarantees a regret of $\tilde{O}(\sqrt{T})$.
  \end{example}
  
  We remark both of the above examples can be extended to the setting where $\Gst$ may be unstable, but is placed in feedback with a known stabilizing controller (\Cref{a:stablize}) via Theorems~\Cref{thm:fast_rate_known_stablz,thm:fast_rate_unknown_stablz}; assumption of such a stabilizing control is standard in the LQR setting. We note that for general partially-observed stabilized settings, the strong convexity modulus is somewhat more opaque, but still yields $\BigOhTil{\sqrt{T}}$ regret asymptotically.

  \subsection{Extensions}
  Our framework easily admits many extensions, which we sketch here:
  \begin{itemize}
  	\item \textbf{Functions of ``histories'':} Our OCO-with-memory framework can easily be extended to accomodate losses $\ell_t$ which depend on histories of past outputs and inputs; that is, loss functions of the form $\ell_t(\maty_{t:t-\tau},\matu_{t:t-\tau})$ for some fixed $\tau \in \mathbb{N}$. Here, we would require that $\ell_t$ satisfy appropriate Lipschitz and quadratic growth properties (\Cref{asm:lip}), and the ``unary specialization''  $\ell_t(y,y,\dots,y,u,u,\dots,u)$ is convex. Functions of past histories can be used to capture notions like costs that depend on rates of change:  for example if $\maty_t$ is the position of the system, $\ell_t(\maty_t,\maty_{t-1}) = \|\maty_t - \maty_{t-1}\|^2$ penalizes instanteous velocity. 
  	\item \textbf{Combining Open and Closed Loop Control Policies:} While our theoretical guarantees  consider a benchmark of inputs $\matu^{\pi}_t$ selected from a closed-loop LDCs (\Cref{def:lin_output_feedback}), we can also allow for open-loop components as well. For example, for a fixed $k \in \K$, and fixed functions $\Psi_1,\dots,\Psi_k$  of $t \in [T]$, we  can compete with policies of the form $\matu^{\pi}_t + \sum_{i=1}^k \alpha_i \Psi_i(t) $,  where $\matu^{\pi}_t$ is dictacted by an LDC $\pi$, and $\alpha_1,\dots,\alpha_n$ are arbitrary (though boudned) coefficients. In particular, we can compete with the superpositions of LDC controllers and finite sums of open loop sine and cosine inputs. The addition of open loop input may be useful for certain tasks, like tracking a reference signal. 
  	\item \textbf{Non-linear features:} In \Cref{thm:policy_approximation}, we show that \drc{} controllers $M$ are essentially in one-to-one correspondence with LDC controllers $\pi$. However, rather than selecting inputs $\matu^M_t := \sum_{i=0}^{m-1} M^{[i]}\naty_{t-i}$, we can in fact select \emph{non-linear features} $\matu^{M;\Psi}_t := \sum_{i=0}^{m-1} M^{[i]}\Psi_i(\naty_{t-i},t)$ where $\Psi_1(\cdot,\cdot),\dots,\Psi_m(\cdot,\cdot)$ are any fixed, potentially non-linear features maps which themselves may vary with $t$. These feature maps can potentially provide richer controller policies, which in practice may lead to better performance on certain control tasks (depending on the structure of the losses and noise). In particular, this can be used as part of a pipeline where first useful control features are learned via another procedure, such as a deep neural network.
  \end{itemize} }


\section{Analysis for Known System\label{sec:main_known_system}}
  
In this section, we prove \Cref{thm:know}. We begin with the following regret decomposition, for simplicity, we abbreviate $\calM \leftarrow \calM(m,\radM)$:
  \begin{align}
  &\Regret_T(\Pi(\psi)) = \sum_{t=1}^T  \loss_t(\yalg_t,\ualg_t) - \inf_{\pi \in \Pi(\psi)} \sum_{t=1}^T \ell_t(\maty^\pi_t,\matu^\pi_t) \\
&   \le \underbrace{\left(\sum_{t=1}^{m +h} \loss_t(\yalg_t,\ualg_t)\right)}_{\text{burn-in loss}} + \underbrace{ \left( \sum_{t=m+h+1}^T \loss_t(\yalg_t,\ualg_t) - \sum_{t=m+h+1}^T F_t[\matM_{t:t-h} \midst]\right)}_{\text{algorithm truncation error}}\nonumber\\
  &\quad+\underbrace{\left(\sum_{t=m+h+1}^T F_t[\matM_{t:t-h} \midst] - \inf_{M \in \calM}\sum_{t=m+h+1}^{T}f_t(M\midst )\right)}_{\text{\color{blue} $f$ policy regret}} \nonumber
  \end{align}
  \begin{align}
  &\qquad + \underbrace{\left(\inf_{M \in \calM}  \sum_{t=m+h+1}^{T}f_t(M \midst ) - \inf_{M \in \calM} \sum_{t=m+h+1}^{T}\ell_t(\maty^{M}_t,\matu^{M}_t)\right)}_{\text{comparator truncation error}}  \\
  &\qquad+  \underbrace{\inf_{M \in \calM} \sum_{t=1}^{T}\ell_t(\maty^{M}_t,\matu^{M}_t) - \inf_{\pi \in \Pi(\psi)}\ell_t(\maty^\pi_t,\matu^\pi_t)}_{\color{blue} \text{policy approximation } :=  J(M) - J(\pi)} \label{eq:known_weak_convex_regret_decomposition}
  \end{align}
  Here, the burn-in captures rounds before the algorithm attains meaningful regret guarantees, the truncation errors represent how closely the counterfactual losses track the losses suffered by the algorithm (algorithm truncation error), or those suffered by the algorithm selecting policy $\pi = M$.  The dominant term in the above bound in the {\color{blue} $f$ policy regret}, which we bound using the OCO-with-Memory bound from \Cref{thm:anava}. Lastly, the {\color{blue} policy approximation error} measures how well finite-history policies $M \in \calM$ approximate LDC's  $\pi \in \Pi(\psi)$, and is adressed by \Cref{thm:policy_approximation}; this demonstrates the power of the nature's $y$ parametrization.

  We shall now bound the regret term-by-term. All subsequent bounds hold in the more general setting of stabilized-systems (defined in \Cref{sec:generalization_to_stabilized}), and all ommited proofs are given in \Cref{ssec:app_ommited_known}. Before beginning, we shall need a uniform bound on the magnitude of $\ualg_t$ and $\yalg_t$. This is crucial because the magnitudes and Lipschitz constants of the losses $\ell_t$ depend on the magnitudes of their arguments:
  \begin{lemma}[Magnitude Bound]\label{lem:magnitude_bound_simple} For all $t$, and $M,M_{1},M_2,\dots \in \calM$, we ahve
  \begin{align*}
  &\max\left\{\left\|\ualg_t\right\|_2,\left\|\matu^{M}_t\right\|_2, \left\|\matu_t(M_t \midsty)\right\|_2\right\}\le \radM\radnat\\
  &\max\left\{\left\|\begin{bmatrix}\yalg_t \\\ualg_t\end{bmatrix}\right\|_2,\left\|\begin{bmatrix}\maty^{M}_t\\\matu_t^M\end{bmatrix}\right\|_2,\left\|\begin{bmatrix}\maty_t[M_{t:t-h} \midst]\\\matu_t[M_{t:t-h} \midst]\end{bmatrix}\right\|_2\right\} \le 2\radGst\radM\radnat
  \end{align*}
  \end{lemma}
  \begin{proof} The proof is a special case of \Cref{lem:magnitude_bound_stabilized} in the appendix.
  \end{proof}
  The above lemma directly yields a bound on the first term of the regret decomposition~\eqref{eq:known_weak_convex_regret_decomposition}:
  \begin{restatable}{lemma}{lemburninboundknown}\label{lem:burn-in-bound_known}We have that  $\text{\normalfont (burn-in loss)} \le 4L\radGst^2\radM^2\radnat^2(m+h)$
  \end{restatable}
  \begin{proof} Combine \Cref{asm:lip} on the loss, \Cref{lem:magnitude_bound_simple}, and the fact $\radGst,\radM,\radnat \ge 1$.
  \end{proof}
  The algorithm and comparator truncation errors represent the extent to which the $h$-step truncation differs from the true losses induced by the algorithm: 
  \begin{restatable}[Bound on Truncation Errors]{lemma}{lemtruncationcostsknown}
  \label{lem:truncation_costs_known} We can bound
  \begin{align*}
  \text{\normalfont (algorithm truncation error)} + \text{\normalfont(comparator truncation error)} \le 4LT\radGst\radM^2\radnat^2\psiGst(h+1).
  \end{align*}
  \end{restatable}

Now, we turn the the $f$-regret. We begin by quoting a result of \cite{anava2015online}:
\begin{proposition}\label{thm:anava}
For any a sequence of $(h+1)$-variate $F_t$, define $f_t(x) = F_t(x,\dots x)$. Let $L_{c}$ be an upper bound on the coordinate-wise Lipschitz constant of $F_t$, $L_f$ be an upper bound on the Lipschitz constant of $f_t$, and $D$ be an upper bound on the diameter of $\K$. Then, the sequence $\{x_t\}_{t=1}^T$ produced by executing OGD on the \textit{unary} loss functions $f_t$ with learning rate $\eta$ guarantees
\begin{align*} \text{\normalfont PolicyRegret} := \sum_{t=h+1}^T F_t(x,\dots,x_{t-h}) - \min_{x\in \K} F_t(x,\dots x) \leq \frac{D^2}{\eta} + \eta T\cdot (L_f^2 + h^2 L_c L_f).\end{align*} 
\end{proposition}

In order to apply the OCO reduction, we need to bound the appropriate Lipschitz constants. Notice that, in order to apply projected gradient descent, we require that the functions $f_t$ are  Lipschitz in the Euclidean (i.e., Frobenius) norm:
\begin{restatable}[Lipschitz/Diameter Bounds]{lemma}{lemLipschitzconstantknown}\label{lem:label_Lipschitz_constants_known} Define $L_f := L\sqrt{m}\radnat^2\radGst^2\radM$. Then, 
\begin{itemize}
\item The functions $f_t(\cdot \midst)$ are $L_f$-Lipschitz
\item The functions $F_t[M_{t:t-h} \midst]$ are $L_f$-coordinate-wise Lipschitz on $\Mclass$ in the Frobenius norm $\|M\|_{\fro} = \|[M^{[0]},\dots,M^{[m-1]}]\|_{\fro}$.
\item the Euclidean diameter of $\Mclass$ is at most $D \le 2\sqrt{\dmin}\radM$.
\end{itemize}
\end{restatable}
We now bound the policy regret by appealing to the OCO-with-Memory guarante, \Cref{thm:anava}:
\begin{lemma}[Bound on the $f$-policy regret]\label{lem:f_regret_known_simple} Let $\dmin = \min\{\dimu,\dimy\}$, and $\eta_t = \eta = \sqrt{\dmin}/4Lh\radnat^2\radGst^2\sqrt{2mT}$ for all $t$. Then,
\begin{align*}
\text{\normalfont ($f$-policy regret)} &\leq 2L\sqrt{T \dmin m}h\radnat^2\radGst^2\radM^2
\end{align*}
\end{lemma}
\begin{proof} From~\Cref{thm:anava} with $L_f = L_c$ as in \Cref{lem:label_Lipschitz_constants_known}, and diameter $D \le 2\sqrt{\dmin}\radM$ from the same lemma,  \Cref{lem:loneop_fro}, we
\begin{align*}
\text{\normalfont ($f$-policy regret)} &\leq \frac{D^2}{\eta} + TL_f^2(h^2 + 1)\eta \leq \frac{\radM^2 d\min}{\eta} + T  (L\radnat^2\radGst^2)^2\radM^2 m(h^2 + 1)\eta.
\end{align*}
Selecting $\eta = \sqrt{\dmin}/2Lh\radnat^2\radGst^2\sqrt{2mT}$ and bounding $\sqrt{h^2+1} \le \sqrt{2}h$ concludes the proof.
\end{proof}
Recalling the bound on policy approximation from \Cref{thm:policy_approximation}, we combine all the relevant bounds above to prove our regret guarantee:
\begin{proof}[Proof of Theorem~\ref{thm:know}]
Summing up bounds in \Cref{lem:burn-in-bound_known,lem:f_regret_known_simple,lem:truncation_costs_known}, and \Cref{thm:policy_approximation}, and using $\radM \ge 1$,
\begin{align*}
\Regret_T(\psi) &\lesssim L\radnat^2\left(\radM^2\radGst^2(m+h)  + \sqrt{T}h\sqrt{m\dmin}\radM^2\radGst^2 +  \radM^2\radGst \psiGst(h+1) T +\radM\radGst^2\psi(m)T \right)\\
&\lesssim L\radnat^2\left(\sqrt{T}h\sqrt{m\dmin}\radM^2\radGst^2 + \radM^2\radGst \psiGst(h+1) T +\radM\radGst^2\psi(m)T \right)\\
&= L\radnat^2\radGst^2\radM^2\sqrt{T}\left( h\sqrt{m\dmin} + \sqrt{T}\left(\frac{\psiGst(h+1)}{\radGst} + \frac{\psi(m)}{\radM}\right)  \right).
\end{align*}


\end{proof}

  \subsection{Proof of Theorem~\ref{thm:policy_approximation} \label{ssec:policy_approximation}}

\begin{proof}
Let $(\matypi_t, \matupi_t)$ be the output-input sequence produced on the execution of a LDC $\pi$ on a LDS $(\Ast, \Bst, \Cst)$, and $(\maty_t^M, \matu_t^M)$ be the output-input sequence produced by the execution of an Disturbance Feedback Controller $M$ on the same LDS.  By \Cref{l:augd},  the closed-loop dynamics are given by
\begin{align}
\begin{bmatrix} \matxpi_{t+1} \\ \matzpi_{t+1} \end{bmatrix} 
&= \underbrace{\begin{bmatrix}
\Ast + \Bst \Dpi \Cst & \Bst \Cpi \\
\Bpi\Cst  & \Api 
\end{bmatrix}}_{\Apicl} \begin{bmatrix} \matxpi_t \\ \matzpi_t\end{bmatrix} +
\underbrace{\begin{bmatrix} I & \Bst \Dpi  \\ 0 & \Bpi \end{bmatrix}}_{\Bpicl}\begin{bmatrix} \matw_t \\ \mate_t \end{bmatrix} ,  \\
\begin{bmatrix} \matypi_{t} \\ \matupi_{t} \end{bmatrix}
&= \underbrace{\begin{bmatrix}
   \Cst & 0 \\
    \Dpi \Cst & \Cpi
\end{bmatrix}}_{\Cpicl} \begin{bmatrix} \matxpi_t \\ \matzpi_t\end{bmatrix}+ \begin{bmatrix}   I \\ \Dpi  \end{bmatrix}\mate_t.
\end{align}
Further, define $\Cpiclu$ as the second row of $\Cpicl$, and $\Bpiclw$ and $\Bpicle$  as the first and second columns of $\Bpicl$. We then have
\begin{align*}
\matupi_t &= 
 \Dpi \mate_t + \sum_{s=1}^{t-1}\Cpiclu \Apicl^{s-1} \Bpicle \mate_{t-s} + \sum_{s=1}^{t-1}\Cpiclu \Apicl^{s-1} \Bpiclw\matw_{t-s}.
\end{align*}
Our argument hinges on the following claim which we establish shortly below:
\begin{claim}[Control Approximation Identity]\label{claim:Youla} Define the matrices $M^{[0]} = \Dpi$, and $M^{[i]} = \Cpiclu \Apicl^{[i-1]}\Bpicle$ for $i \ge 1$. Then,
\begin{align*}
\matupi_t = \sum_{i=0}^{t-1} M^{[i]}\naty_{t-i}.
\end{align*}
\end{claim} As a consequence, we have that, for $\matu_t^M = \sum_{i=0}^{m-1}M^{[i]} \matu_{t-i}$, we find
\begin{align}
 \matupi_t = \matu_t^M + \sum_{i=m}^{t-1} M^{[i]} \naty_{t-i}\label{eq:ut_indentity_simple} 
\end{align}

 This, in particular, implies the following bounds.
\begin{align*}
  \|\matu^\pi_t - \matu^M_t \| &\leq \left(\sum_{i=m}^{t-1} \|\Cpiclu \Apicl^{i-1} \Bpicle\|\right) \max_i\|\naty_{i}\|  \leq \psi(m)\radnat \\
  \|\maty^\pi_t - \maty^M_t\|& \leq \|\Gst\|_{\loneop} \psi(m) \radnat
\end{align*}
Hence,
\begin{align*}
\left\|\begin{bmatrix} \matu^\pi_t - \matu^M_t \\ \maty^\pi_t - \maty^M_t \end{bmatrix}\right\|_{2} \le (1 + \|\Gst\|_{\loneop}) \psi(m) \radnat := \radGst \psi(m)\radnat.
\end{align*}
Moreover, from \Cref{eq:ut_indentity_simple}, we can show that $\|(\matu^{\pi}_t,\maty^{\pi}_t)\| \le 2\radGst\radM\radnat$, as per \Cref{lem:magnitude_bound_simple}. Thus, from the sub-quadratic assumption (\Cref{asm:lip}),
\begin{align*}
|\loss_t(\maty^\pi_t,\matu^\pi_t)-\loss_t(\maty^M_t,\matu^M_t)| \le 2L\radGst \radM\radnat \left\|\begin{bmatrix} \matu^\pi_t - \matu^M_t \\ \maty^\pi_t - \maty^M_t \end{bmatrix}\right\|_{2} \le   2L\radGst^2\radnat^2 \radM \psi(m).
\end{align*}
\end{proof}

\begin{proof}[Proof of \Cref{claim:Youla}]
\begin{align}
 & \sum_{i=0}^{t-1} M^{[i]}\naty_{t-i} = \sum_{i=0}^{t-1} M^{[i]} \left(\mate_{t-i} + \sum_{s=1}^{t-i-1} \Cst \Ast^{t-i-s-1} \matw_{s}\right)\nonumber \\
&= \sum_{i=0}^{t-1} M^{[i]} \mate_{t-i} + \sum_{s=1}^{t-1} \sum_{i=0}^{s-1}(M^{[i]}\Cst\Ast^{s-1-i}) \matw_{t-s}\nonumber \\
&= \Dpi \mate_t  +\sum_{s=1}^{t-1}\Cpiclu \Apicl^{s-1} \Bpicle \mate_{t-s} + \Dpi \sum_{s=1}^{t-1} \Cst \Ast^{t-s}\matw_{t} +
\sum_{s=1}^{t-1} \sum_{i=1}^{s-1}(M^{[i]}\Cst\Ast^{s-1-i}) \matw_{t-s}\nonumber\\
\end{align}
Let us unpack the last line:
\begin{align*}
M^{[i]}\Cst\Ast^{s-1-i} &= \Cpiclu \Apicl^{i-1} \Bpicle\Cst \Ast^{s-1-i}\matw_{t-s}\\
&= \Cpiclu \Apicl^{i-1} \begin{bmatrix}\Bst \Dpi \Cst\\
\Bpi\Cst\end{bmatrix}  \Ast^{s-1-i}\matw_{t-s}\\
&= \Cpiclu \Apicl^{i-1} \underbrace{\begin{bmatrix}\Bst \Dpi \Cst & \Bst \Cpi \\
\Bpi\Cst & \Api \end{bmatrix}}_{(i)}  \begin{bmatrix}\Ast^{s-1-i}\matw_{t-s}  \\ 0
\end{bmatrix},
\end{align*}
where we fill the last two columns of the matrix $(i)$ arbitrarily, since $\begin{bmatrix}\Ast^{s-1-i}\matw_t  \\ 0
\end{bmatrix}$ has a zero in its second block component. Define the matrices $(X,Y)=\left(\begin{bmatrix}\Ast & 0 \\ 0 & 0\end{bmatrix}, \begin{bmatrix}\Bst \Dpi \Cst & \Bst\Cpi \\ \Bpi \Cst & \Api \end{bmatrix}\right)$. We then recognize $\Apicl = X + Y$, and can thus express
\begin{align*}
M^{[i]}\Cst\Ast^{s-i - 1}  = \Cpiclu (X+Y)^{i-1}YX^{s-i-1}\begin{bmatrix}\matw_{t-s}\\ 
0
\end{bmatrix}
\end{align*}
Before proceeding, observe the following identity for any positive integer $n$.
\begin{align*}
X^n = Y^{n} + \sum_{i=1}^{n} X^{i-1} (X-Y) Y^{n-i}.
\end{align*}
Thus,
\begin{align*}
 \sum_{i=1}^{s-1}(M^{[i]}\Cst\Ast^{s-1-i}) \matw_{t-s} &= -\Cpiclu X^{s-1}\begin{bmatrix}\matw_{t-s}\\ 
0
\end{bmatrix} + \Cpiclu (X+Y)^{s-1}\begin{bmatrix}\matw_{t-s}\\ 
0
\end{bmatrix}\\
&= -\Dpi \Cst \Ast^{s-1}\matw_{t-s} + \Cpiclu\Apicl^{s-1}\Bpiclw\matw_{t-s}
\end{align*}
Picking up where we left off,
\begin{align}
 & \sum_{i=0}^{t-1} M^{[i]}\naty_{t-i}\nonumber\\
&=\Dpi \mate_t  +\sum_{s=1}^{t-1}\Cpiclu \Apicl^{s-1} \Bpicle \mate_{t-s} + \Dpi \sum_{s=1}^{t-1} \Cst \Ast^{t-s}\matw_{t} +
\sum_{s=1}^{t-1} \sum_{i=1}^{s-1}(M^{[i]}\Cst\Ast^{s-1-i}) \matw_{t-s}\nonumber\nonumber\\
&= \Dpi \mate_t + \sum_{s=1}^{t-1}\Cpiclu \Apicl^{s-1} \Bpicle \mate_{t-s} + \sum_{s=1}^{t-1} \sum_{i=0}^{s-1}\Cpiclu \Apicl^{s-1}\Bpiclw\matw_{t-s}\nonumber\\
&= \Dpi \mate_t + \sum_{s=1}^{t-1}\Cpiclu \Apicl^{s-1} \Bpicle \mate_{t-s} + \sum_{s=1}^{t-1}\Cpiclu \Apicl^{s-1} \Bpiclw \matw_{t-s} = \matupi_t .
\end{align}
\end{proof}

\section{Analysis for Unknown System\label{sec:unknown}}

\subsection{Estimation of Markov Operators}
In this section, we describe how to estimate the hidden system. We prove the following theorem assuming some knowledge about the decay of $\Gst$. We defer the setting where the learner does not have this knowledge to later work. The proof of the following guarantee applies results from \cite{simchowitz2019learning}, and is given in Appendix~\Cref{ssec:thm:ls_estimation}:

\begin{theorem}[Guarantee for \Cref{alg:estimation}]\label{thm:ls_estimation_simple} 
Let $\delta \in (e^{-T},T^{-1})$, $N ,\dimu \le T$, and $\psiGst(h+1) \le \frac{1}{10}$. For universal constants $c,\Cest$, define
\begin{align*}
 \epsG(N,\delta) &=  \Cest\frac{h^2 \radnat}{\sqrt{N}}C_{\delta}, \quad\text{ where } C_{\delta} := \sqrt{\dmax + \log \tfrac{1}{\delta} + \log (1+\radnat)}.
\end{align*}
and suppose that $N \ge c h^4 C_{\delta}^4 \radM^2\radGst^2$. Then  with probability $1 - \delta - N^{-\log^2 N}$, Algorithm~\Cref{alg:estimation} satisfies the following bounds
\begin{enumerate}
	\item For all $t \in [N]$, $\|\matu_t\| \le \Restu(\delta) := 5\sqrt{\dimu + 2\log (3/\delta)}$
	\item The estimation error is bounded as
	\begin{align*}
	\|\Ghat - \Gst\|_{\loneop} \le \|\Ghat^{[0:h]} - \Gst^{[0:h]}\|_{\loneop} + \Restu\psiGst(h+1) \le \epsG(N,\delta)  \le 1/2\max\{\radM\radGst,\Restu\}.
	\end{align*}
\end{enumerate}
\end{theorem}
For simplicity, we suppress the dependence of $\epsG$ on $N$ and $\delta$ when clear from context. Throughout, we shall assume the following condition
\begin{condition}[Estimation Condition]\label{cond:est_condition} We assume that the event of \Cref{thm:ls_estimation_simple} holds.
\end{condition}

\subsection{Stability of Estimated Nature's y\label{ssec:stability_main}}
 One technical challenge in the analyis of the unknown $\Gst$ setting is that the estimates $\natyhat_{1:t}$ \emph{depend on the history of the algorithm}, because subtracting off the contribution of the inexact estimate $\Ghat$ \emph{does not} entirely mitigate the effects of past inputs. Hence, our the first step of the analysis is to show that if $\Ghat$ is sufficiently close to $\Gst$, then this dependence on history does not lead to an unstable feedback. Note that the assumption of the following lemma holds under \Cref{cond:est_condition}:
  \begin{lemma}[Stability of $\natyhat_{1:t}$]\label{lem:estimation_stability_simple} Introduce the notation $\Rubar(\delta) := 2\max\{\Restu(\delta), \radM\radnat\}$, assume that $\epsG(N,\delta) \le 1/2\max\{\radM\radGst\}$.   Then, for $t \in [T]$, we have the bounds 
  \begin{align*}
  \|\ualg_t\|_2 \le \Rubar(\delta), \quad \|\natyhat_t\|_2 \le 2\radnat, \quad \|\yalg_t\| \le \radnat + \radGst\Rubar(\delta)
  \end{align*}
  \end{lemma}
  \begin{proof} Introduce $\|\ualg_{1:t}\|_{2,\infty} := \max_{s \le t}\|\ualg_s\|_{2}$. We then have
  \begin{align}\label{eq:naty_bound}
  \|\naty_t - \natyhat_t\|_2 = \left\|\sum_{s=1}^t \Gst^{[t-s]}\ualg_s - \Ghat^{[t-s]}\ualg_s\right\|_2 \le \|\Ghat-\Gst\|_{\loneop}\max_{s \le t}\left\|\ualg_s\right\|_2 \le \epsG\|\ualg_{1:t}\|_{2,\infty}.
  \end{align}
 We now have that 
  \begin{align*}
\|\ualg_{1:t}\|_{2,\infty} &\le \max\left\{\Restu,\left\|\max_{s \le t} \sum_{j=t-m+1}^{j-1} \matM_{s}^{[t-s]}\natyhat_j\right\|\right\}\\
  &\le \max\{\Restu,\radM \max_{s \le t-1} \|\natyhat_s\|_{2}\}\\
   &\le \max\{\Restu,\radM (\radnat + \|\Gst\|_{\loneop}\epsG\|\ualg_{1:t-1}\|_{2,\infty})\}\\
    &\le \max\{\Restu,\radM \radnat\} +  \radM\|\Gst\|_{\loneop}\epsG\|\ualg_{1:t-1}\|_{2,\infty}
  \end{align*}
   Moreover, by assumption, we have  $\epsG \le 1/2 \radM\radGst$, so that
  \begin{align*}
  \|\ualg_{1:t}\|_{2,\infty} &\le \max\{\Restu, \radM\radnat\} + \|\ualg_{1:t-1}\|_{2,\infty}/2 \\
  &\le \max\{\Restu, \radM\radnat\} + \frac{1}{2}(\Restu + \radM\radnat) + \|\ualg_{1:t-2}\|_{2,\infty}/4 \\
  &\le \dots \le \,2\max\{\Restu, \radM\radnat\}  := \Rubar.
  \end{align*}
 The bound $\|\natyhat_t\|_2 \le 2\radnat$ follows by plugging the above into \Cref{eq:naty_bound}, the the final bound from $\|(\yalg_t,\ualg)\| \le \radnat + \|\Gst\|_{\loneop}\|\ualg_{1:t}\|_{2,\infty} + \|\ualg_t\|_2 \le \radnat + (1+\|\Gst\|_{\loneop})\|\ualg_{1:t}\|_{2,\infty} := \radnat + \radGst \|\ualg_{1:t}\|_{2,\infty}$.
  \end{proof}

\subsection{Regret Analysis\label{ssec:weak_convex_unknown_regret}}

We apply an analogous regret decomposition to the proof of \Cref{thm:know}, again abbreviating $\calM \leftarrow \calM(m,\radM)$:

  \begin{align}
  &\Regret_T(\Pi(\psi)) \le \underbrace{\left(\sum_{t=1}^{m +2h+N} \loss_t(\yalg_t,\ualg_t)\right)}_{\text{\color{blue} estimation \& burn-in loss}} \iftoggle{colt}{\nonumber\\ &\quad+}{+} \underbrace{ \left( \sum_{t=N+m+2h+1}^T \loss_t(\yalg_t,\ualg_t) - \sum_{t=N+m+2h+1}^T F_t[\matM_{t:t-h}\midhatt]\right)}_{\text{\color{blue} loss approximation-error}},\nonumber\\
  &\quad+\underbrace{\left(\sum_{t=N+m+2h+1}^T F_t[\matM_{t:t-h} \midhatt] - \inf_{M \in \calM}\sum_{t=N+m+2h+1}^{T}f_t(M\midhatt )\right)}_{\text{$\fhat$ policy regret}} \nonumber\\
  &\qquad+\underbrace{\left(\inf_{M \in \calM}  \sum_{t=N+m+2h+1}^{T}f_t(M \midhatt ) - \inf_{M \in \calM} \sum_{t=N+m+2h+1}^{T}\ell_t(\maty^{M}_t,\matu^{M}_t)\right)}_{\text{\color{blue}  comparator approximation-error}} \nonumber \\
  &\qquad+  \underbrace{\inf_{M \in \calM} \sum_{t=1}^{T}\ell_t(\maty^{M}_t,\matu^{M}_t) - \inf_{\pi \in \Pi(\psi)}\ell_t(\maty^\pi_t,\matu^\pi_t)}_{\text{policy approximation } :=  J(M) - J(\pi)} \label{eq:unknown_weak_convex_regret_decomposition}
  \end{align}
  Let us draw our attention two the main differences between the present decomposition and that in~\Cref{eq:known_weak_convex_regret_decomposition}: first, the burn-in phase contains the initial estimation stage $N$. During this phase, the system is excited by the Guassian inputs before estimation takes place. Second, the truncation costs are replaced with approximation errors, which measure the discrepancy between using $\Ghat$ and $\natyhat_{1:t}$ and using $\Gst$ and $\naty_{1:t}$. 

  Observe that the policy approximation error is exactly the same as that from the known-system regret bound, and is adressed by \Cref{thm:policy_approximation}. Moreover, the policy regret can be bounded by a black-box reduction to the policy regret in the known-system cases:
  \begin{lemma}\label{lem:policy_regret_unknown_lip} Assume~\Cref{cond:est_condition}. Then, for $\eta_t = \eta \propto  \sqrt{\dmin}/Lh\radnat^2\radGst^2\sqrt{mT}$,  we have 
  \begin{align*}
  \text{\normalfont ($\fhat$-policy regret)} &\lesssim \sqrt{T}L\sqrt{\dmin m}h\radnat^2\radGst^2\radM^2
  \end{align*}
  \end{lemma}
  \begin{proof} 
  Observe that the $\fhat$-regret depends only on the $\natyhat_t$ and $\Ghat$ sequence, but not on any other latent dynamics of the system. Hence, from the proof \Cref{lem:f_regret_known_simple}, we see can see more generally that if $\max_t\|\natyhat_{t}\| \le \radnat'$ and $\|\Ghat\|_{\loneop} \le \radGst'$, then $\eta \propto  \sqrt{\dmin}/Lh(\radnat')^2(\radGst')^2\sqrt{mT}$,
  \begin{align*}
	\text{\normalfont ($\fhat$-policy regret)} &\lesssim L\sqrt{\dmin mh^2T}(\radnat')^2(\radGst')^2\radM^2
	\end{align*}
	In particular, under~\Cref{cond:est_condition} and by \Cref{lem:estimation_stability_simple},  we can take $\radnat' \le 2\radnat$ and $\radGst' \le 2\radGst$.  
  \end{proof}
From \Cref{lem:estimation_stability_simple}, $\twonorm{\yalg_t} \le \radnat + \radGst\Rubar \le 2\radGst\Rubar$. \Cref{asm:lip} then yields
  \begin{lemma}\label{lem:burn-in-bound} Under~\Cref{cond:est_condition}, we have that $\text{\normalfont (estimation \& burn-in)} \le 4L(m+2h+N) \radGst^2\Rubar^2$.
  \end{lemma}
  To conclude, it remains to bound the approximation errors. We begin with the following bound on the accuracy of estimated nature's $y$, proven in \Cref{sssec:lem:acc_estimat_nat}:
  \begin{lemma}[Accuracy of Estimated Nature's $y$]\label{lem:acc_estimat_nat} 
  Assume~\Cref{cond:est_condition}, and let $t \ge N+h+1$, we have that that $ \|\natyhat_t - \naty_t\| \le  2\radM\radnat\epsG$.
  \end{lemma}
  We then use this to show that the estimation error is linear in $T$, but also decays linearly in $\epsG$:
  \begin{lemma}[Approximation Error Bounds]\label{lemma:approx_error_bounds} Under~\Cref{cond:est_condition},
  \begin{align*}
  \text{\normalfont (loss approximation error)} +
  \text{\normalfont (comparator approximation error)} &\lesssim LT\radGst\radM^2\radnat^2 \epsG
  \end{align*}
  \end{lemma}
  \begin{proof}[Proof Sketch] For the ``loss approximation error'', we must control the error introduce by predicting using $\Ghat$ instead of $\Gst$, and by the difference from affine term $\natyhat_t$ in $\maty(\cdot \midhatt)$ from the true natures $y$ $\naty_t$. For the ``comparator approximation error'', we must also adress the mismatch between using the estimated $\natyhat_t$ sequence of the controls in the functions $f_t(\cdot \midhatt)$, and the true natures y's $\natyhat_t$ for the sequence $(\maty_t^M,\matu_t^M)$. A complete proof is given in \Cref{sssec:lemma:approx_error_bounds}
  \end{proof}
  \begin{proof}[Proof of Theorem~\ref{thm:unknown}]
  Assuming \Cref{cond:est_condition}, taking $N \ge m+h$ and combining \Cref{lem:burn-in-bound} with the substituting $\Rubar(\delta) \lesssim \max\{\sqrt{\dimu + \log(1/\delta)},\radM\radGst\}$, and with \Cref{lemma:approx_error_bounds,lem:policy_regret_unknown_lip,thm:policy_approximation},  
  \begin{align*}
  \Regret_T(\psi) &\lesssim L\radGst^2\radM^2 \radnat^2 \left( \left( \frac{\dimu + \log(1/\delta)}{\radnat^2\radM^2} \vee 1\right) N + T\epsG(N,\delta)\radM + \sqrt{dmh^2 T} +  \frac{T\psi(m)}{\radM} \right).
  \end{align*}
  For $\psi(m) \le \radM/\sqrt{T}$, the last term is dominated by the second-to-last. Now, for the constant $C(\delta)$ as in the theorem statment, and for $\radnat\radM \ge \dimu  + \log(1/\delta)$, we have $(\frac{\dimu + \log(1/\delta)}{\radnat^2\radM^2} \vee 1) N + T\epsG(N,\delta)\radM \le  (N + C_{\delta}  T h^2\radM \radnat/\sqrt{N}) $. We see that if we have $N = (Th^{2}\radM\radnat C_{\delta})^{2/3}$, then the above is at most
  \begin{align*}
  \Regret_T(\psi) &\lesssim L\radGst^2\radM^2\radnat^2  \left( (h^2 T \radM \radnat C_{\delta})^{2/3}  + \sqrt{\dmin h^2m T} \right).
  \end{align*}
 Finally, since $\psiGst(h+1) \le 1/10\sqrt{T}$, one can check that Condition~\ref{cond:est_condition} holds with probability $1-3\delta - N^{-\log^2 N} = 1-\delta - T^{-\Omega(\log^2 T)}$ as soon as  $T \ge c'h^4 C_{\delta}^5 \radM^2\radnat^{2}$ for a universal constant $c'$. When $T \ge \dmin m^3$, we can bound the above by $\lesssim L\radGst^2\radM^2\radnat^2  \left(h^2  \radM \radnat C_{\delta}\right)^{2/3} \cdot T^{2/3}$.

  \end{proof}

\section{Logarithmic Regret for Known System \label{sec:proof_fast_rate_known}}
In this section, we prove \Cref{thm:fast_rate_known}. The analoguous result for the strongly-stabilized setting is proved in \Cref{app:strongly_convex}.

\Cref{thm:fast_rate_known} applies the same regret decomposition as \Cref{thm:know}; the key difference is in bounding the $f$-policy regret in Eq.~\eqref{eq:known_weak_convex_regret_decomposition}. Following the strategy of \cite{agarwal2019online}, we first show that the persistent excitation induces strongly convex losses (in expectation). Unlike this work, we \emph{do not assume} access to gradients of expected functions, but only the based on losses and outputs revelead to the learner. We therefore reason about losses conditional on $k \ge m$ steps in the past:
	\begin{definition}[Filtration and Conditional Functions]\label{defn:filt_conditional} Let $\calF_{t}$ denote the filtration generated by the stochastic sequences $\{(\estoch_{s},\wstoch_t)\}_{s \le t}$, and define the conditional losses
	\begin{align*}
	f_{t;k}\left(M \midst\right) := \Exp\left[f_t(M \midst) \mid \calF_{t-k}\right].
	\end{align*}
	\end{definition}
	A key technical component is to show that $f_{t;k}$ are strongly convex:
	\begin{restatable}{proposition}{strongconvexitynocontroller}\label{prop:strong_convex_no_internal_controller} For $\alpha_f$ as in Theorem~\ref{thm:fast_rate_known} and $t \ge k \ge m$, $f_{t;k}(M \midst)$ is $\alpha_f$-strongly convex.

	\end{restatable}
	The above proposition is proven in \Cref{ssec:prop:strong_convex_no_internal_controller}, with the \Cref{app:strong_convexity_check} devoted to establishing a more general bound for strongly-stabilized (but not necessarily stable) systems. 

	Typically, one expects strong-convex losses to yield $\log T$-regret. However, only the condition expectations of the loss are strongly convex; the losses themselves are not. \cite{agarwal2019online} assume that access to a gradient oracle for expected losses, which circumvents this discrepancy. In this work, we show that such an assumption is not necessary if the unary losses are also $\beta$-smooth. 

	We now set up regularity conditions and state a regret bound (\Cref{thm:str_convex_high_prob_regret_known}) under which conditionally-strong convex functions yield logarithmic regret. The proof of \Cref{thm:fast_rate_known} follows by specializing these conditions to the problem at hand.
	\begin{restatable}[Unary Regularity Condition ($\urc$) for Conditionally-Strongly Convex Losses]{condition}{urccondition} \label{cond:urc}   Suppose that $\calK \subset \R^d$. Let $f_t := \calK \to \R$ denote a sequence of functions and $(\calF_t)_{t \ge 1}$ a filtration. We suppose $f_t$ is $\Lf$-Lipschitz, and $\max_{x \in \calK}\|\nablatwo f_t(x)\|_{\op} \le \beta$, and that $f_{t;k}(x) := \Exp[f_t(x) \mid \calF_{t-k}]$ is $\alpha$-strongly convex on $\calK$. 
	\end{restatable}

	Observe that \Cref{prop:strong_convex_no_internal_controller} precisely establishes the strong convexity requirement for $\urc$, and we can verify the remaining conditions below. In the Appendix, we prove a generic high-probability regret bounds for applying online gradient descent to $\urc$ functions (\Cref{thm:high_prob_unary}). Because we require bounds on policy regret, here we shall focus on a consequence of that bound for the ``with-memory'' setting:
	\begin{restatable}[With-Memory Regularity Condition ($\mrc$)]{condition}{mrccondition} \label{cond:mrc} Suppose that $\calK \subset \R^d$ and $h \ge 1$. We let $F_t := \calK^{h+1} \to \R$ be a sequence of $\Lc$ coordinatewise-Lipschitz functions with the induced unary functions $f_t(x) := F_t(x,\dots,x)$ satisfying \Cref{cond:urc}.
	\end{restatable}

	Our main regret bound in the with-memory setting is as follows: 
	\begin{theorem}\label{thm:str_convex_high_prob_regret_known} 
	Let $\calK \subset \R^d$ have Euclidean diameter $D$, consider functions $F_t$ and $f_t$ satisfying \Cref{cond:mrc} with $k \ge h \ge 1$. Consider gradient descent updates $z_{t+1} \leftarrow \Pi_{\calK}(z_t - \eta_{t+1}\nabla f_t(z_t))$, with $\eta_t = \frac{3}{\alpha t}$ applied for $t \ge t_0$ for some $t_0 \le k$, with $z_0 = z_1 = \dots = z_{t_0} \in \calK$. Then, with probability $1-\delta$,
	\begin{align*}
		&\sum_{t=k+1}^T F_t(z_t,z_{t-1},\dots,z_{t-h}) - \inf_{z \in \calK}\sum_{t=k+1}^T f_t(z) \\
		&\qquad\lesssim  \alpha k D^2 + \frac{ (k+h^2)\Lf\Lc + kd\Lf^2 + k\beta \Lf  }{\alpha} \log(T)  \nonumber + \frac{k\Lf^2}{\alpha}\log\left(\frac{1 +\log(e+ \alpha D^2)}{\delta}\right).
	\end{align*}
	\end{theorem}
	The above bound is a special case of \Cref{thm:main_w_memory_high_prob} below, a more general result that adresses complications that arise when $\Gst$ is unknown. Our regret bound incurs a dimension factor $d$ due to a uniform convergence argument\footnote{This is because we consider best comparator $\zst$ for the realized losses, rather than a pseudoregret comparator defined in terms of expectations}, which can be refined for structured $\calK$.

	To conclude the proof of \Cref{thm:fast_rate_known}, it we simply apply \Cref{thm:str_convex_high_prob_regret_known} with the appropriate parameters.  
		
		\begin{proof}[Proof of \Cref{thm:fast_rate_known}] From \Cref{lem:label_Lipschitz_constants_known}, we can take $\Lf := 4L\sqrt{m}\radnat^2\radGst^2\radM$ and $D \le 2\sqrt{d}\radM$. We bound the smoothness in \Cref{app:fast_rate_know}: 
	\begin{lemma}[Smoothness]\label{lem:smoothness_stable}  The functions $f_t(M \midst)$ are $\beta_f$-smooth, where we define $\beta_f := m  \radnat^2\radGst^2 \betaloss$.
	\end{lemma}
	This yields that, for $\alpha \le \alpha_f$ and step sizes $\eta_t = \frac{3}{\alpha t}$, the $f$-policy regret is bounded by
	\begin{align}
	&\lesssim \alpha md\radM^2 + \frac{ (md+h^2 )\Lf^2 + m\beta_f \Lf  }{\alpha} \log(T)  \nonumber + \frac{m\Lf^2}{\alpha}\log\left(\frac{1 +\log(e+\alpha_f d \radM^2)}{\delta}\right)\\
	&\lesssim \alpha md\radM^2 + \frac{L^2 m^{2} \radnat^4\radGst^4\radM^2}{\alpha} \left( \left(d+\frac{h^2}{m} \right) + \frac{\betaloss m^{1/2} }{Ld\radM}\right) \iftoggle{colt}{\nonumber\\&\qquad\cdot}{}\max\left\{\log T,\log\left(\frac{1 +\log(e+\alpha_f d \radM^2)}{\delta}\right)\right\}\nonumber\\
	&\lesssim \alpha md\radM^2 + \frac{L^2 m^{3}d \radnat^4\radGst^4\radM^2}{\alpha} \iftoggle{colt}{\nonumber\\&\qquad\cdot}{}\max\left\{1,  \frac{\betaloss }{L d\radM}\right\}\log\left(\frac{T +\log(e+\alpha d \radM^2)}{\delta}\right)\tag*{($1 \le h \le m$)}\\\
	&\lesssim \alpha md\radM^2 + \frac{L^2 m^{3}d \radnat^4\radGst^4\radM^2}{\alpha} \max\left\{1,  \frac{\betaloss }{L d\radM}\right\}\log\left(\frac{T}{\delta}\right)\tag*{($T \ge \log(e+\alpha d \radM^2)$)}\\
	&\lesssim \frac{L^2 m^{3}d \radnat^4\radGst^4\radM^2}{\alpha \wedge L\radnat^2 \radGst^2} \max\left\{1,  \frac{\betaloss }{L d\radM}\right\}\log\left(\frac{T}{\delta}\right)
	\label{eq:last_line_logT_policy},
	\end{align}
	Therefore, combining the above with \Cref{lem:burn-in-bound_known,lem:truncation_costs_known}, and \Cref{thm:policy_approximation},
		\begin{align*}
		 L\radnat^2\radM^2\radGst^2\left(m+h +  \frac{\psiGst(h+1) T}{\radGst} +  \frac{\psi(m) T}{\radM} + \text{(\Cref{eq:last_line_logT_policy})} \right).
		\end{align*}
		Applying $\psiGst(h+1) \le \radGst/T$, $\psi(m) \le \radM/T$, and $h \le m$, the term in \Cref{eq:last_line_logT_policy} dominates.

		\end{proof}


\section{$\sqrt{T}$-regret for unknown system under strong convexity\label{ssec:fast_rate_unknown_sketch}}

	In this section, we prove of \Cref{thm:fast_rate_unknown}, which requires the most subtle argument of the four settings considered in the paper. We begin with a high level overview, and defer the precise steps to \Cref{ssec:fast_rate_unknown_rigorous}. The core difficulty in proving this result is demonstrating that the error $\epsG$ in estimating the system propagates \emph{quadratically} as $T\epsG^2/\alpha$ (for appropriate $\alpha = \alpha_f/4$), rather than as $T\epsG$ in the weakly convex case. By setting $N = \sqrt{T/\alpha}$, we obtain regret bounds of roughly
	\begin{align}
	\Regret_T(\psi) \,\lessapprox\, N + \frac{T}{\alpha}\epsG(N,\delta)^2 + \frac{1}{\alpha} \log T \,\lessapprox \,\sqrt{\frac{T}{\alpha}} + \frac{\log T}{\alpha}\,, \label{eq:known_cvx_reg_sketch}
	\end{align}
	where we let $\lessapprox$ denote an informal inequality, possibly suppressing problem-dependent quantities and logarithmic factors, and use $\epsG(N,\delta) \lessapprox 1/\sqrt{N}$. To achieve this bound, we modify our regret decomposition by introducing the following a \emph{hypothetical} ``true prediction'' sequence:
	\begin{definition}[True Prediction Losses]\label{defn:true_pred_losses} We define the true prediction losses as
	\begin{align*}
	\ypred_t[M_{t:t-h}] &:= \naty_t + \sum_{i=1}^{h} \Gst^{[i]} \cdot \matu_{t-i}\left(M_{t-i}\midhatytmini\right)\\
	\Fpred_t\left[M_{t:t-h} \right] &:= \loss_t\left(\ypred_t\left[M_{t:t-h}\right] , \matu_t\left(M_t\midhaty\right)\right),
	\end{align*}
	and let $\fpred_t(M) = \Fpred_t(M,\dots,M)$ denote the unary specialization, and define the conditional unary functions $\fpred_{t;k}(M) := \Exp[\fpred_t(M) \mid \calF_{t-k}]$.
	\end{definition} 
	Note that the affine term of $\ypred_t$ is the true nature's $y$\footnote{Note that $\ypred_t[M_{t:t-h}] $ differs from $\maty_t[M_{t:t-h} \mid \Gst, \natyhat_{1:t}]$ (the counterfactual loss given the true function $\Gst$ and estimates $\natyhat_{1:t}$) precisely in this affine term}, and the inputs $\matu_{t-i}\left(M_{t-i}\midhatytmini\right)$ are multiplied by the true transfer function $\Gst$. Thus, up to a truncation by $h$, $\ypred_t$ describes the \emph{true} counterfactual output of system due to the control inputs $\matu_{t-i}\left(M_{t-i}\midhatytmini\right)$ selected based on the \emph{estimated} nature's $y$'s. $\Fpred_t$ and $\fpred_t$ then correspond to the counterfactual loss functions induced by these true counterfactuals. 

	While the algorithm does not access the unary losses $\fpred_t$ directly (it would need to know $\naty_t$ and $\Gst$ to do so), we show in \Cref{sssec:lem:str_cvx_grad_approx} that the gradients of $\fpred_t$ and $f_t(M \midhatt)$ are $\BigOh{\epsG}$ apart:
	\begin{lemma}\label{lem:str_cvx_grad_approx} For any $M \in \calM$, we have that
	\begin{align*}
	\left\| \nabla f_t(M \midhatt) - \nabla \fpred_t(M) \right\|_{\fro} \le \Capprox\,\epsG, 
	\end{align*}
	where we define $\Capprox := \sqrt{m}\radGst\radM\radnat^2(8\betaloss + 12L). $
	\end{lemma}
	As a consequence, we can view \Cref{alg:improper_OGD_main} as performing gradient descent with respect to the sequence $\fpred_t$, but with non-stochastic errors $\adverr_t := \nabla f(M \midhatt) - \nabla \fpred_t(M)$. The key observation is that online gradient descent with strongly convex losses is robust in that the regret grows \emph{quadraticaly} in the errors via $\frac{1}{\alpha}\sum_{t=1}^T \|\adverr_t\|_2^2$.  By modifying the step size slightly, we also enjoy a negative regret term. The following bound applies to the standard strongly convex online learning setup:
	\begin{proposition}[Robustness of Strongly Convex OGD]\label{prop:robust_ogd} 
	Let $\calK \subset \R^d$ be convex with diameter $D$, and let $f_{t}$ denote a sequence of $\alpha$-strongly convex functions on $\calK$. Consider the gradient update rules $z_{t+1} = \Pi_{\calK}(z_t - \eta_{t+1}(\nabla f_t(z_t) + \adverr_t))$, where $\adverr_t$ is an arbitrary error sequence. Then, for step size $\eta_t = \frac{3}{\alpha t}$,
	\begin{align*}
	\forall \zst \in \calK,~\sum_{t=1}^T f_{t}(z_t) - f_{t}(\zst) &\le  \frac{6L^2}{\alpha} \log(T+1) + \alpha D^2 + \frac{6}{\alpha}\sum_{t=1}^T \left\|\adverr_t\right\|_2^2 - \frac{\alpha}{6}\sum_{t=1}^T \|z_t - \zst\|_2^2 
	\end{align*}
	\end{proposition}
	For our setting, we shall need a strengthing of the above theorem to the conditionally strongly convex with memory setting of \Cref{cond:mrc}. But for the present sketch, the above proposition captures the essential elements of the regret bound: (1) logarithmic regret, (2) quadratic sensitivity to $\adverr_t$, yielding a dependence of $T\epsG^2/\alpha$, 
	and (3) negative regret relative to arbitrary comparators. With this observation in hand, we present our regret decomposition in \Cref{eq:unknown_strong_convex_regret_decomposition}, which is described in terms of a comparator $\Mapprox \in \calM$, and restricted set $\calM_0 = \calM(\radM/2, m_0) \subset \calM$, where $m_0 = \floor{\frac{m}{2}} - h$:
	 \begin{align}
	 &\Regret_T(\Pi(\psi)) \le \underbrace{\left(\sum_{t=1}^{m +2h + N} \loss_t(\yalg_t,\ualg_t)\right)}_{\text{burn-in loss}} \iftoggle{colt}{\nonumber\\&\qquad+}{+} \underbrace{ \left( \sum_{t=m+2h+N+1}^T \loss_t(\yalg_t,\ualg_t) - \sum_{t=m+2h+N+1}^T \Fpred_t[\matM_{t:t-h}]\right)}_{\text{algorithm truncation error}},\nonumber\\
	  &\qquad\qquad+\underbrace{\left(\sum_{t=m+2h+N+1}^T \Fpred_t[\matM_{t:t-h} ] - \sum_{t=m+2h+N+1}^{T}\fpred_t(\Mapprox)\right)}_{\text{\color{blue} $\fpred$ policy regret} } \nonumber\\
	  &\qquad\qquad+\underbrace{\sum_{t=N+ m+2h+1}^{T}\fpred_t(\Mapprox ) - \inf_{M \in \calMcomp}\sum_{t=N+m+2h+1}^{T}f_t(M \midst )}_{\text{\color{blue} $\natyhat_t$ control approximation error}} \nonumber\\
	  &\qquad\qquad+ \underbrace{\left(\inf_{M \in \calMcomp}  \sum_{t=N+m+2h+1}^{T}f_t(M \midst ) - \inf_{M \in \calMcomp} \sum_{t=N+m+2h+1}^{T}\ell_t(\maty^{M}_t,\matu^{M}_t)\right)}_{\text{comparator truncation error}} \nonumber \\
	  &\qquad\qquad+  \underbrace{\inf_{M \in \calMcomp} \sum_{t=1}^{T}\ell_t(\maty^{M}_t,\matu^{M}_t) - \inf_{\pi \in \Pi(\psi)}\ell_t(\maty^\pi_t,\matu^\pi_t)}_{\text{policy approximation } :=  J(M) - J(\pi)} \label{eq:unknown_strong_convex_regret_decomposition}
	  \end{align}

	The novelty in this regret decomposition are the ``$\fpred$ policy regret'' and ``$\natyhat_t$ control approximation error'' terms, which are coupled by a common choice of comparator $\Mapprox \in \calM$. The first is precisely the policy regret on the $\Fpred_t,\fpred_t$ sequence, which (as decribed above) we bound via viewing descent on  $f_t(\cdot \midst)$ as a running OGD on the former sequence, corrupted with nonstochastic error. 

	The term  ``$\natyhat_t$ control approximation error'' arises from the fact that, even though $\fpred_t$ describes (up to truncation) the true response of the system to the controls, it considers controls based on \emph{estimates} of natures $y$'s, and not on nature's $y$'s themselves. To bound this term, we show that there exists a specific comparator $\Mapprox$ which competes with the best controller in in the restricted class $\calMcomp$ that access the \emph{true} natures $y$'s.  \Cref{prop:Mapprox} constructs a controller $\Mapprox$ which builds in a correction for the discrepancy between $\natyhat$ and $\naty$. We show that  this controller satisfies for any $c > 0$,:
	  \begin{align}
	  \text{``$\natyhat_t$ control approximation error''}  \lessapprox  \frac{T\epsG^2}{c}  + c \sum_{t} \|\matM_t - \Mapprox\|_{\fro}^2.  \label{eq:information_approx_control_error}
	  \end{align}
	  A proof sketch is given in \Cref{sssec:ynathat_comparator}, which highlights how we use that $\calM$ \emph{overparametrizes} $\calM_0$. 
	  Unlike the coarse argument in the weakly convex case, the first term \Cref{eq:information_approx_control_error} has the desired quadratic sensitivity to $\epsG^2$. However, the second term is a movement cost which may scale linearly in $T$ in the worst case. 

	  Surprisingly, the proof of \Cref{eq:information_approx_control_error} \emph{does not require strong convexity}. However, in the presence of strong convexity, we can cancel the movement cost term with the negative-regret term from the $\fpred$ policy regret.  As decribed above, the $\fpred$ policy regret can bounded using a strengthening of \Cref{prop:robust_ogd}, to 
	  \begin{align*}
	  \text{``$\fpred$ policy regret''} \lessapprox \frac{1}{\alpha}\log T + \frac{T\epsG^2}{\alpha} - \Omega(\alpha) \sum_{t} \|\matM_t - \Mapprox\|_{\fro}^2,
	  \end{align*}
	  for appropriate strong convexity parameter $\alpha$. 
	  By taking $c$ to be a sufficiently small multiple of $\alpha$, 
	  \begin{align*}
	   \text{``$\natyhat_t$ control approximation error''} + \text{``$\fpred$ policy regret''} \lessapprox \frac{1}{\alpha}\log T + \frac{T\epsG^2}{\alpha}.
	  \end{align*}
	  In light of~\Cref{eq:known_cvx_reg_sketch}, we obtain the desired regret bound by setting $N = \sqrt{T/\alpha}$.

\subsection{Rigorous Proof of \Cref{thm:fast_rate_unknown}\label{ssec:fast_rate_unknown_rigorous}}

\subsubsection{$\fpred$-policy regret} 

We begin with by stating our general result for conditionally-strongly convex gradient descent with erroneous gradients. Our setup is as follows:
	\begin{restatable}{condition}{condgd}\label{cond:gd} We suppose that $z_{t+1} = \Pi_{\calK}(z_t - \eta \gradient_t)$, where $\gradient_t = \nabla f_t(z_t) + \adverr_t$. We further assume that the gradient descent iterates applied for $t \ge t_0$ for some $t_0 \le k$, with $z_0 = z_1 = \dots = z_{t_0} \in \calK$. We assume that $\|\gradient_t\|_2 \le \Lgrad$, and $\diam(\calK) \le D$.
	\end{restatable}
	The following theorem is proven in \Cref{sec:GD_conditional_strong}. 
	\begin{theorem}\label{thm:main_w_memory_high_prob} Consider the setting of \Cref{cond:mrc} and \ref{cond:gd}, with $k \ge 1$. Then with step size $\eta_t = \frac{3}{\alpha t}$, the following bound holds with probability $1-\delta$ for all comparators $\zst \in \calK$ simultaenously:
	\begin{align}
		&\sum_{t=k+1}^T f_t(z_t) - f_t(\zst) - \left( \frac{6}{\alpha}\sum_{t=k+1}^T \left\|\adverr_t\right\|_2^2  - \frac{\alpha}{12}\sum_{t=1}^T \|z_t - \zst\|_2^2\right) \nonumber\\
		&\qquad\lesssim   \alpha k D^2 + \frac{ (k\Lf+h^2\Lc)\Lgrad + kd\Lf^2 + k\beta\Lgrad  }{\alpha} \log(T)  \nonumber + \frac{k\Lf^2}{\alpha}\log\left(\frac{1 +\log(e+\alpha D^2)}{\delta}\right) 
		\end{align}
	\end{theorem}
	Observe that when $\adverr_t = 0$, we can take $\Lgrad = \Lf$ and discard the second on third terms on the first line, yielding \Cref{thm:str_convex_high_prob_regret_known}. Let us now specialize the above to bound the $\fpred$-policy regret. First, we verify appropriate smoothness, strong convexity and Lipschitz condiitons; the following three lemmas in this section are all proven in \Cref{sssec:unknown_smooth_str_convex}. 
\begin{restatable}{lemma}{smoothunknown}\label{lem:smoothness_stable_unknown}  Under \Cref{cond:est_condition}, $\fpred_t(M)$ are $4\beta_f$-smooth, for $\beta_f$ as in \Cref{lem:smoothness_stable}.
\end{restatable}
\begin{lemma}\label{lem:strong_cvx_unknown} Let $\alpha_f$ as in \Cref{prop:strong_convex_no_internal_controller} and suppose that 
\begin{align*}
\epsG \le \frac{1}{9\radnat\radM\radGst}\sqrt{\frac{\alpha_f}{m\alphaloss}}
\end{align*}
Then under \Cref{cond:est_condition}, the losses $\fpred_{t;k}(M)$ are $\alpha_f/4$ strongly convex.
\end{lemma}
\begin{restatable}[Lipschitzness: Unknown \& Strongly Convex]{lemma}{lipunknowsc}\label{lem:lip_unknown_strongly_convex} 
Recall the Lipschitz constant $L_f$ from \Cref{lem:label_Lipschitz_constants_known}. Then under \Cref{cond:est_condition}, $\fpred_t(M)$ is $4 L_f$-Lipschitz, $\fpred_t[M_{t:t-h}]$ is $4L_f$ coordinate Lipschitz. Moreover,  $\max_{M \in \Mclass} \| \nabla \fhat_t(M;\Ghat,\natyhat_{1:t})\|_2 \le 4L_f$. 
\end{restatable}

	Specializing the above theorem to our setting, we obtain the following:
	\begin{lemma}[Strongly Convex Policy Regret: Unknown System]\label{lem:str_cvx_unknown_policy_reg} For the step size choose $\eta_t = \frac{12}{\alpha t}$, the following bound holds with probability $1-\delta$:
	\begin{align*}
	(\text{\normalfont $\fpred$-policy regret}) + \frac{\alpha}{48}\sum_{t=N+h+m}^T \|\matM_t - \Mapprox\|_{\fro}^2 \lesssim \text{\normalfont (\Cref{eq:last_line_logT_policy})} + \frac{ T  \Capprox^2 \epsG(N,\delta)^2}{\alpha}
	\end{align*}
	\end{lemma}
	\begin{proof} In our setting, $z_t \leftarrow \matM_t$, $z_\star \leftarrow \Mapprox$, $\adverr_t$, and $\adverr_t \leftarrow \nabla f(M \midhatt) - \nabla \fpred_t(M)$. Moreover, we can can bound the smoothness $\beta \lesssim \beta_f$, the strong convexity $\alpha \gtrsim \alpha_f$, and all Lipschitz constants $\lesssim L_f$, where $L_f$ was as in \Cref{lem:label_Lipschitz_constants_known}. Using the same diameter bound as in that lemma, we see that the term on the right hand side of \Cref{thm:main_w_memory_high_prob} can be bounded as in \Cref{eq:last_line_logT_policy}), up to constant factors. Moreover, in light of \Cref{lem:str_cvx_grad_approx}, we can bound the term  $ \frac{6}{\alpha}\sum_{t=k+1}^T \left\|\adverr_t\right\|_2^2 $ from \Cref{thm:main_w_memory_high_prob} by $\lesssim \frac{1}{\alpha} T (\Capprox \epsG)^2$. This concludes the proof. Lastly, we lower bound $- \frac{\alpha}{12}\sum_{t=1}^T \|z_t - \zst\|_2^2$ by $- \frac{\alpha}{48}\sum_{t=N+h+m}^T \|\matM_t - \Mapprox\|_{\fro}^2 $.
	\end{proof}

\subsubsection{$\natyhat$-comparator approximation error \label{sssec:ynathat_comparator}}

We prove the following theorem in \Cref{ssec:ynat_approx_error}: 
	\begin{proposition}\label{prop:Mapprox} Let $\calM_0 := \calM(m_0,\radM/2)$, suppose that $m \ge 2m_0 - 1 + h$, $\psiGst(h+1) \le \radGst/T$, and that \Cref{cond:est_condition} holds. Them there exists a universal constant $C > 0$ such that, for all $\tau > 0$,
	\begin{align*}
	\text{\normalfont ($\natyhat$-approx error)} &\le  36 m^2 \radGst^4\radnat^4  \radM^3  (m+T\epsG^2) \max\{L,L^2/\tau\} \nonumber \\
	&\quad + \tau \sum_{t=N+m+h+1}^T   \left\| \matM_j  - \Mapprox \right\|_{\fro}^2\nonumber.
	\end{align*}
\end{proposition}
\begin{proof}[Proof Sketch] Let $\Mst$ denote the optimal $M \in \calM_0$ for the loss sequence defined in terms of the true $\naty$ and $\Gst$.
First, consider what happens when the learner selects the controller $\matM_t = \Mst$ for each $t \in [T]$. By expanding appropriate terms, one can show that (up to truncation terms),  the contoller $\Mst$ operating on $\naty_{1:t}$ produces the same inputs as the controller $\Mapprox = \Mst + \Mst * (\Ghat - \Gst) * \Mst$ operating on $\natyhat_{1:t}$, where `$*$' denotes the convolution operator. Since $\calM$ overparametrizes $\calM_0$, we ensure that $\Mapprox \in \calM$. 

Realistically, the learner does not play $\matM_t = \Mst$ at each round. However, we can show that the quality in the approximation for playing $\matM_t$ instead of $\Mst$ degrades as $\sum_{t}\epsG \cdot \|\matM_t - \Mst\|_{\fro}$. By construction,  $\|\Mapprox - \Mst\|_{\fro}$ scales  as $\epsG$, so the triangle inequality gives $\epsG \cdot \|\matM_t - \Mst\|_{\fro} \lessapprox \epsG^2 + \epsG\|\matM_t - \Mapprox\|_{\fro} $.  By the elementary inequality $ab \le a^2/\tau + \tau b^2$, we find that the total penalty for the movement cost scales as $\sum_t \epsG^2 + \epsG^2/\tau + \tau \|\matM_t - \Mapprox\|_{\fro}^2 =  T\epsG^2(1+1/\tau) + \tau \sum_t  \|\matM_t - \Mapprox\|_{\fro}^2$; this argument gives rise quadratic dependence on $\epsG^2$, at the expense of the movement cost penalty. 
\end{proof}
\subsubsection{Concluding the proof of \Cref{thm:fast_rate_unknown}}
We assume that $N$ and $T$ satisfy, for an appropriately large universal constant $c'$,
\begin{align}
&N = mh^2C_{\delta}\radGst\radM \radnat \sqrt{mT(1 + \frac{L}{\alpha} + \frac{\betaloss^2}{L\alpha})} \quad \text{and} \quad T \ge c'h^4C_{\delta}^6\radM^2\radnat^2, \label{eq:T_and_N}
\end{align}
Since we also have $\psiGst(h+1) \le 1/10T$, our choice of $N$ ensures \Cref{cond:est_condition} holds with probability $1-\delta - N^{-\log^2 N} = 1-\delta - T^{-\Omega(\log^2 T)}$. Combining Lemma~\ref{lem:str_cvx_unknown_policy_reg} and \Cref{prop:Mapprox} with $\tau = \frac{\alpha}{48}$, we can cancel the movement cost in the second bound with the negative regret in the first:
\begin{align*}
\iftoggle{colt}{&}{}\text{\normalfont ($\natyhat$-approx error)} +  \text{\normalfont ($\fpred$-policy regret)} \iftoggle{colt}{\\}{} &\lesssim  \text{\normalfont (\Cref{eq:last_line_logT_policy})} +  (T\epsG^2 + m) \left(m^2\radGst^4 \radM^4 \radnat^4\max\{L,\tfrac{L^2}{\alpha}\}  + \tfrac{\Capprox^2}{\alpha}\right)\\
&\lesssim  \text{\normalfont (\Cref{eq:last_line_logT_policy})} +  (\frac{TC_{\delta}^2}{N}+m) m^2 h^4\radGst^4 \radM^4 \radnat^4\left(L+\tfrac{(L+\betaloss)^2}{\alpha}\right)\\
&\lesssim  \text{\normalfont (\Cref{eq:last_line_logT_policy})} +  \frac{TC_{\delta}^2}{N} m^2 h^4\radGst^4 \radM^4 \radnat^4\left(L+\tfrac{(L+\betaloss)^2}{\alpha}\right).
\end{align*}
where in the second line we recall from \Cref{lem:str_cvx_grad_approx} the bound $\Capprox \lesssim  \sqrt{m}\radGst\radM\radnat^2(\betaloss + L)$, and use $\epsG(N,\delta) =  \Cest\frac{h^2 \radnat}{\sqrt{N}}C_{\delta}$ for $ C_{\delta} := \sqrt{\dmax + \log \tfrac{1}{\delta} + \log (1+\radnat)}$. In the third line, we use the assumption that $T \ge m^2$ from the Theorem.

From \Cref{lem:burn-in-bound}, we can bound
\begin{align*}
\text{\normalfont (estimation \& burn-in loss)} \le 4L(m+2h+N) \radGst^2\Rubar^2 \lesssim LN \radGst^2 \Rubar^2,
\end{align*} where $\Rubar$ as in \Cref{lem:estimation_stability_simple}. By assumption $\sqrt{\dimu + \log(1/\delta)} \le \radM\radnat$, we have $\Rubar \lesssim \radM\radnat$. Thus, 
\begin{align*}
&\text{\normalfont ($\natyhat$-approx error)} +  \text{\normalfont ($\fpred$-policy regret)} + \text{\normalfont (estimation \& burn-in loss)} \\
&\lesssim \text{\normalfont (\Cref{eq:last_line_logT_policy})}  + L\left(\frac{TC_{\delta}^2}{N} m^2h^4\radGst^4 \radM^4 \radnat^4\left(1+\frac{L + \betaloss^2/L}{\alpha}\right) + N\radnat^2 \radM^2\radGst^2\right).
\end{align*}
For our choice of $N$, the above is at most 
\begin{align*}
\lesssim \text{\normalfont (\Cref{eq:last_line_logT_policy})} + m h^2L\radGst^3 \radM^3 \radnat^3 C_{\delta}  \sqrt{m(1 + \frac{L}{\alpha} + \frac{\betaloss^2}{\alpha L})}\cdot \sqrt{T}.
\end{align*}
Finally, similar arguments as those in previous bounds show that the truncation costs and policy approximation error are dominated by the above regret contribution under the assumptions $\psi(\floor{\frac{m}{2}} - h) = \psi(m_0) \le \radGst/T$ and $R_{\psi} \le \radM/2$ (note that the policy approximation is for the class $\calM_0 = \calM(m_0,\radM/2)$), and $\psiGst(h+1) \le \radGst/T$. \qed

	\section{Concluding Remarks\label{sec:conclusion}}


This work presented a new adaptive controller we termed Disturbance Response Control via Gradient Descent (\drcgd), inspired by a Youla's parametrization.  This method is particularly suitable for controlling system with partial observation, where we show an efficient algorithm that attains the first sublinear regret bounds under adversarial noise for both known and unknown systems.

This technique attains optimal regret rates for many regimes of interest. Notably, this is the only technique which attains $\sqrt{T}$-regret for partially observed, non-stochastic control model with general convex losses.  Our bound is also the first technique to attain $\sqrt{T}$-regret for the classical LQG problem, and extends this bound to a more general semi-adversarial setting. 


In future work we intend to implement these methods and benchmark them against recent novel methods for online control, including the gradient pertrubation controller \cite{agarwal2019online}. We also intend to compare our guarantees to techniques tailored to the stochastic setting, including Certainty Equivalence Control~\cite{mania2019certainty}, Robust System Level System~\cite{dean2018regret}, and SDP-based relaxations~\cite{cohen2019learning}. We also hope to design variants of \drcgd{} which adaptively select algorithm parameters to optimize algorithm performance, and remove the need for prior knowledge about system properties (e.g. decay of the nominal system). Lastly, we hope to understand how to use these convex parametrizations for related problem formulations, such as robustness to system mispecification, safety contraints, and distributed control. 

\paragraph{Acknowledgements} The authors acknowledge Tuhin Sarkar and Nikolai Matni for their insightful feedback. Elad Hazan acknowledges funding from NSF grant 1704860. Max Simchowitz is generously supported by an Open Philanthropy graduate student fellowship. 

\clearpage
\bibliographystyle{plainnat}
\bibliography{main}
\clearpage

\tableofcontents
\addappheadtotoc
\appendix

\newpage

\section{Appendix Organization and Notation\label{app:a}}
\subsection{Organization\label{ssec:app_organization}}

This appendix presents notation and organization. Appendix~\ref{app:comparison} presents a $\Omega(\sqrt{T})$ lower bound for the online non-stochastic control problem, even with partial observant and benign conditions. It also gives a detailed comparison with prior work, detailed in \Cref{table:comparison_known,table:comparison_unknown}. 

\iftoggle{colt}
{
	\paragraph{Part~\ref{part:high_level}:} Sections~\ref{sec:main_known_system},~\ref{sec:unknown},~\ref{sec:proof_fast_rate_known}, and~\ref{ssec:fast_rate_unknown_sketch} provide the main arguments in the proofs of our main theorems in the order in which they are presented in Section~\ref{sec:alg_and_main}.
	}

{}

\iftoggle{colt}{\paragraph{Part~\ref{part:detail}:}}{} Appendix~\ref{sec:generalization_to_stabilized} introduces the general \emph{stabilized setting}, where the nominal system need not be stable, but is placed in feedback with a stabilizing controller.  All results from the stable setting (Assumption~\ref{asm:a_stab}) extend to the stabilized setting (Assumption~\ref{a:stablize}) with appropriate modifications. Statement which apply to the apply specifically to stabilized setting are denoted by the number of their corresponding statment for the stable setting, with the suffix 'b'. For example, Assumption~\ref{asm:a_stab} stipulates the stable-system setting, and Assumption~\ref{a:stablize} the stabilized setting.

Appendix~\ref{app:non-stochastic} adresses ommited proofs and stabilized-system generalizations of Theorems~\ref{thm:know} and~\ref{thm:unknown}, which give regret bounds for nonstochastic control for known and unknown systems respectively. Appendix~\ref{app:strongly_convex} does the same for the strongly-convex, semi-adversarial setting, namely Theorems~\ref{thm:fast_rate_known} and~\ref{thm:fast_rate_unknown}. This section relies on two technical appendices: Appendix~\ref{app:strong_convexity_check} verifies strong convexity of the induced losses under semi-adversarial noise, and Appendix~\ref{sec:GD_conditional_strong} derives the regret bounds for conditionally-strongly convex losses (see~\Cref{cond:urc}), and under deterministic errors in the gradients(see~\Cref{cond:gd}).

\subsection{Notation for Stable Setting\label{ssec:notation_stable}}
We first present the relevant notation for the \emph{stable} setting, where the transfer function $\Gst$ of the nominal system is assumed to be stable. This is the setting assumed in the body of the text. 
\begin{center}
\begin{longtable}{| l | l |}
\hline
\textbf{Transfer Operators} & \textbf{Definition (Stable Case)}  \\
\hline
$\Gst$ (stable case) & $\Gst^{[i]} = \I_{i \ge 0}\Cst\Ast^{i-1}\Bst$ is nomimal system (\Cref{def:markov}) \\
$\pi$  & refers to an LDC (Definition~\ref{def:lin_output_feedback})  \\
$\Gpicleu$ & transfer function of closed loop system (Lemma~\ref{l:augd}) \\
$M = (M^{[i]})$ & disturbance response controller or \drc{} (Definition~\ref{defn:drc}) \\
\hline
\textbf{Transfer Classes} &  \\
\hline
$\psi$ & proper decay function if $ \sum_{n \ge 0} \psi(n) < \infty$, $\psi(n) \ge 0$  \\
$\psi_G(n)$ & $\sum_{i \ge n}\|G^{[i]}\|_{\op}$ (e.g. $\psiGst$).\\
$\Pi(\psi)$ & Policy Class $\{\pi: \forall n, \psi_{\Gpicl}(n) \le \psi(n) \}$, assuming $\psi$ is proper\\
$\calM(m,R)$ & $\{M = ( M^{[i]})_{i=0}^{m-1} : \|M\|_{\loneop}\leq R  \}$ (calls of \drc s)\\
\hline
\textbf{Input/Output Sequence} & \\
\hline
$\ell_t$ & Loss function\\
$\mate_t,\matw_t$ & output and state disturbances (do not depend on control policy)\\
$\naty_t$ & Nature's $y$ (see \Cref{def:naty}, also does not depend on control policy)\\
$\maty_t^{\pi},\matu_t^{\pi}$ & output, input induced by LDC policy $\pi$\\
$\maty_t^{M},\matu_t^{M}$ & output, input induced by \drc{} policy $M$\\
$\yalg_t$ & output seen by the algorithm\\
$\ualg_t$  & input introduced by the algorithm\\
$\matM_t$ & \drc{} selected by algorithm at step $t$ \\
$\matu_t\left(M_t \midhaty\right)$ & counterfactual input  (\Cref{defn:counterfactual})\\
$\maty_t(M \midhatt)$ & unary counterfactual output (\Cref{defn:counterfactual})\\
$\maty_t\left[M_{t:t-h}\midhatt\right]$  & non-unary counterfactual output (\Cref{defn:counterfactual}) \\
$f_t(M \midhatt)$ & unary counterfactual cost (\Cref{defn:counterfactual})\\
$F_t\left[M_{t:t-h}\midhatt\right]$  & non-unary counterfactual cost (\Cref{defn:counterfactual}) \\
\hline
\textbf{Radius Terms and Alg Parameters} & \\
\hline
$m$ & length of \drc{}\\
$h$ & memory off approximation to transfer function\\
$\radnat$ & $\|\naty_t\| \le \radnat$  \\
$\radGst$ & stable case: $1 \vee \|\Gst\|_{\loneop} \le \radGst$  \\
$R_{\psi}$ & $1 \vee \sum_{n \ge 0} \psi(n) \le R_{\psi}$\\
$\radM$ & $\|\matM_t\|_{\loneop} \le \radM$ (algorithm parameter)\\
$\Restu = \Restu(\delta)$ & $5\sqrt{\dimu + \log\tfrac{3}{\delta}} $\\
$\Rubar = \Rubar(\delta)$ & $\Rubar(\delta) := 2\max\{\Restu(\delta), \radM\radnat\}$\\
$C_{\delta}$ & $\sqrt{\dmax + \log \tfrac{1}{\delta} + \log(1+\radnat)}$ (least squares estimation constant)\\
\hline
\end{longtable}
\end{center}
\addtocounter{table}{-1}

\subsection{Notation for Stabilized Setting \label{ssec:notation_stabilized}}
In general, we do not require that $\Gst$ be a stable matrix, but instead that $\Gst$ is placed in feedback with a stabilizing controller $\pi_0$. In this case, we let $\Gst$ denote the dynamics introduced by the feedback between the nominal system and $\pi_0$; details are given in \Cref{sec:generalization_to_stabilized}; at present, we summarize the relevant notation. 
\begin{table}[h!]
\centering
\begin{tabular}{| l | l |}
\hline
\textbf{Transfer Operators} & \textbf{Definition (Stabilized Case)}  \\
\hline
$\pi_0$  & nomimal stabilizing controller\\
$\mateta_t$ & \makecell[l]{``control-output'' produced by nomimal controller\\ reduces to $\maty_t$ in stable case}\\
$\uex_t$ & \makecell[l]{Exogenous input to controller\\ reduces to $\matu_t$ in stable case}\\
$\Gexyu$ & \makecell[l]{transfer function from \\exogenous inputs to system outputs and inputs  \\ (almost equivalent to $G_{\pi_0,\mathrm{cl}}$, see \Cref {defn:markov_stabilized})}\\
$\Gexeta$   & \makecell[l]{transfer function from \\ exogenous inputs to control-output $\mateta_t$\\ (\Cref {defn:markov_stabilized})}\\
\hline
\textbf{Input/Output Sequence} & \\
\hline
$\matv_t = (\maty_t,\matu_t)$ & Output-Input Pair\\
$\valg_t = (\maty_t,\matu_t)$ & Output-Input Pair produced by algorithm\\
$\natv_t = (\naty_t,\natu_t)$ & Output-Input Pair with no exogenous input\\
$\nateta_t $ & \makecell[l]{Nature's ``control-output'' under zero output\\ reduces to $\naty_t$ in stable case}\\
$\uexalg_t$ & \makecell[l]{exogenous input introduced by algorithm\\ (not including nominal controller)} \\
$\uex_t\left(M_t \midhateta\right)$ & \makecell[l]{ Exogenous input from estimates of $\nateta$ \\ See \Cref{defn:counterfactual_stabilized} for expression, and for below}\\
$\matv_t\left[M_{t:t-h}\midhatetayg\right]$ & Prediction of $\matv_t$ under estimated dynamics and $\nateta$. \\
$F_t\left[M_{t:t-h}\midhatetayg\right]$ & Prediction of loss under estimated dynamics and $\nateta$.\\
$f_t(M \midhatetayg)$ & Unary specialization of the above. \\
\hline
\textbf{Radius Terms} & \\
\hline
$\radGst$ & 
$1 \vee \|\Gstuu\|_{\loneop} \vee \|\Gstyu\|_{\loneop} \le \radGst$ \\
\hline
\end{tabular}
\end{table}

\subsection{Relationship Between Parameters (\Cref{asm:noise_decay_bound}) \label{ssec:asm:noise_decay_bound}}
\begin{itemize}
	\item In typical settings, we might imagine that $(w_t,e_t)$ are a sequence of noise which are possibly biased, by have mean say at most $\sigma$ and subGaussian proxy $\sigma^2$. Then, with probability $1 - \delta$, the  $\max\{\| X e_t\|,\| Xw_t\|\} \le \|X\|_{\op} \BigOh{1} \cdot \sigma d \log(T/\delta)$ for any matrix $X$ of rank at most $d$.
	\item By inflating $\radGst, \psiGst$ if necessary, we can take $\psiGst(n)$ to be an upper bound on 
	\begin{align*}
	\BigOh{1}\cdot \sum_{i \ge n} \max\{\|BA^{i-1}C\|_{\op}, \|A^{i-1} C\|_{\op}\}.
	\end{align*}
	This together with the previous statement yields a bound of $\radnat \le \radGst \sigma d \log(T/\delta)$, where $d = \max\{\dimu,\dimy\}$.
	\item  While this parameter regimes suggest suggests that  $\radnat$ is large relative to $\radGst$, we recal that $\radnat,\radGst$ are upper bounds on various system norms, rather than exact characterizations any given norm. Thus, we can satisfy the relation in the previous bullet by inflating $\radnat$ appropriately. Note that our bound degrade gracefully in $\radnat$, so this does not force an undue increase in regret.
	\item 	The reason for the geometric decay is as follows: any stable matrix $A$ with $\rho(A) < 1$ admits a positive definite Lyapunov matrix $P \succeq I$, for which $A^\top P A \preceq ( 1- \epsilon) P$, for some appropriate $\epsilon$. This implies that, for a suitable constant $\kappa > 0$ depending on $\|C\|_{\op},\|B\|_{\op}$, and $\|P\|$,  $\max\{\|C A^{i-1}B \|, \|C A^{i-1} B\|\} \le \kappa (1-\epsilon)^{i-1}$, Thus, using our inflated definition $ \psiGst = \BigOh{1}\cdot \sum_{i \ge n} \max\{\|BA^{i-1}C\|_{\op}, \|A^{i-1} C\|_{\op}\}$, we have that $\psiGst(n) = \BigOh{\kappa} \sum_{i \ge n} K (1-\epsilon)^{n-1} \le \kappa (1-\epsilon)^{n-1}/\epsilon$. Absorbing these other factors int o$\kappa$  gives the desired geometric decrease.
	\item Most conditions can be relaxed up to constant factors, because online learning methods degrade gracefully when parameters are misspecified. The main exceptions are: (a) one needs to still choose $h,m$ so that $\psiGst(h) \lesssim 1/T$, and similarly $\psi(m) \lesssim 1/T$. If the decay parameters are not known exactly, then the learner must choose a larger h to be conservative. (b), for strongly convex losses, the effective strong convexity parameter used must be *less* than the true strong convexity modulus. Lastly, (c), parameters out to be selected so as to ensure stability in the unknown system setting (see \Cref{lem:estimation_stability_simple})
\end{itemize}

\subsection{Efficient $\loneop$ Projection\label{ssec:efficient_projection}}
	We describe an efficient implemtation of the $\loneop$ projection step in the algorithms above. As with other spectral norms, it suffies to diagonalize and  compute a projection of the singular values onto the corresponding vector-ball, which in this case is the ball: $\{(z^{[i]}) : z^{[i]} = 0,i > m, \,\sum_{i=0}^{m-1}\|z^{[i]}\|_{\infty} \le R\}$; an efficient algorithm for this projection step is given by~\cite{quattoni2009efficient}. 

\newpage
\section{Comparison with Past Work \& Lower Bounds \label{app:comparison}}
\subsection{Comparison to Prior Work}
\Cref{table:comparison_known,table:comparison_unknown} describe regret rates for existing alorithms for known system and unknown system settings, respectively. Within each table,  bold lines further divide the results into nonstochastic and stochastic/semi-adversarial regimes. Specifically, stochastic noise means well conditioned noise that is bounded or light-tailed, non-stochastic noise means noise selected by an arbitrary adversary, and semi-adversarial noise is an intermediate regime described formally by Assumption~\ref{asm:noise_covariance_ind}/~\ref{asm:noise_covariance_mart}. 
Compared to past work in non-stochastic control, we compete with stabilizing LDCs, which strictly generalize state feedback control. We note however that for stochastic linear control with fixed quadratic costs, state feedback is optimal, up to additive constants that do not grow with horizon $T$.

\begin{threeparttable}\small
\caption{Comparison with prior work for known system. See above for explanation of relevant settings.}\label{table:comparison_known}
\begin{tabular}{ |c|c|c|c|c|c| }
\hline
\multicolumn{6}{ |c| }{\vspace{.05in}\makecell{\large \textbf{Comparison with Past Work: Known System}}} \\
\Xhline{2\arrayrulewidth}
\hline
  \textbf{Work} & \textbf{Rate} & \textbf{Obs.} & \textbf{Loss Type} & \textbf{Noise Type} & \textbf{Comparator}\\
\Xhline{2\arrayrulewidth}
\hline
  \cite{agarwal2019online}& $\sqrt{T}$ & Full  & \makecell{Adversarial\\ Lipschitz}  & \makecell{Nonstochastic} & \makecell{Disturbance \& \\ State Feedback}\\
 \hline
  Theorem~\ref{thm:know}& $\sqrt{T}$ & Partial  & \makecell{Adversarial\\ Lipschitz}  & \makecell{Nonstochastic} & \makecell{Stabilizing  \\ LDC}\\
 \Xhline{2\arrayrulewidth}
  \hline
 \makecell{\cite{cohen2018online}  \\(Known System \& Noise)\tnotex{tn:a}\;\,} & $\sqrt{T}$ & \makecell{Full } &  \makecell{Adversarial\\ Quadratic}  & \makecell{Stochastic}  & \makecell{
  State Feedback\\(Pseudo-regret)}\\
\hline
  \makecell{\cite{agarwal2019logarithmic} \\ (Known System \& Noise)\tnotex{tn:a}\;\,} & $\polylog\,T$ & Full & \makecell{Adversarial\\ Strongly \\ Convex}  & \makecell{Stochastic } & \makecell{Disturbace \& \\ State Feedback \\ (Pseudo-regret)\tnotex{tn:b}\;\;}\\ 
\hline
  Theorem~\ref{thm:fast_rate_known}& $\polylog\, T$ & Partial &  \makecell{Adversarial\\ Strongly \\ Convex \&\\
  Smooth\tnotex{tn:c}}  & \makecell{Semi-Adversarial } & \makecell{Stabilizing \\ LDC}\\
\Xhline{2\arrayrulewidth}
\end{tabular}
\begin{tablenotes}
\item[(a)] \label{tn:a}\cite{agarwal2019logarithmic}, \cite{cohen2018online} assume the knowledge of the noise model making the assumption stronger than simply knowing the system
\item[(b)]\label{tn:b} Pseudo-regret refers to the best comparator ``outside the expectation''. It is strictly weaker than regret.
\item[(c)] \label{tn:c} The smoothness assumption is necessary to remove the need for the expected-gradient oracle, and can be removed if such a stronger oracle is provided.
\end{tablenotes}
\end{threeparttable}

\begin{threeparttable}\small
\caption{Comparison with prior work for unknown system. See above for explanation of relevant settings.}\label{table:comparison_unknown}
\begin{tabular}{ |c|c|c|c|c|c| }
\hline
\multicolumn{6}{ |c| }{\vspace{.05in}\makecell{\large \textbf{Comparison with Past Work: Unknown System}}} \\
\hline
  \cite{Hazan2019TheNC}~\tnotex{tn:d}& $T^{2/3}$ & Full  & \makecell{Adversarial\\ Lipschitz}  & \makecell{Nonstochastic} & \makecell{Disturbance \& \\ State Feedback}\\
 \hline
  Theorem~\ref{thm:unknown}& $T^{2/3}$ & Partial & \makecell{Adversarial\\ Lipschitz}  & \makecell{Nonstochastic} & \makecell{Stabilizing  \\ LDC}\\
  \Xhline{2\arrayrulewidth}
\hline
 \makecell{\defcitealias{abbasi2011regret}{Abbasi-Yadkori~\&}\citetalias{abbasi2011regret} \\
 \defcitealias{abbasi2011regret}{Szepesv{\'a}ri~[2011]}\citetalias{abbasi2011regret}
 } & \makecell{$e^d \cdot \sqrt{T}$~~\tnotex{tn:e}\;\;} & Full & \makecell{Fixed \\Quadratic} & \makecell{Stochastic} & State Feedback\\
\hline
 \cite{dean2018regret} & $T^{2/3}$ & Full & \makecell{Fixed \\Quadratic} & \makecell{Stochastic} & State Feedback\\
\hline
   \makecell{\cite{cohen2019learning}\\
   \cite{ambujinput}\\
   \cite{mania2019certainty}\tnotex{tn:f}} & $\sqrt{T}$ & Full & \makecell{Fixed\\ Quadratic}  & \makecell{Stochastic}  & State Feedback\\
  \hline
  Theorem~\ref{thm:fast_rate_unknown}& $\sqrt{T}$ & Partial &  \makecell{Adversarial\\ Strongly \\ Convex \&\\
  Smooth\tnotex{tn:g}}   & \makecell{Semi-Adversarial } & \makecell{Stabilizing \\ LDC}\\
\hline
\Xhline{2\arrayrulewidth}
\hline
\end{tabular}
\begin{tablenotes}
\item[(d)] \label{tn:d} To identify the system, \cite{Hazan2019TheNC}  assumes that the pair $(\Ast,\Bst)$ satisfies a strong contrabillity assumption. Our Nature's y's formulation dispenses with this assumption.
\item[(e)] \label{tn:e} This bound is exponential in dimension $d$.
\item[(f)] \label{tn:f} The authors in \cite{mania2019certainty} present technical guarantees that can be used to imply $T^{2/3}$ regret for the partially observed setting when combined with concurrent results. Since the paper was released, stronger system identification guarantees can be used to establish $\sqrt{T}$ regret for this setting \citep{sarkar2019finite,tsiamis2019finite}. To our knowledge, this complete end-to-end result does not yet exist in the literature.
\item[(g)] \label{tn:g} Unlike Theorem~\ref{thm:fast_rate_unknown}, smoothness is still necessary even when given access to the stronger oracle. Alernatively, certain noise distributions (e.g. Gaussian) can be used to induce smoothness.
\end{tablenotes}
\end{threeparttable}

\subsection{Regret Lower Bounds for Known Systems}
\newcommand{\xalg}{\matx^{\Alg}}
\newcommand{\PseudoRegret}{\mathrm{PseudoRegret}}
\newcommand{\scrD}{\mathscr{D}}
\newcommand{\calP}{\mathcal{P}}
\newcommand{\Bernoulli}{\mathrm{Bernoulli}}

We formally prove our lower bound in the following interaction model:
\begin{definition}[Lower Bound Interaction Model]\label{defn:interaction_lb} We assume that $\matx_{t+1} = \Ast \matx_t + \Bst \matu_t + \matw_t$, where $\matw_t$ are drawn i.i.d. from a fixed distribution. We assume that the learners controlers $\ualg_t$ may depend arbitrarily on $\matx_t$ and $(\matu_s,\matx_s)_{1 \le s < t}$. For a policy class $\Pi$ and joint distribution $\calD$ over losses and disturbances, we define
\begin{align*}
\PseudoRegret^{\Alg}_T(\Pi,\calD) := \Exp_{\calD}\left[\sum_{t=1}^T\loss_t(\xalg_t,\ualg_t)\right] - \inf_{\pi \in \Pi}\Exp_{\calD}\left[\sum_{t=1}^T\loss_t(\matx^{\pi}_t,\matu^{\pi}_t)\right] \le \Exp\left[\Regret_T(\Pi)\right],
\end{align*}
\end{definition}
Informally, our lower bound states that $\sqrt{T}$ regret is necessary to compete with the optimal \emph{state feedback} controller for pseudo-regret in the fully observed regime, either when the noises are stochastic and loss is known to the learner, or the the noises are constant and deterministic, and the losses stochastic. Formally:

\begin{theorem}\label{thm:lb_known}  Let  $\dimx = 2$, and $\dimu = 1$,  $\Ast = \veczero_{\dimx \times \dimx}$, $\Bst =-e_1$, and $\Pi$ denote the set of all state-feedback controllers of the form $\matu_t = K_v \matx_t$, for $K_v = v \cdot e_1^\top$. Then for the interaction Model of Definition~\ref{defn:interaction_lb}, the following hold
\begin{enumerate}
	\item \textbf{\emph{Fixed Lipschitz Loss \& Unknown i.i.d Noise:}} Fix a loss $\ell(x,u) = |x[1]|$, and a family of  distributions $\scrP = \{\calP\}$ over i.i.d. sequences of $\matw_t$ with $\|\matw_t\| \le 2$ for $T \ge 2$
	\begin{align*}
	\inf_{\Alg}\max_{\calD \in \{\I_{\ell_t = \ell}\} \otimes \scrP}\,\PseudoRegret^{\Alg}_T(\Pi,\calD) \ge  -1 + \Omega(T^{1/2}),
	\end{align*}
	where $\{\I_{\ell_t = \ell}\} \otimes \scrP$ is the set of joint loss and noise distribution induced $\ell_t = \ell$ and $\matw_t \iidsim \calP$. 
	\item \textbf{\emph{I.i.d Lipschitz Loss \& Known Deterministic Noise}} Then there exists a family of distributions $\scrP = \{\calP\}$ over i.i.d sequences of $1$-Lipschitz loss functions with $0 \le \ell_t(0) \le 1$ almost surely such that 
	\begin{align*}
	\inf_{\Alg}\max_{\calD \in \scrP \otimes \{I_{\matw_t = (1,0}\}\}}\,\PseudoRegret_T^{\Alg}(\Pi,\calD) \ge -1+ \Omega(T^{1/2})
	\end{align*}
	where $\scrP \otimes \{I_{\matw_t = (0,1)}\}\}$ is the set of  joint loss and noise distribution induced $\ell_1,\ell_2,\dots \iidsim \calP$ and $\matw_t = (1,0)$ for all $t$.
\end{enumerate} 
\end{theorem}

\begin{proof} Let us begin by proving Part 1. 
Let $\scrP$ denote the set of distributions $\calP_p$ where $\matw_t[2] = 1$ for all $t$, and $\matw_t[1] \iidsim \mathrm{Bernoulli}(p)$ for $t \ge 1$.  Let $\Exp_p$ denote the corresponding expectation operator, and let $\PseudoRegret_T^{\Alg}(T,p)$ denote the associated PseudoRegret. We can verify
\begin{align*}
\xalg_{t+1} = \begin{bmatrix}\matw_{t}[1] - \ualg_t[1]\\1\end{bmatrix}, \matu^{K_v}_t = v, \matx^{K_v}_t = \begin{bmatrix}\matw_t[1] - v\\1\end{bmatrix}\,
\end{align*}
For $\ell(x,u) = |x[1]|$, we have
\begin{align*}
\Exp_p[\sum_{t=1}^T \loss(\xalg_t,\ualg_t)] &= \Exp_p[\sum_{t=1}^T |\matw_{t}[1] - \ualg_t[1]|]~\,.
\end{align*}
Since $\matw_t[1]\in \{0,1\}$, we can assume $\ualg_t[1] \in [0,1]$, since projecting into this interval always decreases the regret. In this case, given the interaction model, $\matw_t[1] \mid \ualg_t[1]$ is still $\Bernoulli(p)$ distributed. Therefore, 
\begin{align*}
\Exp_p[\sum_{t=1}^T |\matw_{t}[1] - \ualg_t[1]|] &= \Exp_p[ \sum_{t=1}^T (1-p) \ualg_t + p(1-\ualg_t)] \\
&= pT+ (1-2p)\Exp_p[ \sum_{t=1}^T  \ualg_t]\\
&= pT + (1-2p)T\Exp_p[ Z_t],
\end{align*} 
where we let $Z_t := \frac{1}{T}\sum_{t=1}^T\ualg_t \in [0,1]$. On ther other hand, for any $v \in [0,1]$,
\begin{align*}
\Exp_p[\sum_{t=1}^T \loss(\matx^{K_v}_t,\matu^{K_v})] \le 1 + T\{(1-p)v + p(1-v)\} = 1 + pT + (1-2p)Tv,
\end{align*}
where the additive $1$ accounts for the initial time step. Hence, 
\begin{align*}
\PseudoRegret_T^{\Alg}(T,p) \ge -1 + \max_{v \in [0,1]}(1-2p)T(\Exp_p(Z_t) - v) = -1 + T\left|1 - 2p\right|\cdot\left|\Exp_p(Z_t) - \I(1-2p \ge 0)\right|.
\end{align*}
The lower bound now follows from a hypothesis testing argument. Since $Z_t \in [0,1]$, it follows that there exists an $\epsilon = \Omega(T^{-1/2})$ such that (see e.g. \cite{kaufmann2016complexity})
\begin{align*}
|\Exp_{p = 1/2+\epsilon}(Z_t) - \Exp_{p = 1/2-\epsilon}(Z_t)| \le \frac{7}{8}.
\end{align*}
Combining with the previous display, this shows that for $p \in \{1/2-\epsilon,1/2+\epsilon\}$,
\begin{align*}
\PseudoRegret_T^{\Alg}(T,p) \ge - 1 + T\left|1 - 2p\right|\Omega(1) \ge -1 + \Omega(T\epsilon) = -1 + \Omega(T^{1/2}).
\end{align*}
This proves part 1. Part $2$ follows by observing that the above analysis goes through by moving the disturbacnce into the loss, namely $\loss_t(x,u) = |x[1] -\mate_t|$ where $\mate_t \iidsim \Bernoulli(p)$ and $\matw_t[1] = 0$ for zero.
\end{proof}

\section{Generalization to Stabilized Systems\label{sec:generalization_to_stabilized}}
In this section, we consider a generalization to settings where the system may not be internally stable; that is, where $\rho(\Ast) \ge 1$. Throughout, we assume the system is stabilizable and detectable: a linear system is said to be \emph{stabilizable} if, in the absence of perturbations, there is a state-feedback controller that drives the state of the system asymptotically to zero; a \emph{detectable} system is one where, in absence of perturbations, the state asymptotically tends to zero as long as the observations are all zeros. Relaxing the notions notions of controllability and observability respectively, these requirements do not impose any conditions on the stable modes of the system. In particular, we will employ these assumptions to guarantee the existence of a stabilizing observer-feedback control. See~\cite{anderson2007optimal} for an extensive discussion.

Our general recipe is as follows:
\begin{enumerate}
	\item We assume access to a stabilizing \emph{nominal controller} $\pinot$. This induces an \emph{dynamical} system with exogenous inputs, or \ldcex{} (\Cref{defn:ldcex}).
	\item The \ldcex~produces a control ouput, $\mateta_t$. It's ``natural'' version $\nateta_t$ (\Cref{defn:nateta}) can be computed from input output data, and is what is used to parametrize the controller. This formulation is described in \Cref{ssec:ldcsex}.
	\item In \Cref{ssec:drcex}, we formalally detail our controller parametrization for this framework, which we call \drcex{}, or Disturbance Response Control with Exogenous inputs. We then provide the generalization of our main algorithm, which we term \drcgdex.
	\item In \Cref{ssec:ldcex_examples}, we detail various examples of \ldcex{} parametrizations. 
	\begin{enumerate}
		\item We show that the stable setting can be recovered as a special case, as well as the static-feedback control, and control with nominal stabilizing controllers which are themselves internally-stable (\Cref{exm:stable_system,exm:static_feedback,exm:stable_feedback}). 
		\item In general, unstable systems may require \emph{internally-unstable} controllers to yield stable closed-loop dynamics. To this end, we describe an \ldcex{} parametrization based on \emph{exact observer feedback} (\Cref{exm:observe_feedback}), which yields the classical Youla parametrization \citep{youla1976modern}, and allows us extend our results to arbitrary stabilizable and detectable systems.
		\item The exact Youla parametrization requires \emph{full system knowledge} to construct an exact observer-feedback controller. To circumvent this, we demonstrate a convex parametrization based on approximate observer feedback, \Cref{exm:approx_observe_feedback}. This combines the classical Youla parametrization with a perspective based on Nature's $\eta$'s, which affords convex parametrization without an exact observer-feedback controller. 
\end{enumerate}
\item Finally, in \Cref{ssec:drcex_expressivity}, we demonstrate that all above examples of \drcex parametrizations are \emph{fully expressive}, in that they can approximate the dynamics of any stabilizing linear dynamic controller to arbitrary degrees of accuracy (\Cref{thm:policy_approximation2,thm:full_expressivity}).
\end{enumerate}

\subsection{Preliminaries}
Going forward, it will be useful to slightly formalize our notion of Markov operators, which we shall interchangably refer to as \emph{transfer operators}. We define
\begin{restatable}[Markov Operator]{definition}{transfuncdef}\label{def:trans_func}Let $\scrG^{\dout \times \din}$ denote the set of Markov operators $G = (G^{[i]})_{i \ge 0}$ with $G^{[i]} \in \R^{\dout \times \din}$, such that $\|G\|_{\loneop} < \infty$. Given a system $(A,B,C,D)$ with input dimension $\din$ and output dimension $\dout$, we let $G = \Transfer(A,B,C,D) \in \scrG_{\dout \times \din}$ denote the system $G^{[0]} = D$ and $G^{[i]} = CA^{i-1}B$.
\end{restatable}
Next, we state a computation of the joint evolution of a system under an LDC $\pi$:
\begin{lemma}\label{l:augd}
Let $(\matypi_t, \matxpi_t)$ be the observation-state sequence produced on the execution of a LDC $\pi$ on the LDS parameterized via  $(\Ast, \Bst, \Cst)$. For a given sequence of disturbances $(\mate_t,\matw_t)$, the joint evolution of the system may be described as
\begin{align}
\begin{bmatrix} \matxpi_{t+1} \\ \matvpi_{t+1} \end{bmatrix} 
&= \underbrace{\begin{bmatrix}
\Ast + \Bst \Dpi \Cst & \Bst \Cpi \\
\Bpi \Cst & \Api
\end{bmatrix}}_{\Apicl} \begin{bmatrix} \matxpi_t \\ \matvpi_t\end{bmatrix}
+ \underbrace{\begin{bmatrix} I & \Bst \Dpi \\ 0 & \Bpi  \end{bmatrix}}_{\Bpicl} \begin{bmatrix} \matw_t \\ \mate_t \end{bmatrix},  \\
\begin{bmatrix} \matypi_{t} \\ \matupi_{t} \end{bmatrix} 
&= \underbrace{\begin{bmatrix}
   \Cst & 0 \\
    \Dpi \Cst & \Cpi
\end{bmatrix}}_{\Cpicl} \begin{bmatrix} \matxpi_t \\ \matvpi_t\end{bmatrix}
+ \underbrace{\begin{bmatrix} 0 & I  \\ 0 & \Dpi   \end{bmatrix}}_{\Dpicl}  \begin{bmatrix}
\matw_t \\ \mate_t
\end{bmatrix}.
\end{align}
We refer to this dynamical system as the \emph{closed-loop} system in the main paper. Finally, we define
\begin{align*}
\Cpiclu =  \begin{bmatrix} \Dpi \Cst & \Cpi
\end{bmatrix}, \quad \Bpicle := \begin{bmatrix} \Bst \Dpi \\ \Bpi   \end{bmatrix},
\end{align*}
and let $\Gpicleu := \Transfer(\Apicl,\Bpicle,\Cpiclu,\Dpi)$. 
\end{lemma}
\begin{proof}
The dynamical equations may be verified as an immediate consequence of Equation~\ref{eq:pi_dynamics}.
\end{proof}

\begin{definition}[Markov Operators for closed loop systems]\label{defn:markov_closed_loop}
Given  an LDC $\pi$, we define the systems,
\begin{align*}
\Gpicluyu^{[i]} &:= \I_{i = 0} \Dpicl  + \I_{i > 0} \Cpicl \Apicl^{i-1}\Bpiclin, \quad \Gpicluyu^{[i]} = \begin{bmatrix} \Gpicluy^{[i]}\\
\Gpicluu^{[i]}\\
\end{bmatrix}, \quad \Bpiclin := \begin{bmatrix} \Bst \\
0
\end{bmatrix}
\end{align*}
where $(\Apicl,\Bpicl,\Cpicl,\Dpicl)$ are given by \Cref{l:augd}. 
Furthernote, we define $\psipicl = \psi_{\Gpicl}$ as the decay function of $\Gpicl$, namely, $\psipicl(n) = \sum_{i \ge n}\Gpicl$.
\end{definition}

\subsection{Linear Dynamic Controllers with Exogenous Inputs (\ldcex) \label{ssec:ldcsex}}
In this section, let us set up a general stabilized parametrization. First, let us define the notion of an internal stabilizing controller:
\begin{definition}\label{defn:ldcex} An \emph{linear dynamic controller with exogenous inputs} or $\ldcex$, denoted by a policy $\pinot = (\Apinot,\Bpinot,\Cpinot,\Dpinot)$, as well as matrices $\Bpinotu,\Cpinoteta,\Dpinoteta$, which selects inputs $\ualg_t$ according to the following dynamics:
\begin{align*}
\sopen_{t+1} &= \Apinot \sopen_t + \Bpinot \yalg_t + \Bpinotu \uex_t \\
\uen_{t} &= \Cpinot \sopen_t + \Dpinot \yalg_t  \\
\etalg_{t} &= \Cpinoteta \sopen_t + \Dpinoteta \yalg_t \\
\ualg_t &= \uex_t + \uen_t 
\end{align*}
We refer to $\uex_t$ as the \emph{exogenous input}, $\uen_t$ as the \emph{internal input}, and $\ualg_t$ as the \emph{total input}. We refer to $\etalg_t \in \R^{\dimeta}$ as the control-output. The control policy $\pinot$ is called the \emph{nominal controller}.  Lastly, we also define $\valg_t := (\yalg_t,\ualg_t) \in \R^{\dimy +\dimu}$, which we call the \emph{total output}.
\end{definition} 
Overloading notation, we will alternatively use $\pinot$ to refer to the policy control $(\Apinot,\Bpinot,\Cpinot,\Dpinot)$, and to index objects  associated with both $\pinot$ and the additional matrices $\Bpinotu,\Cpinoteta,\Dpinoteta$.

In an $\ldcex$, the endogenous input $\uen_t$ is chosen so that, in the absence of inputs  -- i.e. $\uex_t \equiv 0$ -- the joint dynamics of the system remain stable. This allows us to generalize to settings where the dynamics of the nominal system may not be stable. For somewhat sophisticated reasons, in stabilized systems, one can be restricted by using Nature's y's for inputs. Instead, we will base our inputs on Nature's $\eta$'s, defining $\nateta_t$ to be the control-output in the absence of exogenous inputs:

\begin{defnmod}{def:naty}{b}[Natures u's, y's, $\eta$'s]\label{defn:nateta} We define $\natu_t$, $\naty_t$, and $\nateta_t$ as the sequence that arises when, for all $s$, $\uex_t = 0$. We set $\natv_t = (\naty_t,\natu_t)$. We note that $\natu_t,\naty_t$ coincide with $\upinot_t,\ypinot_t$, whose dynamics are given by \Cref{l:augd} with the policy $\pi \gets \pinot$.
\end{defnmod}
Rather than requiring the nominal system to be stable, we will use controllers based on $\nateta_t$ (or estimates thereof). This requires only that the $\pinot$ stabilize $\Ast$. Formally:
\begin{asmmod}{asm:a_stab}{b}[Stabilized Setting]\label{a:stablize}
We assume that an $\ldcex$ is \emph{stabilizing}; namely that $\Apinotcl$ is stable, where $\Apicl$ be defined in \Cref{l:augd}.  
\end{asmmod}

In order to define our \drcex~parameterization, we need to introduce the following relevant transfer operators. We note that the `$A$' matrix in each of the following Markov operators is $\Apinotcl$, which is stable by the above assumption, so each of the following operators are stable:
	\begin{defnmod}{def:markov}{b}[Markov Operators for Strongly Stabilized System]\label{defn:markov_stabilized} Fix an \ldcex{} controller, and let $\Apinotcl,\Bpinotcl,\Cpinotcl,\Dpinotcl$ be as in \Cref{l:augd}, with $\pi \leftarrow \pinot$. Further, define 
		\begin{align*}
		\Cpinotcleta &:= \begin{bmatrix}\Dpinoteta \Cst & \Cpinoteta\end{bmatrix}, \quad \Dpinotcleta := \begin{bmatrix} 0 & \Dpinoteta\end{bmatrix}, \\
		\Bpinotclex &= \begin{bmatrix} \Bst \\
		\Bpinotex 
		\end{bmatrix}, \quad \Dpinotclex := \begin{bmatrix} 0 \\
		I \end{bmatrix},
		\end{align*}
		and the transfer functions $\Gexyu \in \scrG^{(\dimy + \dimu) \times \dimu}$ and $\Gexeta \in \scrG^{\dimeta \times \dimu}$ via
		\begin{align*}
		\Gexyu &:=  \Transfer(\Apinotcl,\Bpinotclex,\Cpinotcl,\Dpinotclex)\\
		\Gexeta &:=  \Transfer(\Apinotcl,\Bpinotclex,\Cpinotcleta,0)\\
		\Gweta &:=  \Transfer(\Apinotcl,\Bpinotcl,\Cpinotcleta,\Dpinotcleta)\\
		\Gweyu &= \Transfer(\Apinotcl,\Bpinotcl,\Cpinotcl,\Dpinotcl),
		\end{align*}
		and will decompose $\Gexyu = \begin{bmatrix} \Gexy\\\Gexu\end{bmatrix}$ for appropriate $\Gexy \in \scrG^{\dimy \times \dimu},\Gexu \in \scrG^{\dimu \times \dimu}$. 
	\end{defnmod}
		We can now write a ``Nature's y's'' representation of all relevant quantities:
	

	\begin{lemmod}{lem:supy}{b}\label{lem:supy_stab}
		We have the following identities for $(\yalg_t,\uenalg_t,\ualg_t)$:
		\begin{align*}
		\begin{bmatrix}
		\yalg_t\\
		\uenalg_t 
		\end{bmatrix} = \begin{bmatrix} \naty_t\\
		\natu_t \end{bmatrix} + \sum_{i=1}^{t-1} \Gexyu^{[i]} \uex_{t-i}
		\end{align*}
		Moreover, we have the following identity for $\etalg_t$:
		\begin{align*}
		\etalg_t = \nateta_t + \sum_{i=1}^t \Gexeta^{[i]} \uex_{t-i}
		\end{align*}
		Finally, we can express Nature's y's, u's and $\eta$'s as functions of the noise via
		\begin{align*}
		\natv_t := \begin{bmatrix} \naty_t\\
		\natu_t \end{bmatrix} = \sum_{i=1}^t \Gweyu^{[i]} \begin{bmatrix}\matw_{t-i}\\
		\mate_{t-i}\end{bmatrix}		, \quad \nateta_t = \sum_{i=1}^t \Gweta^{[i]} \begin{bmatrix}\matw_{t-i}\\
		\mate_{t-i}\end{bmatrix}
		\end{align*}
	\end{lemmod}

 The proof of the above lemma is a consequence of computation augmenting that of \Cref{l:augd} computation, whose proof we omit in the interest of brevity. We now state the relevant generalization of \Cref{asm:bound}, which by the above lemma and a similar computation for the mapping of $(w,e) \to (y,u)$, holds for any bounded noise sequence:
\begin{asmmod}{asm:bound}{b}[Bounded Nature's $y$, $u$, $\eta$]\label{asm:bound_b}
We assume that that $\matw_t$ and $\mate_t$ are chosen by an oblivious adversary, and that $\twonorm{\natv_t} := \twonorm{(\naty_t,\natu_t)} \le \radnat$  and $\twonorm{\nateta_t} \le \radnat$ for all $t$. \iftoggle{colt}
{
\footnote{Note that, if the system is stable and perturbations bounded, that $\naty_t$ will be bounded for all $t$.}
}{}
\end{asmmod}

\subsection{\drcex{} Parametrization and Algorithm \label{ssec:drcex}}

	Let us now describe the \drcex{} parametrization. 	Throughout, we will supress dependence on $\pinot$. 
	
	\begin{defnmod}{defn:drc}{b}[Disturbance Response Controller with Exogenous Inputs]\label{defn:drc_stabilize}
		A Disturbance Response Controller with Exogenous Inputs (\drcex), parameterized by a $m$-length sequence of matrices $M = (M^{[i]})_{i=0}^{m-1}$, chooses the control input as $\uex_t = \sum_{s=0}^{m-1} M^{[s]} \nateta_{t-s}$. For a fixed $M$, we denote the resultant inputs, ouputs, and control-outputs $(\maty_t^M,\matu_t^M,\mateta_t^M)$, and let $J_T(M)$ the loss functional. We also set $\matv^M = (\maty^M,\matu^M)$.
	\end{defnmod}
	Parallel to the stable setting, if $\Gexeta$ is known exactly, one can exactly recover $\nateta_{t-s}$ via \Cref{lem:supy_stab}. When unknown, we can approximately recover $\nateta_{t-1}$ using an estimate $\Ghatexeta$, namely (\Cref{line:alg:natv_nateta})
	
	\begin{align*}
	\natetahat_t := \etalg_t - \sum_{i=1}^{t} \Ghatexeta^{[i]}\uex_{t-i}
	\end{align*}
	
	Thus, we propose to use the estimates $\natetahat_t $ to define our controller. Moreover, to estimate the consequence of a given input, we also need to estimate the $\matv_t = (\naty_t,\natu_t)$ so can ascertan the baseline in the absence of exogenous input. Thus we take
	\begin{align*}
	\natvhat_t := \valg_t - \sum_{i=1}^{t} \Ghatexyu^{[i]}\uex_{t-i}.
	\end{align*}
	The above definitions give rise to the following counterfactual dynamics and losses:
	\begin{defnmod}{defn:counterfactual}{b}[Counterfactual Costs and Dynamics, Stabilized Systems]\label{defn:counterfactual_stabilized} Let $\natv_t = (\naty_t,\natu_t)$, and $\natvhat_t$ denote estimates of $\natv_t$. We define the counterfactual costs and dynamics
			\begin{align*}
			&\uex_t\left(M_t \midhateta\right) := \sum_{i=0}^{m-1} M_t^{[i]}\cdot\natetahat_{t-i}\\
			&\matv_t\left[M_{t:t-h}\midhatetayg\right] := \natvhat_t + \sum_{i=1}^{h} \Ghatexyu^{[i]} \cdot \uex_{t-i}\left(M_{t-i}\midhateta[t-i]\right), \\
			  &F_t\left[M_{t:t-h}\midhatetayg\right] := \loss_t\left(\matv_t\left[M_{t:t-h}\midhatetayg\right]\right)
			\end{align*}
			Overloading notation, we let $\matv_t(M_t \mid \cdot) := \matv_t[M_t,\dots,M_t \mid \cdot]$ denote the unary (single $M_t$) specialization of $\matv_t$, and lower case $f_t\left(M|\cdot\right) = F_t\left[M,\dots,M|\cdot \right]$ the specialization of $F_t$. Throughout, we use paranthesis for unary functions of $M_t$, and brackets for functions of $M_{t:t-h}$.
			\end{defnmod}

		The gradient feedback controller (\Cref{alg:improper_OGD_main})  and estimation procedure  (\Cref{alg:estimation}), and \drcgd{} algorithm for unknown algorithm (\Cref{alg:unknown}) are modified in algorithms  \Cref{alg:improper_OGD_main_stabilized,alg:estimation_stabilized,alg:unknown_ex}, respectively.

		\begin{algorithm}[h!]
		\textbf{Input: }  Stepsize $\eta_t$, radius $R$, memory $m$, Markov operators $\Ghatexyu,\Ghatexeta$, rollout $h$.\\
		Define $\mathcal{M} = \mathcal{M}(m,R) = \{M=(M^{[i]})_{i=0}^{m-1} : \|M\|_{\loneop}\leq R\}$.\\
		Initialize $\matM_1\in \mathcal{M}$ arbitrarily.\\
		\For{$t= 1,\dots,T$}
		{
		  Observe $\valg_t = (\yalg_t,\ualg_t)$\\
		  Update $\sopen_t,\uen_t,\etalg_t$ as in \Cref{defn:ldcex}\\
		  Estimate $\natvhat_t$ and $\natetahat_t$ via \label{line:alg:natv_nateta}
		  \begin{align*}
		  \natvhat_t := \valg_t - \sum_{i=1}^{t-1} \Ghatexyu^{[i]}\uex_{t-i}, \quad \natetahat_t := \etalg_t - \sum_{i=1}^{t} \Ghatexeta^{[i]}\uex_{t-i}
		  \end{align*} \\
		  \vspace{-.16in}
		  Choose the exogenous control input as 
		  \begin{align*} 
		  \uexalg_t \leftarrow \uex_t(\matM_t \midhateta) = \sum_{i=0}^{m-1} \matM_t^{[i]} \,\natetahat_{t-i}.
		  \end{align*}\\
		   \vspace{-.16in}
		  Play total input $\ualg_t = \uexalg_t + \uen_t$\\
		  Observe the loss function $\ell_t$ and suffer a loss of $\ell_t(y_t, u_t)$.\\
		  Recalling $f_t(\cdot| \cdot)$ from Definition~\ref{defn:counterfactual_stabilized},update the disturbance feedback controller as $ \matM_{t+1} = \Pi_{\calM}\left(\matM_t - \eta_t \partial f_t\left(\matM_t \midhatetayg\right)\right) $
		}
		\caption{Disturbance Response Controller via Gradient Descent, with Exogenous Inputs (\drcgdex)\label{alg:improper_OGD_main_stabilized}}
		\end{algorithm}
		\begin{algorithm}[h!]
		\textbf{Input: } Number of samples $N$, system length $h$. \\
		\textbf{Initialize} 
		\begin{align*}
		\Ghatexyu^{[0]} = \begin{bmatrix} 0_{\dimy\times\dimu} \\ I_{\dimu} \end{bmatrix}, \quad \Ghatexyu^{[i]} = 0_{(\dimy+\dimu)\times \dimu}, \quad \Ghatexeta^{[i]} = 0_{\dimeta \times \dimu}, \quad i > 0
		\end{align*}\\
		\For{t = $1,2,\dots,N$}
		{
			Play $\uexalg_t \sim \mathcal{N}(0,I_{\dimu})$, recieve $\valg_t = (\ualg_t,\yalg_t)$ and $\etalg_t$
		}
		Estimate $\Ghatuy^{[1:h]}$ and $\Ghatuu^{[1:h]}$
		via least squares:
		\begin{align*}
		&\Ghatexyu^{[1:h]} \leftarrow \argmin_{G^{[1:h]}} \sum_{t=h+1}^N\|\valg_t - \sum_{i=1}^h G^{[i]}\uexalg_{t-i}\|_2^2\\
		&\Ghatexeta^{[0:h]} \leftarrow \argmin_{G^{[0:h]}} \sum_{t=h+1}^N\|\etalg_t - \sum_{i=0}^h G^{[i]}\uexalg_{t-i}\|_2^2
		\end{align*}
		\\
		\text{Return} $\Ghatexyu,\Ghatexeta$.
		\caption{Estimation of Unknown System\label{alg:estimation_stabilized}}
	\end{algorithm}

	\begin{algorithm}[h!]
	\textbf{Input: }  Stepsizes $(\eta_t)_{t \ge 1}$, radius $\radM$, memory $m$, rollout $h$, Exploration length $N$, \\
	Run the estimation procedure (Algorithm~\ref{alg:estimation}) for $N$ steps with system length $h$ to estimate $\Ghatexyu,\Ghatexeta$\\
	Run the regret minimizing algorithm (Algorithm~\ref{alg:improper_OGD_main}) for $T-N$ remaining steps with estimated Markov operators $\Ghatexyu,\Ghatexeta$, stepsizes $(\eta_{t+N})_{t \ge 1}$, radius $\radM$, memory $m$, rollout parameter $h$.
	\caption{\drcgdex{} for Unknown System\label{alg:unknown_ex}}
	\end{algorithm}

\newpage
\subsection{Examples of LDC's with exogenous inputs \label{ssec:ldcex_examples} }

Let us now provide examples of possible LDC's with exogenous inputs $\pinot$ which can be used. The first three examples  (\Cref{exm:stable_system,exm:static_feedback,exm:stable_feedback}) are only pertain to a subset of dynamical systems - namely those that are (a) internally stable, (b) stabilizable by static feedback, or (c) stabilized by an internally stable controller.

In general, the are certain pathological which are unstable, and cannot be stabilized by static feedback or internally stable controller (see e.g. \cite{halevi1994stable}). For general systems, \Cref{sssec:youla_ldcex} describes an LDC-ex formulation based on powerful parametrization known as the ``Youla parametrization'' \cite{youla1976modern}, also attributed to \cite{kuvcera1975stability}, which uses an \emph{observer-feedback} controller to provide an internally stabilizing, convex controller parametrization for \emph{arbitrary} systems. 

Unfortunately, realizing an exact Youla parametrization requires exact system knowledge. To adress this, we consider introduce an LDC-ex parametrization based on \emph{approximate youla parametrization}. Under mild conditions, we shall show that these parametrizations have the same expressive power as the exact Youla parametrization, despite allowing for inexact system knowledge.

\begin{example}[Stable System]\label{exm:stable_system}The internally stable system case (\Cref{asm:a_stab}) corresponds to the setting where $\rho(\Ast) < 1$. Hence, we can $\Apinot,\Bpinot,\Cpinot,\Dpinot$ to be identically zero, $\etalg_t = \yalg_t$, corresponding to  $\Cpinoteta = 0$ and $\Dpinoteta = I$, . This satisfies \Cref{a:stablize} because $\Apinotcl = \Ast$, which is stable by assumption. This identical to the stable system setting in the body of the paper.
\end{example}

\begin{example}[Static Feedback]\label{exm:static_feedback}  Under static feedback, we take $\Apinot,\Bpinot,\Cpinot$ to be zero, but set $\Dpinot = K$ for a static-feedback matrix $K \in \R^{\dimy \times \dimx}$. Again, we set $\etalg_t = \yalg_t$, corresponding to  $\Cpinoteta = 0$ and $\Dpinoteta = I$. From \Cref{l:augd}, the closed-loop matrix $\Apinotcl$ is given by $\Ast + \Bst K \Cst$. Thus, we require $K$ such that $\rho(\Ast + \Bst K \Cst) < 1$. For general partially observed systems, it may not be the case that such a $K$ exists, even if the system is stabilizable (i.e. there exists a control policy $\pinot$ which stabilizes it). However, for stabilizable fully observed systems, such a $K$ is always guaranteed to exist, and can be obtained by solving the discrete algebraic Riccati equation, or DARE \citep{anderson2007optimal}. Observe that static feedback reduces to the stable-system setting when $K = 0$.
 \end{example}
\begin{example}[Stabilizing Feedback]\label{exm:stable_feedback} More generally, we can select a stabilizing controller $\pinot$ such that  $\Apinot,\Bpinot,\Cpinot,\Dpinot$ need not be zero, but both the internal controller dynamics, and the closed-loop dynamics are stable. That is, $\rho(\Apinot) < 1 $ and $\rho(\Apinotcl) < 1$. Yet again, we set $\etalg_t = \yalg_t$, corresponding to  $\Cpinoteta = 0$ and $\Dpinoteta = I$. Note that this strictly generalizes \Cref{exm:stable_system,exm:static_feedback}: Static feedback is  recovered by setting $\Apinot,\Bpinot,\Cpinot = 0$ and $\Dpinot = K$, and stable systems by setting $\Dpinot = 0$ as well.
\end{example}

\subsubsection{Exact Youla LDC-Ex \label{sssec:youla_ldcex}}

As described above, certain pathological systems may not admit any stabilizing controller $\pinot$ satisfying \Cref{exm:stable_feedback}, and thus no controllers satisfying either of the special cases \Cref{exm:stable_system,exm:static_feedback}. However, all stabilizable system and detectable systems \emph{do} admit stabilizing controllers of the following form:
\begin{example}[Exact Observer Feedback]\label{exm:observe_feedback} Consider a stabilizable and detectable system, and fix matrices $L,F$ that satisfy $\rho(\Ast + \Bst F) < 1$ and $\rho(\Ast + L \Cst) < 1$. Exact Observer Feedback with Exogenous inputs denotes the internal state $\sopen_t$ via $\matxtil_t \in \R^{\dimx}$, and has the dynamics
\begin{align*}
	\matxtil_{t+1} &= (\Ast + L \Cst)\matxtil_t - L \maty_t + \Bst \ualg_t \\
	\nateta_t &= \Cst\matxtil_t - \yalg_t, \quad \ualg_t = \uex_t + F \matxtil_t,
	\end{align*}
	with $\matxtil_1 = 0$. This yields an LDC-ex $\dimeta = \dimy$, with $\Apinot = (\Ast + L \Cst + \Bst F)$, $\Bpinot = - L$, $\Cpinot = F$, $\Dpinot = 0$, $\Cpinoteta = \Cst$, and $\Dpinoteta = - I$.
\end{example}
Note that the optimal LQG controller is an observer-feedback controller. However, for this parametrization, we don't need to know this optimal LQG controller. Rather, \emph{any} observer-feedback controller will suffice.

\newcommand{\Aweta}{A_{(w,e)\to \eta}}
\newcommand{\Bweta}{B_{(w,e)\to \eta}}
\newcommand{\Cweta}{C_{(w,e)\to \eta}}
\newcommand{\Dweta}{C_{(w,e)\to \eta}}

\begin{lemma}\label{lem:youla_facts} Under \Cref{exm:observe_feedback}, following identities hold:
\begin{enumerate}
	\item $\Gexeta^{[i]} = 0$ for all $i > 0$. In other words, $\etalg_t = \nateta_t$ for all $t$, regardless of exogenous inputs.
	\item We have the identity.
	\begin{align*}
	\Gweta^{[i]} = \I_{i = 0} \begin{bmatrix} 0 & I_{\dimy} \end{bmatrix} + \I_{i > 0}\Cst (\Ast + L \Cst)^{i-1} \begin{bmatrix} I_{\dimx}  & F \end{bmatrix}.
	\end{align*}
	\item We have the identity
	\begin{align*}
	\Gexyu^{[i]} = \I_{i = 0} \begin{bmatrix} 0 \\ I_{\dimu} \end{bmatrix} + \I_{i > 0}\begin{bmatrix}\Cst \\ F \end{bmatrix} (\Ast + \Bst F)^{i-1} \Bst.
	\end{align*}
	\item We have the identity
	\begin{align*}
	\Gweyu^{[i]} = \I_{i = 0} \begin{bmatrix} 0 & I \\
	0 & 0 \end{bmatrix} +  \I_{i > 0} \begin{bmatrix} I & 0 \\
	0 & -L
	\end{bmatrix} \begin{bmatrix} \Ast & \Bst F\\
	-L \Cst & \Ast + BF + L \Cst \end{bmatrix}^{i-1}  \begin{bmatrix} \Cst & 0 \\
	0 & F
	\end{bmatrix}
	\end{align*}
	Moreover, via a change of basis, we can write
	\begin{align*}
	\Gweyu^{[i]} = \I_{i = 0} \begin{bmatrix} 0 & I \\
	0 & 0 \end{bmatrix} +  \I_{i > 0} \begin{bmatrix} I & 0 \\
	-L & -L
	\end{bmatrix} \begin{bmatrix} \Ast + \Bst F & \Bst F\\
	0 & \Ast + L\Cst  \end{bmatrix}^{i-1}  \begin{bmatrix} \Cst & -F \\
	0 & F
	\end{bmatrix}
	\end{align*}
\end{enumerate}
 \end{lemma}
 \begin{proof} The first four computations may be verified directly. Alternatively, Lemma~\ref{lem:new_approx_youla} suffices to establish this while substituting $(\Ast, \Bst, \Cst) =(\Ahat, \Bhat, \Chat)$. For the last claim, a change of basis conjugating the $\Apinotcl$ matrix by  $T = \begin{bmatrix} I & 0 \\ I & I \end{bmatrix}$, via $T^{-1} \Apinotcl T $ suffices.
 \end{proof}
In particular, since $\rho(\Ast + \Bst F),\rho(\Ast + L \Cst) < 1$ hold by assumption due to stabilizability and detectability, all of the above systems are guaranteed to be stable.

\subsubsection{Approximate Youla LDC-Ex (\ldcex)\label{sssec:approx_youla_ldcex}}
The previously suggested parameterization requires exact specification of the system parameters $(\Ast, \Bst, \Cst)$. However, for an unknown system, one can only hope to estimate parameters approximately. This section details the effects of executing a Youla controller with approximate estimates of the system parameters.

\begin{example}[Approximate Youla LDC-Ex]\label{exm:approx_observe_feedback}
An Approximate Observer-Feedback controller when given parameter estimates $\Ahat,\Bhat,\Chat$ and executed under the influence of exogenous inputs follows:
\begin{align*}
	\matxhat_{t+1} &= (\Ahat + L \Chat)\matxhat_t - L \maty_t + \Bhat \matu_t\\
	\matetahat_t &= \Chat\matxhat_t - \maty_t \\
	\matu_t &= \uex_t + F \matxhat_t.
	\end{align*}
\end{example}
	
	Note that $\matetahat_t$ depends on the history of exogenous inputs $\uex_t$. Still, we can give a closed form representation of the overall system dynamics, and the map from exogenous inputs to outputs/controls:
	\newcommand{\Ainhat}{A_{\hat{\cdot},\mathrm{in}}}
	\newcommand{\Binhat}{B_{\hat{\cdot},\mathrm{in}}}

	\newcommand{\Deltayoul}{\Delta_{\mathrm{youl}}}
	\begin{lemma}\label{lem:new_approx_youla}
	Set $\matdel_t := \matxhat_t - \matx_t$ and $\Deltayoul := \Ahat -  \Ast + L (\Chat - \Cst)$. Then, the dynamics induced by \Cref{exm:approx_observe_feedback} satisfy that
	\begin{align*}
	\begin{bmatrix}
	\matx_{t+1} \\
	\matdel_{t+1} \\
	\end{bmatrix} = \underbrace{\begin{bmatrix} \Ast + \Bst F & \Bst F \\
	\Deltayoul & \Ahat + L \Chat 
	\end{bmatrix}}_{:=\Ainhat} \begin{bmatrix}
	\matx_{t} \\
	\matdel_{t} \\
	\end{bmatrix} + \underbrace{\begin{bmatrix} \Bst\\
	\Bhat - \Bst
	\end{bmatrix}}_{\Binhat}\uex_t +  \begin{bmatrix} I & 0 \\
	-I & -L 
	\end{bmatrix}
	\begin{bmatrix} \matw_t\\
	\mate_t \end{bmatrix}
	\end{align*}
	and
	\begin{align*}
	\begin{bmatrix}
	\maty_t\\
	\matetahat_{t} \\
	\matu_{t} 
	\end{bmatrix} = \begin{bmatrix} \Cst & 0\\
	\Chat - \Cst & \Chat \\
	F & F \\
	\end{bmatrix}\begin{bmatrix}
	\matx_{t} \\
	\matdel_{t} \\
	\end{bmatrix} + \begin{bmatrix} 0 \\
	I \\
	0
	\end{bmatrix}\uex_t + \begin{bmatrix} I \\
	-I \\
	0
	\end{bmatrix}\mate_t.
	\end{align*}
	Denoting by $\Ghatin$ the Markov operator describing the map from $\uex_t \to (\maty_t,\matu_t)$, we then have the identity that
	\begin{align*}
	\begin{bmatrix}\maty_t\\
	\matu_t \\
	\end{bmatrix} = \begin{bmatrix}\naty_t\\
	\natu_t\\
	\end{bmatrix} + \sum_{i=0}^{t-1}\Ghatin^{[i]} \uex_{t-i}.
	\end{align*}
	\end{lemma}
	\begin{proof} Let's change variables.
	\begin{align*}
	\matx_{t+1} &= \Ast \matx_t + \Bst F \matxhat_t + \Bst \uex_t + \matw_t \\
	&= (\Ast + \Bst F) \matx_t + \Bst F \matdel_t + \Bst \uex_t + \matw_t \\
	\matxhat_{t+1} &= (\Ahat + L \Chat)\matxhat_t - L \maty_t + \Bhat F\matxhat_t + \Bhat \uex_t \\
	&= (\Ahat+ \Bhat F)\matxhat_t + L (\Chat \matxhat - \Cst \matx_t)  + \Bhat \uex_t  - L \mate_t\\
	\matdel_{t+1} = \matxhat_{t+1} - \matx_{t+1} &= (\Ahat + L \Chat) \matxhat_t-  (\Ast + L \Cst)\matx_t  - L \mate_t - \matw_t + (\Bhat - \Bst) \uex_t\\
	&= (\Ahat + L \Chat) \matdel_t + (\Ahat -  \Ast + L (\Chat - \Cst))\matx_t  - L \mate_t - \matw_t + (\Bhat - \Bst) \uex_t.
	\end{align*}
	Once again, changing variables, we have
	\begin{align*}
	\matetahat_t &= \Chat\matdel_t + (\Chat - \Cst)\matx_t - \mate_t \\
	\matu_t &= \uex_t + F \matdel_t + F\matx_t.
	\end{align*}
	\end{proof}

\subsection{Expressivity of \drcex \label{ssec:drcex_expressivity}}
In this section generalize the expressivity guarantee of \Cref{thm:policy_approximation} to our more general setting. To begin, let us define a notion of an operator which translates the dynamics under the nominal controller $\pinot$ to target dynamics $\pi$:
\newcommand{\uexpipinot}{\matu^{\mathrm{ex},\pi \to \pinot}}
\begin{definition} Given a dynamical system $\pi$, we say that $\Gpipinot$ is a \emph{$\pinot \to \pi$ conversion operator} if the following under dynamics induced by any noise sequence $(\matw_t,\mate_t)$: If $(\matypi_t,\matupi_t)$ are the input-output sequence under $\pi$ (\Cref{l:augd}), then the sequence defined by
\begin{align*}
\uexpipinot_t := \sum_{i=1}^t \Gpipinot^{[t-i]} \nateta_i
\end{align*}
satisfies the following for all $t$:
\begin{align*}
\begin{bmatrix}
\matypi_t \\ \matupi_t
\end{bmatrix} = \begin{bmatrix}
\naty_t \\ \natu_t
\end{bmatrix}  + \sum_{i=1}^t \Gexyu^{[t-i]} \uexpipinot_{i}.
\end{align*}
\end{definition}
In other words, if one selects exogenous inputs $\uexpipinot_t$, then one recovers the dynamics of the controller $\pi$. Note that, it is enough to show that one recovers the dynamics of $ \matupi_t$, since the inputs and noise to the system uniquely determine the dynamics of $\maty_t$ via \Cref{eq:LDS_system}. With the above definition, we define our comparator class accordingly:
\begin{defnmod}{def:transferclass}{b}[Decay Functions \& Policy Class]\label{def:transfer_stabilize} Given an \ldcex{} $\pinot$, we define the comparator class $\Pi(\psi)$ as the set of all $\pi$ for which there exists a $\pinot \to \pi$ conversion operator $\Gpipinot$ which decay dominated by $\psi$: that is, $\psi_{\Gpipinot}(n) \le \psi(n)$ for all $n$. Moreover, we define $\radGst := 1 \vee \|\Gexeta \|_{\loneop} \vee \|\Gexyu\|_{\loneop}$, and $\psiGst(n) := \max\{\psi_{\Gexyu}(n),\psi_{\Gexeta}(n)\}$.
\end{defnmod}
With this definition in mind, \Cref{thm:policy_approximation2} follows by direct analogy to \Cref{thm:policy_approximation}:
\begin{thmmod}{thm:policy_approximation}{b}\label{thm:policy_approximation2} Let $\Pi(\psi)$ be as in \Cref{def:transfer_stabilize}, $\radpsi = \psi(0)$, and let $J_T(M)$ of the strongly stabilized DRC controller (\Cref{defn:drc_stabilize}). Given a proper decay function $\psi$ and   $\pi \in \Pi(\psi)$, there exists an $M\in \M(m,\radpsi)$ such that
\begin{align}
J_T(M) - J_T(\pi) \le 2LT\radpsi\radGst^2\radnat^2 \,\psi(m).\label{eq:J_subopt_app}
\end{align}
\end{thmmod}
\begin{proof}The proof is analogous to the first part of the proof of \Cref{thm:policy_approximation}, where the control approximation identity (\Cref{claim:Youla}) is built into the definition of $\psi(\cdot)$ by assumption. We omit the proof in the interest of simplicity.
\end{proof}
While quite general, \Cref{thm:policy_approximation2}  guarantees competition with policies whose conversion operators  ($\Pi(\psi)$ in \Cref{def:transfer_stabilize}) have reasonable decay, and unlike \Cref{thm:policy_approximation}, it does not make this explicit. Thus it remains to show that this class $\Pi(\psi)$ is reasonable expressive.

In what follows, we will show that an analogoue holds  in all of our examples. Let's make this formal:
\begin{definition}[Convolution of Markov Operator] Let $G_1 \in \scrG^{d_1 \times d_0}$, and $G_2 \in \scrG^{d_0 \times d_2}$. We define $G = G_1 \odot G_2$ as the operator
\begin{align*}
G^{[i]} = \sum_{j=0}^i G_1^{[j]} \cdot G_2^{[i-j]}.
\end{align*}
\end{definition}

\begin{theorem}\label{thm:full_expressivity} For any policy $\pi$, the matrix $\Gpipinot$ can be represented as follows:
\begin{enumerate}
	\item If the system is internally stable (\Cref{exm:stable_system}), $\Gpipinot=\Gpicleu$, for $\pinot$ which is identically zero.
	\item If the system is stabilized by static feedback $\pi_0$ (\Cref{exm:static_feedback}), $\Gpipinot= \Gpipinotbar \circ \Gpinotyyu$ is as detailed in Proposition~\ref{prop:internal_express} since a static controller is internally stable too, with $\Apinot=0$. Furthermore, since $\rho(\Apinotcl)=\rho(\Ast+\Bst K \Cst)<1$, both  $\Gpipinotbar$ and $\Gpinotyyu$ exhibit geometric decay.
	\item If the system is stabilized by internally stable feedback (\Cref{exm:stable_feedback}), $\Gpipinot= \Gpipinotbar \circ \Gpinotyyu$ is as detailed in Proposition~\ref{prop:internal_express}. In particular, both $\Gpipinotbar$ and $\Gpinotyyu$ exhibit geometric decay as long as $\pi$ is stabilizing, since $\rho(\Apinot)<1$ and $\rho(\Apinotcl)<1$ 
	\item If the system is stabilized by exact observer feedback (\Cref{exm:observe_feedback}), the $\Gpipinot$ is as detailed in Proposition~\ref{prop:observer_express}. The latter exhibits geometric decay as long as $\pi$ is stabilizing.
	\item If the system is stabilized by inexact observer feedback (\Cref{exm:approx_observe_feedback}), then $\Gpipinot$ is as Proposition~\ref{prop:approx_observer_express} details. In particular, it is a convolution of three Markov operators of stable systems, as long as $\max\{\rho(\Ast+\Bst F),\rho(\Ahat+\Bhat F), \rho(\Ast+L\Cst), \rho(\Ahat+L\Chat)\}<1$.
\end{enumerate}
In each of the above cases, $\Gpipinot$ is either the Markov operator of a stable system, or can be expressed by a convolution of two (\Cref{exm:stable_feedback}) or three (\Cref{exm:stable_feedback}) Markov operators of stable systems.
\end{theorem}

Specifically, we show that there we can represent $\Gpipinot$ as an \emph{convolution} of stable transfer operators. Since a convolution of operators with geometric decay itself has geometric decay, we find that we obtain the same expressive power as in the stable system case. 

\subsubsection{Expressivity of Internally Stable Feedback}
Let us begin by defining a closed form expression for the $\pinot \to \pi$ operator that arises under internally stable feedback:
\begin{definition}[Internally Stable Dynamical System Conversion]\label{defn:conversion_dynamical_stable} Given a nominal controller $\pinot$ given by $(\Apinot,\Bpinot,\Cpinot,\Dpinot)$, and a target controller $\pi$  given by $(\Api,\Bpi,\Cpi,\Dpi)$, and recalling the closed loop matrix $\Apicl$ from \Cref{l:augd}, define the matrices $\Apipinot,\Bpipinot,\Cpipinot$ by
\begin{align*}
\Apipinot &:= \left[\begin{array}{@{}cc|c} \Apicl & & 0\\ 
& & 0 \\\hline
\Bpinot \Cst &  0& \Apinot \\
\end{array}\right], \quad 
\Bpipinot := \begin{bmatrix} \Bst \Dpi & -\Bst  \\
\Bpi & 0\\
0 & 0  \end{bmatrix},\\
\Cpipinot &:= \begin{bmatrix} (\Dpi - \Dpinot) \Cst & \Cpi & -\Cpinot  \end{bmatrix} 
\end{align*}
and $\Dpipinot = \begin{bmatrix} \Dpi & 0 \end{bmatrix}$
Define $\Gpipinotbar := \Transfer(\Apipinot,\Bpipinot,\Cpipinot,\Dpipinot)$, and define:
\begin{align*}
\Gpinotyyu^{[i]} = \I_{i = 0} \begin{bmatrix} I \\
\Dpinot \end{bmatrix} + \I_{i \ge 1}\begin{bmatrix} 0 \\
\Cpinot \Apinot^{i-1} \Bpinot\end{bmatrix}.
\end{align*}
Finally, we define the $\pinot \to \pi$ coversion operator 
\begin{align*}
\Gpipinot^{[i]} = \sum_{j=0}^{i}\Gpipinotbar^{[i-j]}\Gpinotyyu^{[j]}.
\end{align*}
\end{definition}
\begin{proposition}\label{prop:internal_express}  For any stabilizing $\pi$ and internally stable $\pinot$, the Markov operator $\Gpipinot$ defined in \Cref{defn:conversion_dynamical_stable} is the convolution of two stable Markov operators, and is a $\pinot\to \pi$ conversion operator. That is, for all $t$, the exogenous inputs
\begin{align*}
\uexpipinot_t = \sum_{i=0}^{t-1} \Gpipinotbar^{[i]} \naty_{t-i}.
\end{align*}
produce the input-output pairs $(\matypi_t,\matupi_t)$ via 
\begin{align*}
	\begin{bmatrix} \matypi_t \\
	\matupi_t \end{bmatrix} = \begin{bmatrix} \naty_t \\
	\natu_t \end{bmatrix} + \sum_{i=0}^{t-1}\Gexyu^{[i]} \uexpipinot_{t-i}.
	\end{align*}
\end{proposition}

\subsubsection{Expressivity of Observer-Feedback (Youla Parametrization)}
\begin{proposition}\label{prop:observer_express}
Define the matrices
\begin{align*}
\Ayoulpi := \Apicl, \quad \Byoulpi := \begin{bmatrix} \Bst \Dpi - L \\ \Bst \end{bmatrix}, \quad \Cyoulpi = \begin{bmatrix} \Dpi \Cst - F \\ \end{bmatrix}, \quad \Dyoulpi = \Dpi
\end{align*}
 Then, 
 \begin{align*}\Gyoulpi = \Transfer(\Ayoulpi, \Byoulpi,\Cyoulpi,\Dyoulpi)
 \end{align*} is a $\pinot\to \pi$ conversion operator for the Youla LDC-Ex of \Cref{exm:observe_feedback}. That is
	\begin{align*}
	\begin{bmatrix} \matypi_t \\
	\matupi_t \end{bmatrix} = \begin{bmatrix} \naty_t \\
	\natu_t \end{bmatrix} + \sum_{i=0}^{t-1}\Gexyu^{[i]} \uexpipinot_{t-i}\quad \text{for}\quad\uexpipinot_s = \sum_{j=0}^{s-1} \Gyoulpi^{[i]}\nateta_{s-i}.
	\end{align*}
\end{proposition}
The statement of the Youla parametrization is standard, though varies source-to-source. We use the expression in cite \citet[Theorem 10.1]{megretski2004lec}.

\subsubsection{Expressivity of Approximative Observer-Feedback (DRC-Youla Parametrization)}
\begin{proposition}\label{prop:approx_observer_express}
Let $\Gyoulpi$ be as in \Cref{prop:observer_express}, and define $\Gyoulpibar \in \scrG^{\dimu + \dimy \times \dimy}$ via 
\begin{align*}\Gyoulpibar^{[i]} = \begin{bmatrix}\Gyoulpi^{[i]} \\ I_{\dimy} \cdot \I_{i = 0}\end{bmatrix}.
\end{align*}
 Further, define the operators
\begin{align*}
	\Gsttohat &:= \Transfer\left(\begin{bmatrix} \Ast + \Bst F & 0\\
	\Bhat F  - L\Cst & \Ahat + L\Chat
	\end{bmatrix}, \begin{bmatrix} \Bst & \Bhat\\
	L & L 
	\end{bmatrix},\begin{bmatrix} F & -F \end{bmatrix}, \begin{bmatrix} I & 0
	\end{bmatrix}\right)\\
	\Ghattost &:= \Transfer\left(\begin{bmatrix} \Ast + L \Cst &  \Bst F - L \Chat \\
	0 & \Ahat + \Bhat F \end{bmatrix}, \begin{bmatrix} L \\ 
	L 
	\end{bmatrix},\begin{bmatrix} \Cst & - \Chat 
	\end{bmatrix}, I\right).
	\end{align*}
	Then, the transfer operator $\Gpipinot := \Gsttohat \odot \Gyoulpibar \odot \Ghattost$ is a $\pinot \to \pi$ conversion operator for the Approximate Youla LCD-Ex of \Cref{exm:approx_observe_feedback}.
\end{proposition}

\subsection{Proofs of Expressivity Guarantes}
\subsubsection{Proof of \Cref{prop:internal_express}}

Define
\begin{align*}
\wpinott := \begin{bmatrix} I & \Bst \Dpinot  \\ 0 & \Bpinot  \end{bmatrix}\begin{bmatrix} \matw_t \\ \mate_t \end{bmatrix}, \quad \epinott := \begin{bmatrix}  I \\ \Dpi  \end{bmatrix} \mate_t
\end{align*}
From the closed loop matrices $\Apinotcl,\Cpinotcl$ described in \Cref{l:augd}, the nomimal system with exogenous inputs is then described by the equations

\begin{align}
\matxbar_{t+1} := \begin{bmatrix}\matx_t\\
\sopen_t \end{bmatrix} &= \Apinotcl \matxbar_t + \underbrace{\begin{bmatrix}\Bst \\ 0 \end{bmatrix}}_{\Bpinotcl} \uex_t + \wpinott \nonumber\\
\matybar_t := \begin{bmatrix} \maty_t\\
\uen_t
\end{bmatrix} &= \Cpinotcl \matxbar_t + \epinott. \label{eq:nomimal_augmented}
\end{align}
We then put \Cref{eq:nomimal_augmented} in feedback with the following system via  $\uex_t = \udel_t$:

\begin{align*}
\adelpi_{t+1} & = \Api \adelpi_t + \underbrace{\begin{bmatrix} \Bpi & 0\end{bmatrix} }_{:= \Bpibar}  \matybar_t \\
\udel_t &= \Cpi \adelpi_t + \underbrace{\begin{bmatrix} \Dpi & -I \end{bmatrix}}_{:=\Dpibar} \matybar_t. \numberthis \label{eq:new_controller}
\end{align*}
Then, the joint dynamics of \Cref{eq:nomimal_augmented,eq:new_controller} are given by
\begin{align*}
\begin{bmatrix}
\matxbar_{t+1}\\
\adelpi_{t+1}
\end{bmatrix} &= \underbrace{\begin{bmatrix} \Apinotcl + \Bpinotcl \Dpibar \Cpinotcl & \Bpinotcl \Cpi \\
\Bpibar\Cpinotcl & \Api.
\end{bmatrix}}_{:=\Adelpicl} \begin{bmatrix}
\matxbar_{t+1}\\
\adelpi_{t+1}
\end{bmatrix} + \underbrace{\begin{bmatrix} I & \Bpinotcl \Dpibar \\
0 & \Bpibar 
\end{bmatrix}}_{\Bdelpicl}  \begin{bmatrix} \wpinott \\
\epinott
\end{bmatrix}\\
\begin{bmatrix}
\matybar_t \\
\udel_t 
\end{bmatrix}&= \underbrace{\begin{bmatrix} 
\Cpinotcl & 0 \\
\Dpibar \Cpinotcl & \Cpi
\end{bmatrix}}_{\Cdelpicl}\begin{bmatrix}
\matxbar_{t}\\
\adelpi_{t}
\end{bmatrix} + \begin{bmatrix}  I \\ \Dpibar\end{bmatrix} \epinott. \numberthis \label{eq:closed_loop_new_controller}
\end{align*}

First, we claim that, for all $t$, the system \Cref{eq:closed_loop_new_controller} yields inputs an outputs equivalent to $\matupi_t,\matupi_t$:
\begin{lemma} Let $\matybar_t = (\maty_t,\uen_t)$ and $\udel_t$ be as given by \Cref{eq:closed_loop_new_controller}. Then,  for $\udel_t$ defined above,
\begin{align*}
\forall t, \maty_t = \matypi_t \text{ and } \uen_t + \udel_t = \matupi_t. 
\end{align*}
In particular, $\begin{bmatrix} \matypi_t  \\
\matupi_t
\end{bmatrix} = \sum_{i=1}^t \Gpinotcluyu^{[i]} \udel_t$.
\end{lemma}
\begin{proof} Let us consider the update of the state $\matx_t$: $\matx_{t+1} = \Ast \matx_{t+1} + \Bst(\uen_t + \udel_t) + \matw_t  = \Ast \matx_{t+1} + \Bst \Cpi \adelpi_t + \Bst \Dpi \yen_t + \matw_t $.  First, note that
\begin{align*}
\matu_t = \uen_t + \udel_t = \Dpi \maty_t + \Cpi \adel_t = \Dpi \Cst \matx_t + D\pi \mate_t.
\end{align*}
Thus,
\begin{align*}
\matx_{t+1} &= \Ast \matx_{t+1} + \uen_t + \udel_t + \matw_t \\
&= (\Ast + \Bst \Dpi \Cst) \matx_t  + \Bst \Cpi \adelpi_t  + \Bst \Dpi  \mate_t + \matw_t
\end{align*}
Moreover, we have that
\begin{align*}
\adelpi_{t+1} = \Api \adelpi_t + \Bpi \maty_t  = \Api \adelpi_t + \Bpi \Cst \matx_t +  \Bpi \Cst \mate_t  \\
= \Api \adelpi_t + (\Bpi \Cst +)\matx_t + \Bpi\mate_t. 
\end{align*}
Thus, $(\matx_t,\adelpi_t)$ have the same evolution as $(\matxpi_t, \mata^{\pi}_t)$, where $\mata^{\pi}_t$ is the internal state of the system when placed in feedback with $\pi$. Thus, 
\begin{align*}
\maty_t &= \Cst \matx_t + \mate_t = \Cst \matxpi_t + \mate_t = \matypi_t\\
\uen_t + \udel_t &= \Dpi \Cst \matx_t  +  \Cpi \adelpi_t  +  \Dpi \Cst \mate_t = \Dpi \Cst \matxpi_t + \Cpi \mata^{\pi}_t + \Dpi  \mate_t  = \matupi_t.
\end{align*}
\end{proof}
Next, we show that $\udel_t$ can be represented as a linear function of the sequence $\ybarpinot_t = (\ypinot,\upinot)$:
\begin{claim}\label{claim:naty_u_Youla} Define
\begin{align*}
\Cdelpiclu := \begin{bmatrix} \Dpibar \Cpinotcl & \Cpi\end{bmatrix}, \quad \Bdelpicle := \begin{bmatrix} \Bpinotclin \Dpibar  \\
\Bpibar \end{bmatrix}
\end{align*}
Then, the matrices $N^{[0]} = \Dpibar$, $N^{[i]} = \Cdelpiclu\Adelpicl^{i-1}\Bdelpicle$ satisfy
\begin{align*}
\udel_t  = \sum_{i=0}^{t-1} N^{[i]} \ybarpinot_{t-i}, \quad \text{ where } \ybarpinot_s =  \begin{bmatrix} \ypinot_t\\
\upinot_t \end{bmatrix}.
\end{align*}
\end{claim}
\begin{proof} Analogous to \Cref{claim:Youla}, and the fact that, in the absence of $\udel_t$, $\matybar_t = \begin{bmatrix} \ypinot_t\\
\upinot_t \end{bmatrix}$.
\end{proof}
Let us now show that $N^{[i]}$ is given by $\Gpipinotbar$:
\begin{claim} For all $i \ge 0$, $N^{[i]} = \Gpipinotbar^{[i]}$. As a consequence, 
\begin{align*}
\udel_t  = \sum_{i=0}^{t-1} \Gpipinotbar^{[i]} \ybarpinot_{t-i}, \quad \text{ where } \ybarpinot_s =  ( \ypinot_t,
\upinot_t).
\end{align*}
\end{claim}
\begin{proof}
By definition $\Dpibar = \Dpipinot$. To establish the identity,  define the block permutation matrix $T$, where the blocks correspond to the $\matx_t,\sopen_t,\adel_t$ states:
\begin{align*}
 T = \begin{bmatrix} I & 0 & 0\\
0 & 0 & I \\
0 & I & 0 \end{bmatrix}.
\end{align*}
Since $T^2 = I$, it suffices to show that 
\begin{align*}
T \Adelpicl T = \Apipinot, \quad T\Bdelpicl = \Bpipinot,  \quad T \Cdelpicl = \Cpipinot.
\end{align*}
Recall that
\begin{align*}
\Adelpicl  := \begin{bmatrix} \Apinotcl + \Bpinotclin \Dpibar \Cpinotcl & \Bpinotclin \Cpi \\
\Bpibar\Cpinotcl & \Api.
\end{bmatrix}
\end{align*}
We begin with
\begin{align*}
\Apinotcl + \Bpinotclin \Dpibar \Cpinotcl &= \begin{bmatrix} \Ast + \Bst \Dpinot \Cst  & \Bst \Cpinot\\
\Bpinot \Cst & \Apinot
\end{bmatrix} + \begin{bmatrix} \Bst \\ 0 \end{bmatrix}\begin{bmatrix} \Dpi & -I \end{bmatrix} \begin{bmatrix}
   \Cst & 0 \\
    \Dpinot \Cst & \Cpinot
\end{bmatrix}\\
 &= \begin{bmatrix} \Ast + \Bst \Dpinot \Cst  & \Bst \Cpinot\\
\Bpinot \Cst & \Apinot
\end{bmatrix} + \begin{bmatrix} \Bst \\ 0 \end{bmatrix}\begin{bmatrix} (\Dpi - \Dpinot) \Cst & -\Cpinot \end{bmatrix} \\
&= \begin{bmatrix} \Ast + \Bst \Dpi \Cst  & 0 \\
\Bpinot \Cst  & \Apinot 
\end{bmatrix} 
\end{align*}
Moreover, recalling $\Bpibar = [\Bpi \mid 0]$, we have
\begin{align*}
\Bpibar\Cpinotcl &= \begin{bmatrix} \Bpi \Dpinot \Cst & 0 
\end{bmatrix} + \begin{bmatrix} (\Dpi - \Dpinot) \Cst & -\Cpinot \end{bmatrix}  = \begin{bmatrix} \Bpi \Cst & 0 \end{bmatrix}
\end{align*}
Finally, since $\Bpinotclin \Cpi = \begin{bmatrix} \Bst \Cpi \\ 0 \end{bmatrix}$, we have
\begin{align*}
\Adelpicl = \begin{bmatrix} \Apinotcl + \Bpinotclin \Ddelpi \Cpinotcl & \Bpinotclin \Cpi \\
\Bpi \Cpinotcl & \Api.
\end{bmatrix} &= \begin{bmatrix} \Ast + \Bst \Dpi \Cst  & 0 & \Bst \Cpi \\
\Bpinot \Cst  & \Apinot  & 0 \\
\Bpi \Cst & 0 & \Api \end{bmatrix}
\end{align*}
Thus, 
\begin{align*}
T\Adelpicl T  = \begin{bmatrix} \Ast + \Bst \Dpi \Cst  & \Bst \Cpi & 0   \\
\Bpi\Cst & \Api & 0 \\
\Bpinot \Cst   & 0   & \Apinot
\end{bmatrix} =  \left[\begin{array}{@{}cc|c} \Apicl & & 0\\ 
& & 0 \\\hline
\Bpinot \Cst &  0& \Apinot \\
\end{array}\right] := \Apipinot
\end{align*}

 Now, recall $\Cpinotcl = \begin{bmatrix}
   \Cst & 0 \\
    \Dpinot \Cst & \Cpi
\end{bmatrix}$ and 
$\Bpinotcl = \begin{bmatrix} I & \Bst \Dpinot \\ 0 & \Bpinot  \end{bmatrix}$. Then,
\begin{align*}
\Cdelpiclu T &:= \begin{bmatrix} \Dpibar \Cpinotcl & \Cpi\end{bmatrix}T\\
&:= \begin{bmatrix} (\Dpi - \Dpinot) \Cst & -\Cpinot & \Cpi\end{bmatrix}T\\
&:= \begin{bmatrix} (\Dpi - \Dpinot) \Cst & \Cpi & -\Cpinot  \end{bmatrix} = \Cpipinot
\end{align*}
and
\begin{align*}
T\Bdelpicle  &:= T \Bdelpicle := \begin{bmatrix} \Bpinotclin \Dpibar  \\
\Bpibar  \end{bmatrix}T\\
&:= \begin{bmatrix} \Bst \Dpi & -\Bst  \\
0 & 0 \\
\Bpi & 0 \end{bmatrix}T := \begin{bmatrix} \Bst \Dpi & -\Bst  \\
\Bpi & 0\\
0 & 0  \end{bmatrix} = \Bpipinot
\end{align*}
\end{proof}
We conclude the proof by showing that $\ybarpinot_t = (\ypinot_t,\upinot_t)$ can be represented in terms of $\ypinot_t$: 
\begin{claim} Recall $\Gpinot^{[i]} = \I_{i = 0}\Dpinot + \I_{i \ge 1} \Cpinot\Apinot^{i-1}(\Bpinot + \Bpinotu \Dpinot)$. Then, $\upinot_t  = \sum_{i=1}^{t} \Gpinot^{[i]}\ypinot_{t-1}$ As a consequence, 
\begin{align*}
\ybarpinot_t  = \sum_{i=0}^{t-1} \Gpinotyyu^{[i]}\ypinot_{t-i}
\end{align*}
\end{claim}
\begin{proof} Directly from the LDC equations.
\end{proof}
In sum,
\begin{align*}
\udel_t  = \sum_{i=0}^{t-1} \Gpipinot^{[i]} \ybarpinot_{t-i} = \sum_{i=0}^{t-1} \sum_{j=0}^{t-i-1} \Gpipinotbar^{[i]}\Gpinotyyu^{[j]}\ypinot_{t-i-j},
\end{align*}
which concludes the proof. \qed.

\subsubsection{Proof of \Cref{prop:approx_observer_express}}

	\newcommand{\natxtil}{\tilde{\matx}^{\mathrm{nat}}}
	\newcommand{\natxch}{\check{\matx}^{\mathrm{nat}}}

	\newcommand{\etast}{\bm{\eta}^{\star}}
	\newcommand{\udelpi}{\matu^{\Delta \pi}}

	\newcommand{\matvdelpi}{\matv^{\Delta \pi}}
	\newcommand{\matxst}{\matx^{\star}}
	\newcommand{\matxtilst}{\tilde{\matx}^{\star}}
	\newcommand{\matxchst}{\check{\matx}^{\star}}
	\newcommand{\xen}{\mathring{\matx}}
	\newcommand{\zen}{\mathring{\matv}}
	\newcommand{\xentil}{\mathring{\tilde{\matx}}}

	\newcommand{\matvst}{\matv^{\star}}
	\newcommand{\Cetainhat}{C_{\hat{\eta},\mathrm{in}}}
	\newcommand{\zinhat}{\matv^{\delta}}
	\newcommand{\matvtil}{\widetilde{\matv}}

	\newcommand{\natetil}{\tilde{\boldsymbol{\eta}}^{\mathrm{nat}}}

	Consider the system
	\begin{align}
	\matxtilst_{t+1} &= (\Ast + L \Cst)\matxtilst_t - L \maty_t + \Bst\matu_t \nonumber\\
	\etast_{t} &= \Cst \matxtilst_{t} - \maty_t \nonumber\\
	\matvdelpi_{t+1} &= \Adelpi \matvdelpi_{t} + \Bdelpi \etast_{t} \nonumber\\
	\udelpi_{t} &= \Cdelpi \matvdelpi_t + \Ddelpi \etast_t \nonumber\\
	\matu_t &= F \matxtilst_t + \udelpi_{t}. \label{eq:st_pi_system}
	\end{align}
	From \Cref{prop:internal_express}, the inputs $\matu_t$ coincide with $\matu_t^{\pi}$ for all $t \ge 1$. Thus, if we set $\uex_t = F (\matxtilst_t - \matxtil_t) + \udelpi_t$, the system
	\begin{align}
	\matxhat_{t+1} &= (\Ahat + L\Chat)\matxhat_{t} - L\maty_t + \Bhat\matu_t \nonumber\\
	\matu_t &= F\matxhat_t + \uex_t \nonumber\\
	\mateta_t &= \Chat \matxhat_{t+1}- \maty_t \label{eq:delpi_syst}
	\end{align}
	also generates $\matu_t = \matu_t^{\pi}$. Now, let us represent the above as a system with inputs $\etast_t,\udelpi_t$. We shall show that these can all be represented in terms of $\natetahat_t$, concluding the proof. 

	First, we write
	\begin{align*}
	\matxtilst_{t+1} &= (\Ast + L \Cst)\matxtilst_t - L \maty_t + \Bst\matu_t\\
	&= (\Ast + L \Cst)\matxtilst_t + L \etast_t + L (\maty_t - \etast_t) + \Bst F \matxtilst_t + \Bst \udelpi_t \\
	&= (\Ast + L \Cst)\matxtilst_t + L \etast_t - L \Cst \matxtilst_t + \Bst F \matxtilst_t + \Bst \udelpi_t \\
	&= (\Ast + \Bst F)\matxtilst_t + L \etast_t  + \Bst \udelpi_t,
	\end{align*}
	where we use the fact that $\etast_t = \Cst \matxtilst_t - \maty_t$. Next, we write
	\begin{align*}
	\matxhat_{t+1} &= (\Ahat + L\Chat)\matxhat_{t} - L\maty_t + \Bhat\matu_t \\
	&= (\Ahat + L\Chat)\matxhat_{t} - L\maty_t + \Bhat F\matxtilst_t  + \Bhat \udelpi_t \\
	&= (\Ahat + L\Chat)\matxhat_{t} + L \etast_t + (\Bhat F  - L\Cst )\matxtilst_t  + \Bhat \udelpi_t,
	\end{align*}
	where in the last line we use $\etast_t = \Cst \matxtilst_t - \maty_t$. This gives that
	\begin{align*}
	\begin{bmatrix} \matxtilst_{t+1}\\
	\matxtil_{t+1}	
	\end{bmatrix}  &= \begin{bmatrix} \Ast + \Bst F & 0\\
	\Bhat F  - L\Cst & \Ahat + L\Chat
	\end{bmatrix} \begin{bmatrix} \matxtilst_{t}\\
	\matxtil_{t}	
	\end{bmatrix} + \begin{bmatrix} \Bst & \Bhat\\
	L & L 
	\end{bmatrix} \begin{bmatrix} \udelpi_t \\ \etast_t
	\end{bmatrix}\\
	\uex_t &= \begin{bmatrix} F & -F \end{bmatrix}\begin{bmatrix} \matxtilst_{t}\\
	\matxtil_{t}	
	\end{bmatrix} + \begin{bmatrix} I & 0
	\end{bmatrix} \begin{bmatrix} \udelpi_t \\ \etast_t
	\end{bmatrix}.
	\end{align*}

	Thus, letting 
	\begin{align*}
	\Gsttohat := \Transfer\left(\begin{bmatrix} \Ast + \Bst F & 0\\
	\Bhat F  - L\Cst & \Ahat + L\Chat
	\end{bmatrix}, \begin{bmatrix} \Bst & \Bhat\\
	L & L 
	\end{bmatrix},\begin{bmatrix} F & -F \end{bmatrix}, \begin{bmatrix} I & 0
	\end{bmatrix}\right)
	\end{align*}
	denote the transfer operator mapping $\begin{bmatrix} \udelpi_t \\ \etast_t
	\end{bmatrix} \to \uex_t$, we can render 
	\begin{align*}
	\uex_t = \sum_{s=1}^{t} \Gsttohat^{[t-s]}\begin{bmatrix} \udelpi_t \\ \etast_t
	\end{bmatrix}.
	\end{align*}
	Next, for $\Gyoulpi := \Transfer(\Ayoulpi,\Byoulpi,\Cyoulpi,\Dyoulpi)$ from \Cref{prop:observer_express}, we have
	
	\begin{align*}
	\udelpi_s = \sum_{j=1}^{s} \Gyoulpi^{[s-j]}\etast_j, 
	\end{align*}
	giving that, for $\uex_t$ defined in \Cref{eq:delpi_syst},
	\begin{align}
	\forall t, \quad \uex_t = \sum_{s=1}^{t}\sum_{j=1}^{s}\Gsttohat^{[t-s]}\Gyoulpibar^{[s-j]}\etast_t., \quad \Gyoulpibar^{[i]} = \begin{bmatrix}\Gyoulpi^{[i]}\\
	I_{\dimy} \I_{i = 0}
	\end{bmatrix} \label{eq:uex_intermed}
	\end{align}
	To conclude, let us represent $\etast_t$ in terms of $\natetahat_t$. Here, we use the crucial fact that the dynamics of $\etast_t$ are \emph{non-counterfactual}. Thus, let us instead consider the following ``natural'' dynamics: 
	\newcommand{\natxhat}{\hat{\matx}^{\mathrm{nat}}}
	\newcommand{\natetast}{\bm{\eta}^{\star,\mathrm{nat}}}

	\begin{align*}
	\natx_{t+1} &= \Ast \natx_t + \Bst \natu_t + \matw_t\\
	\naty_t &= \Cst \matx_t + \mate_t \\
	\natxhat_{t+1} &= (\Ahat + L \Chat)\natxhat_t - L \naty_t + \Bhat \natu_t \\
	\natu_t &= F \natxhat_t \\
	\natxtil_{t+1} &= (\Ast + L \Cst) \natxtil_t - L \naty_t +  \Bst \natu_t \\
	\nateta_t &:= \Chat \natxhat_t - \naty_t \\
	\natetast_t &:= \Cst \natxtil_t  - \naty_t 
	\end{align*}
	From \Cref{lem:youla_facts}, the $\eta$-dynamics under exact observer feedback do not depend on the exogenous inputs; thus,  $\natetast_t = \etast_t$ for all $t$, where $\natetast$ is defined in \Cref{eq:st_pi_system}. Next, we can substitute
	\begin{align*}
	\natxtil_{t+1} &= (\Ast + L \Cst) \natxtil_t - L \naty_t +  \Bst \natu_t\\
	&= (\Ast + L \Cst) \natxtil_t  + L \nateta_t - L \Chat \natxhat_t +  \Bst F \natxhat_t\\
	&= (\Ast + L \Cst) \natxtil_t  + L \nateta_t + (\Bst F - L \Chat) \natxhat_t  
	\end{align*}
	Furthermore, we can write
	\begin{align*}
	\natxhat_{t+1} &= (\Ahat + L \Chat)\natxhat_t + L \nateta_t - L \Chat \natxhat_t + \Bhat F \natxhat_t\\
	&= (\Ahat + \Bhat F)\natxhat_t + L \nateta_t.
	\end{align*}
	Thus,
	\begin{align*}
	\begin{bmatrix} \natxtil_{t+1}\\
	\natxhat_{t+1}\\
	\end{bmatrix} = \begin{bmatrix} \Ast + L \Cst &  \Bst F - L \Chat \\
	0 & \Ahat + \Bhat F \end{bmatrix} \begin{bmatrix} \natxtil_{t}\\
	\natxhat_{t}\\
	\end{bmatrix} + \begin{bmatrix} L \\ 
	L 
	\end{bmatrix}\nateta_t
	\end{align*}
	Moreover,
	\begin{align*}
	\natetast_t = \Cst \natxtil_t  - \naty_t  = \Cst \natxtil_t - \Chat \natxhat_t + \nateta_t ,
	\end{align*}
	or in matrix form
	\begin{align*}
	\natetast_t = \begin{bmatrix} \Cst & - \Chat 
	\end{bmatrix} \begin{bmatrix} \natxtil_{t}\\
	\natxhat_{t}\\
	\end{bmatrix} + I \cdot \nateta_t.
	\end{align*}
	Hence, defining 
	\begin{align*}
	\Ghattost := \Transfer\left(\begin{bmatrix} \Ast + L \Cst &  \Bst F - L \Chat \\
	0 & \Ahat + \Bhat F \end{bmatrix}, \begin{bmatrix} L \\ 
	L 
	\end{bmatrix}, \begin{bmatrix} \Cst & - \Chat 
	\end{bmatrix}, I\right)
	\end{align*} as the $\nateta_t \to \nateta_t$ transfer operator, we see that 
	\begin{align*}
	\etast_j = \natetast_j = \sum_{i=1}^j \Ghattost^{[j-i]}\nateta_i.
	\end{align*}
	Thus, from \Cref{eq:uex_intermed}, the exogenous inputs $\uex_t$ from \Cref{eq:delpi_syst} satisfy
	\begin{align*}
	\uex_t &= \sum_{s=1}^{t}\sum_{j=1}^{s}\sum_{i=1}^j \Gsttohat^{[t-s]} \Gyoulpibar^{[s-j]}  \Ghattost^{[j-i]} \nateta_{i},\\
	&= \sum_{s=1}^{t} (\Gsttohat \odot \Gyoulpibar \odot \Ghattost)^{[t-s]} \nateta_{t}.
	\end{align*}
	Since $\uex_t$ induces the desired inputs $\matupi_t$, the proposition follows.

	\qed

\newcommand{\matUbar}{\overline{\mathbf{U}}}
\newcommand{\calS}{\mathcal{S}}
\newcommand{\calE}{\mathcal{E}}

\section{Regret Analysis: Non-Stochastic\label{app:non-stochastic}}
While the theorems in the main paper hold for stable systems, the stated proofs and claims here hold for the more general setting of stabilizable systems, with the following modifications:
\begin{definition}[Modifications for the Stabilized Case]\label{defn:modifications_stabilized} The following modifications are made for the Stabilized Setting of \Cref{sec:generalization_to_stabilized}:
\begin{enumerate}
  \item We are given access to a stabilizing controller satisfying Assumption~\ref{a:stablize}
  \item We replace Assumption~\ref{asm:a_stab} with Assumption~\ref{a:stablize}.
  \item $\psiGst$ and $\radGst$ are defined as in Definition~\ref{def:transfer_stabilize}. 
  \item We replace Algorithm~\ref{alg:improper_OGD_main} with Algorithm~\ref{alg:improper_OGD_main_stabilized}, and Algorithm~\ref{alg:estimation} with Algorithm~\ref{alg:estimation_stabilized}.
\end{enumerate}
\end{definition}
 (where we are granted access to a sub-optimal stabilizing controller).

\subsection{Omitted Proofs from Section~\ref{sec:main_known_system}\label{ssec:app_ommited_known}}
  In this section, we present all ommited proofs from Section~\ref{sec:main_known_system}, and demonstrate that all bounds either hold verbatim in the more general stabilized system setting, or present generalizations thereof. This ensures that Theorem~\ref{thm:know} holds verbatim in the more general setting as well.  Before continuing, let us review some of the notation from the stabilized setting, and how the stable system setting can be recovered:
  \begin{itemize}
    \item We use $\matv = (\maty,\matu) \in \R^{\dimy + \dimu}$ to denote the pair of outputs and inputs on which the loss is measured. In particular, $\valg_t = (\yalg_t,\ualg_t)$, and $\matv^M = (\maty^M_t,\matu^M_t)$.
    \item The exogenous inputs $\uex_t$ reduce to the  inputs $\matu_t$ in the stable case.
    \item The exogenous inputs $\uex_t$ are linear in $\nateta_t$ or estimates $\natetahat_t$; in the stable case, these correspond to $\naty_t,\natyhat_t$.
  \end{itemize}

  Next, we note that the regret decomposition is the same as in the stable case, given by Eq.~\ref{eq:known_weak_convex_regret_decomposition}. We begin with a magnitude bound that generalizes Lemma~\ref{lem:magnitude_bound_simple}:

    \begin{lemmod}{lem:magnitude_bound_simple}{b}[Magnitude Bound]\label{lem:magnitude_bound_stabilized} Recall the notation  with variants $\valg_t,\matv^M$. For all $t$, and $M,M_{1},M_2,\dots \in \calM$, we have
    \begin{align*}
     &\max\left\{\left\|\uexalg_t\right\|_2,\left\|\uexM_t\right\|_2, \left\|\uex_t(M_t \midetanat)\right\|_2\right\}\le \radM\radnat\\
    &\max\left\{\left\|\matv_t\left[M_{t:t-h}\midetaygst\right]\right\|,\|\valg_t\|,\|\matv^M_t\|\right\} \le 2\radGst\radM\radnat.
    \end{align*}
    \end{lemmod}
    
    \begin{proof} The proofs of all these bounds are similar; let us focus on the $\uexalg_t,\valg_t$ sequence. We have  $\uexalg_t := \sum_{s=t-h \vee 0 }^T\matM_t^{[t-s]}\naty_s$, from which Holder's inequality implies $\|\uexalg_t\|_2 \le \radM\radnat$. Then, $\|\valg_t\| = \|\natv_t + \sum_{s=1}^t \Gexu^{[t-s]}\uexalg_s\|_2 \le \radnat + \|\Gexu\|_{\loneop}\|\ualg_s\|_2  \le \radnat + \radGst\radM\radnat \le 2\radGst\radM\radnat$, since $\radM,\radGst \ge 1$. 
    \end{proof}
    We now restate the burn-in bound, which can be checked to hold in the more general present setting:
    \lemburninboundknown*
    We now turn to the truncation costs (\Cref{lem:truncation_costs_known}):
   \lemtruncationcostsknown*
   \begin{proof}
   Let us bound the algorithm truncation cost; the comparator cost is similar. Note that $\uexalg_t = \uex_t(\matM_t)$. By the magnitude bound (Lemma~\ref{lem:magnitude_bound_stabilized}) and Lipschitz assumption, 
    \begin{align*} 
    &\sum_{t=m+h+1}^T \loss_t(\yalg_t,\ualg_t) - \sum_{t=m+h+1}^T F_t[\matM_{t:t-h}\midst] \le \\
    &\qquad= 2L\radGst\radpsi\radnat\iftoggle{colt}{\\ &\qquad\qquad\cdot}{}\sum_{t=m+h+1}^T\left\|\begin{bmatrix}\maty_t[\matM_{t:t-h} \midetaygst] - \maty_t^{\Alg}  \\
    \matu_t[\matM_{t:t-h} \midetaygst] - \matu_t^{\Alg}\end{bmatrix}\right\|_2\\
    &\qquad= 2L\radGst\radpsi\radnat\sum_{t=m+h+1}^T\left\|\begin{bmatrix}\sum_{s = 1}^{t-h-1}\Gexyu^{[t-s]}\uexalg_{s}\end{bmatrix}\right\|_2\\
    &\qquad\le 2LT\radGst\radpsi\radnat\psiGst(h+1)\max_{s \in [T]}\|\uexalg_s\|_2.
    \end{align*}
    By Lemma~\ref{lem:magnitude_bound_stabilized}, the above is at most $2LT\radGst\radpsi^2\radnat^2\psiGst(h+1)$.
   \end{proof}
  Next, we turn to bounding the Lipschitz constants. For this, we shall need the following bound:
  \begin{lemma}[Norm Relations]\label{lem:loneop_fro} For any $M \in \calK$, we have 
  \begin{align*}
  \|M\|_{\loneop} \le \sqrt{m}\|M\|_{\fro}\quad \text{and} \quad \max_{i=0}^h\|M_{t-i}\|_{\loneop} \le \sqrt{m}\|M_{t:t-h}\|_{\fro}\end{align*}
  \end{lemma}
  \begin{proof} 
  The first inequality follows form Cauchy Schwartz:
  \begin{align*}
  \|M\|_{\loneop} = \sum_{i=1}^{m}\|M^{[i-1]}\|_{\op} \le \sum_{i=1}^{m}\|M^{[i-1]}\|_{\fro} \le \sqrt{m}\|M\|_{\fro}.
  \end{align*}
  The second  follows from using the first to bound 
  \begin{align*}
  \max_{i=0}^h\|M_{t-i}\|_{\loneop} \le \max_{i}\sqrt{m}\|M_{t-i}\|_{\fro} \le \sqrt{m}\|M_{t:t-h}\|_{\fro}.
  \end{align*}
  \end{proof}
  As a second intermediate step, we show that the maps $M \mapsto \maty_t(M\midst)$ and $M \mapsto \matu_t(M\midstyu)$ are Lipschitz:
  \begin{lemma}[Lipschitz Bound on Coordinate Mappings]\label{lem:lip_bound_coord_map}
  For any $M,\Mtil \in \calM(M,R)$,
  \begin{align*}
  &\left\|\matv_t(M\midetaygst) - \matv_t(\Mtil\midetaygst) \right\|_2 \\
  &\quad\le \radnat \|M - \Mtil\|_{\loneop} \le \sqrt{m}\radGst \radnat \|M - \Mtil\|_{\fro}.
  \end{align*}
  Similarly, for any $M_{t:t-h},\Mtil_{t:t-h} \in \calM(M,R)^{h+1}$,
   \begin{align*}
  &\left\|\matv_t[M_{t:t-h}\midetaygst] - \matv_t[\Mtil_{t:t-h}\midetaygst]  \right\|,\\
  &\quad\le  \sqrt{m}\radGst \radnat \max_{s=t-h}^t\|M_s - \Mtil_s\|_{\fro}.
  \end{align*} 
  \end{lemma}
  \begin{proof} Let us prove the bound for  $\|\matv_t(M\midetaygst) - \matv_t(M'\midetaygst) \|_2$,  for time varying $M_{t:t-h}$ and $\Mtil_{t:t-h}$ are similar. We have
  \begin{align*}
  \|\matv_t[M_{t:t-h}\midetaygst] - \matv_t[\Mtil_{t:t-h} \|_2 &= \left\|\sum_{s = t-h}^{t-1} \Gexyu\left(\uex(M \midetanat{s}) - \uex(\Mtil \midetanat[s])\right)\right\|\\
  &\le \radGst\max_{s=t-h}^{t-1}\left\|\uex(M \midetanat[s]) - \uex(\Mtil \midetanat{s})\right\|\\
  &= \radGst\max_{s=t-h}^{t-1}\left\|\sum_{j=s-m+1}^s (M-\Mtil)^{[s-j}] \nateta_j\right\|_2\\
  &\le \radnat\radGst\left\|M -\Mtil\right\|_{\loneop} \le \sqrt{m} \radnat\radGst\left\|M -\Mtil\right\|_{\fro},
  \end{align*}
  where the last step uses Lemma~\ref{lem:loneop_fro}.
  \end{proof}

  We now present and prove the generalization of Lemma~\ref{lem:label_Lipschitz_constants_known} to the stabilized setting:
   \begin{lemmod}{lem:label_Lipschitz_constants_known}{b} Define $L_f := L\sqrt{m}\radnat^2\radGst^2\radM$. Then, the functions $f_t(\cdot \midstyu)$ are $L_f$-Lipschitz, and $F_t[M_{t:t-h} \midstyu]$ are $L_f$-coordinate-wise Lipschitz on $\calM$ in the Frobenius norm $\|M\|_{\fro} = \|[M^{[0]},\dots,M^{[m-1]}]\|_{\fro}$. Moreover, the Euclidean diameter of $\Mclass$ is at most $D \le 2\sqrt{d}\radM$.
   \end{lemmod}
  \begin{proof}
  Let us bound the coordinate-Lipschitz constant of $F_t$, the bound for $f_t$ is similar.
  \begin{align*}
  &\left|F_t[M_{t:t-h} \midetaygst ] - F_t[\Mtil_{t:t-h}\midetaygst]\right| \\
  &= \left|\ell_t\left(\matv_t[M_{t:t-h} \midetaygst ]\right) - \ell_t\left(\matv_t[\Mtil_{t:t-h}\midetaygst]\right)\right|\\
  &\le L\radnat\radGst\radM\left\|\matv_t[M_{t:t-h} \midetaygst ] - \matv_t[\Mtil_{t:t-h}\midetaygst]\right\|\\
  &\le \sqrt{m}L\radnat^2\radGst^2\radM\left(\max_{s\in[t-h:t]}\|M_s - \Mtil_s\|_{\fro}   \right),
  \end{align*}
  where the last inequality is by Lemma~\ref{lem:lip_bound_coord_map}. The bound on the diameter of $\calM$ follows from Lemma~\ref{lem:loneop_fro}

  \end{proof}

\subsection{Estimation Bounds: Proof of Theorems~\ref{thm:ls_estimation_simple} \& \ref{thm:ls_estimation_stabilized}\label{ssec:thm:ls_estimation}}
  We state a generalization of Proof of Theorems~\ref{thm:ls_estimation_simple} for estimating both respose $\Gexeta$ and  $\Gexyu$:
  \begin{thmmod}{thm:ls_estimation_simple}{b}[Guarantee for Algorithm~\ref{alg:estimation_stabilized}, Generalization of Theorem~\ref{thm:ls_estimation_simple}]\label{thm:ls_estimation_stabilized}
  Let $\delta \in (e^{-T},T^{-1})$, $N ,\dimu \le T$, and $\psiGst(h+1) \le \frac{1}{\sqrt{N}}$. Define $\dmax = \max\{\dimy + \dimu,\dimeta\}$, and set
\begin{align*}
 \epsG(N,\delta) &=  \frac{h^2 \radnat}{\sqrt{N}} C_{\delta}, \quad\text{ where } C_{\delta} := 14\sqrt{\dmax + \dimy + \log \tfrac{1}{\delta}}, \quad \text{and }  \Restu := 3\sqrt{\dimu + \log (1/\delta)}.
\end{align*}
and suppose that $N \ge  h^4 C_{\delta}^2 \Restu^2 \radM^2\radGst^2 + c_0 h^2 \dimu^2$ for an appropriately large $c_0$, which can be satisfied by taking
\begin{align*}
N \ge  1764 (\dmax + \dimu +\log  (1/\delta))^2h^4 \radM^2 \radGst^2 + c_0 h^2 \dimu^2.
\end{align*}
Then  with probability $1 - \delta - N^{-\log^2 N}$, Algorithm~\ref{alg:estimation_stabilized} satisfies the following bounds
\begin{enumerate}
  \item $\epsG \le 1/\max\{\Restu,\radM\radGst\}$.
  \item For all $t \in [N]$, $\|\matu_t\| \le \Restu := 3\sqrt{\dimu + \log (1/\delta)}$
  \item For  estimation error is bounded as
  \begin{align*}
  \|\Ghatexeta - \Gexeta\|_{\loneop} &\le \|\Ghatexeta^{[0:h]} - \Gexeta^{[0:h]}\|_{\loneop} + \Restu\psiGst(h+1) \le \epsG\\
   \|\Ghatexyu - \Gexyu\|_{\loneop} &\le \|\Ghatexyu^{[1:h]} - \Gexyu^{[1:h]}\|_{\loneop} + \Restu\psiGst(h+1) \le \epsG.
  \end{align*}
\end{enumerate}
Moreover,  Algorithm~\ref{alg:estimation} satisfies the same for $\Ghat,\Gst$.
  \end{thmmod}

  \begin{proof}
  Let us focus on Algorithm~\ref{alg:estimation_stabilized}, the bound for Algorithm~\ref{alg:estimation} is the special case of $\Ghatexyu = \Ghat$ and $\Gexyu = \Gst$.
  
  The first bound of the lemma is strictly numerical. Lets prove the second part of the lemma. Using standard gaussian concentration (see e.g. \citet[Section 4.2]{vershynin2018high}):
  \newcommand{\eventubounded}{\mathcal{E}^{\matu,\mathrm{bound}}}
  \begin{claim}\label{claim} With probability $1-\delta/3$ and $\delta \le 1/T \le 1/N$ and $\delta \le 1/3$, $\|\matu_t\| \le \frac{4}{3}\sqrt{2}\sqrt{\dimu \log 9 + \log(3N/\delta)} \le \Restu := 3\sqrt{\dimu + \log (1/\delta)}$ for all $t \in [N]$. Denote this event $\eventubounded$. 
  \end{claim}
  Let us turn to the last part of the lemma. To begin, let us bound the truncation error. We have
  \begin{align*}
  \frac{2}{\epsG} \cdot \Restu\|\Gst^{[>h]} - \Ghat^{[>h]}\|_{\loneop} &\le \psiGst(h+1) \frac{2\Restu}{h^2 \radnat C_{\delta}\sqrt{N}} \\
  &\le \psiGst(h+1) \frac{1}{h^2 \radnat\sqrt{N}} \\
  &\le \psiGst(h+1) \frac{1}{\sqrt{N}} \le 1
  \end{align*}
  where the second inequality uses $C_{\delta} \ge 2\Restu$, the thir uses $h^2 \radnat \ge 1$, and the four holds from our choice of $\psiGst(h+1)$. Hence, 
  \begin{align}
  \Restu\|\Gst^{[>h]} - \Ghat^{[>h]}\|_{\loneop} \le \epsG/2 \le 1/6 \label{eq:Restu_trunc_bound}
  \end{align},
  where the last step uses Part 1 of the lemma.

  Let us now bound the estimation error.  We begin by bounding $\|\Ghat^{[0:h]} - \Gst^{[0:h]}\|_{\op}$. To this end, define $\matdel_t = \matv_t - \sum_{i = 1}^{h} \Gst^{[i]}\matu_t$, and define $\matDel = [\matdel_{N_1}^\top \mid \dots \mid \matdel_{N}^{\top}]$. \cite{simchowitz2019learning} develop error bounds in terms of the operator norm of $\matDel$. In the subsubection below, we provide a simplified and self-contaned proof of the estimation guarantees from \cite{simchowitz2019learning}:

  \begin{lemma}[Simplification of Proposition 3.2 in \cite{simchowitz2019learning}]\label{lem:ls_gaurantee}   Then, if $N$ is sufficiently large that  $N \ge c h \dimu \log^4(N)$ for some universal constant $c > 0 $, and $\calE_{\lambda}$ is the event that $\|\matDel\|_{\op} \le \lambda$, then, with probability $1 - N^{-\log^2 (N)} - \delta/4$
  \begin{align*}
  \|\Gls - \Gst^{[1:h]}\|_{\op} \le \frac{5.6}{N}\lambda\sqrt{(h+1)(\dmax+\dimu + \log(h/\delta))}.
  \end{align*}
  In particular, for $h \ge 2$, we have the simplifid bound
  \begin{align*}
  \|\Gls - \Gst^{[1:h]}\|_{\op} \le \frac{5.6}{N}\lambda h\sqrt{\dmax+\dimu + \log(1/\delta)} 
  \end{align*}
  \end{lemma}
  Observe that for a sufficiently large constant $c_0$, taking $N \ge c_0 h^2 \dimu^2$ implies our condition in the above lemma $N \ge c h \dimu \log^4(N)$.  Next, let us bound $\|\matDel\|_{\op}$. We the have on the event $\eventubounded$:
  \begin{align*}
  \|\matDel\|_{\op} &\le \sqrt{N}\max_{t \in [N]}\|\matdel_t\|= \sqrt{N}\max_{t \in [N]}\|\maty_{t} - \sum_{s=1}^h \Gst^{[s]}\matu_{t-s} \|\\
  &= \sqrt{N}\max_{t \in [N]}\|\naty_{t} + \sum_{s>h} \Gst^{[s]}\matu_{t-s} \| \le \sqrt{N}\left(\max_{t \in [N]}\|\naty_{t}\| +  \|\Gst^{[>h]}\|_{\loneop}\max_{t \in [N]}\|\matu_{t-s} \|\right)\\
  &\le \sqrt{N}\left(\radnat +  \|\Gst^{[>h]}\|_{\loneop}\cdot \Restu\right )\\
  &\le \sqrt{N}\left(\radnat +  \frac{1}{6})\right ) \tag{\Cref{eq:Restu_trunc_bound}}\\
  &\le \frac{7}{6}\sqrt{N}\radnat,
  \end{align*}
  where we used the assumed upper bound on $\psiGst$ from  Plugging the above into Lemma~\ref{lem:ls_gaurantee} and using $\radnat \ge 1$ by assumption gives gives
  \begin{align*}
  2\|\Gls - \Gst^{[1:h]}\|_{\op} &\le \frac{7}{3}\radnat\sqrt{N} \cdot \frac{5.6}{N} h\sqrt{\dmax+\dimu + \log(1/\delta)}\\
  &= \frac{14  h \sqrt{\dmax+\dimu + \log(1/\delta)}}{\sqrt{N}} =  h C_{\delta}/\sqrt{N}.
  \end{align*}
  Thus, $2\|\Gls - \Gst^{[1:h]}\|_{\loneop} \le 14 h^2 C_{\delta}/\sqrt{N} := \epsG(N,\delta)$, as needed.
  
  \end{proof}

\subsubsection{Proof of \Cref{lem:ls_gaurantee}}

We adopt the argument of \cite{simchowitz2019learning}, but provide a simpler and self-contained proof.
Let us focus on the $\Gexyu$ case, which we shall denote $\Gst$ for the present argument. We denote the esimtate of $\Ghatexyu$ and $\Ghat$. Further, let $\matUbar$ denote the matrix with rows $\uexalg_{t:t+h}$ for $ 1 \le t \le N-h-1$. Moreover, let $\matdel_t = \maty_t - \sum_{i = 1}^{h} \Gst^{[i]}\matu_t$, and let $\matDel$ denote the matrix with rows $\matdel_t$ for $h+1 \le t \le N$. We then have the identity
\begin{align*}
\Ghat - \Gst = (\matUbar^\top \matUbar)^{-1} \matUbar^\top \matDel
\end{align*}
We can crudely bound
\begin{align*}
\|\Ghat - \Gst\|_{\op} \le \|\matUbar^\top \matUbar)^{-1}\|_{\op} \|\matUbar^\top \matDel\|_{\op}.
\end{align*}
Let us now bound the operator norm of $\|\matUbar^\top \matDel\|_{\op}.$ We have that the columns of $\matUbar^\top \matDel$ are of the form
\begin{align*}
[(\uexalg_{t-i}\matdel_t)_{h+1 \le t \le N}], \quad i \in \{0,1,\dots,h\}
\end{align*}
Thus, by \citet[Lemma A.1]{tsiamis2019finite}, and the definition of the operator norm (with $\calS^{\dimu-1} := \{v: \in \R^{\dimu} : \|v\|  =1 \}$),
\begin{align*}
\|\matUbar^\top \matDel\|_{\op} &\le \sqrt{h+1} \max_{i \in \{0,\dots,h\}}\|[(\uexalg_{t-i}\matdel_t)_{h+1 \le t \le N}]\|_{\op}\\
&\le \max_{i \in \{0,\dots,h\}}\max_{v \in \calS^{\dimu-1}} \|(v^\top\uexalg_{t-i}\matdel_t)_{h+1 \le t \le N}\|_2
\end{align*}
By the self-normalized martingale bounds (\citet[Theorem 3]{abbasi2011online}), and the fact that $\matdel_t$ is $\calF_{t-h-1}$ measurable, where $(\calF_{t})$ is the filtration generated by the random inputs, we have that with probability $1 - \delta$
\begin{align*}
\|(v^\top\uexalg_{t-i}\matdel_t)_{h+1 \le t \le N}(\matDel^\top \matDel + \lambda^2 I)\|_2^2 \le 2 \log \det(\frac{1}{\lambda^2}\matDel^\top \matDel +  I )+ 2 \log (1/\delta)
\end{align*}
In particular, if $\lambda^2$ is any parameter such that the event $\calE_{\lambda} := \|\matDel\|_{\op} \le \lambda$, then we have that with probability $1 - \delta$ that whenever $\calE_{\lambda}$ holds,
\begin{align*}
\frac{1}{\lambda^2}\|(v^\top\uexalg_{t-i}\matdel_t)_{h+1 \le t \le N}\|_2^2 &\le \|(v^\top\uexalg_{t-i}\matdel_t)_{h+1 \le t \le N}(\matDel^\top \matDel + \lambda^2 I)\|_2 \\
&\le 2 \log \det(\frac{1}{\lambda^2}\matDel^\top \matDel +  I )+ 2 \log (1/\delta)\\
&\le 2 \log \det(2I )+ 2 \log (1/\delta) = 2 (\dmax \log 2 + \log(1/\delta)).
\end{align*}
So rearranging, 
\begin{align*}
\|(v^\top\uexalg_{t-i}\matdel_t)_{h+1 \le t \le N}\|_2 \le \lambda\sqrt{2 (\dmax \log 2 + \log(1/\delta))}.
\end{align*}
Next, by a standard covering argument \citet[Section 4.2]{vershynin2018high}, we have that if $\calS_0$ is an $1/5$-net of $\calS^{\dimu-1}$, then  $\max_{v \in \calS^{\dimu-1}} \|(v^\top\uexalg_{t-i}\matdel_t)_{h+1 \le t \le N}\|_2 \le \frac{5}{4} \max_{v \in \calS_0}  \|(v^\top\uexalg_{t-i}\matdel_t)_{h+1 \le t \le N}\|_2$, and that we can take $\log |\calS_0|  \le \dimu \log 9$. Thus, by a union bound over $v \in \calS_0$ and $i \in \{0,\dots,h\}$, the following holds with probability $1-\delta/4$, 
\begin{align*}
\|(v^\top\uexalg_{t-i}\matdel_t)_{h+1 \le t \le N}\|_2 &\le \frac{4\sqrt{2}}{3}\lambda\sqrt{ ((\dmax+\dimu) \log 9 +  \log(4(h+1)/\delta))} \\
&\le \frac{4\sqrt{2}}{3}\lambda\sqrt{ ((\dmax+\dimu) \log 9 +  \log(8h/\delta))} \\
&\le 2.8\lambda\sqrt{\dmax+\dimu + \log(h/\delta)},
\end{align*}
Hence, we have that
\begin{align*}
\|\Ghat - \Gst\|_{\op} \le 2.8 \opnorm{(\matUbar^\top \matUbar)^{-1}}\lambda\sqrt{(h+1)(\dmax+\dimu + \log(h/\delta))}.
\end{align*}
Finally, by constants in the argument modifing the arguments of \citet[Lemma C.2]{oymak2019non}, we have that for any $\epsilon$, we can ensure that for $T \ge c(\epsilon) (h+1)\dimu \log^4(N\dimu)$, we can ensure $\|(\matUbar^\top \matUbar)^{-1}\|_{\op} \le (1 - \epsilon)(N - (h+1))^{-1}$ with probability $1 - N^{-\log^2 N}$. By enforcing $N \ge (h+1)/4$ and taking $\epsilon = 1/2$, we can obtain $\opnorm{(\matUbar^\top \matUbar)^{-1}} \le N/2$, yielding
\begin{align*}
\|\Ghat - \Gst\|_{\op} \le 5.6 \opnorm{(\matUbar^\top \matUbar)^{-1}}\lambda\sqrt{(h+1)(\dmax+\dimu + \log(h/\delta))},
\end{align*}
with probability $1- \delta/4 - N^{-\log^2 N}$ on $\calE_{\lambda}$.

\subsection{Unknown System Regret (Section~\ref{ssec:weak_convex_unknown_regret}) \label{ssec:app_weak_convex_unknown_regret}}

Let us conclude with presenting the omitted proofs from Section~\ref{ssec:weak_convex_unknown_regret}, and generalize to the stabilized case. The regret decomposition is identical to Eq.~\eqref{eq:unknown_weak_convex_regret_decomposition}, modifying the functions if necessary to capture their dependence on $\natuhat_{1:t}$. Throughout, we will assume $\epsG$ satisfies a generalization of \Cref{cond:est_condition} to the stabilized setting:

\begin{condmod}{cond:est_condition}{b}[Estimation Condition]\label{cond:est_general} We assume that the event of Theorem~\ref{thm:ls_estimation_stabilized} holds (i.e. accuracy of estimates $\Ghatexyu$ and $\Ghatexeta$), which entail $\epsG(N,\delta) \le 1/2\max\{\radM\radGst,\Restu\}$, where $\Restu = \Restu(\delta):= 3\sqrt{\dimu + \log (1/\delta)}$. Moreover, these entail that $\|\Ghatexeta\|_{\loneop},\|\Ghatexyu\|_{\loneop} \le 2\radGst$. These also entail that $\max_{s \le N }\|\ualg_s\| \le \Restu$.
\end{condmod}

We begin with the following generalization of the stability guarantee of Lemma~\ref{lem:estimation_stability_simple}:
\begin{lemmod}{lem:estimation_stability_simple}{b}[Stability of $\natyhat_{1:t}$]\label{lem:estimation_stability_generalized} Introduce the notation $\Rubar := 2\max\{\Restu, \radM\radnat\}$. Then, for $\epsG \le 1/2\max\{\radM\radGst,\Restu\}$ (satisfied by  \Cref{cond:est_general}) the following holds $t \in [T]$, 
  \begin{align*}
  &\|\uexalg_t\|_2 \le \Rubar, \quad \|\valg_t\| \le \radnat + \radGst\Rubar \le 2\radGst\Rubar,\\
  &\|\natvhat_t\|_2,\|\natetahat_t\|_2 \le 2\radnat,\,
  \end{align*}
  \end{lemmod}
  \begin{proof} The proof is analogous to that of \Cref{lem:estimation_stability_simple}, but with the following modifications. Let us sketch the major steps in the proof: we first establish the inequality $\|\nateta_t - \natetahat_t\|_2 \le \epsG\|\uexalg_{1:t}\|_{2,\infty}$, where we recall the notation $\|\uexalg_{1:t}\|_{2,\infty} = \max_{s \in \{1,\dots,t\}}\|\uexalg_s\|_2$ introduced in the original proof. Next, we can establish that
  \begin{align*} 
  \|\uexalg_t\|_2 \le \max\{\Restu,\radM \radnat\} +  \epsG\radM\|\Gexeta\|_{\loneop}\|\ualg_{1:t-1}\|_{2,\infty}.
  \end{align*}
  By assumption, $\|\Gexeta\|_{\loneop} \le \radGst$ (\Cref{def:transfer_stabilize}). Hence, for $\epsG \le 1/2\max\{\radM\radGst,\Restu\} \le 1/2\radM\radGst $, we can recursively verify that $\|\uexalg_t\|_2 \le 2\max\{\Restu,\radM \radnat\}$. Lastly, we can bound $\|\natvhat_t - \natv_t\|_2 \le \epsG  \|\uexalg_{1:t}\|_{2,\infty} \le \radnat$ under the conditon of the lemma, giving $\|\natvhat_t\| \le 2\|\natv_t\|$. Similarly, we can bound $\|\natetahat_t\| \le 2\|\nateta_t\|$.
  \end{proof}
  In the stabilized setting, we shall need to slightly modify our magnitude bounds to account for that norms of the controls:
  \begin{lemma}[Magnitude Bounds for Estimated System]\label{lem:estimated_magnitude_bounds} Suppose that Condition~\ref{cond:est_general} holds. Then, for any $t > N$, and all $M \in \calM(m,\radM)$ and $M_{t:t-h} \in \calM(m,\radM)^{h+1}$,
   \begin{align*}
  &\|\ualg_t\|, \|\uex_t(M \midhateta)\| \le 2\radM\radnat  \quad \\
  &\left\|\matv_t\left[ M_{t:t-h} \midhatetayg \right]  \right\|_2 \le 6\radGst\radM\radnat
  \end{align*}
  \end{lemma}
  \begin{proof} We have that $\|\uex_t(M)\|_2 = \|\sum_{i=0}^{m-1}M^{[i]} \natyhat_{t-i}\| \le \radM \max_{s \le t}\|\natyhat_t\| \le 2\radM \radnat$ by Lemma~\ref{lem:estimation_stability_generalized}. The bound on $\ualg_t$ specializes by setting $M \leftarrow \matM_t$.

  By the same lemma, and the fact that $\|\Ghatexyu\|_{\loneop} \le \epsG + \|\Gexyu\|_{\loneop} \le 2\radGst$ by \Cref{cond:est_general},
  \begin{align*}
  \left\|\matv_t\left[ M_{t:t-h} \midhatetayg \right]  \right\|_2 &= \left\|\natvhat_t  + \sum_{s=t-h}^t \Ghatexyu^{[t-s]} \uex_s(M_s \midhateta)\right\| \\
  &\le 2\radnat + \|\Ghatexyu\|_{\loneop}\max_{M\in \calM(m,\radM)}\max_{s \le t}\|\uex_t(M\midhateta)\| \\
  &\le 2\radnat + 4\radGst \radM\radnat \le 6\radGst\radM\radnat.
  \end{align*}
  The bound on $\|\maty_t\left[ M_{t:t-h} \midhatt \right]  \|_2$ is similar.
  \end{proof}

  Next, we check that the proof of Lemma~\ref{lem:policy_regret_unknown_lip} goes through in the general case
  \begin{proof}[ Proof of Lemma~\ref{lem:policy_regret_unknown_lip} for Stabilized Setting] The proof is analogous to the general case, where we replace the dependence no $\radnat$ and $\radGst$ with an upper bound on $\natyhat_t$, $\natuhat_t$ and $\max\{\|\Ghatuy\|_{\loneop},\|\Ghatuu\|_{\loneop}\}$. In light of the above bounds, these quantities are also $\lesssim \radnat$ and $\radGst$, up to additional constant factors, yielding the same regret bound up to constants.
  \end{proof}
  Lastly, we establish Lemma~\ref{lemma:approx_error_bounds}, encompassing both the stabel and stabilized case. Given that the proof is somewhat involved, we organize it in the following subsection.

\subsubsection{Proof of Lemma~\ref{lem:acc_estimat_nat}/\ref{lem:acc_estimat_nat_stab}\label{sssec:lem:acc_estimat_nat}}
We bound the error in estimating natures $y$'s and natures $u$'s:
  \begin{lemmod}{lem:acc_estimat_nat}{b}[Accuracy of Estimated Nature's $y$ and $u$ and $\eta$]\label{lem:acc_estimat_nat_stab} Assume \Cref{cond:est_general}. Then for $t \ge N + h + 1$, we have that 
  \begin{align*}
  \|\natv_t - \natvhat_t\|_2 &\le 3\radM\radnat\epsG\\
  \|\nateta_t - \natetahat_t\|_2 &\le 3\radM\radnat\epsG
  \end{align*}
  \end{lemmod}
  \begin{proof}
  Let us focus on $\|\natv_t - \natvhat_t\|_2 $, the error bound on  $\|\nateta_t - \natetahat_t\|_2$ is similar. Let us use the notation $G^{[0:\ell]}$ to denote the restriction of a Markov operator to $(G^{[i]})_{i = 0}^{\ell}$, and $G^{[>\ell]}$ to restrict to $(G^{[i]})_{i > \ell}$. We can then bound:
  \begin{align*}
  \|\natv_t - \natvhat_t\|_2 &= \left\|\sum_{s=1}^t \Gexyu^{[t-s]}\uexalg_t - \Ghatexyu^{[t-s]}\uexalg_t\right\|_2 \\
  &\le \left\|\Gexyu^{[0:t-(N+1)]} - \Ghatexyu^{[0:t-(N+1)]}\right\|_{\loneop} \cdot \max_{s \in [N+1:t]}\|\uexalg_s\|  \\
  &\qquad+ \left\|\Gexyu^{[>t - (N + 1)]} - \Ghatexyu^{[>t - N + 1]}\right\|_{\loneop} \cdot \max_{s \le N}\|\uexalg_s\|.
  \end{align*}
  For $t \ge N + h - 1$, we have  $\|\ualg_t\| \le 2\radM\radnat$, and under \Cref{cond:est_general}, we have $\max_{s \le N}\|\uexalg_s\| \le \Rubar$. Moreover, we can bound $\left\|\Gexyu^{[0:t-N+1]} - \Ghatexyu^{[0:t-N+1]}\right\|_{\loneop} \le \left\|\Gexyu - \Ghatexyu\right\|_{\loneop} \le \epsG$ under \Cref{cond:est_general}. In addition, since $t \ge N + h+1$, $t - N + 1 \ge h$, so 
  \begin{align*}
  \left\|\Gexyu^{[>t - (N + 1)]} - \Ghatexyu^{[>t - N + 1]}\right\|_{\loneop} &\le  \left\|\Gexyu^{[>h]} - \Ghatexyu^{[>h]}\right\|_{\loneop}  \\
  &= \left\|\Gexyu^{[>h]} \right\|_{\loneop}  \le \psiGst(h+1),
  \end{align*}
  where we use the fact that $\Ghatexyu^{[i]} = 0$ for $i > 0$, and the fact that we define $\psiGst(\cdot) := \max\{\psi_{\Gexyu}(\cdot),\psi_{\Gexeta}(\cdot)\}$ (\Cref{def:transfer_stabilize}). Thus, all in all,
\begin{align*}
  \|\natv_t - \natvhat_t\|_2 &\le 2\radM\radnat \epsG + \psiGst(h+1) \Restu \le 3 \radM\radnat \epsG,
  \end{align*}
  where the last step holds because $\radM\radnat \ge 1$ by assumption, and that under \Cref{cond:est_general}, $\psiGst(h+1) \Restu  \le \epsG$. 
  \end{proof}
\subsubsection{Proof of Lemma~\ref{lemma:approx_error_bounds}\label{sssec:lemma:approx_error_bounds}}
  We prove the lemma in the more general stabilized setting, where we require the stronger \Cref{cond:est_general} instead of \Cref{cond:est_condition}. For completeness, we state this general bound here
   \begin{lemmod}{lemma:approx_error_bounds}{b}[Approximation Error Bounds: Stabilized]\label{lemma:approx_error_bounds_general} Under~\Cref{cond:est_general},
  \begin{align*}
  \text{\normalfont (loss approximation error)} +
  \text{\normalfont (comparator approximation error)} &\lesssim LT\radGst\radM^2\radnat^2 \epsG
  \end{align*}
  \end{lemmod}
  \begin{proof}
  Let us start with the loss approximation error. For $t \ge N+h+1$, and using $\uex_t(\matM_t) = \uexalg_t$, we have
  \begin{align*}
  &\left\|\valg_t -\matv_t[\matM_{t:t-h} \midhatetayg]\right\|_2 \\
  &\quad= \| \natv_t + \sum_{s=1}^{t}\Gexyu^{[t-s]}\uexalg_s  - (\natvhat_t + \sum_{s=t-h}^{t}\Ghatexyu^{[t-s]}\uexalg_s) \|_2\\
  &\quad= \| \natv_t + \sum_{s=1}^{t}\Gexyu^{[t-s]}\uexalg_s  - (\natvhat_t + \sum_{s=1}^{t}\Ghatexyu^{[t-s]}\uexalg_s) \|_2 \tag*{$\Ghatexyu^{[i]} = 0$ for $i > h$}\\
  &\quad\le \| \naty_t - \natyhat_t\|_2 + \left\|\sum_{s=1}^t \Gexyu^{[t-s]}\uexalg_t - \Ghatexyu^{[t-s]}\uexalg_t\right\|_{2}\\
  &\quad= \| \natv_t - \natvhat_t\|_2 + \| \natv_t = \natvhat_t\|_{2} \le 6\radM\radnat \epsG
  \end{align*}
  where we use \Cref{lem:estimation_stability_generalized} in the last inequality. 
  Moreover, recalling the following bound from \Cref{lem:estimated_magnitude_bounds},
  \begin{align*}
  \|\matv_t\left[ \matM_{t:t-h} \midhatetayg \right] \|_2 \le 6\radGst\radM\radnat
  \end{align*}
  we have
  \begin{multline*}
  \max\left\{\|\valg_t\|_2, \left\|\matv_t\left[ M_{t:t-h} \midhatetayg \right]  \right\|_2\right\}  \\
  \le  6\radGst\radM\radnat + 6\radM\radnat \epsG \le 9\radGst\radM\radnat,
  \end{multline*}
  where we use the bound $\epsG \le 1/2\radM\radGst \le \radGst/2$ under \Cref{cond:est_general} and the bounds $\radM,\radGst \ge 1$. Hence, we have
  \begin{align*}
  &\textnormal{(loss approximation error)} \\
  &\quad:=  \sum_{t=N+m+h+1}^T \loss_t(\valg_t) - \sum_{t=N+m+h+1}^T F_t[\matM_{t:t-h}\midhatetayg]\\
  &\quad\le  \sum_{t=N+m+h+1}^T |\loss_t(\yalg_t,\ualg_t) - \loss_t(\matv_t[\matM_{t:t-h}\midhatetayg])]|\\
  &\quad\le 9L\radGst\radM\radnat \iftoggle{colt}{\\&\qquad \cdot}{\cdot} \sum_{t=N+m+h+1}^T \left\|\valg_t -\matv_t[\matM_{t:t-h} \midhatetayg]\right\|_2 \\
  &\quad\le 54LT\radGst\radM^2\radnat^2\epsG.
  \end{align*}
  where we used Assumption~\ref{asm:lip} and the bounds computed above. 
 
  Let us now turn to the comparator approximation error
  \begin{align*}
  &\textnormal{(comparator approximation error)} \\
  &\quad:= \inf_{M \in \calM}  \sum_{t=N+m+h+1}^{T}f_t(M \midhatetayg ) - \inf_{M \in \calM} \sum_{t=N+m+h+1}^{T}\ell_t(\maty^{M}_t,\matu^{M}_t)\\
  &\quad\le \max_{M \in \calM}  \sum_{t=N+m+h+1}^{T}|\ell_t(\matv_t^M) - \ell_t(\matv_t(M \midhatetayg))|\\
  &\quad\le 6LT\radGst\radM\radnat \iftoggle{colt}{\\&\qquad \cdot}{}\max_{M \in \calM}  \sum_{t=N+m+h+1}^{T} \left\| \matv_t^M - \matv_t(M \midhatetayg)\right\|_2
  \end{align*}
  where again we use the magnitude bounds in Lemmas~\ref{lem:magnitude_bound_stabilized} and~\ref{lem:estimated_magnitude_bounds}, and the Lipschitz Assumption (Assumption~\ref{asm:lip}).  Let us bound the differences between the $\matv_t$ terms, taking caree that errors is introduced by both the approximation of the transfer function $\Gexyu$ and the Nature's $\eta$ sequence $\nateta_t$. For $t \ge N+h+1+m$, we obtain
  \begin{align*}
  & \left\| \matv_t^M - \matv_t(M \midhatetayg)\right\|_2\\  
  &= \left\|\natv_t + \sum_{s=1}^t \Gexyu^{[t-s]}\uex_s(M \midetanat[s])  - \left(\natvhat_t + \sum_{s=t-h}^t \Ghatexyu^{[t-s]}\uex_s(M \midhateta[s])\right) \right\|_2\\
  &= \underbrace{\|\natv_t - \natvhat_t\|_2}_{\le 3\radM\radnat\epsG} +  \underbrace{\|\Gexyu - \Ghatexyu\|_{\loneop}}_{\le \epsG}\underbrace{\max_{t - h \le s \le t}\|\uex_s(M \midhateta[s])\|}_{\le 2\radM\radnat} \\
  &\qquad+ \left\|\sum_{s=1}^{t-h+1} \Gexyu^{[t-s]}\uex_s(M \midhateta[s])\right\|_2 \\
  &= 5\radM\radnat \epsG +  \left\|\sum_{s=1}^t \Gexyu^{[t-s]}\uex_s(M \midhateta[s])\right\|_2. 
  \end{align*}
  where we have used the magnitude bounds in Lemma~\ref{lem:estimation_stability_generalized} and~\ref{lem:estimated_magnitude_bounds}. We can further bound
  \begin{align*}
  \left\|\sum_{s=1}^t \Gexyu^{[t-s]}\uex_s(M \midhateta[s])\right\|_2 &\le \psiGst(h+1) \max_{s \le t}\|\uex_s(M \midhateta[s]\|\\
  &\le \psiGst(h+1) 2\max\{\Restu, \radM\radnat\},
  \end{align*}
  where we use Lemma~\ref{lem:estimation_stability_generalized} abolve. From Conditions~\ref{cond:est_condition}/\ref{cond:est_general}, we can bound $\psiGst(h+1)\Restu \le \epsG$. And since $\Restu \ge 1$, this implies that the above is at most $2\psiGst \Restu \radM \radnat \le 2\epsG \radM \radnat$. Thus, from the above previous two displays,
  \begin{align*}
  \left\| \matv_t^M - \matv_t(M \midhatetayg)\right\|_2 \le 7\epsG \radM \radnat,
  \end{align*}
  giving
  \begin{align*}
  \textnormal{(comparator approximation error)} \le 6L\radGst\radM\radnat \cdot 7\epsG \radM \radnat = 42 L T \radGst \radM^2 \radnat^2 \epsG
  \end{align*}
  Combining the two bounds, we and using $\radGst,\radM \ge 1$,
  \begin{align*}
  \textnormal{(comparator approximation error)} + \textnormal{(loss approximation error)} \lesssim TL\radGst\radM^2\radnat^2 \epsG
  \end{align*}
\end{proof}

\newcommand{\vpred}{\matv^{\mathrm{pred}}}
\newcommand{\vst}{\matv^{\star}}
\newcommand{\matvhat}{\widehat{\matv}}

\newpage
\section{ Strongly Convex, Semi-Adversarial Regret \label{app:strongly_convex}}
We begin by stating a slight generalization of the semi-adversarial model described by Assumption~\ref{asm:noise_covariance_ind}.  Recall the assumption that our noises decompose as follows:
	\begin{align*}
	\matw_t &= \wadv_t + \wstoch_t\\
	\mate_t &= \eadv_t + \estoch_t
	\end{align*}
	We make the following assumption on the noise and losses:
	\begin{asmmod}{asm:noise_covariance_ind}{b}[Semi-Adversarial Noise: Martingale Structure]\label{asm:noise_covariance_mart}
	We assume that there is a filtration $(\calF_{t})_{t \ge 0}$ and a matrix $\Sigmanoise \in \R^{(\dimx +\dimy)^2} \succeq 0 $ (possibly degenerate), and $\sigmaw^2,\sigmae^2 \ge 0$ (possibly zero) such that the following hold:
	\begin{enumerate}
		\item The adversarial disturbance sequences $(\wadv_t)$ and $(\eadv_t)$ and the loss sequence $\ell_t(\cdot)$ are \emph{oblivious}, in the sense that they are $\calF_0$-adapted.
		\item The sequences $(\wstoch_t)$ and $(\estoch_t)$ and $(\calF_{t})$-adapted
		\item $\Exp[\estoch_t \mid \calF_{t-1}] = 0$, $\Exp[\wstoch_t \mid \calF_{t-1}] = 0$.
		\item The noises satisfy 
		\begin{align*}
		\Exp\left[\begin{bmatrix} \wstoch_t\\
		\estoch_t \end{bmatrix}\begin{bmatrix} \wstoch_t\\
		\estoch_t \end{bmatrix}^\top \mid \calF_{t-1} \right]  \succeq \Sigmanoise \succeq \begin{bmatrix} \sigmaw^2 I_{\dimx} & 0 \\
		0 & \sigmae^2 I_{\dimy}\end{bmatrix}.
		\end{align*}
	\end{enumerate}
	Moreover, at least one of the following hold: 
	\begin{enumerate}
		\item[(a)] The system is internally stable has no stabilizing controller $\pi_0$, and $\sigmaw^2 + \sigmae^2 > 0$
	\item[(b)] The system is stabilized by a static feedback controller $\pi_0$ (that is, $A_{\pi_0} = 0$ and $\dim(\matz^{\pi}) = 0$), and $\sigmaw^2 > 0$ 
	\item[(c)] The system is stabilized by a general stabilizing controller, and $\sigmae^2 > 0$. \footnote{This condition can be generalized somewhat to a form of ``output controllability'' of the noise transfer function, which can potentially accomodate $\sigmae = 0$. We omit this generalization in the interest of brevity}
	\end{enumerate} 
	\end{asmmod}
	\newcommand{\rhosc}{p_{\mathrm{sys}}}
	\newcommand{\msc}{m_{\mathrm{sys}}}
	\newcommand{\alphasc}{\alpha_{\mathrm{sys}}}

	 As in thes stable setting, the  strong convexity parameter governs the functions
\begin{align}\label{eq:f_t_k_def}
	f_{t;k}\left(M \midstyu\right) := \Exp\left[f_t(M \midstyu) \mid \calF_{t-k}\right].
	\end{align}
	For stabilized settings, \Cref{prop:strong_convex_no_internal_controller} admits the following generalization:
	\begin{propmod}{prop:strong_convex_no_internal_controller}{b}[Strong Convexity for known system]\label{prop:strong_convex_internal_controller} Suppose that we interact with an internally-controlled system (\Cref{def:lin_output_feedback}). Then, under assumptions \Cref{a:stablize,asm:bound,asm:lip,asm:strong_convexity,asm:noise_covariance_mart}, there exists system dependent constants $\msc,\rhosc \ge 0$ and $\alpha_{f;0} > 0 $ such that, for $h = \floor{m/3}$, $k = m+2h$, and $m \ge \msc$, the functions
	$f_{t;k}\left(M\midstyu\right)$ are $\alpha_{f;m}$-strongly convex, where
	\begin{align*}
	\alpha_{f;m} &:= \alphaloss \cdot \alphasc \cdot m^{\rhosc}
	\end{align*}
	In other words, the strong convexity parameter decays at most \emph{polynomially} in $m$.
	\end{propmod}
	The above proposition is given in \Cref{ssec:prop:strong_convex_internal_controller}. For general LDC-Ex controllers, we do not have transparent expressions for $ \alphasc$ and  $m^{\rhosc}$. Nevertheless, we ensure that the above bound is strong enough to ensures rates of $\log^{\BigOh{1} + \rhosc} T$ and $\sqrt{T}\log^{\BigOh{1} + \rhosc}$, where the exponent hidden by  $\BigOh{1}$ does not depend on system parameters (so that the exponents are determined solely by $\rhosc$). We make a couple remarks, which in particular describe how $\rhosc$ is often $0$ in many settings:
	\begin{enumerate}
		\item  In general, the strong convexity parameters of the system are determined by the properties of the Z-transforms for relevant operators. A general expression is given in \Cref{thm:str_cvx_main} , and the  preliminaries and definitions relevant for the theorem are given in \Cref{subsec:strong_convexity_preliminaries}.  \Cref{prop:strong_convex_internal_controller} is proven in \Cref{ssec:main_strong_convexity_bounds} as a consequence of this more general result, and \Cref{app:strong_convexity_check} contains all details related to establishing strong convexity. 
		\item In \Cref{sssec:example:static_feedback}, we show that for systems stabilized via static feedback, we can take $\rhosc = 0$, and  give explict and transparent bounds on $ \alphasc$. This recovers the special case of internally stable systems as a special case, where we can take $ \alphasc = \sigmae^2 + \frac{\sigmaw^2 \sigma_{\min}(\Cst)}{(1+\|\Ast\|_{\op})^2}$.
		\item For the special case of internally systems (Proposition~\ref{prop:strong_convex_no_internal_controller},), we present a smaller self-contained proof that does not appeal to Z-transform machinery (\Cref{prop:strong_convex_no_internal_controller}.  Note that this resut does not require that $m \ge \msc$ restriction required by \Cref{prop:strong_convex_internal_controller}.
		\item The parameter $\msc$ is related to the decay of the system, and can be deduced from the conditions of \Cref{thm:str_cvx_main}.
	\end{enumerate}

\Cref{thm:fast_rate_known,thm:fast_rate_unknown} generalize to the stabilized-system setting:
\begin{thmmod}{thm:fast_rate_known}{b}[Fast Rate for Known System: Stabilized Case]\label{thm:fast_rate_known_stablz} Suppose assumptions \Cref{a:stablize,asm:bound,asm:lip,asm:strong_convexity,asm:noise_covariance_mart} holds. Thenw with the additional condition $h = \floor{m/3}$ and and appropriate modifications as in \Cref{defn:modifications_stabilized}, \Cref{thm:fast_rate_known} holds verbatim when $\alpha_f$ is replaced with the stabilized analgoue $\alpha_{f;m}$ from \Cref{prop:strong_convex_internal_controller}. 
In particular, taking $\alpha = \Omega(\alpha_{f;m})$, we obtain regret bounded by 
  \begin{align}
   \Regret_T(\psi) \lesssim \frac{L^2 m^{3+\rhosc}\dmin \radnat^4\radGst^4\radM^2}{\min\left\{ \alphasc, L\radnat^2 \radGst^2\right\}} \left(1 + \frac{\betaloss }{L \radM}\right)\cdot\log\frac{T}{\delta}\label{eq:str_convex_known_stabilized}.
    \end{align}
    In particular, under Assumption~\ref{asm:noise_decay_bound}, we obtain
    \begin{align*}\Regret_T(\psi) \lesssim (\tfrac{1}{1-\rho} \log \tfrac{T}{\delta})^{\rhosc} \cdot \poly (C,L,\betaloss,\frac{1}{\alpha},\sigmanoise^2,\frac{1}{1-\rho},\log \tfrac{T}{\delta})\cdot (\dmax^{2}\sqrt{T} + \dmax^3),
  \end{align*}
  where the exponents in the $\poly(\cdot)$ term do not depend on system parameters, although $\rhosc$ does.
\end{thmmod}
Again, for general stabilized system, we may suffer exponents which depend on this system-dependent $\rhosc$. But, as discussed above $\rhosc$ may be equal to $0$ in many cases of interest.

For unknown systems, we have the following:

\begin{thmmod}{thm:fast_rate_unknown}{b}[Fast Rate for Unknown System: Stabilized Case]\label{thm:fast_rate_unknown_stablz} Suppose assumptions \Cref{a:stablize,asm:bound,asm:lip,asm:strong_convexity,asm:noise_covariance_mart} holds. Thenw with the additional condition $h = \floor{m/3}$ and and appropriate modifications as in \Cref{defn:modifications_stabilized}, Theorem~\ref{thm:fast_rate_unknown} holds verbatim when $\alpha_f$ is replaced with the stabilized analgoue $\alpha_{f;m}$ from \Cref{prop:strong_convex_internal_controller}. In particular, taking $\alpha = \Omega(\alpha_{f;m})$ and  Assumption~\ref{asm:noise_decay_bound}, we obtain
    \begin{align*}\Regret_T(\psi) &\lesssim  \poly (C,L,\betaloss,\frac{1}{\alpha},\sigmanoise^2,\frac{1}{1-\rho},\log \tfrac{T}{\delta})
    &\quad \cdot (\dmax^{2}\sqrt{T(\tfrac{1}{1-\rho}\log\tfrac{T}{\delta})^{\rhosc}}  + (\tfrac{1}{1-\rho}\log\tfrac{T}{\delta})^{\rhosc}\dmax^3),
  \end{align*}
  where the exponent in $\rhosc$ where the exponents in the $\poly(\cdot)$ term do not depend on system parameters, although $\rhosc$ does.
  \end{thmmod}


\subsection{Proof Details for~\Cref{thm:fast_rate_known,thm:fast_rate_known_stablz} \label{app:fast_rate_know}}

The proof of Theorems~\Cref{thm:fast_rate_known,thm:fast_rate_known_stablz} are identical, except for the difference in strong convexity parameters in view of \Cref{prop:strong_convex_internal_controller} and \Cref{prop:strong_convex_no_internal_controller}. Thus, the proof of \Cref{prop:strong_convex_internal_controller} follows from the proof of \Cref{prop:strong_convex_internal_controller} given in \Cref{sec:proof_fast_rate_known}, ammending $\alpha_f$ to $\alpha_{f;m}$ where it arises.

It remains to supply a the ommited proof of the lemma that establishes smoothness of the objectives, \Cref{lem:smoothness_stable}. We restate the lemma here to include the  $\midstyu$ 
encountered in the stabilized case: 
	\begin{lemmod}{lem:smoothness_stable}{b}[Smoothness]\label{lem:smoothness_stablz} The functions $f_t(M \midstyu)$ are $\beta_f$-smooth, where we define $\beta_f := m  \radnat^2\radGst^2 \betaloss$.
\end{lemmod}
\begin{proof} For brevity, we omit $\midetaygst$. Let $\Differential \matv_t[\cdot]$ the differential of the function as maps from $\R^{(m\dimy\dimu)} \to \R^{\dimy + \dimu}$, these are elements of $\R^{(m\dimy\dimu) \times (\dimy + \dimu)}$. These are affine functions, and thus do not depend on the $M$ argument.From the chain rule (with appropriate transpose conventions), and the fact that affine functions have vanishing second derivative
\begin{align}
\nabla f_t(M) &= \Differential\matv_t (\nabla \ell)(\matv_t)  \nonumber\\
\nablatwo f_t(M) &= \Differential\matv_t \cdot (\nablatwo \ell)(\matv_t) \cdot \Differential\matv_t^\top \, \preceq \betaloss \left\|\Differential\matv_t\right\|_{\op}^2 I   \label{eq:second_der_bound}
\end{align}
Let us now bound the norm of the differentials. Observe that $\|\Differential(\matu_t(M)\|_{\op}$ and $\|\Differential(\maty_t(M)\|_{\op}$ are just the Frobenius norm to $\ell_2$ Lipschitz constant of $M \mapsto \matv_t(M)$ is bounded by $\sqrt{m}\radGst \radnat$ via \Cref{lem:lip_bound_coord_map}. Thus $\nablatwo f_t(M) \preceq \betaloss m \radGst^2\radnat^2 I$, as needed.
\end{proof}

\subsection{Supporting Proofs for Theorems~\Cref{thm:fast_rate_unknown,thm:fast_rate_unknown_stablz}}
We now generalize to the stabilized, unknown setting. Throughout, we shall use the various magnitude bounds on $\natyhat,\natuhat,\dots$ developed in \Cref{ssec:app_weak_convex_unknown_regret} for unknown system / Lipschitz loss setting.

For this strongly convex, stabilized, unknown setting, we generalize the true prediction losses of \Cref{defn:true_pred_losses} as follows:
\begin{defnmod}{defn:true_pred_losses}{b}[True Prediction Losses]\label{defn:true_pred_losses_stablz} We define the true prediction losses as
	\begin{align*}
	\vpred_t[M_{t:t-h}] &:= \matv_t[M_{t:t-h} \midpred]\\
	&= \natv_t + \sum_{i=0}^h \Gexyu^{[i]}\uex_t(M\midhateta[t-i])\\
	\Fpred_t\left[M_{t:t-h} \right] &:= \loss_t\left(\vpred_t[M_{t:t-h}]\right),
	\end{align*}
	and let $\fpred_t(M) = \Fpred_t(M,\dots,M)$ denote the unary specialization. The corresponding conditional functions of interest are
	\begin{align*}
	\fpred_{t;k}\left(M\right) := \Exp\left[\fpred_t(M) \mid \calF_{t-k}\right].
	\end{align*}
	\end{defnmod} 
	Throughout the proof, it will be useful to adopt the shorthand $\vst_t(M) := \matu_t(M \midetaygst)$, and $\yst_t(M) := \maty_t(M \midst)$ to denote the counterfactuals for the true nature's $y$, and $\matvhat_t(M) := (M \midhatetayg)$ denote the counterfactuals for the estimates $\Ghat$ and $\natetahat$ and $\natvhat$. Note that $\upred$ and $\ypred$ can be though as interpolating between these two sequences. 
	
	We shall also let $\Differential \vpred_t$   denote differentials as elements of  are elements of $\R^{(m\dimy\dimu) \times (\dimy+\dimu)}$ , and similarly for $\Differential \vst$ and $\Differential \matvhat$. As these functions are affine, the differential is independent of $M$-argument

	\subsubsection{Preliminary Notation and Perturbation Bounds}

		Before continuing, we shall state and prove two useful lemmas that will help bound the gradients / Lipschitz constants of various quantities of interest. 
		\begin{lemma}[Norm and Perturbation Bounds]\label{lem:norm_and_perturb_str_unknown}  The following bounds hold for $t > N$:
		\begin{enumerate}
			\item[(a)] $\|\Differential \vpred_t  \|_{\op} \le 2\sqrt{m}\radnat\radGst$
			\item[(b)] $\|\vpred_t(M) - \vst_t(M) \|_2\le 3\sqrt{m}\radM\radnat\radGst \|M\|_{\fro}$\\
			\item[(c)] $\|\Differential(\vpred_t-\matvhat_t) \|_{\op} \le 2\sqrt{m}\epsG\radnat$
			\item[(d)] For all $M \in \calM$, $\|\vpred_t(M) - \matvhat_t(M) \|_2 \le 5\radM\radnat\epsG$.
		\end{enumerate}
		\end{lemma}
		\begin{proof} Note that operator norm bounds on the differential are equivalent to the Frobenius-to-$\ell_2$ Lipschitz constants of the associated mappings. The proofs are then analogous to the proof of Lemma~\ref{lem:lip_bound_coord_map}, where the role of $\naty_t$ and $\Gst$ are replaced with the appropraite quantities. For clarity, we provide a relevant generalization of that lemma, without proof. 
		\newcommand{\radGtil}{R_{\Gtil}}
		\newcommand{\radnatilt}{\widetilde{R}_{\mathrm{nat},t}}

		\newcommand{\natytil}{\widetilde{\maty}^{\mathrm{nat}}}
		\newcommand{\natutil}{\widetilde{\matu}^{\mathrm{nat}}}
		\newcommand{\natvtil}{\widetilde{\matv}^{\mathrm{nat}}}
		\newcommand{\natetatil}{\widetilde{\mateta}^{\mathrm{nat}}}
		\begin{lemma}[Lipschitz Bound on Generalized Coordinate Mappings]\label{lem:lip_bound_coord_map_gen}
  		Let $\Gtil,\natytil_{1:t},\natutil_{1:t}$ be arbitrary, let $\radGtil = \|\Gtil\|_{\loneop}$ and $\radnatilt := \max\{\|\natytil_{s}\| : s \in [t-h-m+1:t] \}$. Then, 
  		\begin{align*}
  		&\|\matv_t(M\mid \Gtil,\natvtil_t, \natetatil_{1:t}) - \matv_t(\Mtil \mid \Gtil,\natvtil_t, \natetatil_{1:t}) \|_2 \\
  		&\quad\le \radnatilt \radGtil \|M - \Mtil\|_{\loneop} \le \sqrt{m}\radGtil \radnatilt \|M - \Mtil\|_{\fro}.
  \end{align*}
  The generalized to non-unary functions of $M_{t-h:t}$ is analogous~\Cref{lem:lip_bound_coord_map}. Notice the above bound \emph{does not} depend on $\natvtil_t$, which consitutes an affine term.
  \end{lemma}
  		. For part (a),  the bound follows by bounding $\|\natetahat_t\| \le 2\radnat$ by \Cref{lem:estimation_stability_generalized}, and applying \Cref{lem:lip_bound_coord_map_gen} with $\natetatil_t \leftarrow \natetahat_t$,  and $\natvtil_t \leftarrow \natvhat_t$, and $\Gtil \leftarrow \Gexyu$.
\mscomment{from here}

		In part (b), we can compute
		\begin{align*}
		\|\vpred_t(M) - \vst_t(M) \|_2 &\le \|\matv_t(M\mid \Gexyu,\natv_t, \natetahat_{1:t})  - \matv_t(M\mid \Gexyu,\natv_t, \natetahat_{1:t})\|_2\\
		&= \|\matv_t(M\mid \Gexyu,\natv_t, \natetahat_{1:t} - \natetahat_{1:t}) \|_2.
		\end{align*}
		Let us apply \Cref{lem:lip_bound_coord_map_gen} with $\natetatil_t \leftarrow \nateta_t - \natetahat_t$ and $\Gtil \leftarrow \Gexyu$. The associated value of $\radnatilt$ can be replaced by an upper bound on $\max\{\nateta_t - \natetahat_t: s \in [N+1 :t] \}$, which we can take to be $3\radM\radnat\epsG$ by \Cref{lem:acc_estimat_nat_stab}. This gives a bound of $3\sqrt{m}\radM\radnat\radGst \|M\|_{\fro}.$, as needed.
	
		For part (c), take $\natetatil_t \leftarrow \natetahat_t$  $\Gtil \leftarrow \Ghatexyu - \Gexyu$ playing the role of $\Gst$, yielding a $\radnatilt \le 2\radnat$  by \Cref{lem:estimation_stability_generalized} and  $\radGtil \le \epsG$ from \Cref{lem:estimated_magnitude_bounds}.

		Finally, let us establish part (d). We have
		\begin{align*}
		\|\vpred_t(M) - \matvhat_t(M) \|_2 &\le \|\natv_t - \natvhat_t \|_2 + \|\sum_{i=0}^h (\Gexyu^{[i]} - \Ghatexyu^{[i]})\uex_{t-i}( M  \midhateta[t-i])\|\\
		&\le \|\natv_t - \natvhat_t \|_2 + \|\sum_{i=0}^h (\Gexyu^{[i]} - \Ghatexyu^{[i]})\uex_{t-i}( M  \midhateta[t-i])\|\\
		&\le \|\natv_t - \natvhat_t \|_2 + 2\radM\radnat\sum_{i=0}^h  \|(\Gexyu^{[i]} - \Ghatexyu^{[i]})\|_{\op}\\
		&\le \|\natv_t - \natvhat_t \|_2 + 2\radM\radnat\epsG,
		\end{align*}
		where the second to last step uses \Cref{lem:estimated_magnitude_bounds}. Finally, we can bound $\|\natv_t - \natvhat_t \|_2 \le 3\radM\radnat\epsG$ by \Cref{lem:acc_estimat_nat_stab}. Combining the bounds yields the proof.
		\end{proof}

	\subsubsection{Gradient Error (\Cref{lem:str_cvx_grad_approx,lem:str_cvx_grad_approx_stablz})\label{sssec:lem:str_cvx_grad_approx}}

		\begin{lemmod}{lem:str_cvx_grad_approx}{b}\label{lem:str_cvx_grad_approx_stablz} For any $M \in \calM$, we have that
			\begin{align*}
			\left\| \nabla f_t(M \midhattyu) - \nabla \fpred_t(M) \right\|_{\fro} \le  \Capprox\,\epsG,
			\end{align*}
			where is $\Capprox := \sqrt{m}\radGst\radM\radnat^2(8\betaloss + 12L)$.
		\end{lemmod}
		\begin{proof}
		 Let  denote the differential of the functions as maps from $\R^{(m\dimy\dimu)}$,, respectively. Define differentials analogously for $\upred,\ypred$. Then,
		\begin{align*}
		\iftoggle{colt}{&}{}\nabla \fpred_t(M) - \nabla f_t(M \midhatetayg) \iftoggle{colt}{\\}{} &=  \underbrace{\Differential\vpred_t \cdot \left((\nabla \ell_t)(\vpred_t(M)) -(\nabla \ell_t)(\matvhat_t(M)) \right)}_{(a)}\\
		&+ \underbrace{\Differential(\vpred_t-\matvhat_t)
		  (\nabla \ell_t)(\matvhat_t(M))}_{(b)}
		\end{align*}
		We can bound the first term via
		\begin{align*}
		\|(a)\|_{\op} &\overset{(i)}{\le} \left\|\Differential\vpred_t(M)\right\|_{\op} \cdot \betaloss\cdot\left\|\vpred_t(M)- \matvhat_t(M)\right\|_{\op} \\
		&\overset{(ii)}{\le} \left( 2\sqrt{m}\radGst\radnat\right) \cdot \betaloss  \cdot \left(4\radM\radnat\epsG \right )\\
		&= 8\betaloss \sqrt{m}\radGst \radM\radnat^2\epsG
		\end{align*}
		where $(i)$ uses smoothness of the loss, $(ii)$ uses \Cref{lem:norm_and_perturb_str_unknown}. To bound term $(b)$, we use the Lipschitzness from \Cref{asm:lip} to bound the norm of the gradient: 
		\begin{align*}
		\|(\nabla \ell)(\matvhat_t(M))\|_2 &\le L\max\{1,\|\matvhat_t(M)\|\} \le 6L\radGst\radM\radnat \quad \text{(by Lemma~\ref{lem:estimated_magnitude_bounds})}
		\end{align*}
		Hence, from \Cref{lem:norm_and_perturb_str_unknown},
		\begin{align*}
		\|(b)\|_{\op} &\le 6L\radGst\radM\radnat  \cdot \left\|\Differential(\vpred_t-\matvhat_t)\right\|_{\op}\\
		&\le 6L\radGst\radM\radnat  \cdot 2\sqrt{m}\epsG\radnat = 12\sqrt{m}L\radGst\radnat^2\radM \epsG
		\end{align*}
		Hence, we conclude that
		\begin{align*}
		\|\nabla \fpred_t(M) - \nabla f_t(M \midhattyu) \| \le \underbrace{\sqrt{m}\radGst\radM\radnat^2(8\betaloss + 12)}_{:=\Capprox}\cdot\epsG
		\end{align*}

		\end{proof}

	\subsubsection{Smoothness, Strong Convexity, Lipschitz  (Lemmas~\ref{lem:smoothness_stable_unknown},~\ref{lem:strong_cvx_unknown}/\ref{lem:strong_cvx_unknown_stabilized}, and~\ref{lem:lip_unknown_strongly_convex})\label{sssec:unknown_smooth_str_convex}}
		We begin by checking verifying the smoothness bound, which we recall from Section~\ref{ssec:fast_rate_unknown_rigorous}:
		\smoothunknown*
		\begin{proof} The proof follows by modifying \Cref{lem:smoothness_stablz}, replaced $\Differential \vst_t$ with $\Differential \vpred_t$. By \Cref{lem:norm_and_perturb_str_unknown} part (a), we bound the operator norm of these differentials by twice the corresponding bound in \Cref{lem:smoothness_stablz}, incurring an additional factor of four in the final result.
		\end{proof}
		Next, we check Lipschitznes:
		\lipunknowsc*
		\begin{proof}[Proof of \Cref{lem:lip_unknown_strongly_convex} ] We prove the general stabilized case. Recall that for the known-system setting, the losses $f_t(\midetaygst)$ and $F_t[\mid \midetaygst$ are $L_f := L\sqrt{m}\radnat^2\radGst^2\radM$-Lipschitz and $L_f$-coordinate Lipschitz, respectively. Under \Cref{cond:est_general}, we have that $\|\Ghatexyu\|_{\loneop} \le 2\radGst$, and moreover, by Lemma~\ref{lem:estimation_stability_generalized}, we have that $\twonorm{\natvhat_s},\twonorm{\natetahat_s} \le 2\radnat$ for all $s \in [t]$, Hence, repeating the  computation of the known-system Lipchitz constant in \Cref{lem:label_Lipschitz_constants_known}, but with inflated norms of $\natuhat_{1:t}$ adn $\natyhat_{1:t}$, we find that $\fpred_t$ (resp. $\Fpred_t$) are $\Lbar_f := 4L_f$-Lipschitz (resp. -coordinate Lipschitz).
\end{proof}

		Finally, we verify strong convexity in this setting. The following subsumes \Cref{lem:strong_cvx_unknown}:

		\begin{lemmod}{lem:strong_cvx_unknown}{b}[Strong Convexity: Unknown Stabilized System]\label{lem:strong_cvx_unknown_stabilized} Consider the stabilized setting, with $\alpha_{f;m}$ as in \Cref{prop:strong_convex_internal_controller}. Suppose further that the conditions of that proposition hold, and in addition,
		\begin{align*}
		\epsG \le \frac{1}{9\radM\radnat\radGst}\sqrt{\frac{\alpha_{f;m}}{m\alphaloss}}
		\end{align*}
		 Then, the functions are $\fpred_{t;k}$ are $\alpha_{f;m}/4$-strongly convex. Analogously, replacing $\alpha_{f;m}$ by $\alpha_f$ in the stable setting, the functions $\fpred_{t;k}$ are $\alpha_{f}/4$ strong convex for $\alpha_{f}$ as in \Cref{prop:strong_convex_no_internal_controller}.
		\end{lemmod}
		\begin{proof} Let us consider the stabilized case; the stable case is identical. \Cref{prop:strong_convex_internal_controller} (proved in {\Cref{ssec:prop:strong_convex_internal_controller}}) follows from \Cref{thm:str_cvx_main}, and an can be used to prove the following intermediate bound:
		\begin{align*}
			\Exp\left[\left\| \matv_{t}(M  \midetaygst) - \natv_t\right\|_2^2\mid \calF_{t-k}\right]  \ge \frac{\alpha_{f;m}}{\alphaloss} \|M\|_{\fro}^2, \quad \forall M  = (M^{[i]})_{i=0}^{m-1}
			\end{align*}
		To deduce our desired strong convexity bound, it suffices to show that  $M  = (M^{[i]})_{i=0}^{m-1}$,
		\begin{align*}
		\Exp[\|\vpred_t(M ) - \natv_t\|_2^2 \mid \calF_{t-k}] \ge \frac{\alpha_{f;m}}{4\alpha_{\loss}} \|M\|_{\fro}^2,
		\end{align*}
		with an additional slack factor of $1/2$. To begin, note the elementary inequality 
		\begin{align*}
		\|v + w\|_2^2 = \|v\|_2^2 + \|w\|_2^2 - 2\|v\|\|w\| &\ge \|v\|_2^2 + \|w\|_2^2 - 2( \frac{1}{2} \cdot \frac{1}{2}\|v\|_2^2 + \frac{1}{2}\cdot 2 \|w\|_2^2) \\
		&= \|v\|_2^2/2 - \|w\|_2^2
		\end{align*}
		This yields
		\begin{align*}
		&\Exp[\|\vpred_t(M ) - \natv_t\|_2^2 \mid \calF_{t-k}] \\
		&= \Exp[\|(\vpred_t(M ) - \matv_{t}(M  \midetaygst)  + \matv_{t}(M  \midetaygst)  - \natv_t\|_2^2 \mid \calF_{t-k}]\\
		&\ge \frac{1}{2}\Exp[\| \matv_{t}(M  \midetaygst)   - \natv_t\|_2^2 \mid \calF_{t-k}] \\
		&\qquad- \Exp[\|\vpred_t(M ) -  \matv_{t}(M  \midetaygst) \|_2^2 \mid \calF_{t-k}]\\
		&\ge \frac{\alpha_{f;m}\|M\|_{\fro}^2}{2\alpha_{\loss}} - \underbrace{\Exp[\|\vpred_t(M ) -  \matv_{t}(M  \midetaygst)  \|_2^2 \mid \calF_{t-k}]}_{(i)}.
		\end{align*}
		Moreover, by Lemma~\ref{lem:norm_and_perturb_str_unknown} part (b), we have
		\begin{align*}
		\|\vpred_t(M ) -  \matv_{t}(M  \midetaygst)\|_2^2 \le (3\radM\radnat\radGst\sqrt{m}\epsG)^2.
		\end{align*}
		Hence, if $\epsG \le \frac{\sqrt{\alpha_{f;m}/\alpha_{\loss}}}{9\radM\radnat\radGst\sqrt{m}}$, then the term $(i)$ is bounded by $(i) \le \frac{\alpha_{f;m}\|M\|_{\fro}^2}{4\alpha_{\loss}}$, which concludes the proof.
		\end{proof}

\subsection{Proof of Proposition~\ref{prop:Mapprox} (Approximation Error)\label{ssec:ynat_approx_error}}
We prove the proposition in the general stabilized setting, where assume the corresponding \Cref{cond:est_general} holds. Recall the set $\calMcomp := \calM(m_0,\radM/2)$, and consider a comparator 
\begin{align*}
\Mst \in \arg\inf_{M \in \calMcomp} \sum_{t=N+m+2h+1}^{T}\ell_t(\maty^{M}_t,\matu^{M}_t)
\end{align*}
We summarize the conditions of the \Cref{prop:Mapprox} as follows:
\begin{condition}[Conditions for \Cref{prop:Mapprox}]\label{cond:ynat_approx_error} We assume that (a) $\epsG \le 1/\radM$, (b) $m \ge 2\mcomp  + h$, and (c) $\psiGst(h+1) \le \radGst/T$. 
\end{condition}
Note that the first condition holds from from \Cref{cond:est_condition}/\ref{cond:est_general}., and the secnd two from the definition of the algorithm paramaters. The proof has two major steps. We begin with the following claim, which reduces the proof to controllng the differences $\|\uex_t(\Mapprox \mid \natyhat_{1:t})- \uex_t(\Mst \mid \naty_{1:t}) \|_2$ between algorithmic inputs on the $\natyhat$ sequence using $\Mapprox$, and on the $\naty$ sequence using $\Mst$:
	\begin{lemma}\label{lem:redux_uin_approx} We have the bound:
	\begin{align*}
	\text{\normalfont ($\natyhat$-approx error)} \le 3L\radGst^2 \radM \radnat \sum_{t=N+m+h+1}^T \max_{s=t-h}^t\|\uex_s(\Mapprox \mid \natyhat_{1:s})- \uex_s(\Mst \mid \naty_{1:s}) \|_2
	\end{align*}
	\end{lemma}
	The above lemma is proven in \Cref{sssec:redux_uin_approx}. 

	We will neglect the first $m_0 + 2h$ terms in the above sum. Specifically, defining $N_1 = N+m+3h+1 + m_0$, we have
	\begin{align*}
	\text{\normalfont ($\natyhat$-approx error)} &\le 3L\radGst^2 \radM \radnat \sum_{t=N_1}^T \max_{s=t-h}^t\|\uex_s(\Mapprox \mid \natetahat_{1:s})- \uex_s(\Mst \mid \nateta_{1:s}) \|_2 \\
	&+ 3L \radGst^2 \radM (m_0+2h)\max_{s = N+m+h+1}^{N_1} \|\uex_s(\Mapprox \mid \natetahat_{1:s})- \uex_s(\Mst \mid \nateta_{1:s}) \|_2.
	\end{align*}
	Moreover, by the triangle inequality and \Cref{lem:estimated_magnitude_bounds,lem:magnitude_bound_stabilized} 
	\begin{align*}
	\|\uex_s(\Mapprox \mid \natetahat_{1:s})- \uex_s(\Mst \mid \nateta_{1:s}) \|_2 \le 3 \radM\radnat,
	\end{align*}
	giving 
	\begin{align}
	\text{\normalfont ($\natyhat$-approx error)} &\le 3L\radGst^2 \radM \radnat \sum_{t=N_1}^T \max_{s=t-h}^t\|\uex_s(\Mapprox \mid \natetahat_{1:s})- \uex_s(\Mst \mid \nateta_{1:s}) \|_2 \nonumber\\
	&\quad+ 9L (m_0+2h)\radGst^2 \radM^2 \radnat. \label{eq:natyhat_approx_error}
	\end{align}
	
	We now turn to bounding these $\uex$ differences, which is the main source of difficulty in the proof of \Cref{prop:Mapprox}. The next lemma is proven in \Cref{sssec:bound_uin_approx}:
	\begin{lemma}\label{lem:uin_approx}
	Under \Cref{cond:ynat_approx_error}, there exists an $\Mapprox \in \calM(m,\radM)$, depending only on $\epsG$ and $\Mst$, such that for all $t \ge m+1$ and $\tau > 0$,
	\begin{multline}
	\|\uex_t(\Mst \mid \naty_{1:t}) - \uex_t(\Mapprox \mid \natyhat_{1:t}) \|_2\,\le  \\ 
	\underbrace{\Rubar\radM \psiGst(h+1)}_{\text{\normalfont (truncation term)}} + \underbrace{\radM^2 \epsG^2 \left(\frac{\radnat\radM + \tau^{-1}}{2}\right)}_{\text{\normalfont (estimation term)}}  + \underbrace{\frac{\tau}{2}\max_{j = t-\mcomp + 1-h}^{t} \left\|\uex_j( \matM_j  - \Mapprox) \midhateta[j])\right\|_{2}^2}_{\text{\normalfont (movement term)}} \label{eq:uin_approx_bound}
	\end{multline}
	\end{lemma}
	From the above lemma and \Cref{eq:natyhat_approx_error}, and reparametrizing $\tau \leftarrow 2\tau/3L \radGst^2 \radM \radnat$, and bounding $m_0 + 2h \le m$, have 
	\begin{align*}
	\text{\normalfont ($\natyhat$-approx error)} &\le  3 L  \cdot \radM^2 \radGst^2 \radnat \left( T\Rubar\psiGst(h+1) + 3m  + \radnat \radM T\epsG^2 \cdot \left(1+  L  \radGst^2\tau^{-1}\right) \right)\\
	&\quad + \tau \sum_{t=N_1}^T \max_{j = t-\mcomp + 1-h}^{t} \left\|\uex_j( \matM_j  - \Mapprox \midhateta[j])\right\|_{2}^2
	\end{align*}
	For $\psiGst(h+1) \le \radGst/T$,  the above simplifies to 
	\begin{align*}
	\text{\normalfont ($\natyhat$-approx error)} &\le  3  L\radM^3 \radGst^2 \radnat^2 T \epsG^2\left(1+  \frac{L  \radGst^2}{\tau}\right)  + 3L\radM^2 \radGst^2 \radnat (\Rubar \radGst + 3m)\\
	&\quad + \tau \sum_{t=N_1}^T \max_{j = t-\mcomp + 1-h}^{t} \left\|\uex_j( \matM_j  - \Mapprox \midhateta[j])\right\|_{2}^2
	\end{align*}

	Moreover, we can crudely bound
	\begin{align*}
	\sum_{t=N_1}^T \max_{j = t-\mcomp + 1-h}^{t} \left\|\uex_j( \matM_j  - \Mapprox \midhateta[j])\right\|_{2}^2 &\le  (m_0 + h)\sum_{t=N_1  - m_0 - h}^T  \left\|\uex_j( \matM_j  - \Mapprox \midhateta[j])\right\|_{2}^2 \\
	&=  (m_0 + h)\sum_{t=N + m + 2h + 1}^T  \left\|\uex_j( \matM_j  - \Mapprox \midhateta[j])\right\|_{2}^2 
	\end{align*}
	Thus, again reparametrizing $\tau \leftarrow \tau/(m_0+h)$, and bounding $m_0+h \le m$,
	\begin{align}
	\text{\normalfont ($\natyhat$-approx error)} &\le   3  L\radM^3 \radGst^2 \radnat^2 T \epsG^2\left(1+  \frac{Lm  \radGst^2}{\tau}\right)  + 3L\radM^2 \radGst^2 \radnat (\Rubar \radGst + 3m)\nonumber \\
	&\quad + \tau \sum_{t=N+m+h+1}^T  \left\|\uex_j( \matM_j  - \Mapprox \midhateta[j])\right\|_{2}^2 \label{eq:natyhat_hat_nearly_done}.
	\end{align}
	Finally, let us upper bound 
	\begin{align*}
	\left\|\uex_j( \matM_j  - \Mapprox \midhateta[j])\right\|_{2} &\le \max_{j = 0}^{m-1} \|\natetahat_j\|_2 \cdot \left\| \matM_j  - \Mapprox\right\|_{\loneop}\\ 
	&\le 2\radnat \cdot \left\| \matM_j  - \Mapprox\right\|_{\loneop} \tag{\Cref{lem:estimation_stability_generalized}}\\
	&\le 2\radnat \cdot \sqrt{m}\left\| \matM_j  - \Mapprox\right\|_{\fro} \tag{\Cref{lem:loneop_fro}}.
	\end{align*}
	Thus, $\tau \leftarrow \tau/4 \radnat^2 m$, we obtain
	\begin{align}
	\text{\normalfont ($\natyhat$-approx error)} &\le   3  L\radM^3 \radGst^2 \radnat^2 T \epsG^2\left(1+  \frac{L  m^2 \radnat^2\radGst^2}{\tau}\right)  + 3L\radM^2 \radGst^2 \radnat (\Rubar \radGst + 3m)\\
	&\quad + \tau \sum_{t=N+m+h+1}^T   \left\| \matM_j  - \Mapprox \right\|_{\fro}^2\nonumber.
	\end{align}
	Finally, let us crudely bound the abouve by 
	\begin{align}
	\text{\normalfont ($\natyhat$-approx error)} &\le  36 m^2 \radM^3  \radnat^4 \radGst^4 T\epsG \max\{L,L^2/\tau\}  + 9L\radM^2 \radGst^2 \radnat (\Rubar \radGst + m) \nonumber \\
	&\quad + \tau \sum_{t=N+m+h+1}^T   \left\| \matM_j  - \Mapprox \right\|_{\fro}^2\nonumber.
	\end{align}

	as needed.

	\qed

\subsubsection{Proof of \Cref{lem:redux_uin_approx} \label{sssec:redux_uin_approx}}
Let $\Mapprox,\Mst \in \calM(m,\radM)$, and recall the shorthand
\begin{align*}
\vst_t(M) := \matv_t(M \midetaygst), \quad \vpred_t(M) := \matv_t(M \mid \Gexyu,\natetahat_{1:t},\natv_t)
\end{align*}
Then,
\begin{align*}
&\left| \fpred_t(\Mapprox) - f_t(\Mst \midetaygst)\right| 
\,= \left| \ell_t( \vpred_t(\Mapprox)) - \ell_t(\vst_t(\Mst))\right|\\
&\quad\le L \max\{\underbrace{\|\vpred_t(\Mapprox)\|_2}_{(a)},\underbrace{\|\vst_t(\Mst)\|_2,1\}}_{(b)}\,\|\vpred_t(\Mapprox) - \vst_t(\Mst)\|_2.
\end{align*}

From \Cref{lem:magnitude_bound_stabilized}, we have $(b) \le 2\radGst\radM\radnat$. Moreover, combining with \Cref{lem:estimated_magnitude_bounds}, a similar argument lets us bound $(a) \le \radnat + 2\radGst\radM\radnat  = 3\radGst \radM \radnat$. Since these upper bounds are all assumed to be greater than one, 
\begin{align}
\left| \fpred_t(\Mapprox) - f_t(\Mst \midetaygst)\right| \le 3L\radGst \radM \radnat \cdot \|\vpred_t(\Mapprox) - \vst_t(\Mst)\|_2. \label{eq:fdiff_with_L}
\end{align}
Unfolding 
\begin{align*}
\|\vpred_t(\Mapprox) - \vst_t(\Mst)\|_2 &= \left\|\sum_{i=0}^h \Gexyu^{[i]}(\uex_{t-i}(\Mapprox \midhateta[t-i]) - \uex_{t-i}(\Mst \midetanat[t-i]))  \right\|_2\\
&\le \radGst\max_{i=0}^h\left\|\uex_{t-i}(\Mapprox \midhateta[t-i]) - \uex_{t-i}(\Mst \midetanat[t-i])) \right\|_2.
\end{align*}
Combining with \Cref{eq:fdiff_with_L} gives the bound. \qed
\subsubsection{Proof of \Cref{lem:uin_approx}\label{sssec:bound_uin_approx}}

	\newcommand{\umain}{\matu^{\mathrm{main}}}
	\newcommand{\utrunc}{\matu^{\mathrm{trunc}}}
	\newcommand{\uapprox}{\matu^{\mathrm{apprx}}}
	\newcommand{\umove}{\matu^{\mathrm{move}}}
	\newcommand{\uerr}{\matu^{\mathrm{err}}}
	
	For simplicity, let us use $\Gst,\Ghat$ for $\Gexeta,\Ghatexeta$. Since $\Mst \in \calM(m_0,\radM/2)$, we have $\Mst^{[i]} = 0$ for all $i \ge m_0$. Therefore, we can write 
	\begin{align}
	&\uex_t(\Mst \mid \naty_{1:t}) \\
	&= \sum_{s = t-\mcomp + 1 }^{t} \Mst^{[t-s]}\nateta_s \nonumber\\
	&= \sum_{s = t-\mcomp + 1}^{t} \Mst^{[t-s]}\left(\natetahat_s +  (\nateta_s - \natetahat_s) \right) \nonumber\\
	&= \sum_{s = t-\mcomp + 1}^{t} \Mst^{[t-s]}\left(\natetahat_s + \left(\left(\etalg_s - \sum_{j=1}^s \Gst^{[s-j]} \uexalg_j\right) - \left(\etalg_s - \sum_{j=s-h}^s \Ghat^{[j]} \uexalg_{j} \right)\right) \right) \nonumber\\
	&= \underbrace{\left(\sum_{s = t-\mcomp + 1}^{t} \Mst^{[t-s]}\left(\natetahat_s + \left(  \sum_{j=s-h}^s \left(\Ghat^{[j]} - \Gst^{[j]}\right) \uexalg_{j}\right)\right)\right)}_{:=\,\umain_t} \nonumber\\
	&\qquad- \underbrace{\sum_{s = t-\mcomp + 1}^{t} \Mst^{[t-s]}\sum_{1\le j < s-h} \Gst^{[s-j]}\uexalg_j}_{:=\,\utrunc_{t}} \label{eq:uex_first_decomp}
	\end{align}

	Here, $\utrunc_t$ is a lower order truncation term:
	\begin{claim} For $t \ge N+m+1$, we have that $\|\utrunc_t\|_2 \le \Rubar\radM \psiGst(h+1).$
	\end{claim}
	\begin{proof} We have that $\|\uexalg_j\|_2 \le \Rubar$ by \Cref{lem:estimation_stability_generalized}. This gives
	\begin{align*}
	\left\|\utrunc_{t}\right\| &= \left\|\sum_{s = \plusbrackpls{t-\mcomp}}^{t} \Mst^{[t-s]}\sum_{1\le j < s-h-1} \Gst^{[s-j]}\uexalg_j\right\| \\
	&\le \radM\Rubar\sum_{1\le j < s-h} \|\Gst^{[s-j]}\| ~\le~ \radM\Rubar \psiGst(h+1).
	\end{align*}
	\end{proof}
	To bound the dominant term $\umain_t$, we express $\ualg_j$ in terms of $\natetahat$ and the controller $\Mst$:
	\begin{align*}
	\ualg_j &= \uex_j(\matM_j \midhateta[j]) = \sum_{q = j-\mcomp + 1}^j\matM_j^{[j-q]}\, \natetahat_q \\
	&= \underbrace{\sum_{q = j-\mcomp + 1}^j\Mst^{[j-q]} \,\natetahat_q}_{(b)} + \sum_{q = j-m+1}^j(\matM_j^{[j-q]} - \Mst^{[j-q]}),
	\end{align*}
	where sum only over  $q \in \{\plusbrack{j-\mcomp},\dots,j\}$ in the bracketed term $(b)$  because $\Mst^{[n]} = 0$ for $n > \mcomp$ since $\Mst \in \Mclass(\mcomp,\radM/2)$. Next, the equalities
	\begin{align*}
	\uexalg_{j} &= \uex_j( \matM_j \midhateta[j]) \\
	&= \uex_j( \Mst \midhateta[j]) + \uex_j( \matM_j \midhateta[j])  - \uex_j( \Mst \midhateta[j]) \\
	&= \uex_j( \Mst \midhateta[j]) + (\uex_j( \matM_j  - \Mst \midhateta[j])\\
	&= \left(\sum_{q = j - \mcomp + 1}^j \Mst^{[j-q]}\natetahat_q\right) + \uex_j( \matM_j  - \Mst \midhateta[j]),
	\end{align*}
	and introducing the shorthand $\DelG^{[j]} := \Ghat^{[j]} - \Gst^{[j]}$, we can further develop 
	\begin{align}
	\umain_t &:= \sum_{s = t-\mcomp + 1}^{t} \Mst^{[t-s]}\left(\natyhat_s +  \sum_{j=s-h}^s \DelG^{[j]} \uexalg_{j}\right) \nonumber \\
	&:= \underbrace{\sum_{s = t-\mcomp + 1}^{t} \Mst^{[t-s]}\left(\natyhat_s +  \sum_{j=s-h}^s \sum_{q = j - \mcomp + 1}^j \DelG^{[j]} \Mst^{[j-q]} \natetahat_q\right)}_{:=\,\uapprox_t} \nonumber\\
	&+ \underbrace{\sum_{s = t-\mcomp + 1}^{t} \sum_{j=s-h}^s \sum_{q = j-m + 1}^j \Mst^{[t-s]} \DelG^{[j]}\uex_j( \matM_j  - \Mst \midhateta[j])}_{:=\,\uerr_t}. \label{eq:umain_decomp}
	\end{align}
	Here, the input $\uapprox_t$ is respresents the part of the input which can be represented as $\uapprox_t = \uex_t(\Mapprox \mid \natyhat_t)$ for some $\Mapprox \in \Mclass(m,\radM)$; the remaining error term, $\uerr_t$, will be bounded shortly thereafter.

	\begin{claim}[Existence of a good comparator]\label{lem:Mcomp}  Define the controller
		\begin{align*}
		\Mapprox^{[i]} = \Mst^{[i]}I_{i \le m_0 -1} + \sum_{a = 0}^{m_0 - 1} \sum_{b = 0}^{h}\sum_{c=0}^{m_0 - 1} \Mst^{[a]}\DelG^{[b]}\Mst^{[c]}\I_{a+b+c = i},
		\end{align*}
		which depends only of $\Mst$ and $\epsG$. Then,
		\begin{enumerate}
			\item We have the identity
		\begin{align}
		\uapprox_t = \uex_t(\Mapprox \mid\natetahat_{1:t}),\quad\forall t \ge N \label{eq:uapprox_id},
		\end{align}
		\item $\|\Mapprox - \Mst\|_{\loneop} \le \|\Mst\|_{\loneop}^2 \epsG \le \radM^2\epsG/4$.
		\item If $m \ge 2\mcomp - 1 + h$ and $\epsG \le \frac{1}{\radM}$ (as ensured by \Cref{cond:ynat_approx_error}), then $\Mapprox \in \Mclass(m,\radM)$
	\end{enumerate}
	\end{claim}
	\begin{proof}
		To verify \Cref{eq:uapprox_id},
		\begin{align*}
		\uapprox_t  &= \sum_{s=t - m_0+1}^t\Mst^{[t-s]}\natetahat_s + \sum_{s=t-m_0+1}^t \sum_{s=j-h}^s \sum_{q=j-m_0+1}^j \Mst^{[t-s]}\DelG^{[s-j]}\Mst^{[j-q]}\natetahat_q\\
		&= \sum_{i=0}^{m_0 - 1} \Mst^{[i]}\natetahat_{t-q} + \sum_{a = 0}^{m_0 - 1} \sum_{b = 0}^{h}\sum_{c=0}^{m_0 - 1} \Mst^{[a]}\DelG^{[b]}\Mst^{[c]}\natetahat_{t - (a+b+c)}\\
		&= \sum_{i=0}^{2(m_0 - 1)+h} \left(\Mst^{[i]}I_{i \le m_0 -1} + \sum_{a = 0}^{m_0 - 1} \sum_{b = 0}^{h}\sum_{c=0}^{m_0 - 1} \Mst^{[a]}\DelG^{[b]}\Mst^{[c]}\I_{a+b+c = i}\right)\natetahat_{t-i}\\
		&= \sum_{i=0}^{2(m_0 - 1)+h} \Mapprox^{[i]}\natetahat_{t-i},
		\end{align*}
		Next, since $\|\Mst\| \le \radM/2$,
		\begin{align*}
		\|\Mapprox - \Mst\|_{\loneop} \le \sum_{a = 0}^{m_0 - 1} \sum_{b = 0}^{h}\sum_{c=0}^{m_0 - 1} \|\Mst^{[a]}\|_{\op}\|\DelG^{[b]}\|_{\op}\|\Mst^{[c]}\|_{\op} \le \|\Mst\|_{\loneop}^2 \epsG \le \frac{\radM^2 \epsG}{4},
		\end{align*}
		which verifies point $2$. Therefore, for $\epsG \le 1/\radM$,
		\begin{align*}
		\|\Mapprox\| \le \|\Mst\|_{\loneop}+ \|\Mapprox - \Mst\|_{\loneop}  \le \frac{\radM}{2} + \frac{\radM^2 \eps}{4} \le \radM.
		\end{align*}
		Moreover, by assumption on $m \ge 2m_0 + h - 1$, we have $\Mapprox^{[i]} \ge 0$ for $i > m \ge m_0 2(m_0-1)+h$. 
	\end{proof}
	Lastly, we control the error term. We shall do this incrementally via two successive claims. First, we ``re-center'' $\uerr_t$ arround the comparator $\Mcomp$, rather than $\Mst$, and uses AM-GM to isolate terms $\|\matM_j - \Mapprox\|_{\fro}^2$:
	\begin{claim}  For $m \ge 2\mcomp - 1 + h$, the following bound holds for all $\tau > 0$
	\begin{align*}
	\|\uerr_t\|_2 \le \frac{\radnat\radM^3 \epsG^2}{2} + \frac{\radM^2 \epsG^2}{2\tau}  + \frac{\tau}{2}\max_{j = t-\mcomp + 1-h}^{t} \left\|\uex_j( \matM_j  - \Mapprox \midhateta[j])\right\|_{2}^2.
	\end{align*}
	\end{claim}
	\begin{proof}
	Using $\|\natetahat_q\|_2 \le 2\radnat$ (\Cref{lem:estimation_stability_generalized}) and $\|\Mst\|_{\loneop} \le \radM/2$ by assumption, 
	\begin{align*}
	&\|\uerr_t\|_2 = \left\|\sum_{s = t-\mcomp + 1}^{t} \sum_{j=s-h}^s \Mst^{[t-s]} \DelG^{[j]} \uex_j( \matM_j  - \Mst \midhateta[j])\right\|_2\\
	&\le \radM\epsG\max_{j = t-\mcomp + 1-h}^{t}\left\|\uex_j( \matM_j  - \Mst \midhateta[j])\right\|_{2}\\
	&\le \radM\epsG\max_{j = t-\mcomp + 1-h}^{t}\left(\left\|\uex_j( \Mst  - \Mapprox \midhateta[j])\right\|_{2} + \left\|\uex_j( \matM_j  - \Mapprox \midhateta[j])\right\|_{2}\right).
	\end{align*}
	Next, from (\Cref{lem:estimation_stability_generalized}) and \Cref{lem:Mcomp}, we can bound
	\begin{align*}
	\left\|\uex_j( \Mst  - \Mapprox \midhateta[j])\right\|_{2} &\le \|\Mst - \Mapprox\|_{\loneop} \max_{s \le j} \|\natetahat_s\|_2\\
	&\le 2\radnat\|\Mst - \Mapprox\|_{\loneop}\\
	&\le \frac{\radnat\radM^2 \epsG}{2}.
	\end{align*}
	This yield
	\begin{align*}
	\|\uerr_t\|_2 &\le \frac{\radnat\radM^3 \epsG^2}{2} + \radM \epsG \max_{j = t-\mcomp + 1-h}^{t} \left\|\uex_j( \matM_j  - \Mapprox \midhateta[j])\right\|_{2}\\
	&\le \frac{\radnat\radM^3 \epsG^2}{2} + \frac{\radM^2 \epsG^2}{2\tau}  + \frac{\tau}{2}\max_{j = t-\mcomp + 1-h}^{t} \left\|\uex_j( \matM_j  - \Mapprox \midhateta[j])\right\|_{2}^2\\
	&\le \radM^2 \epsG^2 \left(\frac{\radnat\radM + \tau^{-1}}{2}\right)  + \frac{\tau}{2}\max_{j = t-\mcomp + 1-h}^{t} \left\|\uex_j( \matM_j  - \Mapprox \midhateta[j])\right\|_{2}^2.
	\end{align*}
	\end{proof}
	Putting things together, we have that
	\begin{align*}
	&\| \uex_t(\Mst \mid \naty_{1:t}) - \uex_t(\Mapprox \mid \natyhat_{1:t}) \|_2 \\
	&\quad= \| \utrunc_t + \umain_t - \uex_t(\Mapprox \mid \natyhat_{1:t}) \|_2 \tag{\Cref{eq:uex_first_decomp}}\\ 
	&\quad= \|\utrunc_t +  \uerr_t +\uapprox_t   - \uex_t(\Mapprox \mid \natyhat_{1:t}) \|_2 \tag{\Cref{eq:umain_decomp}}\\ 
	&\quad= \|  \utrunc_t + \uerr_t   \|_2 \tag{by \Cref{lem:Mcomp}}\\ 
	&\quad\le \|\utrunc_t  \|_2 + \| \uerr_t \|_2  \\ 
	&\le \Rubar\radM \psiGst(h+1) + \radM^2 \epsG^2 \left(\frac{\radnat\radM + \tau^{-1}}{2}\right)   \\
	&\qquad+ \frac{\tau}{2}\max_{j = t-\mcomp + 1-h}^{t} \left\|\uex_j( \matM_j  - \Mapprox \midhateta[j])\right\|_{2}^2.\\
	&= \text{(RHS of~\Cref{eq:uin_approx_bound})},
	\end{align*}
	as needed. $\qed$

	
\newcommand{\dimz}{d_{\matz}}

	\newcommand{\alphaunder}{\underline{\alpha}}
\newcommand{\Gleql}{G_{ \ell}}
\newcommand{\Gexyucheck}{\check{G}_{\mathrm{ex}\to(y,u)}}
\newcommand{\Abfch}{\check{A}_{BF}}

\section{Establishing Strong Convexity\label{app:strong_convexity_check}}

This appendix is devoted to establishing strong convexity of the \drc{} and \drcex{} parameterizations under semi-adversarial noise, described by \Cref{asm:noise_covariance_mart} in the previous appendix. The organization is as follows:
\begin{enumerate}
	\item \Cref{subsec:strong_convexity_preliminaries} introduces the necessary preliminaries to state our bound, including the Markov operators of the dynamics that arise from an internal stabilizing controller, and the notion of the Z-transform. 
	\item \Cref{ssec:main_strong_convexity_bounds} presents \Cref{thm:str_cvx_main}, which describes the strong convexity of internally stabilized systems in terms of certain functionals of the Z-transforms of relevant Markov operators. Combining with \Cref{prop:str_cvx_bound} which characterzes the behavior of these functionals, we this section concludes with the proof of \Cref{prop:strong_convex_internal_controller}. This section then specializes this bounds for systems with internal controllers which are given by static feedback (\Cref{sssec:example:static_feedback}), and exact  observer-feedback (\Cref{sssec:example_OF}). 
	\item \Cref{ssec:proof:thm:str_cvx_main} adresses the proof of \Cref{thm:str_cvx_main}.
	\item \Cref{ssec:prop:strong_convex_no_internal_controller} establishes the proof of \Cref{prop:strong_convex_no_internal_controller}. It borrows one lemam from the proof of \Cref{thm:str_cvx_main}, but bypasses the Z-transform to establish bounds via elementary principles.
	\item \Cref{ssec:prop:str_cvx_bound} proves \Cref{prop:str_cvx_bound} via complex-analytic arguments. The focus is to obtain polynomial dependence in the horizon parameters, and no care is paid to specifying system-dependent constants.
\end{enumerate}

\subsection{Strong Convexity Preliminaries \label{subsec:strong_convexity_preliminaries}}

	\paragraph{Transfer Functions and Z-Transforms}

	The strong convexity modulus is most succintly described in the Fourier domain, where we work with Markov operators and their Z-tranfsorms. We recall the definition of an abstract Markov operator as follows:

	\transfuncdef*
	 We shall also use the notation 
	 \begin{align*}
	 G^\top = \Transfer(A,B,C,D)^\top  = \Transfer(A^\top,C^\top,B^\top,D^\top),
	 \end{align*} where $(A^\top,C^\top, B^\top ,D^\top)$ is commonly referred to as the adjoint system. For an abstract Markov operator $G$, its Z-transform is the following power series:
	\begin{definition}[Z-Transform] 
	For $G \in \scrG_{\dout \times \din}$,  the \emph{Z-transform} is the mapping from $\C \to \C^{\dout \times \din}$
	\begin{align*}
	\Gcheck(z): z\mapsto \sum_{i=0}^n G^{[i]} z^{-i} \quad 
	\end{align*}

	\end{definition}
	For finite-order linear dynamical systems, the Z-transform can be expressed in closed form via:
	\begin{lemma} If $G = \Transfer(A,B,C,D)$, then $\Gcheck(z) = D + C(zI - A)^{-1}B.$
	\end{lemma}

	\paragraph{Closed Loop Dynamics:} For stabilized systems, the relevant Markov operators that arise correspond to the closed-loop dyanics of the nominal system placed in feedback with the stabilizing controller $\pi_0$. From \Cref{lem:supy_stab}, we recall the operators $\Gexyu$ and $\Gweta$ which satisfy:
	\begin{align*}
	\begin{bmatrix}\yalg_t\\
		\ualg_t 
		\end{bmatrix} = \begin{bmatrix} \naty_t\\
		\natu_t \end{bmatrix} + \sum_{i=0}^{t-1} \Gexyu^{[i]} \uex_{t-i}
		\end{align*}
		and
		\begin{align*}
		\nateta_t = \sum_{i=1}^t \Gweta^{[i]} \begin{bmatrix}\matw_{t-i}\\
		\mate_{t-i}
		\end{bmatrix}.
		\end{align*}
	The Markov operators in terms of which we bound the strong convexity modulus are as follows:
	\begin{definition}[Markov Operators for Strong Convexity]\label{defn:str_convex_transfer_operators} Recall the Markov operators 
	\begin{align*}
		\Gexyu &:=  \Transfer(\Apinotcl,\Bpinotclex,\Cpinotcl,\Dpinotclex)\\
		\Gweta &:=  \Transfer(\Apinotcl,\Bpinotcl,\Cpinotcleta,\Dpinotcleta)
		\end{align*}
		from \Cref{defn:markov_stabilized}. We define the noise transfer function as
	\begin{align*}
	\Gnoise &= \Gweta \Sigmanoise^{\frac{1}{2}} \in \scrG^{\dimeta \times (\dimx + \dimu) },
	\end{align*}
	where the above notation is short hand for $\Gnoise^{[i]} = \Gweta^{[i]} \Sigmanoise^{\frac{1}{2}}$ for all $i$. Note that $\Gexyu $ has $\din = \dimu$ and $\dout = \dimy + \dimu$, whereas $\Gnoise^\top$ has $\din = \dimy$ and $\dout = \dimx + \dimu$.
	We further define the function $\psi_{\Gexyu}$ and $\psi_{\Gnoise}$ denote the corresponding decay functions, which are proper by \Cref{a:stablize}.
	\end{definition}
	 Here, $\Gexyu$ describes the dependence of $(\maty,\matu)$ on exogenous inputs $\uex$, and $\Gnoise$ is the transpose of the system which describes the effect that the noise in the system has on natures $\naty_t$. Since $\uex_t(M)$ is linear in natures $y$, $\Gnoise$ needs to be sufficiently well conditioned (in a sense we will describe) to ensure strong convexity.  Note that $\Sigmanoise$ above need not be full-covariance, provided that it satisfies \Cref{asm:noise_covariance_mart}.
	Moreover, since $f_t(M)$ depends on $\uex_{s}(M)$ via the Markov operator $\Gexyu$, this operator also needs to be sufficiently well conditioned. \footnote{In the full observation setting, with controllers depending directly on noise $\matw_t$, \cite{agarwal2019logarithmic} only needs to verify that (the appropriate equivalent of) $\Gexyu$ is well conditioned, since the noise terms $\matw_t$ are independent by assumption. }

\subsection{Internally Stabilized Strong Convexity and Proof of \Cref{prop:strong_convex_internal_controller} \label{ssec:main_strong_convexity_bounds}}
	The relevant strong convexity parameter is bounded most precisely in terms of what we call ``$\calH$'' functions, which describe the behavior of the Z-transform $\Gcheck(z)$ of a Markov operator along the torus: $\Torus := \{e^{\iota \theta} \mid \theta \in [0,2\pi]\}$:

	\begin{definition}[$\Hmin$-Functional]\label{defn:functionals} Let $\din \le \dout$\footnote{The restriction $\din \le \dout$ is to remind the reader that, if $\din > \dout$, then $\hminnorm{G}$ is identically zero.}, $G \in \scrGdinout$. We define the $\Hmin$ and $\Hinf$ functionals as 
	\begin{align*}
	 \hminnorm{G} &:= \min_{\theta \in [0,2\pi]}\sigma_{\din}(\Gcheck(\eithe)) \quad \text{ and } \quad \hinfnorm{G} := \max_{\theta \in [0,2\pi]}\|\Gcheck(\eithe)\|_{\op}.
	\end{align*}
	\end{definition}
	 We will show that for $h,k$ sufficiently large, the strong convexity parameter is lower bounded by $ \gtrsim \hminnorm{\Gexyu}\cdot \hminnorm{\Gnoise}$. Unfortunately, for certain pathological systems, one or both of these terms may vanish. To ensure fast rates for \emph{all} systems, we will need a more refined notion:
	\begin{definition}\label{defn:h_h_functional}
	Let $\din \le \dout$, $G \in \scrGdinout$, and let $\omega = (\omega^{[i]})_{i \ge 0}$ denote elements of $\scrGdin := \scrG^{1 \times \din}$, with $\ltwonorm{\omega}^2 := \sum_{i\ge 0}\|\omega^{[i]}\|_2^2$ and Z-transform $\omegacheck$. Further, define $\Omegah := \{\omega \in \scrGdin: \|\omega\|_{\ell_2} = 1,\,\omega^{[i]} = 0,\, \forall i > h \}$. We define the $\calH_{[h]}$-functional as
	\begin{align*}
	\omnorm{G}{h}^2 &:= \min_{\omega \in \Omegah} \frac{1}{2\pi}\int_{0}^{2\pi} \|\Gcheck(\eithe)\omegacheck(\eithe)\|_{2}^2\dtheta.
	\end{align*}
	\end{definition}

	Abusing notation, we also will write $ \hminnorm{\Gcheck}$ and other relevant functionals as a function of the Z-transform, where convenient. We also note that, just as $\sigma_{\min}(\cdot)$ is not a norm,  $\hminnorm{\cdot}$ and $\omnorm{\cdot}{h}$ are not norms as well. 

	Having defined the relevant functions, the following bound gives us a precise bound on the relevant strong convexity parameter. We consider the functions
	\begin{align*}
	f_{t;k}(M \midstyu) = \Exp[ f_t(M \midstyu) \mid \calF_{t-k}],
	\end{align*}
	where $\calF_{t}$ is the filtration from \Cref{asm:noise_covariance_mart}. In what follows, we will adopt the shorthand $f_{t;k}(M \midstyu)$. Our main theorem is as follows:
	\begin{theorem}\label{thm:str_cvx_main} Fix, $m,h$ and let $k = m+2h$. Further, define
	\begin{align*}
	\alpha_{m,h} &=  \frac{1}{2}\cdot\omnorm{\Gexyu}{m}^2 \cdot\omnorm{\Gnoise^\top}{m+h}^2\\
	&\ge  \frac{1}{2}\cdot\hminnorm{\Gexyu}^2 \cdot\hminnorm{\Gnoise^\top}^2 := \alpha_{\infty},
	\end{align*}
	Then
	\begin{align*}
	\Exp\left[\left\| \matv_{t}(M  \midetaygst) - \natv_t\right\|_2^2\mid \calF_{t-k}\right] \ge \alpha_{m,h,k} \|M\|_{\fro}^2\ge \alpha_{\infty} \|M\|_{\fro}^2,
	\end{align*}
	provided that
	\begin{align}
	 \psi_{\Gexyu}(h)  \le \frac{(\omhnorm{\Gexyu})}{8(m+h)}, \quad \text{and} \quad \Psi_{\Gnoise}(h) \le \frac{\omnorm{\Gnoise}{m+h-1}}{2(m+h)} \label{eq:str_cvx_G_condition}.
	\end{align}
	Thus, if each $\ell_t$ is chosen by an oblivious adversary and is $\alpha$-strongly convex, the functions $f_{t;k}(M)$ are $\alphaloss \cdot \alpha_{m,h,k}$ and $\alphaloss \cdot \alpha_{\infty}$strongly-convex, provided that \Cref{eq:str_cvx_G_condition} holds.
	\end{theorem}
	The above theorem is proved in \Cref{ssec:proof:thm:str_cvx_main}. Some remarks are in order:
	\begin{enumerate}
		\item While $m,h$ are algorithm parameters, $k$ appears only in the analysis. The constraints on $h$ reflects how the $m$-history long inputs $\uex_t(M)$ must be given time to propogate through the system, and the constrain $k \ge m+h$ reflects the sufficient excitation required from past noise to ensure the last $m+h$ natures y's are well conditioned. 
		\item As we shall show in \Cref{prop:str_cvx_bound}, the functional $\omnorm{G}{h}$ decays polynomially as a function of $h$. On the other hand, decay functions decay geometrically, so these constrains on $m,h,k$ can always be satisfied for $h$ and $k$ sufficiently large.
		\item We consider $\Gnoise^\top$ to insure that the input dimension is greater than output dimension, as per the restriction in~\ref{defn:functionals}.
	\end{enumerate}

	\Cref{sssec:example:static_feedback} provides a transparent lower bound on $\alpha_{\infty}$ when the system is stabilized by an static feedback controller.  For general controllers, however, $\hminnorm{\Gexyu}$ may be equal to zero. We introduce the following condition. 
	However, we can show that $\omhnorm{G}$ degrades at most polynomially in $h$:
	\begin{proposition}\label{prop:str_cvx_bound} Let $\din \le \dout$ and $G = \Transfer(A,B,C,D) \in \scrGdinout$, with $\sigma_{\din}(D) > 0$, or more generally, that $G$ is Then, there exists constants $c,n$ depending only on $G$ such that, for all $h \ge 0$, $\omhnorm{G} \ge c/(h+1)^{n}$.
	\end{proposition}
	We are now in a place to prove our intended proposition:
	\subsubsection{Proof of \Cref{prop:strong_convex_internal_controller} \label{ssec:prop:strong_convex_internal_controller}}
	Recall the settings $h = \floor{m/3}$, and $k = m + 2h$. First, we lower bound $\alpha_{m,h}$. From \Cref{prop:str_cvx_bound}, there exists constances $c_1,c_2 > 0$ and $n_1,n_2$ such that $\omnorm{\Gexyu}{m} \ge c_1m^{-n_1}$ $\omnorm{\Gnoise^\top}{m+h} \ge c_2(m+h)^{-n_2} \ge \frac{c_2}{2^{n_2}}m^{n_2}$, since $h \le m$. Thus, there exists somes $\alphasc > 0$ and $\rhosc $ such that $\alpha_{m,h} \ge \alphasc m^{\rhosc}$.

	Now, let us show that there exist an $m \ge \msc$ for which conditions of \Cref{thm:str_cvx_main} hold. From the stability assumption of the stabilized system (\Cref{a:stablize}), there exists constants $C > 0$ and $\rho \in (0,1)$ for which $\psi_{\Gexyu}(n) \vee \psi_{\Gnoise}(n) \le C\rho^n$. Thus, for $m \ge 4$ and $h = \floor{m/3} \ge m/4$, we have
	\begin{align*}
	(\frac{(\omhnorm{\Gexyu})}{8(m+h)}) \cdot \psi_{\Gexyu}(h)  \le \frac{C\rho^{\floor{m/3}}}{8(m+\floor{m/3}) c_1\floor{m/3}^{-n_1}},
	\end{align*}
	which is at most $1$ for all $m$ sufficiently large. A similar argume,tm applies to checking the bound $\Psi_{\Gnoise}(h) \le \frac{\omnorm{\Gnoise}{m+h-1}}{2(m+h)}$. \qed

	\subsubsection{Example: Static Feedback Controllers \label{sssec:example:static_feedback}}
		Consider the static feedback setting (\Cref{exm:static_feedback}), where we have a stabilizing controller with  $A_{\pi_0} = 0$, and $\mateta_t = \maty_t$. For consistency with conventiona notational, we set $F = D_{\pi_0} \in \R^{\dimu \dimy}$. This includes the full observation setting via the laws $(\Ast+\Bst F)$, but may also include settings with partial observation which admit a matrix $F$ such that $(\Ast+\Bst K\Cst)$ is stable: Note that taking $K = 0$ subsumes full-feedback as well. The proposition shows that $\alpha_{\infty}$ from \Cref{thm:str_cvx_main} admits a transparent lower bound:
		\begin{proposition}\label{prop:static_feedback_alpha} Consider a static feedback controller with $K = D_{\pi_0}$, and recall $\alpha_{\infty} := \frac{1}{2}\cdot\hminnorm{\Gexyu}^2 \cdot\hminnorm{\Gnoise^\top}^2$. Then,
		\begin{enumerate}
			\item $\Sigmanoise \succeq \sigma^2 I$, then $\alpha_{\infty} \ge \frac{\sigma^2}{32} \min\left\{1,\|K\|_{\op}^{-2}\right\}\cdot\min\left\{1,\|\Bst K\|_{\op}^{-2}\right\}.$
		\item If only $\sigmaw^2 > 0$ (but $\Sigmanoise$ may not be positive definite),  then 
		\begin{align*}
		\alpha_{\infty} \ge \frac{\sigmaw^2}{16} \min\left\{1,\|K\|_{\op}^{-2}\right\}\cdot \frac{\sigma_{\min}(\Cst)^2}{(1+\|A_K\|_{\op})^2}
		\end{align*}
		\item  Finally, if $K  = 0$, then 
		\begin{align*}
		\alpha_{\infty} \ge \frac{\sigmae^2}{2}+ \frac{\sigmaw^2}{2}\frac{\sigma_{\min}(\Cst)}{(1+\|\Ast\|_{\op})^2}.\end{align*}
		\end{enumerate}
		\end{proposition}
		Note that if $\sigma_{\min}(\Cst) > 0$, then we only need $\sigmaw^2 > 0$ to ensure $\alpha_{\infty} > 0$.  In particular, with with $\Cst = I$, we  recover the bounds from \cite{agarwal2019logarithmic}, even with stabilizing feedback. Note that, unlike \cite{agarwal2019logarithmic}, these bounds don't require any assumptions on the system, or any approximate diagonalizability.\footnote{However, to conclude strong convexity via \Cref{thm:str_cvx_main}, we require $h - m > 0$. Still, we note that these bounds apply to more general settings where one has (a) observation noise and (b) partial observation.} It order to illustrate how useful it is to the represent strong convexity in terms of Z-transform and $\calH$-functionals, we provide a proof of the above proposition
		\begin{proof}[Proof of \Cref{prop:static_feedback_alpha}]
		In static feedback, we have a stabilizing controller with  $A_{\pi_0} = 0$, and set $D_{\pi_0} = K$ and $A_K = \Ast + \Bst K \Cst$, and $\Achk(z) = (zI - A_K)^{-1}$. Then, we can verify
		\begin{align*}
		\Gexyu^{[i]} = \I_{i = 0}\begin{bmatrix} 0 \\
		I
		\end{bmatrix} + \I_{i > 0}\begin{bmatrix} \Cst \\
		K\Cst
		\end{bmatrix} A_K^{i-1} \Bst, \quad \Gweta^{[i]} = \I_{i = 0}\begin{bmatrix} 0  & I
		\end{bmatrix} + \I_{i > 0}\Cst A_K^{i-1} \begin{bmatrix} 0  & K
		\end{bmatrix}
		\end{align*}

		Thus,
			\begin{align*}
			\Gexyucheck(z) = \begin{bmatrix} \Cst \Achk(z) \Bst \\
			I + K \Cst \Achk(z) \Bst
			\end{bmatrix}, \quad
			\Gnoisech(z)^\top = \Sigmanoise^{1/2}\begin{bmatrix}\Achk(z)^{\top}\Cst^\top \\
		 	I + \Bst^\top K^\top \Achk(z) ^{\top}\Cst^\top
			\end{bmatrix}.
			\end{align*}

			We now invoke a simple lemma:
		\begin{lemma}\label{lem:lem_simple_well_conditioned}
			Consider a matrix of the form $W = \begin{bsmallmatrix}
			YZ\\
			 I + X Z
			 \end{bsmallmatrix} \in \R^{(d_1 + d) \times d}.$, with $Y \in \R^{d_1 \times d_1}$, $X,Z^\top \in \R^{d \times d_1}$. Then, $\sigma_{\min}(W) \ge \frac{1}{2}\min\{1, \frac{\sigma_{\min}(Y)}{\|X\|_{\op}}\}$.  
			\end{lemma}
			\begin{proof}[Proof of Lemma~\ref{lem:lem_simple_well_conditioned}] Consider $\|Wv\|_{2}$ for $v \in \R^d$ with $\|v\| = 1$. If $\|Zv\|_2 \le 1/2\|X\|_{\op}$, then 
			\begin{align*}
			\|Wv\|_2 \ge \|I + XZv\|_{2} \ge 1 - \|X\|_{\op}\|Zv\|_{2} \ge 1 - \frac{1}{2} = \frac{1}{2}.
			\end{align*}
			Otherwise, $\|Wv\|_2 \ge \|YZv\| \ge \sigma_{\min}(Y)\cdot \|Zv\| \ge \frac{\sigma_{\min}(Y)}{2\|X\|_{\op}}$
			\end{proof}

			By Lemma~\ref{lem:lem_simple_well_conditioned}, we see that 
			\begin{align*}
			\hminnorm{\Gexyu} \ge \frac{1}{2}\min\{1,\|K\|_{\op}^{-1}\},
			\end{align*}
			and if $\Sigmanoise \succeq \sigma^2 I$, then 
			\begin{align*}
			\hminnorm{\Gnoise^\top}\ge \frac{\sigma}{2}\min\{1,\|\Bst K\|_{\op}^{-1}\}.
			\end{align*} 
			This establishes the first result of the Proposition. Moreover, if we just have state noise $\sigma_w^2$ but possibly no observation noise, then $\hminnorm{\Gnoise} \ge \sigma_{\min}(\Cst)\hminnorm{\Achk} \ge \frac{\sigma_{\min}(\Cst)}{1+\|A_J\|_{\op}}$, where we note that $\sigma_{\min}(\Achk(z)) = \frac{1}{\|\Achk(z)^{-1}\|_{\op}} = \frac{1}{1 + \|zI - A_K\|_{\op}}$, which is at least $\frac{1}{1+\|A_K\|_{\op}}$ for $z \in \Torus$. 

			Lastly, when $K = 0$, we can direclty lower bound $\hminnorm{\Gexyu} \ge 1$, and  lower bound 
			$\hminnorm{\Gnoise}^2 \ge \sigmaw^2 \sigma_{\min}(\Cst)^2\hminnorm{\Achk}^2 + \hminnorm{I + \Bst^\top K^\top \Achk(z)^{\top}\Cst^\top}^2\sigmae^2$  By specializing $F = 0$ in the argument adopted for the previous part of the proposition, $\hminnorm{\I + \Bst^\top K^\top \Achk(z)^{\top}\Cst^\top} \ge \frac{\sigma_{\min}(\Cst)}{1+\|\Ast\|_{\op}}$, and by setting $K = 0$, $\hminnorm{I + \Bst^\top K^\top \Achk(z)^{\top}\Cst^\top}^2 = I$.
		\end{proof}

	\subsubsection{Example: Youla LDC-Ex with Exact Observer Feedback \label{sssec:example_OF}}
		In general, static feedback is not sufficient to stabilize a partially observed linear dynamic system. Let us consider what arises from the the Youla LDC-Ex parametrization. From \Cref{lem:youla_facts}, we have
	\begin{align*}
	\Gweta^{[i]} = \I_{i = 0} \begin{bmatrix} 0 & I_{\dimy} \end{bmatrix} + \I_{i > 0}\Cst (\Ast + L \Cst)^{i-1} \begin{bmatrix} I_{\dimx}  & F \end{bmatrix}.
	\end{align*}
	and
	\begin{align*}
	\Gexyu^{[i]} = \I_{i = 0} \begin{bmatrix} 0 \\ I_{\dimu} \end{bmatrix} + \I_{i > 0}\begin{bmatrix}\Cst \\ F \end{bmatrix} (\Ast + \Bst F)^{i-1} \Bst.
	\end{align*}
	Thus, introducing $\Abfch(z) = (zI - (A+BK))^{-1}$, and $\Alcch(z) = (zI - (A-LC))^{-1}$, we have
	\begin{align*}
	 \Gexyucheck(z) = \begin{bmatrix} 0 \\ I_{\dimu} \end{bmatrix} + \begin{bmatrix}\Cst \\ F \end{bmatrix} \Abfch(z) \Bst = \begin{bmatrix}\Cst \Abfch(z) \Bst \\ I_{\dimu} + F \Abfch(z) \Bst  \end{bmatrix} 
	\end{align*}
	Moreover,
	\begin{align*}
	\check{G}_{(w,e)\to\eta}(z) = \begin{bmatrix} 0 & I_{\dimy} \end{bmatrix} + \Cst \Alcch(z) \begin{bmatrix} I_{\dimx}  & F \end{bmatrix}, 
	\end{align*}
	giving
	\begin{align*}
	\Gnoisech(z) = \Sigmanoise^{\frac{1}{2}} \begin{bmatrix} \Alcch(z)^\top \Cst^\top  \\
	  I_{\dimy} +F^\top \Alcch(z)^\top \Cst^\top  \end{bmatrix}.
	\end{align*}
	From \Cref{lem:lem_simple_well_conditioned}, we have 
	\begin{align*}
	\hminnorm{\Gnoisech(z)} \ge \min\left\{1, \frac{1}{2\|F\|_{\op}}\right\}.
	\end{align*}
	Lower bounding $\Gexyucheck(z)$ is a little trickier. Define $X(z) = zI - \Ast$. We have
	\begin{align*}
	I_{\dimu} + F \Abfch(z) \Bst &= I + F(zI - \Ast - \Bst F)^{-1} \Bst  = I + F(X(z) - \Bst F)^{-1}\Bst\\
	&=  I + F(X(z) - \Bst F)^{-1}\Bst\\
	&=  -( (-I) - F(X(z) + \Bst (-I) F)^{-1}\Bst)\\
	&=  -( -I  + F X(z)^{-1}\Bst)^{-1}.
	\end{align*}
	Then, 
	\begin{align*}
	\sigma_{\min}(I_{\dimu} + F \Abfch(z) \Bst) = \frac{1}{\| -I  + F X(z)^{-1}\Bst\|_{\op} } \le \frac{1}{1 + \|F\|_{\op}\|\Bst\|_{\op}\|X(z)^{-1}\|_{\op} }.
	\end{align*}
	Substituting $X(z) = (zI - \Ast)$, we have
	\begin{align*}
	\hminnorm{\Gexyucheck} \ge \frac{1}{1 + \|F\|_{\op}\|\Bst\|_{\op} \max_{z \in \Torus}\|(zI - \Ast)^{-1}\|_{\op} }
	\end{align*}
	In otherwise, if the eigenvalues of $\Ast$ are bounded away from $1$ in magnitude, then $\hminnorm{\Gexyucheck}  > 0$, yielding a bound of $\alpha_{\infty} > 0$.

\subsection{Proof of \Cref{thm:str_cvx_main}\label{ssec:proof:thm:str_cvx_main}}
	The proof of \Cref{thm:str_cvx_main} proceeds by first representing the strong convexity in terms of the Toeplitz operator defined below:
	\begin{definition}\label{defn:ToepOp}Let $k,\ell,h \in \N$ with $k \ge h \ge 0$, and $\ell > 0$. Given a Markov operator $G \in \scrGdinout$ with $G \in \R^{\dout \times \din}$, let $\Gleql \in \R^{\dout \times \din}$ denote the Markov operator with $\Gleql^{[i]} = \I_{i \le \ell}G^{[i]}$. The \emph{Toeplitz} operator is defined by 
	\begin{align*}
	\Toep_{h;k,\ell}(G) := \begin{bmatrix}  \Gleql^{[0]} & 0 & 0 & \dots & 0\\
	\Gleql^{[1]} & \Gleql^{[0]} & 0 & \dots & 0\\
	\dots &\dots& \dots & \dots & \dots \\
	\Gleql^{[h]} & \Gleql^{[h-1]} & \dots & \Gleql^{[1]} &\Gleql^{[0]}\\
	\Gleql^{[h+1]} & \Gleql^{[h]} & \dots & \Gleql^{[1]} &\Gleql^{[1]}\\
	\dots &\dots& \dots & \dots & \dots \\
	\Gleql^{[k]} & \Gleql^{[k-1]} & \dots & \dots &  \Gleql^{[k-h]}
	\end{bmatrix} \in \R^{(k+1)\dout \times (h+1)\din}
	\end{align*}
	We use the shorthand $\Toep_{h;k}(G) = \Toep_{h;k,k}(G)$.
	\end{definition}

	Our first lemma establishes strong convexity in terms of the above operator:
	\begin{lemma}\label{lem:strong_convex_from_Toep} For any $M = (M^{[i]})_{i = 0}^{m-1}$, we have the bound
	\begin{align*}
	\Exp\left[\left\| \matv_{t}(M  \midetaygst) - \natv_t\right\|_2 ^2\mid \calF_{t-k}\right] \ge \alphaunder_{m,h,k} \|M\|_{\fro}^2
	\end{align*}
	 where $\|M\|_{\fro}^2 := \sum_{i=0}^{m-1}\|M^{[i]}\|_{\fro}^2$, and
	\begin{align*}
	\alphaunder_{m,h,k} =  \sigma_{\dimu m}(\Toep_{m-1;m+h-1,h}(\Gexyu))^2  \cdot \sigma_{\dimy(m+h)}(\Toep_{m+h-1;k}(\Gnoise^\top))^2 
	\end{align*}
	\end{lemma}

	Next, we show that the smallest singular value of a Toeplitz operators is lower bounded by the $\omhnorm{G}$
	\begin{lemma}\label{lem:main_transfer_lb} Let $G  \in \scrGdinout$ be a Markov operator, which in particular means $\|G\|_{\loneop} < \infty$. Further, let $\ell,k,h \in \N$, with $\ell \ge 1$, and $k \ge h \ge 1$. Finally, set $c_1 = c_2$ if $\ell \ge k$, and otherwise, let $c_1 = (1-\tau)$, $c_2 = \frac{1}{\tau}-1$ for some $\tau > 0$.  Then,
	\begin{align*}
	\sigma_{\din(h+1)}(\Toep_{h;k,\ell}(G))^2 &\ge c_1  \omhnorm{G}^2 -   c_2(h+1)^2 \psi_{G}(k-h \vee \ell)^2,
	\end{align*}
	where $\omhnorm{G}$ is as in \Cref{defn:h_h_functional}.
	\end{lemma}
	The above lemma is proved in \Cref{sec:proof:lem:main_transfer_lb}. 
	\Cref{thm:str_cvx_main} noq follows readily:
	\begin{proof}[Proof of \Cref{thm:str_cvx_main} ] From  \Cref{lem:strong_convex_from_Toep}, the functions $f_{t;k}(M) = \Exp[ f_t(M) \mid \calF_{t-k}]$ are $\alphaloss \cdot \alphaunder_{m,h,k}$ strongly-convex, where
	\begin{align*}
	\alphaunder_{m,h,k} =  \sigma_{\dimu m}(\Toep_{m-1;m+h-1,h}(\Gexyu))^2  \cdot \sigma_{\dimy(m+h)}(\Toep_{m+h-1;k}(\Gnoise^\top))^2 
	\end{align*}
	Applying \Cref{lem:main_transfer_lb} with $\tau = \frac{1}{8}$, 
	\begin{align*}
	\sigma_{\dimu m}(\Toep_{m-1;m+h-1,h}(\Gexyu))^2 \ge \frac{7}{8} \omhnorm{\Gexyu}^2 -   7 (h+1)^2 \psi_{\Gexyu}(h)^2
	\end{align*}. 
	Taking 
	\begin{align*}\psi_{\Gexyu}(h) \le \frac{h+1(\omhnorm{\Gexyu})^2}{8},
	\end{align*}
	we obtain 
	\begin{align*}\sigma_{\dimu m}(\Toep_{m-1;m+h-1,h}(\Gexyu))^2 \ge \frac{3}{4} \omhnorm{\Gexyu}^2 
	\end{align*}.  Further, applying \Cref{lem:main_transfer_lb} with $\ell = k$, we obtain 
	\begin{align*}\sigma_{\dimu (m+h)}(\Toep_{m+h-1;k}(\Gnoise))^2 \ge \omnorm{\Gnoise}{m+h-1}^2 - (m+h-1)^2 \Psi_{\Gnoise}(k - (m+h-1))^2.\end{align*} 

	Taking $k = m+2h$, it suffices that $\Psi_{\Gnoise}(h)^2 \le \frac{\omnorm{\Gnoise}{m+h-1}}{2(m+h)}$, we obtain $\sigma_{\dimu (m+h)}(\Toep_{m+h-1;k}(\Gnoise))^2 \ge \frac{3}{4}\omnorm{\Gnoise}{m+h}^2$. Thus, 
	\begin{align*}
	\alphaunder_{m,h,k} \ge   \frac{3}{4} \omhnorm{\Gexyu}^2  \cdot \frac{3}{4}\omnorm{\Gnoise}{m+h-1}^2 \ge \frac{1}{2}\omhnorm{\Gexyu}^2\omnorm{\Gnoise}{m+h-1}^2 := \alpha_{m,h,k}.
	\end{align*}
	Since $\omhnorm{G} \ge \hminnorm{G}$, we conclude that $\alpha_{m,h} \ge \alpha_{\infty}$. Therefore,
	\begin{align}\label{eq:norm_lb}
	\Exp\left[\left\| \matv_{t}(M  \midetaygst) - \natv_t\right\|_2^2\mid \calF_{t-k}\right] \ge \alpha_{m,h,k} \|M\|_{\fro}^2\ge \alpha_{\infty} \|M\|_{\fro}^2
	\end{align}
	Finally, we observe that if $f(z) = \ell(X z + v)$ is a function with a random variable $X$ and $\alpha$-strongly convex loss $\ell$, then $\Exp[f( z)]$ is $\alpha \cdot \alpha'$ strongly convex as long as $\Exp[\|Xz\|_2^2] \ge \alpha ' \|z\|_2^2$ for all $z$. This means that \Cref{eq:norm_lb} entails that $f_{t;k}$ is both $\alphaloss \cdot \alpha_{m,h,k}$
	and $\alphaloss \cdot \alpha_{\infty}$-strongly convex.

	\end{proof}

\subsubsection{Proof of \Cref{lem:strong_convex_from_Toep}\label{sec:proof:strong_convex_from_Toep}}
	For simplicity, we assume that $M \in \Mclass(m+1,\radM)$; this simplifies the indexing. Further, introduce the row-toeptliz operator
	\begin{align*}
	\ToepRow_{h}(G) :=  \begin{bmatrix} G^{[0]} & G^{[1]} &\dots  &  G^{[h]} \end{bmatrix}.
	\end{align*}
	
	\newcommand{\delmatv}{\delta\matv}
	Further, lets us introduce the shorthand 
	\begin{align*}
	\delmatv_t(M) = \matv_t(M \midetaygst) - \natv_t, \quad \uex_t(M) := \uex_t(M \midetanat).
	\end{align*}
	 We can directly check that
	\begin{align*}
	\delmatv_{t}(M)
	 &=  \ToepRow_{h}(\Gexyu)\begin{bmatrix} \uex_{t}(M) \\
	\uex_{t-1}(M)\\
	\dots\\
	\uex_{t-h}(M)
	\end{bmatrix}.
	\end{align*}
	Moreover,
	\begin{align*}
	\begin{bmatrix} \uex_{t}(M) \\
	\uex_{t-1}(M)\\
	\dots \\
	\uex_{t-h}(M)\\
	\end{bmatrix} &= \begin{bmatrix} M^{[0]} & M^{[1]} &\dots & M^{[m]} & 0 & \dots \\
	0 & M^{[0]} & \dots & M^{[m-1]} & M^{[m]} & \dots \\
	\dots & \dots & \dots & \dots & \\
	\dots & \dots & M^{[0]} & M^{[1]} & \dots  & M^{[m]}\\
	\end{bmatrix}\begin{bmatrix} \nateta_{t}\\
	\nateta_{t-1} \\
	\dots\\
	\nateta_{t-(m+h)}
	\end{bmatrix} \\
	&= \Toep_{h;m+h}(M)\begin{bmatrix} \nateta_{t}\\
	\nateta_{t-1} \\
	\dots\\
	\nateta_{t-(m+h)}
	\end{bmatrix}.
	\end{align*}
	\newcommand{\natEta}{\mathbf{N}^{\mathrm{nat}}}
	Letting $\natEta_{t:t-(m+h)}$ denote the vector above, this us gives the compact representation:
	\begin{align*}
	\delmatv_t(M) = \ToepRow_{h}(\Gexyu)\Toep_{h;m+h}(M)\natEta_{t:t-(m+h)}.
	\end{align*}
	Recall that, to establish the lemma, we wish to lower bound
	\begin{align*}
	\Exp[\|\delmatv_t(M)\|_2^2 \mid \calF_{t-k}] \ge \alphaunder_{h,m,k} \|M\|_{\fro}^2,
	\end{align*}
	where $\|M\|_{\fro}^2 := \sum_{i=0}^{m-1}\|M^{[i]}\|_{\fro}^2$. To this end, define the random variable 
	\begin{align}
	\natEta_{t;t-(m+h);k} := \natEta_{t;t-(m+h)} - \Exp[\natEta_{t;t-(m+h)} \mid \calF_{t-k}]. \label{eq:Nateta_mid_k}
	\end{align}
	Since $\natEta_{t;(t-(m+h);k}$ is uncorrelated with $\Exp[\natEta_{t;(t-(m+h)} \mid \calF_{t-k}]$, we have
	\begin{align*}
	&\Exp\left[\left\|\delmatv_t(M)\right\|_2^2 \mid \calF_{t-k}\right]\\
	&=\Exp\left[\left\|\ToepRow_{h}(\Gexyu)\Toep_{h;m+h}(M)\natEta_{t:t-(m+h)}\right\|_{2}^2 \mid \calF_{t-k}\right] \\
	&= \Exp\left[\left\|\ToepRow_{h}(\Gexyu)\Toep_{h;m+h}(M)\natEta_{t:t-(m+h);l}\right\|_{2}^2 \mid \calF_{t-k}\right] \\
	&\qquad+ \left\|\ToepRow_{h}(\Gexyu)\Toep_{h;m+h}(M) \cdot \Exp\left[\natEta_{t:t-(m+h);k} \mid \calF_{t-k}\right]\right\|_{2}^2 \tag{By uncorrelation}\\
	&\ge \Exp\left[\left\|\ToepRow_{h}(\Gexyu)\Toep_{h;m+h}(M)\natEta_{t:t-(m+h);k}\right\|_{2}^2 \mid \calF_{t-k}\right]\\
	&= \left\|\ToepRow_{h}(\Gexyu)\Toep_{h;m+h}(M)\right\|_{2}^2\\
	&\qquad\times \sigma_{\dimy(1+m+h)}\left(\Exp\left[\natEta_{t:t-(m+h);k}(\natEta_{t:t-(m+h);k})^\top  \mid \calF_{t-k}\right]\right).
	\end{align*}

	Thus, to conclude the proof, it suffices to establish that, for $\natEta_{t;t-(m+h);k}$ defined in \Cref{eq:Nateta_mid_k}, we have
	\begin{align}
	&\sigma_{\dimy(1+m+h)}\left(\Exp\left[\natEta_{t:t-(m+h);k}(\natEta_{t:t-(m+h);k})^\top  \mid \calF_{t-k}\right]\right) \ge \sigma_{\dimy(m+1+h)}(\Toep_{m+h;k}(\Gnoise))^2  \label{eq:Gnoise_wts}\\
	\intertext{and}
	&\left\|\ToepRow_{h}(\Gexyu)\Toep_{h;m+h}(M)\right\|_{\fro}^2\ge \sigma_{\dimu(m+1)}(\Toep_{m;m+h,h}(\Gexyu))^2 	\label{eq:Gstuuy_wts}
	 \end{align}
	 Note that $m$ in the above display in fact corresponds to $m-1$ in the statement of the lemma, since for the proof we assume $M \in \calM(m+1,R)$ to simplify indices.

	 Let us now establish both equations in the above display. 
	 \paragraph{Proving \Cref{eq:Gnoise_wts}} Recall that $\Gweta$ denotes the Markov operator mapping disturbances to outputs, which satisfies by \Cref{lem:supy_stab} the following
	 \begin{align*}
	\nateta_t = \sum_{i=1}^t \Gweta^{[i]} \begin{bmatrix}\matw_{t-i}\\
		\mate_{t-i}
		\end{bmatrix}.
	\end{align*}
	Then, since $\natEta_{t;t-(m+h)}$ is the component of nature's y's depending only on noises $(\matw_{s},\mate_s)$ for $s \in \{t-k,t-k+1,\dots,t\}$, we deduce:
	\begin{align*}
	\natEta_{t:t-(m+h);k} = \ToepTrans_{m+h;k}(\Gweta)\begin{bmatrix} \begin{bmatrix}\estoch_t \\ \wstoch_t \end{bmatrix} \\
	\begin{bmatrix}\estoch_{t-1} \\ \wstoch_{t-1} \end{bmatrix} \\
	\dots\\
	\begin{bmatrix}\estoch_{t-k} \\ \wstoch_{t-k} \end{bmatrix}
	\end{bmatrix},
	\end{align*}
	where we have defined the Toeplitz Transpose operator
	\begin{align*}
	\ToepTrans_{h;k}(G) &:= \begin{bmatrix} G^{[0]} & G^{[1]} & G^{[2]}  & \dots & G^{[k-h]} & \dots & G^{[k]}\\
	0 & G^{[0]} & G^{[1]} & \dots & G^{[k-h-1]} & \dots & G^{[k-1]}\\
	\dots &\dots& \dots & \dots & \dots &  \dots\\
	\dots & \dots &\dots  & 0 & G^{[0]} & \dots & G^{[k-h]}.
	\end{bmatrix}
	\end{align*}
	Thus, letting $\Diag_{k+1}(\Sigmanoise)$ denote a block diagonal matrix with $\Sigmanoise$ along the diagonal, we have from \Cref{asm:noise_covariance_mart}  
	\begin{align*}
	&\Exp\left[\natEta_{t:t-(m+h);k}(\natEta_{t:t-(m+h);k})^\top \mid \calF_{t-k}\right] \\
	&= \ToepTrans_{m+h;k}\left(\Gweta\right)\Diag_{k+1}(\Sigmanoise)\ToepTrans_{m+h;k}\left(\Gweta\right)^\top\\
	&= \ToepTrans_{m+h;k}\left(\Gweta\cdot\Sigmanoise^{\frac{1}{2}}\right) \ToepTrans_{m+h;k}\left(\Gweta\cdot\Sigmanoise^{\frac{1}{2}}\right)^\top\\
	&= \ToepTrans_{m+h;k}\left(\Gnoise\right) \ToepTrans_{m+h;k}(\Gnoise)^\top,
	\end{align*}
	where we use the convention $\Gweta\cdot\Sigmanoise^{\frac{1}{2}}$ denotes the Markov operator whose $i$-th component is $\Gweta^{[i]}\Sigmanoise^{\frac{1}{2}}$, and recall the definition  $\Gnoise :=  \Gweta\cdot\Sigmanoise^{\frac{1}{2}}$ (\Cref{defn:str_convex_transfer_operators}).

	The following fact is straightforward:
	\begin{claim}\label{fact:Toep_Trans} For all  $\sigma_{d}(\ToepTrans_{h;k}(G)) = \sigma_d(\Toep_{h;k}(G^\top))$ for all $d$.
	\end{claim}
	Thus, combining the above with \Cref{fact:Toep_Trans},
	\begin{align*}
	\sigma_{\dimy(1+m+h)}\left(\Exp[\natEta_{t:t-(m+h);k}(\natEta_{t:t-(m+h);k})^\top  \mid \calF_{t-k}]\right) &\ge \sigma_{\dimy(1+m+h)}\left(\ToepTrans_{m+h;k}(\Gnoise^\top)\right)^2   \\
	&= \sigma_{\dimy(1+m+h)}\left(\Toep_{m+h;k}(\Gnoise^\top)\right)^2,
	\end{align*}
	concluding the proof of \Cref{eq:Gnoise_wts}. 

	\paragraph{Proof of \Cref{eq:Gstuuy_wts}} This bound is a direct consequence of the following claim, which thereby concludes the proof of \Cref{lem:strong_convex_from_Toep}.
	\begin{claim} Suppose that $G$ and $M$ are of conformable shapes. Then,
	\begin{align*}
	\|\ToepRow_{h}(G)\Toep_{h;m+h}(M)\|_{\fro} \ge \sigma_{\dimu(m+1)}(\Toep_{m;m+h,h}(G^\top))  \fronorm{\begin{bmatrix} M^{[0]}\\
	M^{[1]} \\
	\dots \\
	M^{[m]}\end{bmatrix}} 
	\end{align*}
	\end{claim}
	\begin{proof} Keeping the convention $M^{[i]} = 0$ for $i > m$, we can write
	\begin{align*}
	&\ToepRow_{h}(G)\Toep_{h;m+h}(M) \\
	&\quad=\begin{bmatrix} G^{[0]} M^{[0]}\mid G^{[0]} M^{[1]} + G^{[1]} M^{[0]} \mid  \dots\mid  \sum_{i=0}^h G^{[i]}M^{[h-i]} \mid \sum_{i=0}^h G^{[i]}M^{[1+h-i]}\mid \dots \mid \sum_{i=0}^h G^{[i]}M^{[m+h-i]}\end{bmatrix}.\\
	&\quad = \begin{bmatrix}\symbN^{[0]} \mid \symbN^{[1]} \mid \dots & \symbN^{[m+h]} \end{bmatrix}, \quad \text{where } \symbN^{[j]} := \sum_{i=0}^{m+h} \I_{i \le h} G^{[i]} M^{[j-i]}.
	\end{align*}
	The above block-row matrix has Frobenius norm equal to that of the following block-column matrix,
	\begin{align*}
	\overline{\symbN} := \begin{bmatrix} \symbN^{[0]} \\ \symbN^{[1]}  \\  \dots \\ \symbN^{[m+h]} 
	\end{bmatrix} = \begin{bmatrix} \sum_{i=0}^0 \I_{i \le h} G^{[i]} M^{[-i]} \\ \sum_{i=0}^1\I_{i \le h} G^{[i]} M^{[1-i]}  \\  \dots \\ \sum_{i=0}^{m+h} \I_{i \le h} G^{[i]} M^{[m+h-i]}.
	\end{bmatrix} 
	\end{align*}
	which can be expressed as the product
	\begin{align*}
	\mathscr{N} = \begin{bmatrix} G^{[0]} & 0 &\dots  & 0 & 0 & 0 & \dots & 0 & \dots \\
	G^{[1]} & G^{[0]} & 0 &\dots  &  0 & 0 & \dots & \dots & 0 \\
	\dots &\dots& \dots & \dots & \dots &  \dots & \dots & \dots & \dots\\
	G^{[h]} &G^{[h-1]}& \dots & \dots & \dots & G^{[0]} & 0 &  \dots & \dots & 0 \\
	0 & G^{[h]} &G^{[h-1]}  & \dots & \dots & \dots & G^{[0]} & 0 & \dots & 0   \\
	\dots &\dots& \dots & \dots & \dots &  \dots & \dots \\
	0 & \dots & 0  &G^{[h]}  & G^{[h-1]} & \dots &  \dots & \dots & \dots & G^{[0]}  \\
	\end{bmatrix} \cdot \begin{bmatrix} M^{[0]}\\
	M^{[1]} \\
	\dots \\
	M^{[m]}\\
	0 \\
	\dots\\
	0
	\end{bmatrix},
	\end{align*}
	where we have use $M^{[i]} = 0$ for $i > m$. Let us denote the first $m$ column blocks of the above matrix as $X_m$. 
	\newcommand{\Gleqh}{G_h}
	Letting $\Gleqh^{[i]} = \I_{i \le h} \Gleqh^{[i]}$, we have 
	\begin{align*}
	X_m := \begin{bmatrix} \Gleqh^{[0]} & 0 &\dots  & 0 & 0 & 0 \\
	\Gleqh^{[1]} & \Gleqh^{[0]} & 0 &\dots  &  0 & \\
	\dots &\dots& \dots & \dots & \dots &  \\
	\Gleqh^{[m]} &\Gleqh^{[m-1]}& \dots & \dots & \dots & \Gleqh^{[0]} \\
	\dots &\dots& \dots & \dots & \dots &  \dots \\
	\Gleqh^{[h+m]} & \Gleqh^{[m+h-1]}& \dots & \dots & \dots & \Gleqh^{[h]} \\
	\end{bmatrix}
	\end{align*}
	Then, we have that $\overline{\symbN} = X_m  \begin{bmatrix} M^{[0]}\\
	M^{[1]} \\
	\dots \\
	M^{[m]}\end{bmatrix}$
	and in particular, 
	\begin{align*}
	\fronorm{\overline{\symbN}} \ge \sigma_{\dimu(m+1)}(X_m)  \fronorm{\begin{bmatrix} M^{[0]}\\
	M^{[1]} \\
	\dots \\
	M^{[m]}\end{bmatrix}}
	\end{align*}
	To conclude, we recognize $X_m$ are the matrix $\Toep_{m;m+h,h}(G)$, so that 
	\begin{align*}
	\fronorm{\overline{\symbN}} \ge \sigma_{\dimu(m+1)}(\Toep_{m;m+h,h}(G))  \fronorm{\begin{bmatrix} M^{[0]}\\
	M^{[1]} \\
	\dots \\
	M^{[m]}\end{bmatrix}}
	\end{align*}
	\end{proof}

\subsubsection{Proof of \Cref{lem:main_transfer_lb} \label{sec:proof:lem:main_transfer_lb}}

	Let $G = (G^{[i]})_{i \ge 0}$ be a Markov operator. We define its z-series as the series
	\begin{align*}
	\Gcheck(z) := \sum_{i \ge 0}  z^{-i} \cdot G^{[i]}.
	\end{align*}
	Our goal is to prove a lower bound on $\sigma_{\min}(\Toep_{h;k}(G))$. Introduce a ``signal'' $\{\omega^{[i]}\}_{i \ge 0}$, and annotate the $\ell_2$-ball $\Omegah := \{\omega: \sum_{i=0}^{h} \|\omega^{[i]}\|_{2}^2 \le 1, \omega^{[i]} = 0, \forall {}i > h\}$. Let us introduce the convention $G^{[i]} = 0$ for $i < 0$. Then, we can express
	\begin{align}
	\sigma_{\min}(\Toep_{h;k,\ell}(G))^2 &= \min_{\omega \in \Omegah}\left\|\Toep_{h;k,\ell}(G) (\omega^{[0]},\dots,\omega^{[h]})\right\|_2^2\nonumber\\
	&= \min_{\omega \in \Omegah} \sum_{i=0}^k \left\|\sum_{j=0}^h  \I_{(i-j) \le \ell}G^{[i-j]} \omega^{[j]}\right\|_2^2\label{eq:sigma_min}
	\end{align} 
	Let us first pass to the $k,\ell \to \infty$ limit. 
	\begin{lemma}\label{lem:infinite_index} 
	Let $c_1 = c_2$ if $\ell \ge k$, and otherwise, let $c_1 = (1-\tau)$, $c_2 = \frac{1}{\tau}-1$ for some $\tau > 0$.  Then, for $\|G\|_{\loneop} < \infty$, we have 
	\begin{align*}
	\sigma_{\min}(\Toep_{h;k,\ell}(G))  \ge c_1\min_{\omega \in \Omegah}\sum_{i \in \Z}\left\|\sum_{j \in \Z}\I_{(i-j) \le \ell} G^{[i-j]}\omega^{[j]}\right\|_2^2 -  c_2(h+1)^2 \psi_{G}(k-h)^2,
	\end{align*}
	where if $\ell \ge k$, $c_1,c_2$ sa $c_1 = c_2 = 1$.

	$c_1,c_2$ satisfy either , or, 
	\end{lemma}
	\begin{proof} 
	For $\omega \in \Omegah$, we have
	\begin{align*}
	\left\|\Toep_{h;k,\ell}(G) (\omega^{[0]},\dots,\omega^{[h]})\right\|_2^2 &= \sum_{i=0}^k \left\|\sum_{j=0}^h  \I_{(i-j) \le \ell}G^{[i-j]} \omega^{[j]}\right\|_2^2\\
	&= \sum_{i=0}^k \left\|\sum_{j=0}^h  (G^{[i-j]} - \I_{(i-j) > \ell}G^{[i-j]}) \omega^{[j]}\right\|_2^2\\
	&= \sum_{i=0}^{\infty} \left\|\sum_{j=0}^h  (G^{[i-j]} - \I_{(i-j) > \ell \text{ or } i > k}G^{[i-j]}) \omega^{[j]}\right\|_2^2.
	\end{align*}
	where in the last line we use that $\ell \le k$. Let us introduce the shorthand $\I_{i,j} := \I_{(i-j) > \ell \text{or} i > k}$. Using the elementary vector inequality $\|v+w\|_2^2 \ge (1 - \tau)\|v\|_2^2 + (1-\frac{1}{\tau})\|w\|_2^2$. This gives
	\begin{align*}
	\left\|\Toep_{h;k,\ell}(G) (\omega^{[0]},\dots,\omega^{[h]})\right\|_2^2 \ge c_1\sum_{i=0}^{\infty} \left\|\sum_{j=0}^h  G^{[i-j]}  \omega^{[j]}\right\|_2^2 - c_2 \sum_{i=0}^{\infty} \left\|\sum_{j=0}^h  \I_{i,j} G^{[i-j]}  \omega^{[j]}\right\|_2^2,
	\end{align*}
	where $c_1 = (1-\tau)$ and $c_2 = (\frac{1}{\tau}-1)$, and where all sums converge due to $\|G\|_{\loneop} < 1$. Moreover, one can see that if $\ell \ge k$, then we can simplify the above argument and take $c_1 = c_2 = 1$, as in this case
	\begin{align*}
	\sum_{i=0}^k \left\|\sum_{j=0}^h  (G^{[i-j]} - \I_{(i-j) > \ell}G^{[i-j]}) \omega^{[j]}\right\|_2^2 = \sum_{i=0}^{\infty} \left\|\sum_{j=0}^h  G^{[i-j]}  \omega^{[j]}\right\|_2^2 -  \sum_{i=0}^{\infty} \left\|\sum_{j=0}^h  \I_{i \ge k } G^{[i-j]}  \omega^{[j]}\right\|_2^2
	\end{align*}

	Observe that since $\omega$ has  $\ell_2$-norm bounded by $1$, we have 
	\begin{align*}
	\|\sum_{j=0}^h   \I_{i,j} G^{[i-j]}\omega^{[j]}\|_2^2 \le \opnorm{\begin{bmatrix} \I_{i,j}G^{[i]}] \mid \dots \mid \I_{(i-j) > \ell} G^{[i-h]} \end{bmatrix}}^2 \le \sum_{j=0}^h \opnorm{\I_{i,j} G^{[i-j]}}^2
	\end{align*}
	Thus, 
	\begin{align*}
	\sum_{i=0}^{\infty} \left\|\sum_{j=0}^h  \I_{i,j} G^{[i-j]}  \omega^{[j]}\right\|_2^2 &\le \sum_{i=0}^{\infty} \sum_{j=0}^h \opnorm{\I_{i,j} G^{[i-j]}}^2 \le \left(\sum_{i = 0}^{\infty} \sum_{j=0}^h \opnorm{\I_{i,j} G^{[i-j]}}\right)^2\\
	&\le (h+1)\left(\sum_{i \ge 0} \opnorm{\I_{i \le \min\{\ell,k-h\}} G^{[i-j]}}\right)^2 ~= (h+1)^2 \psi_{G}( \ell \wedge k-h)^2.
	\end{align*}
	Thus, we conclude that, for any $\tau > 0$,
	\begin{align*}
	\left\|\Toep_{h;k,\ell}(G) (\omega^{[0]},\dots,\omega^{[h]})\right\|_2^2 \ge (1-\tau)\sum_{i=0}^{\infty} \left\|\sum_{j=0}^h  G^{[i-j]}  \omega^{[j]}\right\|_2^2 + (1-\frac{1}{\tau}) (h+1)^2 \psi_{G}( \min\{\ell,k-h\})^2.
	\end{align*}
	Finally, since $\omega^{[j]} = 0$ for $j \notin \{0,\dots,h\}$, and $G^{[i-j]} = 0$ for $j \ge 0$ and $i < 0$, we can pass to a double-sum over all indices $i, j \in \Z$.
	\end{proof}
	Next, for each $\omega \in \Omegah$, we introduce
	\begin{align*}
	G_{*\omega}^{[i]} := (G*\omega)^{[i]} = \sum_{j \in \Z}  G^{[i-j]}\omega^{[j]},
	\end{align*}
	and let $\Gcheck_{*\omega}[z]$ denote its Z-transform. By the convolution theorem,
	\begin{align*}
	\Gcheck_{*\omega}(z) = \Gcheck(z)\omegacheck(z),
	\end{align*}
	where $\omegacheck(z) = \sum_{i\in \Z}\omega^{[i]} z^i$ is the Z-transform induced by $\omega$. Moreover, by parseval's indentity, 
	\begin{align*}
	\sum_{i \in \Z}\|G_{*\omega}^{[i]}\|_2^2 &= \frac{1}{2\pi}\int_{0}^{2\pi}\|\Gcheck_{*\omega}(e^{\iota \theta})\|_{2}^2 \rm\dtheta  =\frac{1}{2\pi}\int_{i}^{2\pi} \| \Gcheck(e^{\iota \theta})\omegacheck(e^{\iota \theta}\|_2^2\dtheta = \omhnorm{G}^2
	\end{align*}
	Therefore, for $c_1,c_2$ as in  Lemma~\ref{lem:infinite_index}, 
	\begin{align*}
	\sigma_{\min}(\Toep_{h;k}(G))  &\ge c_1 \min_{\omega \in \Omegah}\sum_{i \in \Z}\left\|\sum_{j \in \Z}  G^{[i-j]}\omega^{[j]}\right\|_2^2 -  c_2(h+1)^2 \psi_{G}(\ell \wedge k-h)^2\\
	&= c_1 \min_{\omega \in \Omegah} \sum_{i \in \Z}\|G_{*\omega}^{[i]}\|_2^2 -   c_2 (h+1)^2 \psi_{G}(\ell \wedge k-h)^2\\
	&= c_1\omhnorm{G}^2 -   c_2 (h+1)^2 \psi_{G}(\ell \wedge k-h)^2.
	\end{align*}

\subsection{Proof of \Cref{prop:strong_convex_no_internal_controller} (Strong Convexity for the Stable Case)\label{ssec:prop:strong_convex_no_internal_controller}}

For \emph{stable systems} -- that is, systems without a stabilizing controller -- we can directly lower bound bound the strong convexity without passing to the Z-transform. This has the advantage of not requiring the conditions on $\Psi_{\Gexyu}$ and $\Psi_{\Gnoise}$ stipulated by \Cref{thm:str_cvx_main}. As our starting bound, we recall from \Cref{lem:strong_convex_from_Toep} the bound that $f_{t;k}(M)$ are $\alphaloss \cdot \alphaunder_{m,h,k}$ strongly-convex, where
	\begin{align*}
	\alphaunder_{m,h,k} =  \sigma_{\dimu m}(\Toep_{m-1;m+h-1,h}(\Gexyu))^2  \cdot \sigma_{\dimy(m+h)}(\Toep_{m+h-1;k}(\Gnoise^\top))^2
	\end{align*}
	The followign lemma bounds the quantities in the above display, directly implying \Cref{prop:strong_convex_no_internal_controller}:
	\begin{lemma}\label{lem:stable_toep_lb} For any $m \ge 1$ and any $k \ge m+h$, we have that $\sigma_{\dimu(h+1)}(\Toep_{h;m+h-1}(\Gexyu)) \ge 1$, and $\sigma_{\dimu(m+h)}(\Toep_{m+h-1;k}(\Gnoise))^2 \ge \sigmae^2 + \frac{\sigmaw^2\sigma_{\min}(\Cst)^2}{(1+\|\Ast\|_{\op})^2}$
	\end{lemma} 

\begin{proof}

	In the stable case, we have that 
	\begin{align*}
	\Gexyu^{[i]} = \begin{bmatrix} \Gst^{[i]} \\
	 I_{\dimu} \I_{i = 0}.
	  \end{bmatrix}
	\end{align*}
	Thus, for $m \ge 1$, $\Toep_{m-1;m+h-1,h-1}(\Gexyu))$ can be repartioned so as to contain a submatrix $I_{ \dimu m \times \dimu m}$. Thus, $\sigma_{\dimu m}(\Toep_{h;m+h-1}(\Gexyu)) \ge 1$. 

	To lower bound $\sigma_{\dimy(m+h)}(\Toep_{m+h-1;k}(\Gnoise)^\top)^2$, we use the diagonal covariance lower bound $\Sigmanoise = \begin{bmatrix} \sigmaw^2 I_{\dimx} & 0\\
	0 & \sigmae^2 I_{\dimy} \end{bmatrix}$. For this covariance, $\Gnoise$ takes the bform
	\begin{align*}
	\Gnoise^{[i]} = \begin{bmatrix} \I_{i \ge 1}  \cdot (\Cst A^{[i-1]})^\top \sigmaw \\ \I_{i = 0} \cdot I_{\dimu}  \sigmae. 
	\end{bmatrix}
	\end{align*}
	Thus, $\Toep_{m+h-1;k}(\Gnoise)$ can be partitioned as a two-block row matrix, where one block is an 
	\begin{align*} X = 
	\begin{bmatrix}
	 \sigmae \cdot I_{(h+m)\dimu}  \\ 
	\mathbf{0}_{(k - (m+h-1))\dimu \times (m+h)\dimu}\\
	\sigmaw  \Toep_{m+h-1;k}(\Gstw^\top),
	\end{bmatrix}
	\end{align*}
	where $\Gstw = \I_{i \ge 1}\Cst \Ast^{i-1}$. From this structure (and use the short hand )
	\begin{align*}
	&\sigma_{\dimy(m+h)}(\Toep_{m+h-1;k}(\Gnoise)^\top)^2 \\
	&= \sigma_{(m+h)\dimy}^2(X) \\
	&= \sigma_{(m+h)\dimy}( \sigmae^2 I_{(h+m)\dimy}  + \sigmaw^2 \Toep_{m+h-1;k}(\Gstw^\top)^\top \Toep_{m+h-1;k}(\Gstw^\top)) \\
	&\ge \sigmae^2  + \sigmaw^2\cdot  \sigma_{(m+h)\dimy}(\Toep_{m+h-1;k}(\Gstw^\top))^2
	\end{align*}
	It remains to lower bound $\sigma_{(m+h)\dimu}(\Toep_{m+h-1;k}(\Gstw^\top))^2$. We can recognize that, for $k \ge m+h$, $\Toep_{m+h-1;k}(\Gstw^\top)$ has an $(m+h)\dimy \times (m+h)\dimy$ submatrix which takes the form 
	\begin{align*}
	\left(\Diag(\underbrace{\Cst,\dots,\Cst}_{\text{$m+h$ times}} ) \cdot \powtoep_{m+h}(\Ast) \right),
	\end{align*}
	where we have defined
	\begin{align*}
	\powtoep_p(A) := \begin{bmatrix} I & A & A^2 & \dots & A^{p-1}\\
		0 & I & A & \dots & A^{p-2} \\
		\dots \\
		0 & 0 & 0 & \dots & I
		\end{bmatrix}
		\end{align*}
	Thus, 
	\begin{align*}
	\sigma_{(m+h)\dimy}(\Toep_{m+h-1;k}(\Gstw^\top)) &\ge \sigma_{\min}(\Cst) \cdot \sigma_{\min}(\powtoep_{m+h}(A))\\
	&= \sigma_{\min}(\Cst) \cdot \|\powtoep_{m+h}(A)^{-1}\|_{\op}
	\end{align*}
	We can verify by direct computation that 
	\begin{align*}
	\powtoep(A)^{-1} = \begin{bmatrix} I & -A & 0 & 0 & \dots & 0\\
	0 & I & - A &  0 & \dots & 0\\
	\dots \\
	0 & 0 & \dots &0& 0 & I\end{bmatrix},
	\end{align*}
	giving $\|\powtoep_{m+h}(A)^{-1}\|_{\op} \le 1 + \|A\|_{\op}$. Thus, $\sigma_{(m+h)\dimy}(\Toep_{m+h-1;k}(\Gstw^\top)) \ge \frac{\sigma_{\min}(\Cst)}{1+\|A\|_{\op}}$, yielding $\sigma_{\dimy(m+h)}(\Toep_{m+h-1;k}(\Gnoise)^\top)^2 \ge \sigmae^2  + \sigmaw^2 \left(\frac{\sigma_{\min}(\Cst)}{1+\|A\|_{\op}}\right)^{-1}$, as needed.
\end{proof}

\subsection{Proof of \Cref{prop:str_cvx_bound}\label{ssec:prop:str_cvx_bound}}
	The proof of all subsequence lemmas are provided in sequence at the end of the section. As we continue, set $\Torus := \{e^{\iota \theta}: \theta \in [2\pi]\}\subset \C$. For $\omega \in \scrGdin$, we the following signal norms:
	\begin{definition}[Signal norms] For a complex function $p(z): z \to  \C^{d}$, we define
	\begin{align*}
	\|p\|_{\Htwo}^2 := \frac{1}{2\pi}\int_{0}^{2\pi}\|p(e^{\iota \theta})\|_2^2\dtheta, \quad \quad \|p\|_{\Hinf} &:= \max_{\theta \in [0,2\pi]}\|p(e^{\iota \theta})\|_2
	\end{align*}

	\end{definition}

	By Parseval's theorem, $\ltwonorm{\omega} = \|\omegacheck\|_{\Htwo}^2$ whenever $\ltwonorm{\omega} < \infty$. Importantly, whenever $\omega \in \Omegah$, the signals $\omegacheck$ are not too ``peaked'', in the sense that they haved bounded $\Hinf$-norm:
	\begin{lemma}\label{lem:Hinf_bound}  Suppose that $\omega \in \Omegah$. Then, $\omegacheck$ is a rational function, $\|\omegacheck\|_{\Htwo} = 1$, and $\|\omegacheck\|_{\Hinf}^2 \le h+1$.
	\end{lemma} 
	Using this property, we show that integrating against $\omegacheck$, for $\omega \in \Omegah$, is lower bounded by integrating against the indicator function of a set with mass proportional to $1/h$:
	\begin{lemma}[Holder Converse]\label{lem:holder_converse} Let Let $\scrC(\epsilon)$ denote the set of Lebesgue measure subsets $\calC \subseteq [0,2\pi]$ with Lebesgue measure $|\calC| \ge \epsilon$. Then,
	\begin{align*}
	\omhnorm{G} = \min_{\omega \in \Omegah} \|\Gcheck(z)\omegacheck(z)\|_{\Htwo} \ge \frac{1}{8\pi}\cdot\min_{\calC \in \scrC(\frac{\pi}{h+1})}  \int_{\theta \in \calC} \sigma_{\din}\left(\Gcheck(e^{\iota \theta})\right)^2\rm\dtheta
	\end{align*}
	\end{lemma}

	The next steps of the proof argue that the function $\sigma_{\din}\left(\Gcheck(e^{\iota \theta})\right)^2$ can be lower bounded by a function which, roughly speaking, cannot spend ``too much time'' close to zero. First, we verify that $\sigma_{\din}(G(z))$ can only reach zero finitely many times:
	\begin{lemma}\label{lem:sigma_vanish} Let $G = (A,B,C,D)$ be a Markov operator from $\R^{\din} \to \R^{\dout}$. Suppose that $\dout \ge \din$, and $\sigma_{\din}(D) > 0$. Then, $\sigma_{\din}(G(z)) = 0$ for at most finitely many $z \in \C$.
	\end{lemma}
	Using this property, we lower bound $\sigma_{\min}(\Gcheck(\eithe))^2$ by an \emph{analytic} function. Recall that $f:\R \to \R$ is analytic if it is infinitely differentiable, and for each $x \in \R$, there exists a radius $r$ such that, for all $\delta \in (0,r)$, the Taylor series of $f$ converges on $(x-\delta,x+\delta)$, as is equal to $f$.
	\begin{lemma}\label{lem:sigma_lb_f} There exists a non-negative, analytic, function $f: \R \to \R$, which is is not identically zero such that, for any $\theta \in \R$, $\sigma_{\min}(\Gcheck(\eithe))^2 \ge f(\theta)$, for all $\theta \in \R$. 
	\end{lemma}
	Finally, we use the fact that analytic functions cannot spend ``too much time'' close to zero (unless of course they vanish identically):
	\begin{lemma}\label{lem:f_integral} Let $f: \R \to \R$ be a real analytic, nonegative, periodic function with period $2\pi$, which is not identically zero. Then, there exists a constants $c > 0$ and $n \in \N$ depending on $f$ such that the following holds: for all $\epsilon \in (0,2\pi]$, and any set $\calC \subseteq [0,2\pi]$ of  Lebesgue measure $|\calC| \ge \epsilon$, then, 
	\begin{align*}
	\int_{\calC} f(\theta)\dtheta \ge c\epsilon^{n}.
	\end{align*}
	\end{lemma}
	The proof of Proposition~\ref{prop:str_cvx_bound} follows from applying the above lemmas in sequence. Recalling that $\scrC(\frac{\pi}{h+1})$ denotes the set of $\calC\subset [0,2\pi]$ with $|\calC| \ge \frac{\pi}{h+1}$, for some non-vanishing, periodic, analytic $f$, we obtain
	\begin{align*}
	\omhnorm{G(z)} &\ge \frac{1}{8\pi}\cdot\min_{\calC \in \scrC(\frac{\pi}{h+1})}  \int_{\theta \in \calC} \sigma_{\din}\left(\Gcheck(e^{\iota \theta})\right)^2\rm\dtheta \tag*{(Lemma~\ref{lem:holder_converse})}\\
	&\ge \frac{1}{8\pi}\cdot\min_{\calC \in \scrC(\frac{\pi}{h+1})}  \int_{\theta \in \calC} f(\theta)\rm\dtheta \tag*{(Lemma~\ref{lem:sigma_lb_f})} \\
	&\ge \frac{c}{8\pi} \min_{\calC \in \scrC(\frac{\pi}{h+1})} |\calC|^n \tag*{(Lemma~\ref{lem:f_integral})} \\
	&\ge \frac{c'}{8\pi}(h+1)^n,
	\end{align*}
	for some $c'>0$, and $n\in \N$.

	\subsubsection{Proof of Lemma~\ref{lem:Hinf_bound}}

	\begin{proof} 
	Since $\omegacheck(z) = \sum_{i=0}^{h} \omega^{[i]} z^{-i}$, $\omegacheck(z)$ is rational. The bound $ \|\omegacheck\|_{\Htwo} = 1$ follows from Parsevals identity with $\sum_{i\ge 0}\|\omega^{[i]}\|_2 =1$ for $\omega \in \Omegah$. The third point explicitly uses that $\omega \in \Omegah$ is an $h+1$-length signal. Namely, by Cauchy-Schwartz,
	\begin{align*}
	\|\omegacheck(z)\|_2 &= \|\sum_{i=0}^h z^i \omega^{[i]}\|_2^2 \le \sqrt{\sum_{i=0}^h |z^i|^2} \sqrt{\sum_{i=0}^h \|\omega^{[i]}\|_2^2} \overset{(i)}{=} \sqrt{\sum_{i=0}^h |z^i|^2},
	\end{align*}
	where we use that the $\ell_2$ norm of $\omega$ is bounded by $1$. If $z = e^{\iota\theta}$, then $|z^i|^2= 1$, so $\|\omegacheck(z)\|_2 \le \sqrt{h+1}$.
	\end{proof}

	\subsubsection{Proof of Lemma~\ref{lem:holder_converse}}

	We argue that rational functions $p$ with unit $\Htwo$-norm and bounded $\Hinf$ norm must be large on a set of sufficiently large measure:
	\begin{lemma}\label{lem:holder_cheby} Let $|\cdot|$ denote Lebesgue measure. Let $p$ be a rational function on $\C$ with $\|p\|_{\Htwo} = 1$ and $\|p\|_{\Hinf}^2 \le B$. Then, there exists a Lebesgue measurable $\calC$ which is a finite union of intervals with Lebesgue measure $|\calC| \ge \frac{\pi}{B}$ for which
	\begin{align*}
	\forall z \in \calC, \|p(z)\|_2 \ge \frac{1}{2}.
	\end{align*}
	\end{lemma}
	\begin{proof} Let $\calC_t := \{\theta \in [0,2\pi]: \|p(e^{\iota\theta})\|_2 \ge t\}$, which is Lebesgue measurable by rationality of $p$. Then, by a Chebyschev-like arugment,
	\begin{align*}
	1 = \|p\|_{\Htwo}^2 &\le \frac{1}{2\pi}\int_{\theta \in \calC_t}\|p(e^{\iota \theta})\| + \frac{1}{2\pi}\int_{\theta \in [0,2\pi] - \calC_t}\|p(e^{\iota \theta})\|^2\\
	&\le  \frac{|\calC_t|\|p\|_{\Hinf}^2}{2\pi} +  (1 - \frac{|\calC_t|}{2\pi})t\\
	&\le  \frac{B|\calC_t|}{2\pi} +  t
	\end{align*}
	Hence, $\frac{|\calC_t|}{2\pi} \ge \frac{1 - t}{B}$. In particular, if we $\calC = \calC_{1/2}$, then $\frac{|\calC_t|}{2\pi} \ge \frac{1}{2B}$, as needed.
	\end{proof}
	We can now prove \Cref{lem:holder_converse} as follows. For each $\omega \in \Omegah$, let $\calC_{\omega})$ denote the corresponding subset of $[0,2\pi]$ guaranteed by Lemma~\ref{lem:holder_cheby}, that is $\|\omegacheck(\eithe)\|_2 \ge \frac{1}{2}\I(\theta \in \calC_{\omega}$. Then, for $\omega \in \Omegah$,
	\begin{align*}
	 \|\Gcheck(z)\omega(z)\|_{\Hinf} &= \frac{1}{2\pi}\min_{\omega \in \Omegah}
	\int_{0}^{2\pi}\|G(\eithe)\omegacheck(\eithe)\|_{2}^2\dtheta\\
	&\ge \frac{1}{2\pi}
	\int_{0}^{2\pi}\sigma_{\din}(G(\eithe))^2\|\omegacheck(\eithe)\|_{2}^2\dtheta\\
	&\ge \frac{1}{2\pi}
	\int_{0}^{2\pi}\sigma_{\din}(G(\eithe))^2 \cdot (\frac{1}{2}\I(\theta \in \calC_{\omega}))^2 \dtheta\\
	&= \frac{1}{8\pi}
	\int_{\theta \in \calC_{\omega}}^{2\pi}\sigma_{\din}(G(\eithe))^2 \dtheta.
	\end{align*}
	Since each set $\calC_{\omega} \subset [0,2\pi]$ has Lebesgue measure at least $\pi/(h+1)$, and since $\scrC(\pi/(h+1))$ denotes the collection of all subsets with this property, $ \min_{\omega\in \Omegah}\|\Gcheck(z)\omega(z)\|_{\Hinf}$ is lower bounded by 
	\begin{align*}\min_{\calC \in \scrC(\pi/(h+1))}\frac{1}{8\pi}
	\int_{\theta \in \calC_{\omega}}^{2\pi}\sigma_{\din}(G(\eithe))^2 \dtheta,
	\end{align*}
	as needed.

	\subsubsection{Proof of Lemma~\ref{lem:sigma_vanish}}
	\begin{proof}
	Next, Note that if $G = (A,B,C,D)$, then $\Gcheck(z) = D + C(zI - A)^{-1}B$. Since $\sigma_{\din}(D) > 0$, there exists a projection matrix $P \in \R^{\din \dout}$ such that $PD$ is rank $\din$. Moreover, if $\sigma_{\din}(P\Gcheck(z)) = 0$, then $\sigma_{\din}(\Gcheck(z)) = 0$, so it suffices to show that  $\sigma_{\din}(P\Gcheck(z)) = 0$ for only finitely many $z$.  Since $P\Gcheck(z) \in \R^{\din \times \din}$ is square-matrix valued, it suffices to show that determinant $\det(P\Gcheck(z)) = 0$ for at most finitely $z$. Since $\det$ is a polynomial function, and $P\Gcheck(z)$ has rational-function entries, (this can be verified by using Cramers rule), there exists polynomials $f,g$ such that $\det(P\Gcheck(z)) = \frac{f(z)}{g(z)}$ for $z \in \C$. This means that either $\det(P\Gcheck(z)) = 0$ for all $z \in \C$, or is identically zero on $\C$. Let us show that the second option is not possible. Consider taking $z \to \infty$ (on the real axis). Then $\lim_{z \to \infty}\Gcheck(z) = \lim_{z \to \infty} D + C(zI - A)^{-1}B =  D$. Hence, $\lim_{z \to \infty}\det(P\Gcheck(z)) = \det(PD) > 0$, since $\sigma_{\din}(PD) > 0$.
	\end{proof}

	\subsubsection{Proof of Lemma~\ref{lem:sigma_lb_f}}
	\begin{proof} We have the following lower bound
	\begin{align*}
	\lambda_{\min}(\Gcheck(z)^{\herm}\Gcheck(z)) &= \frac{\prod_{i=1}^{\din}\lambda_{i}\Gcheck(z)^{\herm}\Gcheck(z)}{\prod_{i=1}^{\din-1}\lambda_{i}\Gcheck(z)^{\herm}\Gcheck(z)} \\
	&\ge \frac{\prod_{i=1}^{\din}\lambda_{i}\Gcheck(z)^{\herm}\Gcheck(z)}{(\max_{z \in \Torus}\opnorm{\Gcheck(z)}^{2})^{\din-1}}\\
	&\ge \frac{\det(\Gcheck(z)^\herm \Gcheck(z))}{\|\Gcheck\|_{\Hinf}^{2(\din - 1)}} := \phi(z)
	\end{align*}
	By assumption, $\hinfnorm{\Gcheck(z)} < \infty$, and so $\phi(z)$ only vanishes when $\det(\Gcheck(z)^\herm \Gcheck(z))$ does, which itself only vanishes when $\lambda_{\min}(\Gcheck(z)^{\herm}\Gcheck(z))  = 0$.  By Lemma~\ref{lem:sigma_vanish}, this means that $\phi(z)$ is not indentically zero on $\Torus$.

	Now, let $\phi(\theta) = \det(\Gcheck(e^{\iota\theta})^\herm \Gcheck(e^{\iota\theta}))$. It suffices to show that this function is real analytic. We argue by expressing $\phi(\theta)= \Phi_2(\Phi_1(\theta))$, where $\Phi_{1} $ is a real analytic map from $\R \to (\R^2)^{\dout \times \din}$, and $\Phi_2$ an analytic map from  $(\R^2)^{\dout \times \din} \to \R$. 
	\newcommand{\embed}{\mathrm{embed}}

	Let $\embed : \C^{\dout \times \din} \to (\R^2)^{\dout \times \din}$ denote the cannonical complex to real embedding. We define $\Phi_1(\theta) = \embed(\Gcheck(e^{\iota \theta})$. To see that $\Phi_1(\theta)$ is real analytic, we observe that the map $z \mapsto e^{\iota \theta}$ is complex analytic, and since $\Gcheck$ is a rational function, $u \mapsto \Gcheck(u)$ is analytic away from the poles of $\Gcheck$. Since $\Gcheck$  has no poles $u \in \Torus$ (by assumption of stability/bounded $\Hinf$ norm), we conclude that $z\mapsto \Gcheck(e^{\iota z})$ is complex analytic at any $z \in \R$. Thus, $\theta \mapsto \embed(\Gcheck(e^{\iota z}))$ is real analytic for $\theta \in \R$.

	Second, given $X \in (\R^2)^{\dout \times \din}$, let $\Phi_2(X) = \det((\embed^{-1}(X))^{\herm} (\embed^{-1}(X)))$. It is easy to see that $\Phi_2(X)$ is a polynomial in the entries of $X$, and thus also real analytic. Immediately, we verify that $\phi(\theta)= \Phi_2(\Phi_1(\theta))$, demonstrating that $\phi$ is given by the composition of two real analytic maps, and therefore real analytic.
	\end{proof}

	\subsubsection{Proof of Lemma~\ref{lem:f_integral}}

	We begin with a simple claim:
	\begin{claim} $f$ has finitely many zeros on $[0,2\pi]$.
	\end{claim}
	\begin{proof}
	\newcommand{\fbar}{\bar{f}}
	Since $f$ is real analytic on $\R$, it can be extended to a complex analytic function $\fbar$ on a open subset $U \subset \C$ containing the real line $\R$. Since $f$ is not identically zero on $\R$ assumption, $\fbar$ is not identically zero on $\R$, and thus by ``Principle of Permanence'', $\overline{f}$ can have no accumulation points of zeros on $U$. In particular, its restriction $f$ can have no accumulation points of zeros on $\R$. As $[0,2\pi]$ is compact, $f$ has finitely many zeros on $[0,2\pi]$.\footnote{As a proof of this fact, note that if $f$ has no accumulation points, then for each $x \in [0,2\pi]$, there exists an open set set $U_x \subset \R$ containing $x$ which has at most $1$ zero. The sets $U_x$ form an open cover of $[0,2\pi]$. By compactness, there exists a finite number of these sets $U_{x_1},\dots,U_{x_m}$ which cover $[0,2\pi]$. Since each $U_{x_i}$ has at most one zero, there are at most $m$ zeros of $f$ on $[0,2\pi]$.}
	\end{proof}
	We now turn to the proof of our intended lemma: 
	\begin{proof}[Proof of Lemma~\ref{lem:f_integral}]

	Let $\theta_1,\dots,\theta_m$ denote the zeros on $f(\theta)$ which lie on $[0,2\pi)$, of which there are finitely many by the above argument. The Taylor coefficients of $f$ cannot be all zero at any of these $\theta_i$, for otherwise analyticity would imply that $f$ would locally vanish. Thus, by Taylor's thoerem, at each zero $\theta_i$, we have that for some constants $c_i > 0$, $r_i > 0$, $n_i \in \N$, 
	\begin{align*}
	f(\theta) \ge c_i|\theta - \theta_i|^{n_i}, \,\forall \theta \in\R:|\theta - \theta_i| \le r_i
	\end{align*}
	Letting $c = \min_{i} c_i$, $n = \max_i n_i$, and $r_i = \min\{1,\min_i r_i\}$, we have that for all
	\begin{align*}
	f(\theta) \ge c|\theta - \theta_i|^{n}, \,\forall \theta \in \R:|\theta - \theta_i| \le r
	\end{align*}
	By shrinking $r$ if necessary, we may assume that the intervals $I_i = [\theta_i-r,\theta_i+r]$ are disjoint, and that there exists a number $\theta_0$ such that $[\theta_0,2\pi + \theta_0] \supset \bigcup_{i=1}^m I_i$. By periodicity, one can check then $f(\theta)$ only vanishes on $[\theta_0,2\pi + \theta_0]$ at $\{\theta_1,\dots,\theta_m\} \subset \bigcup_{i=1}^{m} \mathrm{Interior}(I_i)$. Compactness of the set  $S := [\theta_0,2\pi + \theta_0]  - \bigcup_{i=1}^{m} \mathrm{Interior}(I_i) $ and the fact that $f(\theta)  \ne 0$ for all $\theta \in S$ implies that $\inf_{\theta \in S} f(\theta) > 0$. By shrinking $c$ if necessary, we may assume $\inf_{\theta \in S} f(\theta) \ge c$. Therefore, we have shown that 
	\begin{align*}
	\forall \theta \in [\theta_0,2\pi + \theta_0], \quad f(\theta) \ge \underbar{f}(\theta) := c\left( \I\left(\min_{i\in [m]}|\theta - \theta_i| >r\right)  + \sum_{i=1}^m \I((\theta - \theta_i) > r) |\theta - \theta_i|^n  \right).
	\end{align*}
	Now, let $\scrC(\epsilon)$ denote the set of subsets $\calC \subset [\theta_0, 2\pi + \theta_0]$ with Lebesgue measure $\epsilon$. By translation invariance of the Lebesgue measure, and periodicity of $f(\theta)$, it suffices to show that, for constants $c',n'$, the following holds for all $\epsilon \in (0,1)$, the following holds
	\begin{align}
	\min_{ \calC \in \scrC(\epsilon)} \int_{\theta \in \calC} \underbar{f}(\theta)\dtheta \ge c' \epsilon^{n'}.\label{eq:calc_min}
	\end{align}
	In fact, by shrinking $c'$ if necessary, it suffices to show that the above holds only for $\epsilon \in (0,2mr)$. Examining the above display, we that any set of the form $\calC:=  \{\theta: \underbar{f}(\theta) \le t\}$ with $|\calC| = \epsilon$ is a minimizer. Assuming the restriction $\epsilon < 2mr$, this implies that the minimum  \eqref{eq:calc_min} is attained by the set $\calC_{\epsilon} := \bigcup_{i}I_i(\epsilon)$, where we define the intervals $I_i(\epsilon) = [\theta_i - \epsilon/2m,\theta_i + \epsilon/2m]$. We can compute then that,
	\begin{align*}
	\int_{\theta \in \calC_{\epsilon}} \underbar{f}(\theta)\dtheta  &=  \sum_{i=1}^m\int_{\theta_i-\epsilon/2m}^{\theta_i-\epsilon/2m} \underbar{f}(\theta)\dtheta \\
	&=  \sum_{i=1}^m\int_{\theta_i-\epsilon/2m}^{\theta_i-\epsilon/2m} c|\theta - \theta_i|^n \dtheta \\
	&=  \sum_{i=1}^m\int_{-\epsilon/2m}^{\epsilon/2m} c|\theta|^n \dtheta \\
	&=  2 c m \int_{0}^{\epsilon/2m} |r|^n \dtheta \\
	&=  \frac{2 c m}{n+1} (\epsilon/2m)^{n+1},
	\end{align*}
	which has the desired form.

	\end{proof}

\newpage
\section{Gradient Descent with Conditional Strong Convexity \label{sec:GD_conditional_strong}}
We begin by recalling Condtions~\ref{cond:urc},~\ref{cond:mrc}and~\ref{cond:gd} under which we argue the subsequent bounds. First:
\urccondition*
Note that, by Jensen's inequality, $f_{t;k}$ are $\beta$-smooth and $\Lf$-Lipschitz on $\calK$. Second, we recall the with-memory analogue:
\mrccondition*
Lastly, we formalize the fashion in which the iterates are generated:

\condgd*
The remainder of the section is as follows. Section~\ref{ssec:basic_regret_lemma} proves Lemma~\ref{ssec:basic_regret_lemma}, which relates the regret on the non-conditioned unary sequence $f_t$ to standard strongly convex $\frac{\log T}{\alpha}$, plus additional correction for the errors $\adverr_t$, the negative regret, and a correction $\stocherr_t$ for the mismatch between $f_t$ and $f_{t;k}$. For $k = 0$, $f_{t;k} = t_t$ and $\stocherr_t$ is zero, recovering Proposition~\ref{prop:robust_ogd}. Next, Section~\ref{ssec:debias_regret} proves Lemma~\ref{lem:deterministic_str_cvx_regret}, which bounds theterms $\stocherr_t$ in terms of a mean-zero sequence $Z_t(\zst)$ depending on the comparator $\zst$. 

Next, Section~\ref{ssec:high_prob_unary} states and proves our main high-probability regret bound for unary functions, Theorem. Lastly, Section~\ref{ssec:high_prob_memory} extends

\subsection{Basic Regret Lemma\label{ssec:basic_regret_lemma} and  Proposition~\ref{prop:robust_ogd}}
	We begin by proving the following ``basic'' inequality for the unary setting, which provides a key intermediate regret bound adressing both conditional strong convexity and error in the gradients, as well as incorporating negative regret:
	\begin{restatable}[Basic Inequality for Conditional-Expectation Regret]{lemma}{basiccondexpregret}\label{lem:basic_conditional_expectation} 
		Consider the setting of Conditions~\ref{cond:urc} and \ref{cond:gd}. For step size $\eta_t = \frac{3}{\alpha t}$,
		\begin{align*}
		\forall \zst \in \calK,~\sum_{t=k+1}^T f_{t;k}(z_t) - f_{t;k}(\zst) &\le - \frac{\alpha}{6}\sum_{t=k+1}^T \|z_t - \zst\|_2^2 +  \frac{6\Lf^2}{\alpha} \log(T+1)\nonumber\\
		&\qquad + \frac{6}{\alpha}\sum_{t=k+1}^T \left\|\adverr_t\right\|_2^2  - \sum_{t=1}^T\llangle \stocherr_t, z_t - \zst\rrangle + \frac{\alpha D^2(k+1)}{2},
		\end{align*}
		where we define $\stocherr_t :=  \nabla f_t(z_t) - \nabla f_{t;k}(z_t)$
	\end{restatable}
	Note that Proposition~\ref{prop:robust_ogd} in the body arises in the special case where $k = 0$. For $k > 1$, the error $\stocherr_t$ is required to relate the updates based on  $\nabla f_t(z_t)$ to to based on $\nabla f_{t;k}(z_t)$, the latter of which corresponding to functions which are strongly convex.
	\begin{proof}
	Let $\gradient_t := \nabla f_t(z_t) + \adverr_t$. From \citep[Eq. 3.4]{hazan2016introduction}, strong convexity of $f_{t;k}$ implies that 
		\begin{align}
		2(f_{t;k}(z_t) - f_{t;k}(\zst)) \le 2\nabla_{t;k}^\top(z_t - \zst) - \alpha\|\zst - z_t\|^2_2\label{eq:subopt}
		\end{align}
		Now, we let gradient descent correspond to the update $\maty_{t+1} = z_t - \eta_{t+1} \gradient_t$, where $\gradient_t = \nabla_{t;k} + \stocherr_t + \adverr_t$, for stochastic error $\stocherr_t$ and deterministic noise $\adverr_t$. The Pythagorean Theorem implies
		\begin{align}
		\|z_{t+1} - \zst\|_2^2 &\le \|z_t - \zst - \eta_{t+1} \gradient_t \|^2_2 = \|z_t - \zst\|^2_2 + \eta_{t+1}^2 \|\gradient_t\|^2 - 2\eta_{t+1} \gradient_t^\top (z_t - \zst),
		\end{align}
		which can be re-expressed as 
		\begin{align}
		-2 \gradient_t^\top (z_t - \zst) \ge \frac{\|z_{t+1} - \zst\|_2^2 -  \|z_t - \zst\|^2_2}{\eta_{t+1}} - \eta_{t+1} \|\gradient_t\|^2 \label{eq:first_GD_ineq}
		\end{align}
		Furthermore, using the elementary inequality $ab \le \frac{a^2}{2\tau} + \frac{\tau}{2} b^2$ for any $a,b$ and $\tau > 0$, we have that for any $\tau > 0$
		\begin{align}
		- \langle\gradient_t, z_t - \zst\rangle &= -  \left\langle\nabla_{t;k} + \stocherr_t, z_t - \zst\right\rangle   -  \left\langle\adverr_t, z_t - \zst\right\rangle\nonumber\\
		&\le -  \llangle\nabla_{t;k} + \stocherr_t,z_t - \zst\rrangle   + \frac{ \alpha \tau}{2} \|z_t - \zst\|_2^2  +  \frac{1}{2\alpha \tau} \left\|\adverr_t\right\|_2^2 \label{eq:second_GD_ineq}
		\end{align}
		Combining Equations~\eqref{eq:first_GD_ineq} and~\eqref{eq:second_GD_ineq}, and rearranging, 
		\begin{align*}
		2\nabla_{t;k}^\top (z_t - \zst) &\le \frac{ \|z_t - \zst\|^2-  \|z_{t+1} - \zst\|^2}{\eta_{t+1}} + \eta_{t+1}\|\gradient_t\|^2 + \frac{1}{\tau \alpha} \left\|\adverr_t\right\|_2^2 \\
		&\qquad+ \tau \alpha \|z_t - \zst\|^2  - 2 \llangle \stocherr_t, z_t - \zst\rrangle\\
		&\le \frac{ \|z_t - \zst\|^2-  \|z_{t+1} - \zst\|^2}{\eta_{t+1}} + 2\eta_{t+1}L^2 + \left(2\eta_{t+1} + \frac{1}{\tau \alpha} \right) \left\|\adverr_t\right\|_2^2 \\
		&\qquad+ \tau \alpha \|z_t - \zst\|^2  - 2 \llangle \stocherr_t, z_t - \zst\rrangle,
		\end{align*}
		where we used $\|\gradient_t\|_2^2 \le 2(\|\nabla f(z_t)\|_2^2 + \|\adverr_t\|_2^2) \le 2(L^2 + \|\adverr_t\|_2^2) $. Combining with~\eqref{eq:subopt}, we have
		\begin{align*}
		\iftoggle{colt}{&}{}\sum_{t=k+1}^T f_{t;k}(z_t) - f_{t;k}(\zst)\iftoggle{colt}{\\}{} &\le \frac{1}{2}\sum_{t=k+1}^T \left(\frac{1}{\eta_{t+1}} - \frac{1}{\eta_t}  -(1-\tau) \alpha \right)\|z_t - \zst\|_2^2 \nonumber\\
		&\qquad+ \sum_{t=k+1}^T2\eta_{t+1}L^2 + \left(\frac{1}{\tau \alpha} + 2\eta_{t+1}\right)\left\|\adverr_t\right\|_2^2  - \sum_{t=1}^T\llangle \stocherr_t, z_t - \zst\rrangle + \frac{\|z_{k+1}\|_2^2}{\eta_{1+k}}
		\end{align*}
		Finally, let us set $\eta_t = \frac{3}{\alpha t}$,  $\tau = \frac{1}{3}$, and recall $D = \diam(\calK)$ and $\|\nabla_t\|_2^2 \le \Lf^2$. Then, we have that 
		\begin{enumerate}
			\item $\frac{1}{2}\sum_{t=1}^T \left(\frac{1}{\eta_{t+1}} - \frac{1}{\eta_t}  -(1-\tau) \alpha \right)\|z_t - \zst\|_2^2 = \frac{1}{2}\sum_{t=1}^T\left(\alpha/3 - 2\alpha/3\right)\|z_t - \zst\|_2^2$, which is equal to $\frac{-\alpha}{6}\sum_{t=1}^T\|z_t - \zst\|^2$
			\item $\left(\frac{1}{\tau \alpha} + 2\eta_{t+1}\right) \le \frac{6}{\alpha}$, and $2\sum_{t=k+1}^T\eta_{t+1}\Lf^2 \le  2\cdot 3\Lf^2\log (T+1)/\alpha$
			\item $\frac{\|z_{k+1}\|_2^2}{\eta_{1+k}} \le  \alpha(k+1)D^2/3$.
		\end{enumerate}
		Putting things together,
		\begin{align*}
		\sum_{t=k+1}^T f_{t;k}(z_t) - f_{t;k}(\zst) &\le - \frac{\alpha}{6}\sum_{t=k+1}^T \|z_t - \zst\|_2^2 + 6 \Lf^2 \log(T+1)\nonumber\\
		&\qquad + \frac{3}{\alpha}\sum_{t=k+1}^T \left\|\adverr_t\right\|_2^2  - \sum_{t=1}^T\llangle \stocherr_t, z_t - \zst\rrangle + \frac{\alpha D^2(k+1)}{3}
		\end{align*}
		Finally, to conclude, we bound
		\begin{align*}
		\frac{\alpha D^2(k+1)}{3}  -\frac{\alpha}{6}\sum_{t=k+1}^T \|z_t - \zst\|_2^2 &\le \frac{(k+1)\alpha D^2}{3} + \frac{(k+1)\alpha D^2}{6}  -\frac{\alpha}{6}\sum_{t=1}^T \|z_t - \zst\|_2^2 \\
		&= \frac{(k+1)\alpha D^2}{2}  -\frac{\alpha}{6}\sum_{t=1}^T \|z_t - \zst\|_2^2. 
		\end{align*}
		\end{proof}
\subsection{De-biasing the Stochastic Error \label{ssec:debias_regret}}
	The next step in the proof is to unpack the stochastic error term from Lemma~\ref{lem:basic_conditional_expectation}, yielding a bound in terms of a mean-zero sequence $Z_t(\zst)$:
	\begin{lemma}[De-biased Regret Inequality]\label{lem:deterministic_str_cvx_regret} Under Conditions~\ref{cond:urc} and~\ref{cond:gd}  step size $\eta_t = \frac{3}{\alpha t}$, the following bound holds deterministically for any $\zst \in \calK$
	\begin{align*}
	\sum_{t=k+1}^T f_t(z_t) - f_t(\zst) &\le   \frac{\alpha D^2(k+1)}{2} +\frac{6\Lf^2 + k\Lgrad(6\beta + 12\Lf)}{\alpha} \log(T+1) + \frac{6}{\alpha}\sum_{t=k+1}^T \left\|\adverr_t\right\|_2^2  \nonumber\\
		&\qquad   + \sum_{t=k+1}^T Z_{t}(\zst) -  \frac{\alpha}{6}\sum_{t=1}^T \|z_t - \zst\|_2^2 
	\end{align*}
	where we define $Z_t(\zst) := (f_{t;k} - f)(z_{t-k}) - (f_{t;k} - f)(\zst) + \llangle \nabla (f_t - f_{t;k})(z_{t-k}), z_{t-k} - \zst\rrangle$.
	\end{lemma}
	One can readily check that $\Exp[Z_t(\zst) \mid \calF_{t-k}] = 0$.
	\begin{proof}

	Let $\zst \in \calK$ denote an arbitrary competitor point. We recall that $f_{t;k} := \Exp[f_t \mid \calF_{t-k}]$, and set  $\stocherr_t :=  \nabla f_t(z_t) - \nabla f_{t;k}(z_t)$. Proceeding from Lemma~\ref{lem:basic_conditional_expectation}, there are two challenges: (a) first, we wish to convert a regret bound on the conditional expectations $f_{t;k}$ of the functions to the actual functions $f_t$ and (b) the errors $\langle \stocherr_t, z_t - \zst\rangle$ do not form a martingale sequence, because the errors $\stocherr_t$ are correlated with $z_t$. We adress both points with a decoupling argument. Begin by writing 
	\begin{align*}
	\stocherr_t = \nabla f_{t}(z_t) - \nabla f_{t;k}(z_t) = \nabla (f_t - f_{t;k})(z_{t-k}) + \nabla (f_{t} - f_{t;k})(z_t)  - \nabla (f_{t} - f_{t;k})(z_{t-k}).
	\end{align*}
	We then have that
	\begin{align*}
	\iftoggle{colt}{&}{}\sum_{t=k+1}^T \llangle \stocherr_t, z_{t} - \zst \rrangle  \iftoggle{colt}{\\}{}&= \sum_{t=k+1}^T \llangle \nabla (f_t - f_{t;k})(z_{t-k}), z_{t-k} - \zst\rrangle\\
	&\qquad +\sum_{t=k+1}^T \llangle \nabla (f_{t} - f_{t;k})(z_t)  - \nabla (f_{t} - f_{t;k})(z_{t-k}), z_{t} - \zst \rrangle  + \llangle \nabla (f_t - f_{t;k})(z_{t-k}), z_{t} - z_{t-k}\rrangle
	\end{align*}
	Since $f_{t;k}$ and $f_t$ are $L$-Lipschitz and $\beta$-smooth, we have
	\begin{align*}
	&\left|\sum_{t=k+1}^T \llangle \nabla (f_{t} - f_{t;k})(z_t)  - \nabla (f_{t} - f_{t;k})(z_{t-k}), z_{t} - \zst \rrangle  + \llangle \nabla (f_t - f_{t;k})(z_{t-k}), z_{t} - z_{t-k}\rrangle\right| \\
	&\quad\le \sum_{t=k+1}^T 2\beta\|z_t - z_{t-k}\|_2\|z_t-\zst\|_2  + 2\Lf\|z_t - z_{t-k}\| \quad\le \sum_{t=k+1}^T 2(\beta D+\Lf)\|z_t - z_{t-k}\|. 
	\end{align*}
	Similarly, we decouple, 
	\begin{align*}
	\sum_{t=k+1}^T f_{t;k}(z_t) - f_{t;k}(\zst) &= \sum_{t=k+1}^T f_t(z_t) - f_t(\zst) + \sum_{t=k+1}^T (f_{t;k} - f)(z_{t-k}) - (f_{t;k} - f)(\zst)\\
	 &+\sum_{t=k+1}^T f_t(z_{t-k}) - f_{t}(z_t) + f_{t;k}(z_t) - f_{t;k}(z_{t-k}).
	\end{align*}
	Similarly, we can bound
	\begin{align*}
	\left|\sum_{t=k+1}^T f_t(z_{t-k}) - f_{t}(z_t) + f_{t;k}(z_t) - f_{t;k}(z_{t-k})\right| \le \sum_{t=k+1}^2 2\Lf\|z_t - z_{t-k}\|.
	\end{align*}
	Putting the above together, we find that 
	\begin{align*}
	&\sum_{t=k+1}^T f_{t;k}(z_t) - f_{t;k}(\zst) + \sum_{t=k+1}^T \llangle \stocherr_t, z_{t} - \zst \rrangle \le \underbrace{\sum_{t=k+1}^T f_t(z_t) - f_t(\zst)}_{(i)}\\
	&+ \sum_{t=k+1}^T \underbrace{(f_{t;k} - f)(z_{t-k}) - (f_{t;k} - f)(\zst) + \llangle \nabla (f_t - f_{t;k})(z_{t-k}), z_{t-k} - \zst\rrangle}_{:= Z_t(\zst)} + \underbrace{\sum_{t=k+1}^T (2\beta D+4\Lf)\|z_t - z_{t-k}\|}_{(iii.a)},
	\end{align*}
	To conclude, let us bound the term $(iii.a)$:
	\begin{align*}
	(iii.a) &\le \sum_{t=k+1}^T (2\beta D+4\Lf)\sum_{i=1}^k\|z_{t-i} - z_{t-k-j}\|\\
	&\le \sum_{t=k+1}^T (2\beta D+4\Lf)\sum_{i=1}^k\eta_{t+1}\|\nabla f_t(z_t) + \adverr_t \|\\
	&\le k\Lgrad(2\beta D+4\Lf)\sum_{t=k+1}^T \eta_{t+1} \le \frac{3}{\alpha}\cdot k\Lgrad(2\beta D+4\Lf)\log(T+1).
	\end{align*}
	Hence,
	\begin{align*}
	&\sum_{t=k+1}^T f_{t;k}(z_t) - f_{t;k}(\zst) + \sum_{t=k+1}^T \llangle \stocherr_t, z_{t} - \zst \rrangle \\
	&\qquad\le \sum_{t=k+1}^T f_t(z_t) - f_t(\zst) +  \sum_{t=k+1}^T Z_t(\zst) + \frac{k\Lgrad(6\beta D+12\Lf)}{\alpha}\log(T+1),
	\end{align*}
	Lemma~\ref{lem:deterministic_str_cvx_regret} follows directly from combining the above with Lemma~\ref{lem:basic_conditional_expectation}.

	\end{proof}

\subsection{High Probability Regret}
\subsubsection{High Probability for Unary Functions\label{ssec:high_prob_unary}}
	Our main high-probability guarantee for unary functions is as follows:
	\begin{theorem}\label{thm:high_prob_unary} Consider a sequence of functions $f_1,f_2,\dots$ satisfying Condtions~\ref{cond:urc} and~\ref{cond:gd}. Then, with step size $\eta_t = \frac{3}{\alpha t}$, the following bound holds with probability $1-\delta$ for all $\zst \in \calK$ simultaenously:
	\begin{align*}
	&\sum_{t=k+1}^T f_t(z_t) - f_t(\zst) - \left( \frac{6}{\alpha}\sum_{t=k+1}^T \left\|\adverr_t\right\|_2^2  - \frac{\alpha}{12}\sum_{t=1}^T \|z_t - \zst\|_2^2\right)\\
	&\qquad\lesssim  \alpha k D^2 + \frac{ kd\Lf^2 + k\Lf\Lgrad + k\beta \Lgrad}{\alpha} \log(T)  \nonumber + \frac{k\Lf^2}{\alpha}\log\left(\frac{1 +\log_+(\alpha D^2)}{\delta}\right).
	\end{align*}
	\end{theorem}
	\begin{proof}
	Starting from Lemma~\ref{lem:deterministic_str_cvx_regret}, we have
	\begin{align*}
		\sum_{t=k+1}^T f_t(z_t) - f_t(\zst) &\le   \frac{\alpha D^2(k+1)}{2} +\BigOh{\frac{\Lf^2 + k(\beta + \Lf)\Lgrad}{\alpha}} \log(T+1) + \frac{6}{\alpha}\sum_{t=k+1}^T \left\|\adverr_t\right\|_2^2  \nonumber\\
			&\qquad   + \underbrace{\sum_{t=k+1}^T Z_{t}(\zst) -  \frac{\alpha}{12}\sum_{t=1}^T \|z_t - \zst\|_2^2}_{(i)} - \frac{\alpha}{12}\sum_{t=k+1}^T \|z_t - \zst\|_2^2
		\end{align*}
		where we we recall $Z_t(\zst) := (f_{t;k} - f)(z_{t-k}) - (f_{t;k} - f)(\zst) + \llangle \nabla (f_t - f_{t;k})(z_{t-k}), z_{t-k} - \zst\rrangle$. We now state a high-probability upper bound on term $(i)$, proved in Section~\ref{ssec:lem_pointwise_martingale} below:
	\begin{lemma}[Point-wise concentration]\label{lem:pointwise_martingale} Fix a $\zst \in \calK$. Then, with probability $1-\delta$, the following bound holds
	\begin{align*}
	\sum_{t=k+1}^T Z_t(\zst) - \frac{\alpha}{12}\sum_{t=1}^T \|z_t - \zst\|_2^2 \le \BigOh{\frac{k\Lf^2}{\alpha}}\log\left(\frac{k(1 +\log_+(\alpha TD^2)}{\delta}\right), 
	\end{align*}
	where $\log_+(x) = \log(x \vee 1)$.
	\end{lemma}
	Together with $k \le T$ and some algebra, the following holds probabilty $1-\delta$ for any fixed $\zst \in \calK$,
	\begin{align*}
		\sum_{t=k+1}^T f_t(z_t) - f_t(\zst) &\le   \frac{\alpha D^2(k+1)}{2} +\BigOh{\frac{\Lf^2 + k(\beta + \Lf)\Lgrad}{\alpha}} \log(T+1) + \frac{6}{\alpha}\sum_{t=k+1}^T \left\|\adverr_t\right\|_2^2  \nonumber\\
			&\qquad   + \BigOh{\frac{k\Lf^2}{\alpha}}\log\left(\frac{T(1 +\log_+(\alpha D^2)}{\delta}\right) - \frac{\alpha}{12}\sum_{t=1}^T \|z_t - \zst\|_2^2
		\end{align*}
	To extend from a fixed $\zst$ to a uniform bound, we adopt a covering argument. Note that the only terms that depend explicitly on the comparators $\zst$ are $-f_t(\zst)$ and $\|z_t - \zst\|_2^2$. We then establish the following bound: 
	\begin{claim}\label{claim:covering_str_convex} Let $\calN$ denote a $D/T$-cover of $\calK$. Then, for ay $\zst \in \calK$, there exists a $z \in \calN$ with
	\begin{align*}
	\left|-\sum_{t=k+1}^T (f_t(\zst) - f_t(z)) + \frac{\alpha}{12}\sum_{t=1}^T\|z_t - \zst\|_2^2 - \|z_t - z\|_2^2\right| \lesssim \frac{\Lf^2}{\alpha}+ \alpha D^2.
	\end{align*}
	\end{claim} 
	\begin{proof} $\left|\|z_t - \zst\|_2^2 - \|z_t-z\|_2^2\right| = \left|\llangle z_t - \zst, (z_t - \zst) - (z_t - z)\rrangle + \llangle z_t - \zst, (z_t - \zst) - (z_t - z) \rrangle\right| \le 2D\|\zst - z\|$. Moreover, $|f_t(z) - f_t(\zst)| \le \Lf\|z - \zst\|$. From the triangle inequality, we have that the sum in the claim is bounded by $ (\Lf + \alpha D)T\|z- \zst\| \le \Lf D + \alpha D^2 \lesssim \Lf^2/\alpha + \alpha D^2$.
	\end{proof}
	Next, we bound the size of our covering
	\begin{claim}\label{claim:covering_number_str_convex} There exists an $D/T$ covering of $\calN$ with cardinality at most $(1+2 T)^{d}$. 
	\end{claim}
	\begin{proof} Observe that $\calK$ is contained in ball of radius $D$. Set $\epsilon = D/T$. By a standard volumetric covering argument, it follows that we can select $|\calN| \le ((D+\epsilon/2)/(\epsilon/2))^{d} = (1 + \frac{2D}{\epsilon})^d = (1 + 2 T)^d$.
	\end{proof}
	Absorbing the approximation error of $L^2/\alpha + \alpha D^2$ from Claim~\ref{claim:covering_str_convex}, and applying a union bound over the cover from Claim~\ref{claim:covering_str_convex}, we have with probability $1-\delta$ that
	\begin{align}
	&\sum_{t=k+1}^T f_t(z_t) - f_t(\zst) - \left(\frac{\alpha}{12}\sum_{t=1}^T \|z_t - \zst\|_2^2 + \frac{6}{\alpha}\sum_{t=k+1}^T \left\|\adverr_t\right\|_2^2\right) \nonumber \\
	&\lesssim   \alpha D^2k + \frac{\Lf^2 + k(\beta + \Lf)\Lgrad}{\alpha} \log(T)   + \frac{kL^2}{\alpha}\log\left(\frac{T(1+2T)^d(1 +\log_+(\alpha D^2)}{\delta}\right)\nonumber\\ 
	&\lesssim   \alpha k D^2 + \frac{ kdL^2 + k(\beta + \Lf)\Lgrad}{\alpha} \log(T)  + \frac{k\Lf^2}{\alpha}\log\left(\frac{1 +\log_+(\alpha D^2)}{\delta}\right) \nonumber .
	\end{align}
	\end{proof}

	\subsubsection{Proof of Lemma~\ref{lem:pointwise_martingale}\label{ssec:lem_pointwise_martingale}}
		\newcommand{\Zbar}{\overline{Z}}

		For simplicity, drop the dependence on $\zst$, and observe that, since $f_t,f_{t;k}$ are $\Lf$-Lipschitz, we can bound $|Z_t| \le 4\Lf\|z_t - \zst\|_2$.  Moreover, $\Exp[Z_t \mid \calF_{t-k}] = 0$. We can therefore write
		\begin{align*}
		\sum_{t=k+1}^T Z_t = \sum_{t=k+1}^T  U_t \cdot \Zbar_t,
		\end{align*}
		where we set $U_t := 4\Lf\|z_t - \zst\|_2$ and $\Zbar_t := Z_t/U_t$. We can check that
		\begin{align*}
		|\Zbar_t| \mid \calF_{t-k} \le 1 \text{ a.s.} \quad \text{and} \quad \Exp[\Zbar_t \mid \calF_{t-k}] = 0.
		\end{align*}
		Hence, $\Zbar_t$ is a bounded,  random variable with $\Exp[\Zbar_t \mid \calF_{t-k}] = 0 $ multiplied by a $\calF_{t-k}$-measurable non-negative term. Note that this does not quite form a martingale sequence, since $\Zbar_t$ has mean zero conditional on $\calF_{t-k}$, not $\calF_{t-1}$. 

		This can be adressed by a blocking argument: let $t_{i,j} = k + i + jk - 1$ for $ i \in [k]$, and $j \in \{1,\dots,T_i\}$, where $T_i := \max\{j: t_{i,j} \le T\}$.  Then, we can write
		\begin{align*}
		\sum_{t=k+1}^T  U_t \cdot \Zbar_t = \sum_{i=1}^k \left(\sum_{j=1}^{T_i}U_t \cdot \Zbar_{t_{i,j}}\right)
		\end{align*}
		Now, each term in the inner sum is a martingale sequence with respect to the filtration $\{\calF_{t_{i,j}}\}_{j\ge 1}$. Moreover, $\Zbar_{t_{i,j}} \mid \calF_{t_{i,j-1}}$ is $\frac{1}{4}$-sub-Gaussian. We now invoke the a modification of \citet[Lemma 4.2 (b)]{simchowitz2018learning}, which follows straightforwardly from adjusting the last step of its proof
		\begin{lemma} Let $X_j,Y_j$ be two random processes. Suppose $(\calG_j)_{j \ge 0}$ is a filtration such that $(X_j)$ is $(\calG_j)$-adapted, $Y_j$ is $(\calG_{j-1})$ adapted, and  $X_j \mid \calG_{j-1}$ is $\sigma^2$ subGaussian. Then, for any $0 < \beta_{-} \le \beta_{+} $,
		\begin{align*}
		\Pr\left[\I(\sum_{j=1}^T Y_j^2 \le \beta_+)\sum_{j=1}^T X_jY_j \ge u \max\left\{\sqrt{\sum_{j=1}^T Y_j^2},\beta_-\right\}\cdot  \right] \le \log \ceil{\frac{\beta_+}{\beta_-}}\exp( - u^2/6\sigma^2)
		\end{align*}
		\end{lemma}
		For each $i$, apply the above lemma with  $\beta_- = \Lf^2/\alpha $  and $\beta_+ = \max\{T\Lf^2 D^2,\Lf^2/\alpha\}$, $X_j = \Zbar_{t_{i,j}}$ and $Y_j = U_{t_{i,j}}$, $\sigma^2 = 1/4$, and $u = \sqrt{3\log(1/k\delta)/2}$. Then, we have $\sum_{j=1}^{T_i} Y_j^2 \le \beta_+$ almost surely, so we conclude that, with probability $(1+\log(1 \vee \alpha TD^2))\delta$, the following holds for any $\tau > 0$
		\begin{align*}
		\forall i: \sum_{j=1}^{T_i}U_t \cdot \Zbar_{t_{i,j}} &\le \beta_- \sqrt{3\log(k/\delta)/2} + \sqrt{\frac{3}{2}\log(k/\delta)\sum_{j=1}^{T_i}U_{t_{i,j}}^2}\\
		&= \beta_- \sqrt{3\log(k/\delta)/2} + \sqrt{\frac{3\cdot 16\Lf^2}{2}\log(k/\delta)\sum_{j=1}^{T_i}\|z_{t_{i,j-1}} - \zst\|_2^2}\\
		&= \beta_- \sqrt{3\log(k/\delta)/2} + \sqrt{24\Lf^2\log(k/\delta)\sum_{j=1}^{T_i}\|z_{t_{i,j-1}} - \zst\|_2^2}\\
		&\le \frac{\Lf^2}{\alpha}  (\sqrt{3\log(k/\delta)/2} + \frac{12}{\tau} \log(k/\delta)) +  \frac{\tau}{2}\sum_{j=1}^{T_i}\|z_{t_{i,j-1}} - \zst\|_2^2\\
		&\le \Lf^2\log(k/\delta) (\frac{3}{2\alpha} + \frac{12}{\tau} ) +  \frac{\tau}{2}\sum_{j=1}^{T_i}\|z_{t_{i,j-1}} - \zst\|_2^2\\
		\end{align*}
		Therefore, with probability with probability $1-(1+\log(\alpha TD^2))\delta$, for any $\tau, \tau_1 > 0$, 
		\begin{align*}
		\sum_{t=k+1}^T Z_t &\le \sum_{i=1}^k\sum_{j=1}^{T_i}\sum_{j=1}^{T_i}U_t \cdot \Zbar_{t_{i,j}} \\
		&\le  k \Lf^2\log(k/\delta) \left(\frac{3}{2\alpha} + \frac{ 27}{\tau} \right) +  \frac{\tau}{2}\sum_{t=k+1}^T\|z_{t-k} - \zst\|_2^2
		\end{align*}
		$\tau = \alpha/6$, $\delta \leftarrow \delta/(1 + \log(1 \vee \alpha TD^2))$, we have that with probability $1-\delta$
		\begin{align*}
		\sum_{t=k+1}^T Z_t - \frac{\alpha}{12}\sum_{t=1}^{T}\|z_{t}-\zst\|_2^2  &\le \sum_{t=k+1}^T Z_t - \frac{\alpha}{12}\sum_{t=k+1}^{T}\|z_{t-k}-\zst\|_2^2 \\
		&\lesssim \frac{k\Lf^2}{\alpha}\log\left(\frac{k(1 +\log_+(\alpha TD^2)}{\delta}\right).
		\end{align*}
		\qed.

\subsubsection{High Probability Regret with Memory: Proof of Theorem~\ref{thm:main_w_memory_high_prob} \label{ssec:high_prob_memory}}

\begin{proof}
We reitarate the argument of \cite{anava2015online}. Decompose
\begin{align*}
\sum_{t=k+1}^T \ftil_t(z_t,z_{t-1},z_{t-2},\dots,z_{t-h}) - f_t(\zst) &= \sum_{t=k+1}^T f_t(z_t) - f_t(\zst) \iftoggle{colt}{\\&\quad+}{+} \sum_{t=k+1}^T \ftil_t(z_t,z_{t-1},z_{t-2},\dots,z_{t-h}) - f_t(z_t).
\end{align*}
We can bound  the first sum directly from Theorem~\ref{thm:high_prob_unary}. The second term can be bounded as follows:
\begin{align*}
\sum_{t=k+1}^T \ftil_t(z_t,z_{t-1},z_{t-2},\dots,z_{t-h}) - f_t(z_t) &\le \Lc\sum_{t=k+1}^T \|(0,z_{t-1} - z_t,\dots,z_{t-h}-z_t)\|_{2} \\
&\le \sum_{t=k+1}^T \sum_{i=1}^h\|z_t - z_{t-i}\|_2\, \\
&\le \Lc\sum_{t=k+1}^T \sum_{i=1}^h\sum_{j=1}^i \|z_{t-j+1} - z_{t-j}\|_2 \,\\
&\le \Lc\sum_{t=k+1}^T \sum_{i=1}^h\sum_{j=1}^i \eta_{t-j+1}\|\gradient_{t-j}\|  \\
&\le h\Lc\sum_{t=k+1}^T \sum_{i=1}^h \eta_{t-j+1}\|\gradient_{t-j}\| \\
&\le h\Lc \Lgrad\sum_{t=k+1}^T \sum_{i=1}^h \eta_{t-j+1}\\
&\le h^2\Lc \Lgrad\sum_{t=1}^T\eta_{t+1} \lesssim \frac{h^2 \Lc\Lgrad \log T}{\alpha}
\end{align*}
This establishes the desired bound.
\end{proof}

\end{document}